%% file: main_arxiv.tex
\PassOptionsToPackage{numbers,compress}{natbib}
\documentclass{article}

\usepackage[preprint]{neurips_2026}

\input{preamblecommon}
\usepackage{geometry}

\newcommand{\CodeLink}{\url{https://github.com/FUJI1229/Hallucination_Detection}}

\title{\PaperTitle}

\author{%
  Shota Fujikawa \\
  Department of Computer Science\\
  University of Tokyo\\
  \texttt{fj@g.ecc.u-tokyo.ac.jp} \\
  \And
  Issei Sato \\
  Department of Computer Science\\
  University of Tokyo\\
}

\begin{document}
\raggedbottom
\maketitle
\input{bodycommon}

\bibliography{ref}
\clearpage
\appendix
\input{sections/Appendix}
\end{document}

%% file: preamblecommon.tex
\usepackage[utf8]{inputenc} 
\usepackage[T1]{fontenc}    
\usepackage{hyperref}       
\usepackage{url}            
\usepackage{booktabs}       
\usepackage{amsmath}
\usepackage{amsfonts}       
\usepackage{amssymb}
\usepackage{amsthm}
\usepackage{microtype}      
\usepackage{xcolor}         
\usepackage{multirow}
\usepackage{pgffor}
\usepackage{graphicx}
\usepackage{float}
\usepackage{subcaption}
\usepackage{listings}

\newtheorem{theorem}{Theorem}
\newtheorem{lemma}[theorem]{Lemma}
\newtheorem{proposition}{Proposition}
\newtheorem{assumption}{Assumption}
\newtheorem{definition}{Definition}

\newcommand{\PaperTitle}{Max-pooling Network Revisited: Analyzing the Role of
Semantic Probability in Multiple Instance Learning for Hallucination Detection}

%% file: bodycommon.tex
\input{sections/Abstract}
\input{sections/Introduction}
\input{sections/Related_work}
\input{sections/Preliminaries}
\input{sections/HaMI}
\input{sections/Revisit}
\input{sections/EXP}
\input{sections/Conclude}

%% file: sections/Abstract.tex
\begin{abstract}
Hallucination detection is essential for reliable large language models (LLMs).
Hybrid methods such as HaMI combine semantic consistency with internal states
through multiple instance learning (MIL), but repeated sampling and semantic
comparisons impose substantial computational overhead. We analyze HaMI's
decision margins and identify conditions under which semantic-probability
scaling enlarges the average logit margin. Motivated by this analysis, we
revisit classical sentence classification from a local margin dynamics
perspective: a shared MLP extracts token-level features, feature-wise max
pooling aggregates them, and a linear head scores the response. The detector
requires no additional sampling or semantic consistency computation. Across
four QA benchmarks and three LLMs, it is competitive with semantically weighted
HaMI in detection AUC while adding a median of $0.5\%$ to end-to-end response
time, whereas semantically weighted HaMI takes $7$--$34$ times as long.
On long-form answers with web-verified entity labels (LongFact++, two LLMs), its larger
gains over fixed-position probes support aggregating evidence across token
positions.
\ifdefined\CodeLink Code is available at \CodeLink.\fi
\end{abstract}

%% file: sections/Introduction.tex
\section{Introduction}
\label{sec:introduction}

Hallucinations hinder the reliable deployment of large language models (LLMs) \citep{hallucination_survey2,survey,hall}, motivating detectors that identify factual errors in generated responses.
Existing approaches to hallucination detection fall into three main categories.  
First, uncertainty-based methods estimate confidence from model probabilities or from the semantic diversity of multiple sampled responses, as in Semantic Entropy \citep{SE},
which accounts for semantic equivalence.  
Second, retrieval-based approaches query external knowledge sources or LLMs to verify factual consistency \citep{gaibu}.  
Third, representation-based methods exploit internal hidden states \citep{saplma,h-neuron}, often training classifiers on fixed token positions such as the final token \citep{knowmorethan}.  
A recent hybrid method, HaMI \citep{base}, formulates hallucination detection as a multiple instance learning (MIL) problem and improves robustness through adaptive token selection and semantic-probability weighting.

HaMI's semantic weighting requires repeated generation and external semantic comparisons, adding latency and API costs. Single-token hidden-state probes are fast, but commit to one token position and can miss evidence elsewhere in the response. We ask whether sequence-level detection can keep HaMI's accuracy while adding little to the cost of generation.

Despite its empirical success, it remains unclear why semantic-probability weighting improves detection performance.
We analyze HaMI and show that semantic-probability scaling enlarges the average logit margin under an explicit condition on class proportions and scaling sensitivities. This motivates studying whether aggregation of internal token features can support margin expansion without semantic weighting.

Specifically, we revisit max-over-time sentence classification \citep{max-over-time-pooling}: a shared MLP transforms token states, feature-wise max pooling aggregates them, and a linear head scores the response. Different features can retain evidence from different tokens. We analyze each bag's margin increase under its own gradient update and bound how much max pooling avoids gradient dilution in terms of the share of inactive token features. This motivates aggregating informative features across token positions.

Our contributions are summarized as follows:
\begin{itemize}
    \item We characterize when semantic weighting enlarges HaMI's average margin (Theorem~\ref{theorem:hami}).
    \item We analyze margin dynamics under each bag's own update and sparse-signal aggregation in max-pooling networks (Lemma~\ref{thm:global} and Theorem~\ref{thm:sparse}), and bound their capacity (Proposition~\ref{prop:rc_bounds}).
    \item On four QA benchmarks with three LLMs, our detector is competitive with HaMI (SP); on long-form LongFact++ answers with two LLMs, it outperforms HaMI and fixed-position probes and leads or ties mean pooling; and it adds almost no time to generation (Sections~\ref{ex:auc}--\ref{ex:efficiency}).
\end{itemize}

\begin{table}[H]
\centering
\small
\setlength{\tabcolsep}{6pt}
\renewcommand{\arraystretch}{1.2}
\caption{Comparison of architectural designs between HaMI and our study.}
\label{tab:hami_comparison}
\begin{tabular}{p{3cm} p{4.5cm} p{4.5cm}}
\toprule
    \textbf{Aspect} & \textbf{HaMI \citep{base}} & \textbf{Ours} \\
\midrule
\textbf{Aggregation} 
& Instance-based
& Embedding-based \\
\textbf{Pooling Method} 
& TopK Pooling
& Max Pooling \citep{max-over-time-pooling} \\
\textbf{Signals} 
& Semantic Metric + Hidden States 
& Hidden States\\

\bottomrule
\end{tabular}
\end{table}

%% file: sections/Related_work.tex
\section{Related Work}
\paragraph{LLM Hallucination and Detection Metrics}
Factual errors can reflect failures to access internal knowledge \citep{GPT}. Detection methods include logit-based indicators \citep{logit} and consistency checks across sampled responses, such as SelfCheckGPT \citep{selfcheck}.

Semantic Entropy (SE) quantifies uncertainty via semantic clustering \citep{SE}; SE Probes approximate it from hidden representations to avoid repeated sampling and semantic comparisons \citep{sep}.
\paragraph{Internal States and Latent Truthfulness}
Recent probing studies suggest that truthfulness signals are not uniformly scattered throughout the model, but instead are concentrated in specific layers and token positions \citep{saplma,knowmorethan}. At the level of individual neurons, prior work identified a small subset of hallucination-related neurons, termed H-Neurons, and showed that manipulating their activations can directly increase or decrease the model's tendency to generate false information \citep{h-neuron}.

\paragraph{Multiple Instance Learning (MIL) for Hallucination Detection}
MIL assigns labels to bags of instances \citep{MIL_survey,MIL}; attention-based deep MIL learns to weight their contributions \citep{deepMIL}.

Building on this perspective, the HaMI framework \citep{base} models hallucination as a local, token-dependent phenomenon rather than a uniform sentence-level property. It formulates detection as a MIL task, adaptively weighting token contributions using uncertainty metrics and internal representations, thereby outperforming fixed aggregation baselines. 

In contrast, our work revisits this aggregation process \citep{max-over-time-pooling} by relying solely on internal representations, eliminating the need for external uncertainty estimation while retaining the ability to capture informative token-level signals. RAGLens \citep{raglens} also max-pools token features, but fixed sparse-autoencoder features for faithfulness to retrieved context, whereas we learn the features from response-level labels.

%% file: sections/Preliminaries.tex
\section{Preliminaries}
\subsection{Problem Settings}

Hallucination refers to the phenomenon in which responses contradict factual reality, typically arising when LLMs lack sufficient internal knowledge or fail to properly use their internal knowledge~\citep{GPT}. 

For short-answer QA, let $D = \{(q_i, a_i)\}_{i=1}^{N}$, where $q_i$ is a question and $a_i$ its reference answer. LongFact++ supplies prompts and model-written answers with entity-level labels instead of reference answers. Given a question or prompt $q$, an LLM produces a response sequence $\mathcal{A} = (t_1, t_2, \dots, t_n)$. 

For each generated token $t_i$, we extract the hidden representation $\mathbf{h}_i \in \mathbb{R}^d$ from an intermediate layer. An external evaluator LLM assigns $y=+1$ to a hallucinated response and $y=-1$ to a factual response. This evaluator supplies supervision; our detector requires no external evaluator at inference.

\subsection{Semantic Probability}

Following the Semantic Entropy framework \citep{SE}, we quantify uncertainty by analyzing the semantic diversity of stochastic samples $\mathcal{G} = \{\mathcal{A}_{1}, \dots, \mathcal{A}_{K}\}$ for each question $q$.

\paragraph{Semantic Clustering}
To group linguistically diverse but semantically equivalent responses, we use an external LLM to evaluate the relation $R_{ij}:=R(\mathcal{A}_i, \mathcal{A}_j) \in \{E, C, N\}$ (\textit{Entailment, Contradiction, Neutral}). We define equivalence as:
\begin{equation*}
\mathcal{A}_i \equiv \mathcal{A}_j \iff (R_{ij}, R_{ji}) \in \{(E,E), (E,N), (N,E)\}
\end{equation*}
The sampled responses in $\mathcal{G}$ are grouped into semantic clusters $\{\mathcal{C}_{1}, \dots, \mathcal{C}_{s}\}$.

\paragraph{Uncertainty Metrics}
Following HaMI~\citep{base}, a cluster's probability is its empirical frequency among the $K$ sampled responses:
\begin{equation*}
P_{\mathcal{C}} = \frac{1}{K}\sum_{k=1}^{K}\mathbf{1}\{\mathcal{A}_k \in \mathcal{C}\}.
\end{equation*}
Each sampled response is counted separately, including repeated identical texts. Its \textbf{semantic probability} is:
\begin{equation*}
P^{(i)}_{\text{sem}} = P_{\mathcal{C}(i)}, \quad \text{where } \mathcal{C}(i) \text{ is the cluster containing $\mathcal{A}_i$}.
\end{equation*}

\subsection{Hallucination Detection via Multiple Instance Learning}
We formulate hallucination detection as MIL~\citep{MIL_survey,deepMIL,MIL}. Each response forms a bag of token states,
$\mathbf{B}=\{\mathbf{h}_{\mathbf{B},i}\}_{i=1}^{T_\mathbf{B}}$, with label $y_\mathbf{B}$.
The dataset $\mathcal{S}$ contains the following positive and negative subsets:
\begin{equation*}
    \mathcal{S}_{pos} = \{ \mathbf{B} \in \mathcal{S} \mid y_{\mathbf{B}} = 1 \}, \quad \mathcal{S}_{neg} = \{ \mathbf{B} \in \mathcal{S} \mid y_{\mathbf{B}} = -1 \}
\end{equation*} To ensure permutation invariance, the scoring function $S( \mathbf{B} )$ must be symmetric.
\paragraph{Aggregation Strategies}
Two orders are possible. \emph{Instance-based} aggregation scores each token
first and pools the scores,
$S(\mathbf{B}) = g_1(\{ f_1(\mathbf{h}_{\mathbf{B},i}) \}_{i=1}^{T_{\mathbf{B}}})$;
\emph{embedding-based} aggregation pools the hidden states into a bag embedding
and scores once,
$S(\mathbf{B}) = f_2(g_2(\{\mathbf{h}_{\mathbf{B},i}\}_{i=1}^{T_\mathbf{B}}))$.
HaMI takes the first route, ours the second.

%% file: sections/HaMI.tex
\section{Analysis of HaMI \texorpdfstring{\citep{base}}{}}
\subsection{HaMI Architecture Formulation}
The HaMI architecture is illustrated in Figure~\ref{fig:hami_arch} (Appendix~\ref{app:hami_architecture}). For a bag
$\mathbf{B} = \{\mathbf{h}_{\mathbf{B},i}\}_{i=1}^{T_\mathbf{B}}$ with
$\mathbf{h}_{\mathbf{B},i} \in \mathbb{R}^d$, a network $f$ turns each instance
into a logit, and HaMI scales the instances by the semantic probability
$P_{\mathrm{sem}}^\mathbf{B}$ before doing so:
\begin{equation*}
    z_{\mathbf{B},i} = f(\mathbf{h}_{\mathbf{B},i})
      = w^\top \mathrm{ReLU}(\mathrm{BN}(\mathbf{x}_{\mathbf{B},i} + \mathbf{b}_1)) + b_2,
    \qquad
    \tilde{\mathbf{h}}_{\mathbf{B},i} = (1 + \lambda P_{\mathrm{sem}}^\mathbf{B}) \mathbf{h}_{\mathbf{B},i},
\end{equation*}
where $\mathbf{x}_{\mathbf{B},i} = \mathbf{W}\mathbf{h}_{\mathbf{B},i}$,
$\mathrm{BN}$ is batch normalization with per-feature scale $\gamma_j$ and
standard deviation $\sigma_j$, $\lambda > 0$, and
$\tilde{z}_{\mathbf{B},i} = f(\tilde{\mathbf{h}}_{\mathbf{B},i})$. The bag score
$S(\mathbf{B})$ averages the scores of the $k$ TopK instances $\mathcal{K}_{\mathbf{B}}$.

\paragraph{Margin Formulation}
\label{sec:hami_margin}
For a dataset $\mathcal{S}$ of bags, define the average logit-space margin by
\begin{equation*}
M_{\text{hami}}(\lambda) = \frac{1}{|\mathcal{S}|}
\sum_{\mathbf{B} \in \mathcal{S}} y_{\mathbf{B}}\tilde{Z}_{\mathbf{B}},
\qquad
\tilde{Z}_{\mathbf{B}} = \frac{1}{k}
\sum_{i \in \mathcal{K}_{\mathbf{B}}}\tilde{z}_{\mathbf{B},i}.
\end{equation*}
We use logit space because the sigmoid is monotone but not distance-preserving:
at large $|z|$, it compresses margin differences into negligible probability differences.
\subsection{Margin Enhancement via Asymmetric Scaling}
\label{sec:hami_margin_scaling}
We analyze the effect of the scaling factor
$p_\mathbf{B} = \lambda P_{\mathrm{sem}}^\mathbf{B}$. Batch
normalization is affine at inference, so as long as the ReLU active set
$\mathcal{U}_{\mathbf{B},i}$ does not change under the scaling, the scaled
logit separates exactly into an unscaled part and a term linear in
$p_\mathbf{B}$ (Appendix~\ref{app:bn_decomposition}):
\begin{equation*}
    \tilde{z}_{\mathbf{B},i} = z_{\mathbf{B},i} - p_\mathbf{B} C(\mathbf{x}_{\mathbf{B},i}),
    \qquad
    C(\mathbf{x}_{\mathbf{B},i}) = - \frac{\partial \tilde{z}_{\mathbf{B},i}}{\partial p_\mathbf{B}}
    = - \sum_{j \in \mathcal{U}_{\mathbf{B},i}} w_j \gamma_j \frac{x_{\mathbf{B},i,j}}{\sigma_j}.
\end{equation*}

For clarity, we present the derivation assuming a locally invariant active set. 
However, using the property that ReLU networks are continuous piecewise-affine (CPWA) maps \citep{relu_CPWA}, we can generalize this sensitivity analysis via path integration. 
The pooled logit then satisfies $\tilde Z_{\mathbf B}(p_{\mathbf B}) = Z_{\mathbf B} - p_{\mathbf B} \bar{C}_{\mathbf B}^{\mathrm{int}}$ with an integrated sensitivity $\bar{C}_{\mathbf B}^{\mathrm{int}}$, even when activation patterns change (Appendix~\ref{app:active_non_invariant}).

Empirically, we observe the following three properties (Appendix \ref{app:assumption_check}):
\begin{gather}
    \bar{C}_{\mathbf{B}} > 0, \quad \text{where} \quad \bar{C}_{\mathbf{B}} = \frac{1}{k} \sum_{i \in \mathcal{K}_{\mathbf{B}}} C(\mathbf{x}_{\mathbf{B},i})
    \label{E[C_pos]}\\
    \mathbb{E}_{neg}[P_{\mathrm{sem}}^{\mathbf{B}}] > \mathbb{E}_{pos}[P_{\mathrm{sem}}^{\mathbf{B}}]
    \label{E[P]}\\
    \mathbb{E}_{neg}[\bar{C}_{\mathbf{B}}] > \mathbb{E}_{pos}[\bar{C}_{\mathbf{B}}].
    \label{E[C]}
\end{gather}
These properties can be understood from the following perspectives:
\begin{itemize}
    \item \textbf{Logit Suppression Effect:} The positivity in Equation \eqref{E[C_pos]} indicates that the scaling operation inherently acts as a penalty. The scaling operation suppresses the overall logits.
    \item \textbf{Semantic Consistency:} Equation \eqref{E[P]} reflects the intuition that models produce more semantically consistent outputs for factual responses.
    \item \textbf{Input-Scale Sensitivity:} Driven by the classification loss during training, the model adaptively adjusts $\bar{C}_{\mathbf{B}}$. It learns to increase $\bar{C}_{\mathbf{B}}$ for negative bags to amplify logit suppression, while decreasing it for positive bags to prevent undesired attenuation. This adaptive behavior naturally leads to Equation \eqref{E[C]}.
\end{itemize}

Motivated by these observations, we define the ratio of the expected product of the scaling factor and logit sensitivity between the negative and positive classes as $\gamma$:
\begin{equation}
\label{eq:gamma_definition}
    \gamma = \frac{\mathbb{E}_{neg}[P_{\mathrm{sem}}^{\mathbf{B}}\bar{C}_{\mathbf{B}}]}
    {\mathbb{E}_{pos}[P_{\mathrm{sem}}^{\mathbf{B}} \bar{C}_{\mathbf{B}}]}.
\end{equation}
The empirical validation of this property and the observed range of $\gamma$ are provided in Appendix \ref{app:expected_values_summary}.

\begin{theorem}[Margin Enhancement via Semantic Probability Weighting]
\label{theorem:hami}
Weighting the input instances by their semantic probabilities strictly increases the empirical expected margin $\mathbb{E}_{\mathbf{B}}[M_{hami}]$, provided that the dataset class ratio satisfies the following condition:
\begin{equation}
\label{eq:margin_condition}
    \frac{|\mathcal{S}_{neg}|}{|\mathcal{S}_{pos}|} > \frac{1}{\gamma}.
\end{equation}
\end{theorem}

\paragraph{Interpretation.}
Because the scaling suppresses logits, it widens the margins of negative bags and narrows those of positive bags. $\gamma$ is the ratio of the expected gain on negative bags to the expected loss on positive bags, so by Equation~\eqref{eq:margin_condition} the average margin grows whenever negative bags are frequent enough.

%% file: sections/Revisit.tex
\section{Max Pooling Network Revisited~\texorpdfstring{\citep{max-over-time-pooling}}{}}
HaMI adjusts token-level scales using semantic information.
However, it scores each token separately and then pools the scores, so it cannot combine different features from different tokens.
We therefore revisit the classical sentence classification
model~\citep{max-over-time-pooling} and study how aggregation affects
a bag's margin increase.
\subsection{Model Architecture}
\label{sec:maxpool_arch}
The model's architecture is shown in Figure~\ref{fig:max_arch_appendix} (Appendix~\ref{app:maxpool_architecture}). It treats each LLM token's hidden state as an instance and computes the hallucination probability for the entire bag.

Let a bag be $\mathbf{B} = \{\mathbf{h}_1, \dots, \mathbf{h}_{T_\mathbf{B}}\}$
with $\mathbf{h}_i \in \mathbb{R}^d$ the hidden state of the $i$-th generated
token. A shared projection maps each instance into a $D$-dimensional feature
space, the bag representation takes the maximum along the token axis
independently per dimension, and a linear layer produces the score:
\begin{align*}
\mathbf{u}_i = \mathrm{ReLU}(W^\top \mathbf{h}_i),
\qquad
v_j = \max_{i \le T_{\mathbf{B}}} (u_i)_j,
\qquad
z_\mathbf{B} = w^\top \mathbf{v},
\qquad
S(\mathbf{B}) = \sigma(z_\mathbf{B}).
\end{align*}
This is the max-over-time pooling of \citet{max-over-time-pooling}, not a
maximum over instance scores as in TopK MIL: the bag representation may combine
evidence from different tokens in different dimensions. With
$y_\mathbf{B}=+1$ for hallucinated bags and $-1$ for factual bags, training minimizes the
logistic loss
$\mathcal{L}_\mathbf{B} = \log(1 + \exp(-y_\mathbf{B} z_\mathbf{B}(\theta)))$.
Max pooling could be applied to the raw $d$-dimensional hidden states
directly. Section~\ref{sec:rademacher} bounds the capacities of both operation
orders, and Section~\ref{ex:ablations} compares their detection accuracy.
\subsection{Margin Analysis}
\label{sec:margin_max_pooling}
\paragraph{Bag-wise Margin.}
Our analysis and Section~\ref{ex:margin} use the signed bag logit as the
\emph{bag-wise margin}:
\begin{equation}
\label{eq:bag_margin}
    m_\mathbf{B}(\theta) := y_\mathbf{B}\, z_\mathbf{B}(\theta).
\end{equation}
Positive margins indicate correct classification, and their magnitude is the
logit-space distance from zero. The logistic loss
$\mathcal{L}_\mathbf{B} = \log(1 + \exp(-m_\mathbf{B}(\theta)))$ decreases with
the margin. This matches $M_{\text{hami}}$ for the TopK-pooled logit
$\tilde Z_\mathbf B$ (Section~\ref{sec:hami_margin});
$\mathcal{M}(\theta) = \mathbb{E}_\mathbf{B}[m_\mathbf{B}(\theta)]$ is the
expected bag-wise margin.

\begin{lemma}[Expected Margin Expansion under Per-Bag Updates]
\label{thm:global}
For a finite empirical bag distribution, define each bag's own update
$\theta_{\mathbf{B},\eta} := \theta - \eta\nabla_\theta\mathcal{L}_{\mathbf{B}}(\theta)$
and $\mathcal{M}_\eta(\theta) := \mathbb{E}_{\mathbf{B}}[m_{\mathbf{B}}(\theta_{\mathbf{B},\eta})]$,
so $\mathcal{M}_0(\theta)=\mathcal{M}(\theta)$.
For sufficiently small $\eta>0$ with each update staying in a smooth region,
\begin{equation*}
\mathcal{M}_\eta(\theta) = \mathcal{M}_0(\theta) + \eta \mathbb{E}_{\mathbf{B}} \left[ \frac{\| \nabla_\theta z_\mathbf{B}(\theta) \|^2}{1 + \exp(m_\mathbf{B}(\theta))} \right] + O(\eta^2).
\end{equation*}
The increase is strict when the expected first-order coefficient is positive.
\end{lemma}
\label{global_margin}
Each bag is evaluated after its own update; this does not assert an increase
in the mean margin of a shared updated model.
Appendix~\ref{app:beta-smooth-margin} gives an explicit step-size condition
under a $\beta$-smoothness assumption.
\paragraph{Motivation: Sparse Evidence.}
Factual errors are often confined to a small part of a response \citep{span,base}. Averaging then dilutes the few informative features with background features, whereas max pooling keeps the largest activation of each feature and routes the gradient through the token that produced it. Following prior analyses of max and mean pooling \citep{max_vs_mean}, we compare the two at the same parameters, with $z^{\mathrm{mean}}_\mathbf{B}=w^\top\frac{1}{T_\mathbf{B}}\sum_i\mathbf{u}_i$, and quantify the difference by how many tokens leave each feature inactive.
\begin{definition}[Activation Sparsity] \label{ass:sparse-mil}
Fix $\theta=(W,w)$ and a bag $\mathbf{B}$, and let $J_\mathbf{B}:=\{j: v_j>0\}$. Feature $j$ is \emph{active} on token $i$ if $u_{\mathbf{B},i,j}>0$. For $j\in J_\mathbf{B}$, let $\alpha_{\mathbf{B},j}$ be the fraction of tokens on which $j$ is active, $i^*_j$ a token attaining $v_j$, and $\kappa_{\mathbf{B},j}:=\max_{i:\,u_{\mathbf{B},i,j}>0}\|\mathbf{h}_{\mathbf{B},i}\|_2/\|\mathbf{h}_{\mathbf{B},i^*_j}\|_2\ge 1$.
\end{definition}
Since $\partial z_\mathbf{B}/\partial W_j$ is a combination of $\mathbf{h}_{\mathbf{B},i}\mathbf{1}\{u_{\mathbf{B},i,j}>0\}$, inactive tokens affect neither $v_j$ nor its gradient, and no assumption about them is needed.
\begin{theorem}[Gradient Norm Ratio under Activation Sparsity]
\label{thm:sparse}
For any $\theta$ and any bag $\mathbf{B}$, positive or negative, with $J_\mathbf{B}\neq\emptyset$, let $a_{\mathbf{B},j}:=v_j^2+w_j^2\|\mathbf{h}_{\mathbf{B},i^*_j}\|_2^2$ be the contribution of feature $j$ to $\|\nabla_\theta z_\mathbf{B}^{\max}\|^2$ and $\pi_{\mathbf{B},j}:=a_{\mathbf{B},j}/\sum_{k\in J_\mathbf{B}}a_{\mathbf{B},k}$. Then
\begin{equation*}
\frac{\|\nabla_\theta z_\mathbf{B}^{\max}\|^2}{\|\nabla_\theta z_\mathbf{B}^{\mathrm{mean}}\|^2}
\ \ge\ \Big(\sum_{j\in J_\mathbf{B}}\pi_{\mathbf{B},j}\,\alpha_{\mathbf{B},j}^2\,\kappa_{\mathbf{B},j}^2\Big)^{-1}.
\end{equation*}
\end{theorem}
The bound is the reciprocal of the average of $\alpha_{\mathbf{B},j}^2\kappa_{\mathbf{B},j}^2$, weighted by each feature's share $\pi_{\mathbf{B},j}$ of the max-pooling gradient, so it grows as the features become sparser. When the maximizing token has the largest norm among the active tokens ($\kappa_{\mathbf{B},j}=1$), the bound exceeds one as soon as some feature in $J_\mathbf{B}$ is inactive on at least one token.
\paragraph{Gradient Dynamics under Per-Bag Updates.}
Lemma~\ref{thm:global} links each bag's margin gain under its own update to
its squared logit-gradient norm and the factor $(1+\exp(m_\mathbf{B}))^{-1}$.
Theorem~\ref{thm:sparse} compares the gradient norms at fixed parameters:
max pooling avoids their dilution by averaging in sparse bags.
At comparable margins, max pooling therefore yields larger first-order
margin gains; the comparison concerns per-bag updates rather than the
training trajectory of a shared model.
The result does not use the bag label, so it applies to hallucinated and
factual responses alike. It shows, through quantities measured on trained detectors
(Appendix~\ref{app:validation}), how max pooling retains sparse but
informative signals that would otherwise be diluted by averaging.
Lemma~\ref{thm:global} and Theorem~\ref{thm:sparse} therefore provide insight into
how training improves class separation by increasing the classification
margin.

\subsection{Rademacher Complexity Analysis}
\label{sec:rademacher}
We bound the Rademacher complexity \citep{rademacher} of both pooling orders under shared norm constraints.
\paragraph{Setup.}
Let $\mathcal{S}=\{ \mathbf{B}_i \}_{i=1}^n$ be a set of bags, where $\mathbf{B}_i = \{h_{i,t}\}_{t=1}^{T_i} \subset \mathbb{R}^d$, $\|h_{i,t}\|_2 \le R$, and $T=\max_i T_i$. 
We compare two max-pooling architectures: the model $\mathcal{F}_{\mathrm{base}}$ that pools raw input features directly, and the model $\mathcal{F}_{\mathrm{feat}}$ that transforms instances into a latent space before pooling.

With ReLU activation $a(x)=\max\{0,x\}$, feature-wise max pooling $\rho$ over instances, and weight bounds $\|W_j\|_2 \le B_1$ and $\|w\|_2 \le B_2$, their hypothesis classes are defined as:
\begin{align*}
\mathcal{F}_{\mathrm{feat}} &= \left\{ w^\top \rho\big(\{a(W_j^\top h_t)\}_{j=1,\dots,D;\; t=1,\dots,T}\big) \right\}, \\
\mathcal{F}_{\mathrm{base}} &= \left\{ w^\top a\big(W^\top \rho(\mathbf{B})\big) \right\}.
\end{align*}

\begin{proposition}[Rademacher Complexity Bounds]
\label{prop:rc_bounds}
Under the setup above, the empirical Rademacher complexities are bounded as follows:
\begin{align*}
\hat{\mathcal{R}}_S(\mathcal{F}_{\mathrm{feat}}) \;\le\; \frac{2\sqrt{2}\, R B_1 B_2 \sqrt{D T}}{\sqrt{n}}, \label{eq:rc_feat} \quad 
\hat{\mathcal{R}}_S(\mathcal{F}_{\mathrm{base}}) \;\le\; \frac{2 R B_1 B_2 \sqrt{D d}}{\sqrt{n}}. 
\end{align*}
\end{proposition}

\paragraph{Discussion.}
These norm-based upper bounds characterize the capacity of each architecture,
but their comparison alone does not establish an ordering of the actual
complexities or a generalization advantage for either architecture
(Appendix~\ref{app:proof_proposition_1}). The benefit of transforming features
before pooling is supported empirically by the placement ablation in
Section~\ref{ex:ablations}.

%% file: sections/EXP.tex
\section{Experiments}
\subsection{Experimental Setup}
\label{ex:setup}
\paragraph{Models and data.}
We evaluate LLaMA-3.1-8B (base), Mistral-Nemo-12B (Instruct), and
LLaMA-3.3-70B (Instruct, 4-bit quantized)~\citep{llama,mistral}
on TriviaQA~\citep{trivia_qa},
SQuAD~\citep{squad}, NQ~\citep{NQ} and BioASQ~\citep{bioasq}, with a
$100$-token generation limit. Mean test-response lengths are $6$--$35$ tokens;
long responses from LongFact++~\citep{obeso2026realtimedetectionhallucinatedentities}
are evaluated separately (Section~\ref{ex:length}) with LLaMA-3.1-8B-Instruct
and LLaMA-3.3-70B-Instruct, using their own $1{,}000$/$1{,}000$/$1{,}500$
train/validation/test answers. Unless marked Instruct, LLaMA-3.1-8B denotes the
base model.
For QA, we generate five responses per question in every split at temperature $0.5$,
using only one for detector training and all five for validation and test.
Before filtering, QA train/validation/test counts are
$4{,}000$/$4{,}500$/$5{,}500$, except BioASQ
($2{,}300$/$3{,}000$/$3{,}000$; Appendix~\ref{app:sample_counts}).

\paragraph{Baselines.}
We compare Semantic Entropy Probes (SEP)~\citep{sep}, the last-token probe of
\citet{knowmorethan} (LLMs Know), SAPLMA~\citep{saplma}, HaMI and HaMI
(SP)~\citep{base}, and mean and max pooling. Methods share generations,
factuality labels, splits and hardware; hidden-state extraction is also
shared.

\paragraph{Labels and training.}
GPT-5 mini~\citep{gpt5mini} judges QA against reference answers; LongFact++
provides entity labels verified with web search (Section~\ref{ex:length}).
Initially hallucinated
responses are re-checked using internal knowledge and discarded only if
judged factual. Agreement with human TruthfulQA labels~\citep{truthfulqa} is
$95.3\%$ (Appendix~\ref{app:judge_validation}). We train for $100$ epochs with
batch size $128$, select checkpoints and layers by validation AUC, and report
five-run means. Hidden-state methods share the stride-two layer search of
\citet{base}; mean and max pooling also share $D=256$ and tuning budgets.
Prompts, implementation and hyperparameters are in Appendix~\ref{app:set};
selected layers and standard errors are in
Appendices~\ref{app:layer_selection} and~\ref{app:std}.

\subsection{Classification Performance}
\label{ex:auc}
Table~\ref{tab:performance_comparison} tests whether aggregating token
features improves response-level detection. Max pooling outperforms mean
pooling in all twelve model--dataset pairs, with average gains of
$1.5$--$2.1$ AUC points across the four QA benchmarks. The two methods share
the feature-extraction architecture, feature dimension and tuning budget,
making the pooling rule the key difference in this comparison. A
feature-wise maximum can retain a strong activation even when most tokens
do not express that feature. These gains are consistent with preserving
token-level cues that averaging may attenuate.

Max pooling also outperforms unweighted HaMI in all twelve settings. Its QA
averages are $0.866$, $0.897$ and $0.900$, versus $0.877$, $0.889$ and
$0.885$ for HaMI (SP). Thus, aggregating learned features from a single
response can be competitive with semantic weighting that requires
additional responses. Semantic weighting remains strongest on
LLaMA-3.1-8B, while max pooling leads on the larger models.

Max pooling improves the four-benchmark average over SAPLMA by $1.0$--$4.8$
points and over LLMs Know by $1.1$--$4.8$ points. SAPLMA remains best on
LLaMA-3.3-70B/BioASQ ($0.900$ versus $0.884$), so gains are not universal.
Pooling gives the classifier direct access to features from multiple token
positions.

\begin{table}[t]
\centering
\small
\setlength{\tabcolsep}{1.4pt}
\renewcommand{\arraystretch}{1.15}
\caption{Mean QA AUC over five runs. TQ: TriviaQA; SQ: SQuAD; BQ: BioASQ. \textcolor{red}{Red}/\textcolor{blue}{blue}: best/second-best.}
\label{tab:performance_comparison}
\begin{tabular}{@{}l*{12}{c}@{}}
\toprule
& \multicolumn{4}{c}{\textbf{LLaMA-3.1-8B}}
& \multicolumn{4}{c}{\textbf{Mistral-Nemo-12B}}
& \multicolumn{4}{c}{\textbf{LLaMA-3.3-70B}} \\
\cmidrule(lr){2-5} \cmidrule(lr){6-9} \cmidrule(lr){10-13}
Method & TQ & SQ & NQ & BQ & TQ & SQ & NQ & BQ & TQ & SQ & NQ & BQ \\
\midrule
SEP & 0.821 & 0.759 & 0.729 & 0.792 & 0.842 & 0.773 & 0.730 & 0.836 & 0.870 & 0.824 & 0.829 & 0.855 \\
LLMs Know & 0.879 & 0.779 & 0.766 & 0.872 & 0.902 & 0.810 & 0.802 & 0.882 & \textcolor{blue}{0.934} & 0.863 & 0.871 & \textcolor{blue}{0.888} \\
SAPLMA & 0.870 & 0.776 & 0.765 & 0.861 & 0.905 & 0.837 & 0.834 & 0.880 & 0.909 & \textcolor{blue}{0.866} & \textcolor{blue}{0.884} & \textcolor{red}{0.900} \\
\midrule
HaMI & 0.865 & 0.795 & 0.777 & 0.862 & 0.915 & 0.840 & 0.837 & 0.889 & 0.907 & 0.851 & 0.871 & 0.874 \\
HaMI (SP) & \textcolor{blue}{0.923} & \textcolor{red}{0.840} & \textcolor{red}{0.840} & \textcolor{red}{0.903} & \textcolor{blue}{0.932} & \textcolor{blue}{0.869} & \textcolor{blue}{0.859} & 0.894 & 0.926 & 0.856 & 0.878 & 0.880 \\
Mean & 0.897 & 0.815 & 0.789 & 0.885 & 0.924 & 0.854 & 0.846 & \textcolor{blue}{0.902} & 0.911 & 0.857 & 0.870 & 0.877 \\
\textbf{Max} & \textcolor{red}{0.924} & \textcolor{blue}{0.827} & \textcolor{blue}{0.815} & \textcolor{blue}{0.899} & \textcolor{red}{0.936} & \textcolor{red}{0.871} & \textcolor{red}{0.867} & \textcolor{red}{0.913} & \textcolor{red}{0.937} & \textcolor{red}{0.880} & \textcolor{red}{0.897} & 0.884 \\
\bottomrule
\end{tabular}
\end{table}

\subsection{Detection on Long Responses}
\label{ex:length}
To see whether combining evidence across token positions matters more for
long answers, we use
LongFact++~\citep{obeso2026realtimedetectionhallucinatedentities}, which
extends the LongFact prompts~\citep{longfact} and labels the entities in
model-written answers by web search. It contains answers from
LLaMA-3.1-8B-Instruct and LLaMA-3.3-70B-Instruct, about $740$ tokens long at
the median. We feed each answer to the model
that wrote it and treat it as hallucinated when at least $10\%$ of its checked
entities are not supported. Training answers are class-balanced
(Appendix~\ref{app:length_study}).

On both models, the four MIL detectors outperform the three fixed-position
probes (Table~\ref{tab:length_study}), suggesting that on long answers,
combining evidence from many positions is more effective than reading a single one. Among these methods, max pooling is
best on LLaMA-3.1-8B-Instruct, significantly ahead of mean
pooling and HaMI, and ties with mean pooling on LLaMA-3.3-70B-Instruct.
Semantic weighting does not improve HaMI here, and answer length alone is a
weak predictor (AUC at most $0.59$), so the gains do not come from length.

\begin{table}[t]
\centering
\small
\setlength{\tabcolsep}{5pt}
\renewcommand{\arraystretch}{1.15}
\caption{Mean LongFact++ test AUC and AUPRC over five runs. Chance AUPRC equals
the hallucination rate ($0.674$ for 8B, $0.449$ for 70B). Colors follow
Table~\ref{tab:performance_comparison}.}
\label{tab:length_study}
\begin{tabular}{@{}lcccc@{}}
\toprule
& \multicolumn{2}{c}{\textbf{LLaMA-3.1-8B-Instruct}}
& \multicolumn{2}{c}{\textbf{LLaMA-3.3-70B-Instruct}} \\
\cmidrule(lr){2-3} \cmidrule(lr){4-5}
Method & AUC & AUPRC & AUC & AUPRC \\
\midrule
SEP & 0.493 & 0.665 & 0.534 & 0.469 \\
LLMs Know & 0.623 & 0.773 & 0.633 & 0.558 \\
SAPLMA & 0.647 & 0.774 & 0.637 & 0.581 \\
\midrule
HaMI & 0.711 & 0.827 & 0.665 & 0.602 \\
HaMI (SP) & 0.711 & 0.828 & 0.663 & 0.598 \\
Mean & \textcolor{blue}{0.759} & \textcolor{blue}{0.864} & \textcolor{red}{0.739} & \textcolor{blue}{0.689} \\
\textbf{Max} & \textcolor{red}{0.783} & \textcolor{red}{0.881} & \textcolor{red}{0.739} & \textcolor{red}{0.691} \\
\bottomrule
\end{tabular}
\end{table}

\subsection{Efficiency Comparison}
\label{ex:efficiency}
We measure the time from the prompt to the hallucination score, including
generation, with every detector scoring the same answers
(Appendix~\ref{app:throughput}). Max pooling adds almost nothing to this time.
It reads each new token's state through a forward hook while the answer is
being generated and only updates a running maximum, so nothing needs to be
stored or copied off the GPU. In Table~\ref{tab:throughput_comparison} its
mean time is within $0.01$\,s of generation alone, and its median overhead per response is $0.5\%$
(Appendix~\ref{app:added_baselines}). The other single-pass detectors
behave the same way. HaMI (SP) is different: it generates four more answers
and clusters them with LLM entailment calls, which makes it $7$--$34$ times
slower on QA and more than five times slower on long answers. The appendix
also reports the cost of storing hidden states and moving them off the GPU
(Table~\ref{tab:extraction}).

\begin{table}[t]
\begin{minipage}[t]{0.49\linewidth}
\centering
\small
\setlength{\tabcolsep}{2.5pt}
\renewcommand{\arraystretch}{1.15}
\hyphenpenalty=10000\exhyphenpenalty=10000
\caption{Mean end-to-end time per QA response (s), from prompt to score, including generation.}
\label{tab:throughput_comparison}
\begin{tabular}{@{}lccc@{}}
\toprule
\textbf{Method} & \textbf{\shortstack{LLaMA\\3.1-8B}} & \textbf{\shortstack{Mistral\\12B}} & \textbf{\shortstack{LLaMA\\3.3-70B}} \\
\midrule
\textit{Generation only} & \textit{0.40} & \textit{0.35} & \textit{6.12} \\
SEP & 0.40 & 0.35 & 6.11 \\
LLMs Know & 0.40 & 0.34 & 6.11 \\
SAPLMA & 0.40 & 0.35 & 6.11 \\
\midrule
HaMI & 0.41 & 0.35 & 6.11 \\
HaMI (SP) & 13.73 & 10.00 & 42.04 \\
\midrule
Mean & 0.40 & 0.35 & 6.11 \\
\textbf{Max} & 0.41 & 0.36 & 6.12 \\
\bottomrule
\end{tabular}
\end{minipage}\hfill
\begin{minipage}[t]{0.49\linewidth}
\centering
\small
\setlength{\tabcolsep}{2.5pt}
\renewcommand{\arraystretch}{1.15}
\hyphenpenalty=10000\exhyphenpenalty=10000
\captionsetup{justification=raggedright,singlelinecheck=false}
\caption{Mean QA test margins. \textcolor{red}{Red}/\textcolor{blue}{blue}: higher/lower within each block.}
\label{tab:margin_comparison}
\begin{tabular}{@{}lccc@{}}
\toprule
\textbf{Method} & \textbf{\shortstack{LLaMA\\3.1-8B}} & \textbf{\shortstack{Mistral\\12B}} & \textbf{\shortstack{LLaMA\\3.3-70B}} \\
\midrule
HaMI                       & \textcolor{blue}{3.93} & \textcolor{blue}{4.35} & \textcolor{blue}{4.10} \\
HaMI (SP)                  & \textcolor{red}{6.80} & \textcolor{red}{5.86} & \textcolor{red}{5.66} \\
\midrule
Mean                       & \textcolor{blue}{1.28} & \textcolor{blue}{1.71} & \textcolor{blue}{1.89} \\
\textbf{Max}               & \textcolor{red}{1.48} & \textcolor{red}{1.99} & \textcolor{red}{2.14} \\
\bottomrule
\end{tabular}
\end{minipage}
\end{table}

\subsection{Margin Results}
\label{ex:margin}
Table~\ref{tab:margin_comparison} reports mean test margins
$m_\mathbf{B}=y_\mathbf{B}z_\mathbf{B}$ after $100$ epochs at
validation-selected layers; HaMI uses its TopK-pooled logit. HaMI (SP) raises
the QA mean by $35$--$73\%$ over HaMI, and max pooling by $13$--$16\%$ over
mean pooling (full results in Appendix~\ref{app:margin_full}). These directions
agree with Sections~\ref{sec:hami_margin_scaling}
and~\ref{sec:margin_max_pooling}. However, margins depend on logit scale, so
we compare them descriptively within architecture families, as support for the mechanism of
Section~\ref{sec:margin_max_pooling} rather than as a predictor of AUC, which
Sections~\ref{ex:auc}--\ref{ex:length} evaluate directly.

\subsection{Architectural Ablations}
\label{ex:ablations}
To examine how token features should be aggregated, we keep the MLP
architecture fixed and apply it either before or after pooling.
Transforming tokens first improves max-pooling AUC by
$0.7$--$5.7$ points, while mean pooling changes by at most $0.2$ points
(Table~\ref{tab:raw_vs_reduced}). The operation order also changes the
relative performance of the pooling rules: max pooling falls below mean
pooling in nine of twelve settings when applied to raw states, but leads
in all twelve after feature extraction. On LLaMA-3.3-70B/TriviaQA, max
pooling rises from $0.880$ to $0.937$, while mean pooling changes from
$0.913$ to $0.911$.

This reversal suggests that max pooling benefits from selecting learned
token features. The token-wise MLP can form useful patterns before
pooling combines coordinates across tokens. Pooling raw states first may
lose within-token relationships before the MLP can use them. The results
therefore support learning token-level features before aggregation. This
is empirical evidence for the operation order;
the capacity bounds alone do not establish an accuracy ordering.

\begin{table}[t]
\centering
\small
\setlength{\tabcolsep}{1.4pt}
\renewcommand{\arraystretch}{1.15}
\caption{Feature-placement ablation (mean AUC, five shared seeds; same
validation-selected layer). Abbreviations and colors follow
Table~\ref{tab:performance_comparison}.}
\label{tab:raw_vs_reduced}
\begin{tabular}{@{}l*{12}{c}@{}}
\toprule
& \multicolumn{4}{c}{\textbf{LLaMA-3.1-8B}}
& \multicolumn{4}{c}{\textbf{Mistral-Nemo-12B}}
& \multicolumn{4}{c}{\textbf{LLaMA-3.3-70B}} \\
\cmidrule(lr){2-5} \cmidrule(lr){6-9} \cmidrule(lr){10-13}
Pooling & TQ & SQ & NQ & BQ & TQ & SQ & NQ & BQ & TQ & SQ & NQ & BQ \\
\midrule
\multicolumn{13}{@{}l}{\textit{Raw ($d$): pool the hidden states, then apply the MLP}} \\
Mean & \textcolor{blue}{0.898} & \textcolor{blue}{0.815} & \textcolor{blue}{0.790} & \textcolor{blue}{0.885} & 0.925 & 0.855 & 0.847 & \textcolor{blue}{0.903} & \textcolor{blue}{0.913} & \textcolor{blue}{0.857} & \textcolor{blue}{0.870} & 0.876 \\
Max & 0.891 & 0.807 & 0.782 & 0.872 & \textcolor{blue}{0.929} & \textcolor{blue}{0.856} & \textcolor{blue}{0.849} & 0.896 & 0.880 & 0.848 & 0.861 & 0.868 \\
\midrule
\multicolumn{13}{@{}l}{\textit{Features ($D$): extract features, then pool the $D$-dimensional projections}} \\
Mean & 0.897 & \textcolor{blue}{0.815} & 0.789 & \textcolor{blue}{0.885} & 0.924 & 0.854 & 0.846 & 0.902 & 0.911 & \textcolor{blue}{0.857} & \textcolor{blue}{0.870} & \textcolor{blue}{0.877} \\
\textbf{Max} & \textcolor{red}{0.924} & \textcolor{red}{0.827} & \textcolor{red}{0.815} & \textcolor{red}{0.899} & \textcolor{red}{0.936} & \textcolor{red}{0.871} & \textcolor{red}{0.867} & \textcolor{red}{0.913} & \textcolor{red}{0.937} & \textcolor{red}{0.880} & \textcolor{red}{0.897} & \textcolor{red}{0.884} \\
\bottomrule
\end{tabular}
\end{table}

A matched mean/max sweep over $D$ on LLaMA-3.1-8B/TriviaQA keeps max pooling ahead at every dimension (Appendix~\ref{app:dim}). Layer-wise curves show performance across representation depths (Appendix~\ref{app:layer_sensitivity}).

%% file: sections/Conclude.tex
\section{Conclusion}
We revisited feature-wise max pooling for hallucination detection. Our analysis characterizes when semantic weighting enlarges HaMI's margin and relates each bag's margin increase under its own gradient update to sparse aggregation. The placement ablation, rather than the capacity bounds alone, supports preserving task-relevant token features until aggregation, so that the classifier can combine evidence from different positions.

On four QA benchmarks with three LLMs, max pooling is competitive with HaMI (SP), and on long-form LongFact++ answers with two LLMs, it outperforms HaMI and fixed-position probes. It adds almost no time to generation and needs no extra sampling. Future work includes evaluation against human annotations and localizing errors at the token level.

%% file: sections/Appendix.tex
\appendix
\noindent\textbf{Guide to the appendix.} The table lists, for each part of the main text, the appendices that support it.

\begin{center}
\small
\renewcommand{\arraystretch}{1.15}
\begin{tabular}{@{}p{0.40\linewidth}p{0.56\linewidth}@{}}
\toprule
\textbf{Main text} & \textbf{Appendix} \\
\midrule
Theorem~\ref{theorem:hami} (Section~\ref{sec:hami_margin_scaling}) & proof \ref{app:proof_hami}; extension \ref{app:active_non_invariant}; empirical checks \ref{app:assumption_check} and \ref{app:integrated_checks} \\
Architectures (Section~\ref{sec:maxpool_arch}) & \ref{app:hami_architecture}, \ref{app:maxpool_architecture}, \ref{app:bn_decomposition} \\
Lemma~\ref{thm:global}, Theorem~\ref{thm:sparse} (Section~\ref{sec:margin_max_pooling}) & proofs \ref{app:proof_margin_dynamics}--\ref{app:proof_sparse} and \ref{app:beta-smooth-margin}; empirical check \ref{app:validation} \\
Proposition~\ref{prop:rc_bounds} (Section~\ref{sec:rademacher}) & proof \ref{app:proof_proposition_1} \\
Setup (Section~\ref{ex:setup}) & implementation \ref{app:set}, \ref{app:layer_selection}; data \ref{app:qa_data}; labels \ref{app:label_validation} and \ref{app:judge_sensitivity} \\
QA results (Section~\ref{ex:auc}) & variability \ref{app:std}; additional baselines \ref{app:added_baselines}; domain generalization \ref{app:lodo} \\
Long responses (Section~\ref{ex:length}) & data and labels \ref{app:length_study} \\
Efficiency (Section~\ref{ex:efficiency}) & \ref{app:throughput}, \ref{app:added_baselines} \\
Margins (Section~\ref{ex:margin}) & per-dataset margins \ref{app:margin_full}; empirical check \ref{app:validation} \\
Ablations (Section~\ref{ex:ablations}) & \ref{app:ablations} \\
\bottomrule
\end{tabular}
\end{center}

\section{Implementation Details}
\label{app:implementation}
This appendix gives the training details, the architectures and the selected layers used in Section~\ref{ex:setup}.
\subsection{Implementation and Training}
\label{app:set}

\paragraph{Architecture and software.}
HaMI follows \citet{base}: a two-layer MLP projects each hidden state to $256$ dimensions, applies batch normalization and ReLU, and produces an instance logit. Mean and max pooling map each hidden state with one linear layer and a ReLU to $D$ features, aggregate the features over tokens, and use a linear output head. We use PyTorch $2.5.1$ and Transformers $5.5.3$.

\paragraph{Optimization and selection.}
Mean and max pooling use $D=256$ and Adam with learning rate $2\times10^{-4}$ and weight decay $5\times10^{-3}$. HaMI uses Adam with learning rate $10^{-3}$, weight decay $5\times10^{-4}$, and TopK aggregation over $10\%$ of tokens; HaMI (SP) uses $\lambda=1$. Detectors train for $100$ epochs at batch size $128$ on one NVIDIA A100, retaining the checkpoint with the highest validation AUC. Mean and max pooling share the layer search (Appendix~\ref{app:layer_selection}) and all remaining hyperparameters; $D$ is varied only in the ablation of Appendix~\ref{app:dim}.

Training takes approximately $140$, $155$, and $220$ seconds for LLaMA-3.1-8B, Mistral-Nemo-12B, and LLaMA-3.3-70B, respectively, excluding data generation and labeling. Time per response is measured end to end on one A100 (two for the $4$-bit LLaMA-3.3-70B), processing each response independently without batching (Appendix~\ref{app:throughput}).
\label{app:setup}
\paragraph{Questions and generated responses.}
\label{app:sample_counts}
We generate five responses per question in every split, without an additional context passage. Detector training uses one response per question; validation and test use all five before filtering.
The training splits contain $4{,}000$ questions each for TriviaQA, SQuAD, and NQ and $2{,}300$ for BioASQ. Five responses per question yield $20{,}000$ and $11{,}500$ training-split generations, respectively, but only one response per question is used as a detector-training example, giving $4{,}000$ and $2{,}300$ examples before filtering. Validation contains $900$ and $600$ questions, respectively, yielding $4{,}500$ and $3{,}000$ responses. Test contains $1{,}100$ and $600$ questions, yielding $5{,}500$ and $3{,}000$ responses. Thus, the combined validation and test totals are $10{,}000$ and $6{,}000$ generated responses. Table~\ref{tab:class_balance} reports retained QA responses after filtering; these counts differ from question counts and total generations. LongFact++ provides one answer per prompt, and we use $1{,}000$/$1{,}000$/$1{,}500$ answers (Table~\ref{tab:length_counts}).

\noindent\textbf{Prompt 1: Short-answer QA.}\\*
The four QA benchmarks share the following template from \citet{SE}, as in \citet{base}, with generation capped at $100$ new tokens and stopped at the first blank line. LongFact++ answers are read as described under Prompt 2.

\begin{lstlisting}[
    breaklines=true, 
    breakatwhitespace=true, 
    columns=fullflexible, 
    keepspaces=true,
    basicstyle=\ttfamily\small,
    frame=single
]
Answer the following question in a single but complete sentence only.
Question: {question}
Answer:
\end{lstlisting}

\noindent\textbf{Prompt 2: LongFact++ input.}\\*
We do not generate LongFact++ answers. Each prompt and its dataset answer are
formatted with the chat template of the model that wrote the answer and passed
through that model with teacher forcing; the token states of the answer are the
detector input. For HaMI (SP), we sample four additional answers per prompt
from the same model with the same chat template (vLLM~\citep{vllm};
temperature $0.5$, top-$p$ $1.0$, top-$k$ $50$, repetition penalty $1.2$, at
most $2{,}048$ new tokens; an AWQ~\citep{awq} $4$-bit checkpoint for
LLaMA-3.3-70B) and cluster them together with the dataset answer using
Prompt 4.

\paragraph{Labeling and semantic clustering.}
GPT-5 mini compares QA responses with reference answers. LongFact++ answers keep the dataset's web-search labels (Appendix~\ref{app:length_study}) and are not relabeled. The re-check below is applied only to initially hallucinated responses; those judged factual on re-check are discarded, and all others retain their initial labels (Appendix~\ref{app:discard}).

\noindent\textbf{Prompt 3: Knowledge-only labeling re-check.}\par\nobreak

\begin{lstlisting}[
    breaklines=true, 
    breakatwhitespace=true, 
    columns=fullflexible, 
    keepspaces=true,
    basicstyle=\ttfamily\small,
    frame=single
]
We are assessing the quality of answers to the following question: {question}
The proposed answer is: {predicted_answer}
Based on the context of question and your own knowledge, is the proposed answer correct?
Please think carefully and respond only with yes or no.
Response:
\end{lstlisting}

We employ GPT-4o mini as a natural language inference (NLI) judge to assess the semantic relationship between responses.

\noindent\textbf{Prompt 4: Entailment.}\par\nobreak

\begin{lstlisting}[
    breaklines=true, 
    breakatwhitespace=true, 
    columns=fullflexible, 
    keepspaces=true,
    basicstyle=\ttfamily\small,
    frame=single
]
Context: Question asked to an AI: '{question}'
Answer A: {text1}
Answer B: {text2}
Determine if Answer A entails Answer B. 
Respond with only one word: entailment,contradiction, or neutral.
\end{lstlisting}

\subsection{HaMI Architecture}
\label{app:hami_architecture}
The architecture of HaMI is illustrated in Figure~\ref{fig:hami_arch}. 
The model operates on an instance-based MIL framework, where each instance in a bag is first scaled by $1+\lambda P_{\mathrm{sem}}$, where $P_{\mathrm{sem}}$ is the semantic probability of the response. 
These scaled representations are then processed through a two-layer MLP $f$ to compute individual instance logits. 
To determine the bag-level score, HaMI employs a TopK pooling strategy, which selects and averages the most discriminative instance scores. 
\begin{figure}[H]
    \centering
    \includegraphics[width=1.0\linewidth]{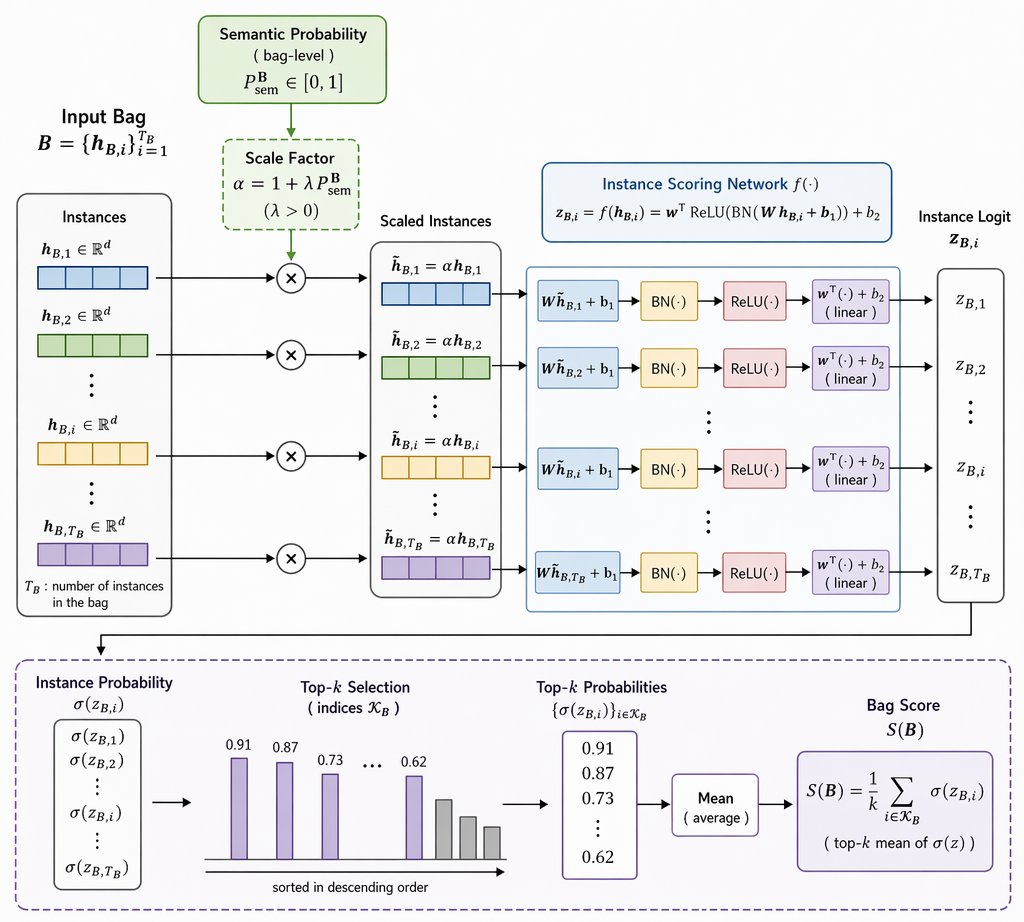}
    \caption{HaMI architecture.}
    \label{fig:hami_arch}
\end{figure}

\subsection{Max Pooling Architecture}
\label{app:maxpool_architecture}
Figure~\ref{fig:max_arch_appendix} illustrates the response-level detector. Each state is projected into a lower-dimensional space, and max pooling is used to extract a bag representation by selecting the most salient features. A linear classifier then produces a hallucination probability.

The model is trained with a binary classification objective. Section~\ref{sec:margin_max_pooling} analyzes its local margin dynamics and gradient dilution under activation sparsity. 

\begin{figure}[H]
    \centering
    \includegraphics[width=0.85\linewidth]{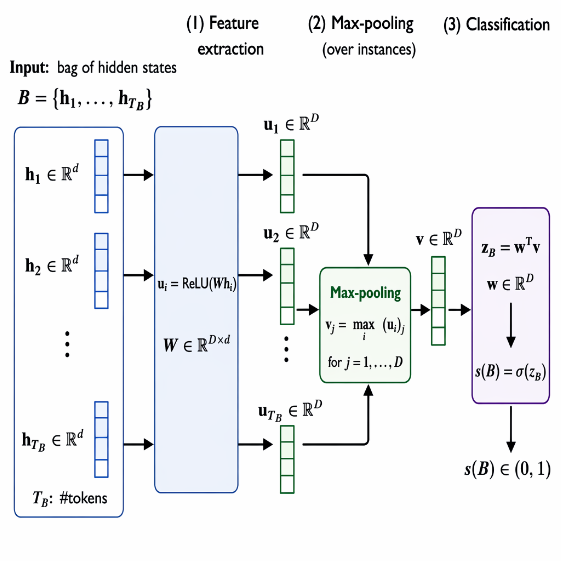}
    \caption{Max-pooling architecture. The projection matrix drawn as $W$ corresponds to $W^\top$ in Section~\ref{sec:maxpool_arch}.}
    \label{fig:max_arch_appendix}
\end{figure}

\subsection{Selected Layers}
\label{app:layer_selection}
The main-text results are reported at the layer that achieves the best validation AUC for each method. To make that choice auditable, we state the protocol and list the resulting layer indices.

On QA, transformer layers were searched with a stride of two (layers $0, 2, 4, \dots$), a detector was trained independently at each searched layer, and the layer with the highest validation AUC was retained. On LongFact++, the stored layers are every second layer of LLaMA-3.1-8B-Instruct and every fourth of LLaMA-3.3-70B-Instruct (Appendix~\ref{app:length_study}). \textbf{The identical protocol, search stride, and search range were applied to our method and to every hidden-state baseline}, so no method received a wider layer search than any other. Selection used the validation split only; the test split was never consulted during selection.

Tables~\ref{tab:layer_llama31}--\ref{tab:layer_llama33} list the selected indices. On the four short-answer benchmarks the selection concentrates near the middle of the network: the median selected layer is $14$ of $32$ for LLaMA-3.1-8B, $18$ of $40$ for Mistral-Nemo-12B and $36$ of $80$ for LLaMA-3.3-70B, i.e. $44$--$45\%$ of depth in every case, consistent with prior probing results \citep{saplma,knowmorethan}. The spread around that median is small for the two smaller generators (layers $12$--$16$ and $16$--$20$), so competing methods are read out at essentially the same depth and the reported differences reflect the aggregation strategy rather than the layer choice. LLaMA-3.3-70B is less uniform: most methods select layers $30$--$46$, but on BioASQ the probe baselines move to $50$--$54$ and HaMI~(SP) to $76$, and mean pooling selects $76$ on TriviaQA. Those are the cells where the validation curve is flattest near its peak, so the arg-max is weakly determined. LongFact++ layers are selected separately over all stored layers (Table~\ref{tab:layer_length}). Appendix~\ref{app:layer_sensitivity} shows the full layer-wise curves behind these choices.

\begin{table}[H]
\centering
\small
\setlength{\tabcolsep}{5pt}
\renewcommand{\arraystretch}{1.2}
\caption{Selected transformer layers for LLaMA-3.1-8B (32 layers).}
\label{tab:layer_llama31}
\begin{tabular}{lcccc}
\toprule
\textbf{Method} & TriviaQA & SQuAD & NQ & BioASQ \\
\midrule
SEP                     & 16 & 14 & 14 & 14 \\
LLMs Know  & 14 & 14 & 14 & 14 \\
SAPLMA                  & 14 & 14 & 14 & 14 \\
HaMI (Original)         & 12 & 14 & 14 & 14 \\
HaMI (SP)               & 12 & 12 & 12 & 14 \\
Mean Pooling            & 14 & 14 & 14 & 14 \\
\textbf{Max Pooling}    & 14 & 14 & 14 & 14 \\
\bottomrule
\end{tabular}
\end{table}

\begin{table}[H]
\centering
\small
\setlength{\tabcolsep}{5pt}
\renewcommand{\arraystretch}{1.2}
\caption{Selected transformer layers for Mistral-Nemo-12B (40 layers).}
\label{tab:layer_mistral}
\begin{tabular}{lcccc}
\toprule
\textbf{Method} & TriviaQA & SQuAD & NQ & BioASQ \\
\midrule
SEP                     & 16 & 18 & 18 & 18 \\
LLMs Know  & 16 & 18 & 18 & 18 \\
SAPLMA                  & 16 & 18 & 20 & 18 \\
HaMI (Original)         & 16 & 18 & 18 & 18 \\
HaMI (SP)               & 16 & 16 & 18 & 18 \\
Mean Pooling            & 16 & 18 & 18 & 18 \\
\textbf{Max Pooling}    & 16 & 18 & 18 & 18 \\
\bottomrule
\end{tabular}
\end{table}

\begin{table}[H]
\centering
\small
\setlength{\tabcolsep}{5pt}
\renewcommand{\arraystretch}{1.2}
\caption{Selected transformer layers for LLaMA-3.3-70B (80 layers).}
\label{tab:layer_llama33}
\begin{tabular}{lcccc}
\toprule
\textbf{Method} & TriviaQA & SQuAD & NQ & BioASQ \\
\midrule
SEP                     & 42 & 36 & 34 & 50 \\
LLMs Know  & 36 & 38 & 36 & 50 \\
SAPLMA                  & 46 & 36 & 36 & 54 \\
HaMI (Original)         & 38 & 34 & 34 & 36 \\
HaMI (SP)               & 36 & 34 & 32 & 76 \\
Mean Pooling            & 76 & 34 & 32 & 32 \\
\textbf{Max Pooling}    & 32 & 34 & 34 & 30 \\
\bottomrule
\end{tabular}
\end{table}

\begin{table}[H]
\centering
\small
\setlength{\tabcolsep}{6pt}
\renewcommand{\arraystretch}{1.2}
\caption{Layers selected by validation AUC on LongFact++.}
\label{tab:layer_length}
\begin{tabular}{lcc}
\toprule
\textbf{Method} & LLaMA-3.1-8B-Instruct & LLaMA-3.3-70B-Instruct \\
\midrule
SEP                      &   4 &  60 \\
LLMs Know   &  22 &   4 \\
SAPLMA                   &  30 &  72 \\
HaMI                     &  12 &  64 \\
HaMI (SP)                &  12 &  64 \\
Mean                     &  12 &  32 \\
\textbf{Max}             &  12 &  20 \\
\bottomrule
\end{tabular}
\end{table}

\clearpage
\section{Proofs and Additional Theory}
\label{app:proofs}
This appendix proves the results of Sections~\ref{sec:hami_margin_scaling}--\ref{sec:rademacher}; their empirical checks are in Appendices~\ref{app:validation}--\ref{app:integrated_checks}.
\subsection{Batch-Normalization Decomposition}
\label{app:bn_decomposition}
The step condensed in Section~\ref{sec:hami_margin_scaling} in full.

We analyze the effect of the scaling factor $p_\mathbf{B} = \lambda P_{\mathrm{sem}}^\mathbf{B}$. The pre-normalization feature with scaling during the forward pass is $(1+p_\mathbf{B})\mathbf{x}_{\mathbf{B},i} + \mathbf{b}_1$. During inference, the running mean $\mu_j$ and standard deviation $\sigma_j$ are fixed. The $j$-th dimension of the BN output is decomposed as:
\begin{align*}
    \mathrm{BN}((1+p_\mathbf{B})\mathbf{x}_{\mathbf{B},i} + \mathbf{b}_1)_j &= \gamma_j \frac{(1+p_\mathbf{B})x_{\mathbf{B},i,j} + b_{1,j} - \mu_j}{\sigma_j} + \beta_j \\
    &= \mathrm{BN}(\mathbf{x}_{\mathbf{B},i} + \mathbf{b}_1)_j + p_\mathbf{B} \gamma_j \frac{x_{\mathbf{B},i,j}}{\sigma_j}.
\end{align*}

Let $\mathcal{U}_{\mathbf{B},i}$ denote the active set of indices where the ReLU activation is strictly positive. Assuming the active set remains invariant under the scaling $p_\mathbf{B}$, the exact scaled logit $\tilde{z}_{\mathbf{B},i}$ becomes:
\begin{equation*}
    \tilde{z}_{\mathbf{B},i} = \sum_{j \in \mathcal{U}_{\mathbf{B},i}} w_j \left( \mathrm{BN}(\mathbf{x}_{\mathbf{B},i} + \mathbf{b}_1)_j + p_\mathbf{B} \gamma_j \frac{x_{\mathbf{B},i,j}}{\sigma_j} \right) + b_2 = z_{\mathbf{B},i} - p_\mathbf{B} C(\mathbf{x}_{\mathbf{B},i}),
\end{equation*}
where $C(\mathbf{x}_{\mathbf{B},i}) = - \frac{\partial \tilde{z}_{\mathbf{B},i}}{\partial p_\mathbf{B}} = - \sum_{j \in \mathcal{U}_{\mathbf{B},i}} w_j \gamma_j \frac{x_{\mathbf{B},i,j}}{\sigma_j}$.

\subsection{Proof of Theorem \ref{theorem:hami}}
\label{app:proof_hami}
The theorem analyzes the scaling on a fixed network. Separately trained HaMI and HaMI (SP) models, as compared in Table~\ref{tab:margin_comparison}, are related to it only approximately, and the theorem makes no statement about ranking or AUC.
Fix the network and the local ReLU and TopK selections. The same derivation applies to the integrated sensitivity in Appendix~\ref{app:active_non_invariant}. Assume both class-conditional product expectations are positive, as in the theorem.
Given the exact decomposition $\tilde{z}_{\mathbf{B},i} = z_{\mathbf{B},i} - \lambda P_{\mathrm{sem}}^\mathbf{B} C(\mathbf{x}_{\mathbf{B},i})$, the change in the bag-level logit, $\Delta Z_{\mathbf{B}} = \tilde{Z}_{\mathbf{B}} - Z_{\mathbf{B}}$, is exactly:
\begin{equation*}
    \Delta Z_{\mathbf{B}} = - \lambda P_{\mathrm{sem}}^\mathbf{B} \bar{C}_{\mathbf{B}}.
\end{equation*}
The total margin change $\Delta M_{hami} = M_{hami}(\lambda) - M_{hami}(0)$ is the sum of contributions from the positive and negative bags:
\begin{align*}
    \Delta M_{hami} &= \frac{1}{|\mathcal{S}|} \left( \sum_{\mathbf{B} \in \mathcal{S}_{pos}} (+1) \Delta Z_{\mathbf{B}} + \sum_{\mathbf{B} \in \mathcal{S}_{neg}} (-1) \Delta Z_{\mathbf{B}} \right) \\
    &= \frac{\lambda}{|\mathcal{S}|} \left( \sum_{\mathbf{B} \in \mathcal{S}_{neg}} P_{\mathrm{sem}}^{\mathbf{B}} \bar{C}_{\mathbf{B}} - \sum_{\mathbf{B} \in \mathcal{S}_{pos}} P_{\mathrm{sem}}^{\mathbf{B}} \bar{C}_{\mathbf{B}} \right).
\end{align*}
By expressing the sums as empirical averages over the respective subsets, the overall expected margin change over the dataset $\mathcal{S}$ is exactly:
\begin{align*}
    &\Delta M_{hami} \\
    &= \lambda \left( \frac{|\mathcal{S}_{neg}|}{|\mathcal{S}|} \mathbb{E}_{\mathbf{B} \in \mathcal{S}_{neg}}[P_{\mathrm{sem}}^{\mathbf{B}} \bar{C}_{\mathbf{B}}] - \frac{|\mathcal{S}_{pos}|}{|\mathcal{S}|} \mathbb{E}_{\mathbf{B} \in \mathcal{S}_{pos}}[P_{\mathrm{sem}}^{\mathbf{B}} \bar{C}_{\mathbf{B}}] \right)\\
    &=\lambda \mathbb{E}_{\mathbf{B} \in \mathcal{S}_{pos}}[P_{\mathrm{sem}}^{\mathbf{B}} \bar{C}_{\mathbf{B}}]\frac{|\mathcal{S}_{pos}|}{|\mathcal{S}|} \left( \gamma \frac{|\mathcal{S}_{neg}|}{|\mathcal{S}_{pos}| } - 1\right),
\end{align*}
where $\mathbb{E}_{\mathbf{B} \in \mathcal{S}_{neg}}$ and $\mathbb{E}_{\mathbf{B} \in \mathcal{S}_{pos}}$ denote the empirical expectations. From this expression, it follows that the expected margin increases, if and only if:
\begin{equation*}
\frac{|\mathcal{S}_{neg}|}{|\mathcal{S}_{pos}|} > \frac{1}{\gamma}.
\end{equation*}
The scaling lowers every logit, but it lowers negative-bag logits more than positive-bag logits in aggregate, which increases the expected logit-space margin.

\subsection{Proof of Bag-level Margin Dynamics under Logistic Loss}
\begin{lemma}[Bag-level Margin Dynamics]
\label{lem:margin_dynamics}
For a sampled bag $\mathbf{B}$ at parameters in a smooth region, its own logistic-loss update produces a positive first-order change in its signed margin $m_B=y_{\mathbf{B}} z_B$, whenever $\nabla_\theta z_B \neq 0$.
The magnitude of this increase is determined by the squared gradient norm scaled by the current margin's confidence.
\end{lemma}
This lemma states that a gradient step on the logistic loss of a bag enlarges that bag's margin to first order, by an amount proportional to the squared gradient norm and scaled by $(1+\exp(m_\mathbf{B}))^{-1}$.
\paragraph{Proof.}
\label{app:proof_margin_dynamics}
We first derive the first-order change in the bag-level margin. Let $\mathcal{L}_\mathbf{B}(\theta) = \log(1 + \exp(-y_\mathbf{B} z_\mathbf{B}(\theta)))$ be the logistic loss for a bag $B$. Its gradient with respect to $\theta = \{W, w\}$ is:
\begin{align*}
\nabla_\theta \mathcal{L}_\mathbf{B}(\theta) = \frac{\partial \mathcal{L}_\mathbf{B}}{\partial z_\mathbf{B}} \nabla_\theta z_\mathbf{B}(\theta) = -\frac{y_\mathbf{B}}{1 + \exp(y_\mathbf{B} z_\mathbf{B}(\theta))} \nabla_\theta z_\mathbf{B}(\theta).
\end{align*}
For a fixed starting parameter $\theta_t$, the update induced by bag $\mathbf B$ is $\theta_{\mathbf B,\eta} = \theta_t - \eta \nabla_\theta \mathcal{L}_\mathbf{B}(\theta_t)$. Its parameter shift is:
\begin{align*}
\theta_{\mathbf B,\eta} - \theta_t = \eta \frac{y_\mathbf{B}}{1 + \exp(m_\mathbf{B}(\theta_t))} \nabla_\theta z_\mathbf{B}(\theta_t),
\end{align*}
where $m_\mathbf{B}(\theta) = y_\mathbf{B} z_\mathbf{B}(\theta)$. For a sufficiently small learning rate $\eta$, the first-order Taylor expansion of $z_\mathbf{B}(\theta_{\mathbf B,\eta})$ around $\theta_t$ yields:
\begin{align*}
z_\mathbf{B}(\theta_{\mathbf B,\eta}) &= z_\mathbf{B}(\theta_t) + \nabla_\theta z_\mathbf{B}(\theta_t)^\top (\theta_{\mathbf B,\eta} - \theta_t) + O(\eta^2) \nonumber \\
&= z_\mathbf{B}(\theta_t) + \eta \frac{y_\mathbf{B}}{1 + \exp(m_\mathbf{B}(\theta_t))} \| \nabla_\theta z_\mathbf{B}(\theta_t) \|^2 + O(\eta^2).
\end{align*}
Multiplying both sides by $y_\mathbf{B}$ and using $y_\mathbf{B}^2 = 1$, we obtain the margin increase:
\begin{align*}
\Delta m_\mathbf{B}(\theta_t) = \eta \frac{\| \nabla_\theta z_\mathbf{B}(\theta_t) \|^2}{1 + \exp(m_\mathbf{B}(\theta_t))} + O(\eta^2).
\end{align*}
For a nonzero gradient and a sufficiently small step remaining in a smooth region, the positive first-order term dominates the remainder.

The squared gradient norm $\| \nabla_\theta z_\mathbf{B} \|^2$ can be decomposed into contributions from the classification weights $w$ and feature extraction weights $W$. Since $\theta$ consists of $\{w_j, W_j\}_{j=1}^D$, we have:
\begin{align*}
\| \nabla_\theta z_\mathbf{B} \|^2 = \sum_{j=1}^D \left( \left| \frac{\partial z_\mathbf{B}}{\partial w_j} \right|^2 + \left\| \nabla_{W_j} z_\mathbf{B} \right\|_2^2 \right) = \sum_{j=1}^D q_{\mathbf{B},j}.
\end{align*}

For \textbf{Max Pooling}, where $v_j = \max_i a(W_j^\top h_{\mathbf{B},i})$, the output depends only on the maximally activated instance $i^*_{\mathbf{B},j}$. The derivatives are:
\begin{align*}
\frac{\partial z_\mathbf{B}}{\partial w_j} = u_{\mathbf{B},i^*_{\mathbf{B},j}, j}, \quad \frac{\partial z_\mathbf{B}}{\partial W_j}
= w_j \, h_{\mathbf{B},i^*_{\mathbf{B},j}} \, \mathbf{1}\{W_j^\top h_{\mathbf{B},i^*_{\mathbf{B},j}} > 0\}.
\end{align*}
Thus, $q_{\mathbf{B},j}^{\max} = u_{\mathbf{B},i^*_{\mathbf{B},j}, j}^2 + w_j^2 \mathbf{1}\left\{W_j^\top h_{\mathbf{B},i^*_{\mathbf{B},j}}>0\right\}\left\|h_{\mathbf{B},i^*_{\mathbf{B},j}}\right\|^2$.

For \textbf{mean Pooling}, where $v_j = \frac{1}{T_{\mathbf{B}}} \sum_{i=1}^{T_{\mathbf{B}}} a(W_j^\top h_{\mathbf{B},i})$, the derivatives are averaged over the bag:
\begin{align*}
\frac{\partial z_\mathbf{B}}{\partial w_j} = \frac{1}{T_{\mathbf{B}}} \sum_{i=1}^{T_{\mathbf{B}}} u_{\mathbf{B},i,j}, \quad
\frac{\partial z_\mathbf{B}}{\partial W_j}
=\frac{w_j}{T_{\mathbf{B}}} \sum_{i=1}^{T_{\mathbf{B}}}
h_{\mathbf{B},i} \, \mathbf{1}\{W_j^\top h_{\mathbf{B},i} > 0\}.
\end{align*}
This yields 
$q_{\mathbf{B},j}^{\mathrm{mean}}
= \left( \frac{1}{T_{\mathbf{B}}} \sum_{i=1}^{T_{\mathbf{B}}} u_{\mathbf{B},i,j} \right)^2
+ \frac{w_j^2}{T_{\mathbf{B}}^2}
\left\|
\sum_{i=1}^{T_{\mathbf{B}}}
h_{\mathbf{B},i}\,\mathbf{1}\{W_j^\top h_{\mathbf{B},i} > 0\}
\right\|_2^2.
$

Combining these with the margin update equation completes the proof.
\subsection{Proof of Lemma \ref{thm:global}}
\label{app:proof_global_margin}
For a fixed starting parameter $\theta$, define $\theta_{\mathbf B,\eta}:=\theta-\eta\nabla_\theta\mathcal L_{\mathbf B}(\theta)$ and $\mathcal M_\eta(\theta):=\mathbb E_{\mathbf B}[m_{\mathbf B}(\theta_{\mathbf B,\eta})]$, as in Lemma~\ref{thm:global}. The expectation is over the finite empirical bag distribution, with $\mathcal M_0(\theta)=\mathcal M(\theta)$. Each term evaluates the same bag that induces its update, starting from the same $\theta$.

$z_\mathbf{B}$ is piecewise affine in its input for fixed weights. Jointly in $\theta=(W,w)$ it is piecewise polynomial and smooth inside regions with fixed ReLU and max-pooling selections. We take a sufficiently small update within such a region.
The derivative with respect to $\eta$ at $\eta=0$ is:
\begin{align*}
\left. \frac{\mathrm{d}}{\mathrm{d}\eta} m_\mathbf{B}(\theta_{\mathbf B,\eta}) \right|_{\eta=0} = - \nabla_\theta m_\mathbf{B}(\theta)^{\top} \nabla_\theta \mathcal{L}_\mathbf{B}(\theta).
\end{align*}
Substituting the gradient $\nabla_\theta \mathcal{L}_\mathbf{B}(\theta) = -\frac{y_\mathbf{B}}{1 + \exp(m_\mathbf{B}(\theta))} \nabla_\theta z_\mathbf{B}(\theta)$:
\begin{align*}
- \nabla_\theta m_\mathbf{B}(\theta)^{\top} \nabla_\theta \mathcal{L}_\mathbf{B}(\theta) &= -(y_\mathbf{B} \nabla_\theta z_\mathbf{B}(\theta))^{\top} \left( -\frac{y_\mathbf{B}}{1 + \exp(m_\mathbf{B}(\theta))} \nabla_\theta z_\mathbf{B}(\theta) \right) \nonumber \\
&= \frac{\| \nabla_\theta z_\mathbf{B}(\theta) \|^2}{1 + \exp(m_\mathbf{B}(\theta))}.
\end{align*}
For the finite empirical bag distribution, differentiation commutes with averaging:
\begin{align*}
\left. \frac{\mathrm{d}}{\mathrm{d}\eta} \mathcal{M}_\eta(\theta) \right|_{\eta=0} = \mathbb{E}_\mathbf{B} \left[ \frac{\| \nabla_\theta z_\mathbf{B}(\theta) \|^2}{1 + \exp(m_\mathbf{B}(\theta))} \right] \ge 0.
\end{align*}
The Taylor expansion around $\eta=0$ gives:
\begin{align*}
\mathcal{M}_\eta(\theta) = \mathcal{M}_0(\theta) + \eta \mathbb{E}_\mathbf{B} \left[ \frac{\| \nabla_\theta z_\mathbf{B}(\theta) \|^2}{1 + \exp(m_\mathbf{B}(\theta))} \right] + O(\eta^2).
\end{align*}
For sufficiently small $\eta > 0$, $\mathcal{M}_\eta(\theta) \ge \mathcal{M}_0(\theta)$, with strict increase when the expected first-order term is positive. Each bag is evaluated after its own update; this is not a statement about the population margin after one shared update.
\subsection{Proof of Theorem~\ref{thm:sparse}}
\label{app:proof_sparse}
Write $\mathbf{h}_i=\mathbf{h}_{\mathbf{B},i}$, $u_{ij}=u_{\mathbf{B},i,j}$, $T=T_\mathbf{B}$, $\alpha_j=\alpha_{\mathbf{B},j}$, $\kappa_j=\kappa_{\mathbf{B},j}$, and $A_j:=\{i:u_{ij}>0\}$, so $|A_j|=\alpha_jT$. Since $\partial u_{ij}/\partial W_j=\mathbf{h}_i\mathbf{1}\{i\in A_j\}$, the squared gradient norm of $z_\mathbf{B}=w^\top\mathbf{v}$ decomposes over features as $\sum_jq_j$ with $q_j=(\partial z_\mathbf{B}/\partial w_j)^2+\|\partial z_\mathbf{B}/\partial W_j\|_2^2$. For $j\notin J_\mathbf{B}$, every $u_{ij}$ is zero, so $q_j=0$ for both poolings.

\paragraph{1. Max pooling.}
For $j\in J_\mathbf{B}$, $v^{\max}_j=u_{i^*_j j}$ and $\partial z^{\max}_\mathbf{B}/\partial W_j=w_j\mathbf{h}_{i^*_j}$, so $q_j^{\max}=v_j^2+w_j^2\|\mathbf{h}_{i^*_j}\|_2^2=a_{\mathbf{B},j}$.

\paragraph{2. Mean pooling.}
For $j\in J_\mathbf{B}$, $v^{\mathrm{mean}}_j=\frac1T\sum_{i\in A_j}u_{ij}\le\frac{|A_j|}{T}\max_iu_{ij}=\alpha_jv_j$ and
\[
\Big\|\frac{\partial z^{\mathrm{mean}}_\mathbf{B}}{\partial W_j}\Big\|_2=\frac{|w_j|}{T}\Big\|\sum_{i\in A_j}\mathbf{h}_i\Big\|_2\le\alpha_j\,|w_j|\max_{i\in A_j}\|\mathbf{h}_i\|_2=\alpha_j\,\kappa_j\,|w_j|\,\|\mathbf{h}_{i^*_j}\|_2 .
\]
Hence $q_j^{\mathrm{mean}}\le\alpha_j^2v_j^2+\alpha_j^2\kappa_j^2w_j^2\|\mathbf{h}_{i^*_j}\|_2^2\le\alpha_j^2\kappa_j^2\,a_{\mathbf{B},j}$, using $\kappa_j\ge1$.

\paragraph{3. Conclusion.}
Summing over $J_\mathbf{B}$ gives $\|\nabla_\theta z^{\mathrm{mean}}_\mathbf{B}\|^2\le\sum_{j\in J_\mathbf{B}}\alpha_j^2\kappa_j^2a_{\mathbf{B},j}$ and $\|\nabla_\theta z^{\max}_\mathbf{B}\|^2=\sum_{j\in J_\mathbf{B}}a_{\mathbf{B},j}$; dividing both by $\sum_{k\in J_\mathbf{B}}a_{\mathbf{B},k}$ gives the stated bound. The argument does not use the label of $\mathbf{B}$; $J_\mathbf{B}\neq\emptyset$ only excludes bags on which every feature is inactive, where both gradients vanish. A first-layer bias is covered by appending a constant $1$ to each $\mathbf{h}_i$; an output bias has gradient $1$ under both poolings and is omitted, as in the theorem.

\paragraph{4. Sparse bags.}
In particular, suppose each feature is active on at most $s_\mathbf{B}$ tokens of the bag and the active hidden states have norms in $[\underline{g},\overline{g}]$. Then $\alpha_{\mathbf{B},j}\le s_\mathbf{B}/T_\mathbf{B}$ and $\kappa_{\mathbf{B},j}\le\overline{g}/\underline{g}$, and since $\sum_{j}\pi_{\mathbf{B},j}=1$ the bound gives
\[
\frac{\|\nabla_\theta z^{\max}_\mathbf{B}\|^2}{\|\nabla_\theta z^{\mathrm{mean}}_\mathbf{B}\|^2}\ \ge\ \Big(\frac{\underline{g}}{\overline{g}}\Big)^2\Big(\frac{T_\mathbf{B}}{s_\mathbf{B}}\Big)^2=\Omega\Big(\Big(\frac{T_\mathbf{B}}{s_\mathbf{B}}\Big)^2\Big)
\]
for positive and negative bags alike: when few tokens carry each feature, mean pooling dilutes the gradient by the square of the token share, while max pooling does not.

\subsection{Proof of Proposition \ref{prop:rc_bounds}}
\label{app:proof_proposition_1}
We provide the derivations for the Rademacher complexity bounds of both the feature-extraction-based model ($\mathcal{F}_{\mathrm{feat}}$) and the baseline model ($\mathcal{F}_{\mathrm{base}}$).

\paragraph{1. Bound for the Model with Feature Extraction Layer ($\mathcal{F}_{\mathrm{feat}}$)}
We first bound the empirical Rademacher complexity of the hypothesis class:
\[
\mathcal{F}_{\mathrm{feat}} = \left\{ f_\theta(\mathbf{B}) = w^\top \rho\big(\{a(W_j^\top h_t)\}_{j,t}\big) \;\middle|\; \|W_j\|_2 \le B_1, \|w\|_2 \le B_2 \right\}.
\]

Applying the Cauchy-Schwarz inequality to separate the classification weights $w$, we have:
\begin{equation*}
    \hat{\mathcal{R}}_S(\mathcal{F}_{\mathrm{feat}}) \le \frac{B_2}{n} \mathbb{E}_\sigma \left[ \sup_{\|W_j\|_2 \le B_1} \left\| \sum_{i=1}^{n} \sigma_i \rho(\{a(W^\top h_{i,t})\}_t) \right\|_2 \right].
\end{equation*}
Since the norm constraint $\|W_j\|_2 \le B_1$ applies independently to each column $j \in \{1, \dots, D\}$, the supremum of the $L_2$ norm factors into independent suprema across the $D$ dimensions:
\begin{equation*}
    \sup_{W} \left\| \sum_{i=1}^{n} \sigma_i \rho(\{a(W^\top h_{i,t})\}_t) \right\|_2 \le \sqrt{D} \sup_{\|w_1\|_2 \le B_1} \left| \sum_{i=1}^{n} \sigma_i \max_{1 \le t \le T_i} a(w_1^\top h_{i,t}) \right|.
\end{equation*}
By introducing a factor of 2 to handle the absolute value, and sequentially applying Maurer's vector contraction inequality~\citep{maurer} for the max operator which introduces independent Rademacher variables $\sigma_{i,t} \in \{-1, +1\}$ for each instance $t$ and the Ledoux-Talagrand contraction inequality for the ReLU function $a(\cdot)$, we obtain:
\begin{align*}
    \hat{\mathcal{R}}_S(\mathcal{F}_{\mathrm{feat}}) &\le \frac{2 \sqrt{2} B_2 \sqrt{D}}{n} \mathbb{E}_\sigma \left[ \sup_{\|w_1\|_2 \le B_1} \sum_{i=1}^{n} \sum_{t=1}^{T_i} \sigma_{i,t} w_1^\top h_{i,t} \right] \nonumber \\
    &\le \frac{2 \sqrt{2} B_1 B_2 \sqrt{D}}{n} \mathbb{E}_\sigma \left[ \left\| \sum_{i=1}^{n} \sum_{t=1}^{T_i} \sigma_{i,t} h_{i,t} \right\|_2 \right] \nonumber \\
    &\le \frac{2 \sqrt{2} R B_1 B_2 \sqrt{DT}}{\sqrt{n}},
\end{align*}
where $T = \max_i T_i$. This gives $\mathcal{O}(\sqrt{DT/n})$ scaling with no explicit dependence on $d$ under the stated norm bounds.

\paragraph{2. Bound for the Baseline Model ($\mathcal{F}_{\mathrm{base}}$)}
For the baseline class $\mathcal{F}_{\text{base}}$, max pooling $\rho(\mathbf{B}_i)$ is performed directly on the $d$-dimensional input space. Let $z_i = \rho(\mathbf{B}_i) \in \mathbb{R}^d$. Under the assumption $\|h_{i,t}\|_2 \le R$, which implies $|h_{i,t,\ell}| \le R$ for each coordinate, the norm of the pooled vector is bounded by:
\[
\|z_i\|_2 \le \sqrt{\sum_{k=1}^d (\max_{t} |h_{i,t,k}|)^2} \le \sqrt{d R^2} = R\sqrt{d}.
\]
Treating this as a standard two-layer neural network with weight constraints $B_1, B_2$ acting on $\{z_i\}$, we apply standard Rademacher complexity results:
\begin{align*}
    \hat{\mathcal{R}}_S(\mathcal{F}_{\mathrm{base}}) &\le \frac{2 B_1 B_2\sqrt D}{n} \mathbb{E}_\sigma \left\| \sum_{i=1}^n \sigma_i z_i \right\|_2 \nonumber \\
    &\le \frac{2 B_1 B_2\sqrt D}{\sqrt{n}} \max_i \|z_i\|_2 \le \frac{2 R B_1 B_2 \sqrt{Dd}}{\sqrt{n}}.
\end{align*}
This establishes the displayed coordinate-wise bound. A sharper estimate uses
\[
\|z_i\|_2^2\le\sum_{k=1}^d\max_t|h_{i,t,k}|^2
\le\sum_{t=1}^{T_i}\|h_{i,t}\|_2^2\le T_iR^2.
\]
Combining the estimates gives
$\hat{\mathcal R}_{\mathcal S}(\mathcal F_{\mathrm{base}})
\le 2RB_1B_2\sqrt{D\min(d,T)/n}$.
Thus, these norm-based upper bounds alone do not establish a smaller complexity for feature-first pooling. The feature-placement ablation (Section~\ref{ex:ablations}, Table~\ref{tab:raw_vs_reduced}) supplies the empirical comparison.
\subsection{Margin Expansion under Per-Bag Updates and \texorpdfstring{$\beta$}{beta}-Smoothness}
\label{app:beta-smooth-margin}

Lemma~\ref{thm:global} describes the average change in each bag's margin under its own update, $\mathcal M_\eta(\theta)-\mathcal M_0(\theta)$. Its first-order Taylor expansion is local in $\eta$: the $O(\eta^2)$ remainder can dominate the positive first-order term. We use a $\beta$-smoothness assumption to bound this remainder and specify a sufficient step-size condition for this average to increase.

\begin{assumption}[$\beta$-Smoothness of the Margin]
\label{assum:beta_smoothness}
For any sampled bag $\mathbf{B}$, assume the margin function $m_\mathbf{B}(\theta)$ is $\beta$-smooth on a neighborhood containing its update segment. That is, its gradient is Lipschitz continuous with a constant $\beta > 0$ there (ReLU and max pooling do not guarantee smoothness across selection boundaries, hence the local statement), implying that for parameters $\theta$ and update steps $\Delta\theta$ within that neighborhood:
\begin{equation*}
m_\mathbf{B}(\theta + \Delta \theta) \ge m_\mathbf{B}(\theta) + \nabla_\theta m_\mathbf{B}(\theta)^\top \Delta \theta - \frac{\beta}{2} \| \Delta \theta \|^2.
\end{equation*}
\end{assumption}

For the logistic-loss update induced by a sampled bag $\mathbf{B}$, let $\Delta \theta = -\eta \nabla_\theta \mathcal{L}_\mathbf{B}(\theta)$ and $\theta_{\mathbf B,\eta} := \theta + \Delta \theta$. Substituting this bag-specific update into the $\beta$-smoothness inequality yields:
\begin{equation*}
m_\mathbf{B}(\theta_{\mathbf B,\eta}) - m_\mathbf{B}(\theta) \ge - \eta \nabla_\theta m_\mathbf{B}(\theta)^\top \nabla_\theta \mathcal{L}_\mathbf{B}(\theta) - \frac{\beta \eta^2}{2} \| \nabla_\theta \mathcal{L}_\mathbf{B}(\theta) \|^2.
\end{equation*}

Recall from Appendix~\ref{app:proof_global_margin} that $\nabla_\theta m_\mathbf{B}(\theta) = y_\mathbf{B} \nabla_\theta z_\mathbf{B}(\theta)$ and $\nabla_\theta \mathcal{L}_\mathbf{B}(\theta) = - \frac{y_\mathbf{B}}{1 + \exp(m_\mathbf{B}(\theta))} \nabla_\theta z_\mathbf{B}(\theta)$. Since $y_\mathbf{B} \in \{-1, 1\}$, the squared norm of the loss gradient simplifies to:
\begin{equation*}
\| \nabla_\theta \mathcal{L}_\mathbf{B}(\theta) \|^2 = \frac{\| \nabla_\theta z_\mathbf{B}(\theta) \|^2}{(1 + \exp(m_\mathbf{B}(\theta)))^2}.
\end{equation*}

Substituting the gradient terms back into the inequality gives:
\begin{align*}
m_\mathbf{B}(\theta_{\mathbf B,\eta}) - m_\mathbf{B}(\theta) &\ge \eta \frac{\| \nabla_\theta z_\mathbf{B}(\theta) \|^2}{1 + \exp(m_\mathbf{B}(\theta))} - \frac{\beta \eta^2}{2} \frac{\| \nabla_\theta z_\mathbf{B}(\theta) \|^2}{(1 + \exp(m_\mathbf{B}(\theta)))^2} \nonumber \\
&= \eta \frac{\| \nabla_\theta z_\mathbf{B}(\theta) \|^2}{1 + \exp(m_\mathbf{B}(\theta))} \left( 1 - \frac{\beta \eta}{2(1 + \exp(m_\mathbf{B}(\theta)))} \right).
\end{align*}

Taking the expectation over the finite empirical bag distribution bounds the average change in each bag's own margin:
\begin{equation*}
\mathcal M_\eta(\theta)-\mathcal M_0(\theta) \ge \mathbb{E}_\mathbf{B} \left[ \eta \frac{\| \nabla_\theta z_\mathbf{B}(\theta) \|^2}{1 + \exp(m_\mathbf{B}(\theta))} \left( 1 - \frac{\beta \eta}{2(1 + \exp(m_\mathbf{B}(\theta)))} \right) \right].
\end{equation*}

If at least one bag has a nonzero gradient, a sufficient condition for $\mathcal M_\eta(\theta)>\mathcal M_0(\theta)$ is that the factor in parentheses is positive for every bag and every update stays in its smooth neighborhood. The step-size condition is:
\begin{equation*}
0<\eta < \frac{2(1 + \exp(m_\mathbf{B}(\theta)))}{\beta}\quad\text{for every }\mathbf B.
\end{equation*}

Since $\exp(m_\mathbf{B}(\theta)) > 0$, we have $1 + \exp(m_\mathbf{B}(\theta)) > 1$. Therefore, a sufficient condition for the learning rate to guarantee margin expansion, independent of the current margin value, is:
\begin{equation*}
0<\eta \le \frac{2}{\beta}.
\end{equation*}
Provided each update remains in its smooth neighborhood, this condition guarantees an increase in each bag's own margin whenever its gradient is nonzero, and hence in their average when at least one bag has a nonzero gradient.

\clearpage
\section{Extension of \texorpdfstring{Theorem~\ref{theorem:hami}}{Theorem 1} to Changing Activation Patterns}
\label{app:active_non_invariant}
In the analysis of ReLU networks, the fact that input perturbations cause changes in activation patterns poses a significant challenge. In this section, we show that the entire network is a Continuous Piecewise-Affine (CPWA) map, mathematically guaranteeing the validity of sensitivity analysis through path integration.

\subsection{Properties and Compositions of Piecewise-Affine Maps}
Here, we summarize the properties regarding the composition of piecewise-affine maps defined on a finite number of polyhedral regions.

\paragraph{Definition: Piecewise-Affine Map}
A function $F: \mathbb{R}^d \rightarrow \mathbb{R}^m$ is defined as piecewise-affine if there exists a partition of the input space $R_1, \ldots, R_M$, such that on each region $R_i$, it can be described using $A_i \in \mathbb{R}^{m \times d}$ and $b_i \in \mathbb{R}^m$ as follows:
\begin{align*}
F(x) = A_i x + b_i \quad (\forall x \in R_i).
\end{align*}

\paragraph{Primary Properties of Composition}
The following properties hold for the composition of affine maps and piecewise-affine maps. Furthermore, since the composition of continuous functions is continuous, if each component is continuous, the composite function is also a continuous piecewise-affine (CPWA) map.

\begin{enumerate}
    \item \textbf{(Affine) $\circ$ (Affine)} is an affine map.
    \item \textbf{(Affine) $\circ$ (PWA)} is a piecewise-affine map.
    \item \textbf{(PWA) $\circ$ (Affine)} is a piecewise-affine map.
    \item \textbf{(PWA) $\circ$ (PWA)} is a piecewise-affine map.
\end{enumerate}

\paragraph{Mathematical Proofs}
\begin{itemize}
    \item \textbf{Proof of 1:} Let $H(x) := Ax + b$ and $G(y) := Cy + d$. Then $(G \circ H)(x) = C(Ax+b)+d = (CA)x + (Cb+d)$, which is an affine map.
    \item \textbf{Proof of 2:} Let $F$ be an affine map $A_i x + b_i$ on region $R_i$. Then $(G \circ F)(x) = C(A_i x + b_i) + d = (CA_i)x + (Cb_i + d)$, making it a PWA map on the same region $R_i$.
    \item \textbf{Proof of 3:} Let $G$ be an affine map $C_j y + d_j$ on region $S_j$. The inverse image $T_j = H^{-1}(S_j)$ is also a polyhedral region, and on each $T_j$, the map becomes $(C_j A)x + (C_j b + d_j)$, thus it is a PWA map.
    \item \textbf{Proof of 4:} Let the regions for $F$ and $G$ be $R_i$ and $S_j$, respectively. Defining the region where the composite function is affine as $T_{ij} = R_i \cap \{x : A_i x + b_i \in S_j\}$, this forms a polyhedral region as the intersection of polyhedral regions. On each $T_{ij}$, it can be written as $C_j(A_i x + b_i) + d_j$, meaning it is a PWA map.
\end{itemize}

\subsection{Analysis of Continuous Piecewise-Affine Properties in ReLU Networks \texorpdfstring{\citep{relu_CPWA}}{}}
A function $F: \mathbb{R}^d \rightarrow \mathbb{R}^m$ is defined as \textbf{continuous piecewise-affine (CPWA)} if the input space is partitioned into a finite number of polyhedral regions $R_1, \ldots, R_L$, and on each region $R_\ell$, it can be described as an affine map $F(x) = A_\ell x + b_\ell$ using $A_\ell \in \mathbb{R}^{m \times d}$ and $b_\ell \in \mathbb{R}^m$.

In a ReLU network, each region $R_\ell$ corresponds to a region where the \textbf{active pattern} of all ReLU units is fixed. That is, within a certain region, the On/Off state of each unit is invariant, and the ReLU can be viewed as a linear transformation $D u$ using a diagonal mask matrix $D = \operatorname{diag}(\mathbb{1}_{\{u > 0\}})$ \citep{relu_oresen}. For example, in a two-layer ReLU network $F(x) = A_2 \operatorname{ReLU}(A_1 x + b_1) + b_2$, since $D_1 = \operatorname{diag}(\mathbb{1}_{\{A_1 x + b_1 > 0\}})$ is fixed within a specific activation region,
\begin{align*}
F(x) = A_2 D_1 (A_1 x + b_1) + b_2 = A_2 D_1 A_1 x + (A_2 D_1 b_1 + b_2)
\end{align*}
which reduces to an affine map $A'x + c$.

Generalizing this to an $L$-layer deep ReLU network, let the output of each layer be $h_0 = x$, $h_\ell = \operatorname{ReLU}(A_\ell h_{\ell-1} + b_\ell)$ for $\ell=1, \ldots, L-1$, and the final output be $F(x) = A_L h_{L-1} + b_L$. Within a region where a certain activation pattern $D = (D_1, \ldots, D_{L-1})$ is fixed, it can be written as follows:
\begin{align*}
F(x) = \left( A_L D_{L-1} A_{L-1} \cdots D_1 A_1 \right) x + c(D)
\end{align*}
Here, $c(D)$ is a constant term that depends only on the activation pattern and is independent of the input $x$ within the region:
\begin{align*}
c(D) = b_L + A_L D_{L-1} b_{L-1} + A_L D_{L-1} A_{L-1} D_{L-2} b_{L-2} + \cdots + A_L D_{L-1} \cdots A_2 D_1 b_1.
\end{align*}
Thus, a ReLU network partitions the input space into numerous polyhedral regions and possesses the structure of a CPWA function that behaves locally linearly within each region.

\subsection{Sensitivity Analysis via Path Integration of Continuous Piecewise-Affine Functions \texorpdfstring{\citep{path_integration}}{}}
Let the function $F: \mathbb{R}^d \rightarrow \mathbb{R}$ be a continuous piecewise-affine (CPWA) function described as $F(x) = a_m^{\top} x + b_m$ on a finite number of polyhedral regions $R_1, \ldots, R_M$. As discussed previously, deep ReLU networks also reduce to this form. In this subsection, we derive the change in output for a finite perturbation from an input $x_0 \in \mathbb{R}^d$ to $x_1 = x_0 + \Delta x$ as the integral of local sensitivities along a path.

First, we define a linear path $\gamma(t) := x_0 + t \Delta x$ (where $t \in [0, 1]$) connecting the two points, and consider a single-variable function $g(t) := F(\gamma(t))$ along this path. Even if $F$ is a multivariable CPWA function, $g(t)$ becomes an affine map with respect to $t$ on each interval, thus forming a continuous piecewise-linear (CPWL) function on $[0, 1]$. Specifically, when the path is within a region $R_m$,
\begin{align*}
g(t) = a_m^{\top}(x_0 + t \Delta x) + b_m = (a_m^{\top} \Delta x) t + (a_m^{\top} x_0 + b_m)
\end{align*}
and its slope changes each time the path enters a different region.
Therefore, letting the points where the path crosses region boundaries be $0 = \tau_0 < \tau_1 < \cdots < \tau_M = 1$, $g(t)$ can be written as $g(t) = \alpha_m t + \beta_m$ in each interval $(\tau_{m-1}, \tau_m)$. Applying the fundamental theorem of calculus in each interval and summing over all intervals yields:
\begin{align*}
g(1) - g(0) = \sum_{m=1}^M [g(\tau_m) - g(\tau_{m-1})] = \sum_{m=1}^M \int_{\tau_{m-1}}^{\tau_m} g^{\prime}(t) \mathrm{d} t = \int_0^1 g^{\prime}(t) \mathrm{d} t.
\end{align*}
Since the breakpoints $\tau_m$ are finite in number (measure zero), they do not affect the integral value. Here, by the chain rule, at points where $F$ is differentiable along $\gamma(t)$, $g^{\prime}(t) = \nabla F(\gamma(t))^{\top} \gamma^{\prime}(t) = \nabla F(x_0 + t \Delta x)^{\top} \Delta x$, finally leading to the following equation:
\begin{align*}
F(x_1) - F(x_0) = \int_0^1 \nabla F(x_0 + t \Delta x)^{\top} \Delta x \mathrm{d} t.
\end{align*}
Viewing this result in the form of a finite sum gives $F(x_1) - F(x_0) = \sum_{m=1}^M (\tau_m - \tau_{m-1}) a_m^{\top} \Delta x$, which explicitly shows the process where the local sensitivity $a_m^{\top} \Delta x$ obtained in each region accumulates according to the residence time.

\subsection{Application to HaMI: Semantic Scaling and Integral Sensitivity Analysis}
In this section, we apply the theoretical foundations established so far to the semantic scaling of HaMI under inference-time. First, we show that the bag-level logit score $Z_{\mathbf{B}}(p) = \operatorname{TopK\text{-}Mean}_i f((1+p)h_{\mathbf{B},i})$ when scaling the hidden state as $h(p) = (1+p)h$ preserves the CPWA property.

To this end, we prove that the aggregation operation $\operatorname{TopKSum}$ is a CPWA map. For two continuous functions $f$ and $g$, their maximum can be written as $\max \{f(x), g(x)\} = \frac{1}{2}(f(x)+g(x)+|f(x)-g(x)|)$. Because the sum, difference, absolute value, and composition of continuous functions are continuous, the maximum of a finite number of continuous functions is also continuous.
The sum of the top $k$ elements for a vector $u \in \mathbb{R}^T$ can be defined using the linear form $a_S^{\top} u$ for the index set $S \subseteq \{1, \ldots, T\}$ with $|S|=k$ as:
\begin{align*}
\operatorname{TopKSum}_k(u) = \max_{S:|S|=k} a_S^{\top} u
\end{align*}
This is an operation that takes the maximum of a finite number of linear functions (Max-affine); it is affine on each region $R_S = \{u : a_S^{\top} u \geq a_{S'}^{\top} u, \forall S'\}$ and continuous overall. Therefore, $\operatorname{TopKSum}$ and its scalar multiple, $\operatorname{TopK\text{-}Mean}$, are CPWA maps.

From the above discussion, even if the model $f$ contains ReLU, Batch Normalization, and various pooling operations, their composition $Z_{\mathbf{B}}(p)$ becomes a single-variable continuous piecewise-linear (CPWL) function with respect to $p$. Based on this property, applying the sensitivity analysis from the previous section allows the change in output to be described as follows:
\begin{align*}
Z_{\mathbf{B}}(p_{\mathbf{B}}) - Z_{\mathbf{B}}(0) &= \int_0^{p_{\mathbf{B}}} Z_{\mathbf{B}}^{\prime}(p) \mathrm{d} p, \\
Z_{\mathbf{B}}(p_{\mathbf{B}}) &= Z_{\mathbf{B}}(0) - p_{\mathbf{B}} \bar{C}_{\mathbf{B}}^{\mathrm{int}}(p_{\mathbf{B}}).
\end{align*}
Here, $\bar{C}_{\mathbf{B}}^{\mathrm{int}}(p_{\mathbf{B}}) := -\frac{1}{p_{\mathbf{B}}} \int_0^{p_{\mathbf{B}}} Z_{\mathbf{B}}^{\prime}(p) \mathrm{d} p$ is the average cost term calculated across changes in the active set. Furthermore, by setting $p_{\mathbf{B}} = \lambda P_{\mathrm{sem}}^{\mathbf{B}}$, the general formula for HaMI is obtained:
\begin{align*}
Z_{\mathbf{B}}(\lambda P_{\mathrm{sem}}^{\mathbf{B}}) = Z_{\mathbf{B}}(0) - \lambda P_{\mathrm{sem}}^{\mathbf{B}} \bar{C}_{\mathbf{B}}^{\mathrm{int}}(\lambda P_{\mathrm{sem}}^{\mathbf{B}}).
\end{align*}
This formulation maintains mathematical consistency even under dynamic situations where the activation pattern changes.
\paragraph{Numerical Approximation:} 
In our empirical evaluations, we set $\lambda$ to $1$ and approximate the integral by discretizing the interval $p \in [0, p_{\mathbf{B}}]$ into $N=1000$ uniform steps. For each step $k$, we compute the local gradient $Z_{\mathbf{B}}^{\prime}(p_k)$. This high-resolution discretization provides an accurate numerical approximation of the integrated sensitivity, even when the active set changes frequently.

\section{Data and Label Validation}
\label{app:data_dist}
Shared generation and filtering procedures are described in
Appendix~\ref{app:set}.

\subsection{Short-answer QA}
\label{app:qa_data}

\paragraph{Counts and class balance.}
Table~\ref{tab:class_balance} reports class balance for the four QA benchmarks
on the training and test splits, after the two-pass label filter of
Appendix~\ref{app:discard}. Labels follow the convention of
Section~\ref{ex:setup}: the positive class is a hallucinated generation
($1-\mathrm{accuracy}=1$). Five responses are generated per question, but detector
training uses only one and test uses all five. The splits have different question counts, and $N$ also reflects
filtering (Appendix~\ref{app:sample_counts}).

On the short-answer benchmarks the balance shifts systematically with generator
strength: the negative-to-positive ratio on TriviaQA rises from $3.26$
(LLaMA-3.1-8B) to $4.69$ (Mistral-Nemo-12B) to $12.21$ (LLaMA-3.3-70B).
LongFact++ class counts are reported in Table~\ref{tab:length_counts}.

The class ratio enters Theorem~\ref{theorem:hami}. Semantic scaling
enlarges the expected margin only when
$|\mathcal{S}_{neg}|/|\mathcal{S}_{pos}| > 1/\gamma$, and the values of
$1/\gamma$ we measure on the short-answer benchmarks lie between $0.37$ and
$0.49$ (Appendix~\ref{app:expected_values_summary}). Every short-answer ratio
in Table~\ref{tab:class_balance} clears that band, the smallest being $0.52$
for LLaMA-3.1-8B on SQuAD, so the condition holds wherever we have checked it.

\begin{table}[H]
\centering
\small
\setlength{\tabcolsep}{4pt}
\renewcommand{\arraystretch}{1.15}
\caption{QA class counts after filtering ($N$). ``Halluc.'' is positive and ``Correct'' negative; Neg/Pos is their negative-to-positive ratio. Of five responses generated per question, training uses one and test uses all five.}
\label{tab:class_balance}
\begin{tabular}{llrrrc@{\hskip 12pt}rrrc}
\toprule
& & \multicolumn{4}{c}{\textbf{Train}} & \multicolumn{4}{c}{\textbf{Test}} \\
\cmidrule(lr){3-6} \cmidrule(lr){7-10}
\textbf{Model} & \textbf{Dataset} & $N$ & Halluc. & Correct & Neg/Pos & $N$ & Halluc. & Correct & Neg/Pos \\
\midrule
\multirow{4}{*}{LLaMA-3.1-8B} & TriviaQA & 3,870 & 996 & 2,874 & \textbf{2.89} & 5,354 & 1,257 & 4,097 & \textbf{3.26} \\
 & SQuAD & 3,290 & 2,163 & 1,127 & \textbf{0.52} & 4,480 & 2,957 & 1,523 & \textbf{0.52} \\
 & NQ & 3,606 & 2,222 & 1,384 & \textbf{0.62} & 4,968 & 3,151 & 1,817 & \textbf{0.58} \\
 & BioASQ & 2,141 & 1,081 & 1,060 & \textbf{0.98} & 2,766 & 1,334 & 1,432 & \textbf{1.07} \\
\midrule
\multirow{4}{*}{Mistral-Nemo-12B} & TriviaQA & 3,866 & 705 & 3,161 & \textbf{4.48} & 5,329 & 937 & 4,392 & \textbf{4.69} \\
 & SQuAD & 3,154 & 1,810 & 1,344 & \textbf{0.74} & 4,274 & 2,449 & 1,825 & \textbf{0.75} \\
 & NQ & 3,469 & 1,626 & 1,843 & \textbf{1.13} & 4,751 & 2,356 & 2,395 & \textbf{1.02} \\
 & BioASQ & 2,054 & 765 & 1,289 & \textbf{1.68} & 2,687 & 1,035 & 1,652 & \textbf{1.60} \\
\midrule
\multirow{4}{*}{LLaMA-3.3-70B} & TriviaQA & 3,875 & 343 & 3,532 & \textbf{10.30} & 5,362 & 406 & 4,956 & \textbf{12.21} \\
 & SQuAD & 3,046 & 1,353 & 1,693 & \textbf{1.25} & 4,114 & 1,826 & 2,288 & \textbf{1.25} \\
 & NQ & 3,349 & 1,040 & 2,309 & \textbf{2.22} & 4,531 & 1,462 & 3,069 & \textbf{2.10} \\
 & BioASQ & 2,006 & 603 & 1,403 & \textbf{2.33} & 2,629 & 784 & 1,845 & \textbf{2.35} \\
\bottomrule
\end{tabular}
\end{table}

\paragraph{Response length.}
Figure~\ref{fig:tokendist} plots response lengths for the four QA benchmarks
and LongFact++. Table~\ref{tab:token_stats} summarizes QA
lengths; Table~\ref{tab:length_tokens} summarizes LongFact++. QA counts are read from the stored per-token
log-likelihoods, so they count exactly the positions a pooling detector
aggregates over --- that is, the bag size $T_\mathbf{B}$ of
Section~\ref{sec:maxpool_arch}.

The four short-answer benchmarks were prompted for a single sentence, and their
distributions are correspondingly short and right-skewed, with medians between
$4$ and $33$ tokens and $95\%$ of answers below about $63$ tokens. Length nonetheless varies
widely between generators under an identical prompt, though not monotonically
in model size: LLaMA-3.3-70B answers TriviaQA in $24.5$ tokens on average
against $8.7$ for LLaMA-3.1-8B and $5.5$ for Mistral-Nemo-12B, so the bags reaching
the detector differ in size by a factor of three to four and a half across
generators on the same questions.

LongFact++ lengths count the LLaMA-3 tokens of each dataset answer, i.e. the
positions the detectors read (Appendix~\ref{app:length_study}).

\begin{table}[H]
\centering
\small
\setlength{\tabcolsep}{5pt}
\renewcommand{\arraystretch}{1.15}
\caption{Retained QA response lengths in tokens, counted from stored per-token log-likelihoods.}
\label{tab:token_stats}
\begin{tabular}{llrrr@{\hskip 12pt}rrr}
\toprule
& & \multicolumn{3}{c}{\textbf{Train}} & \multicolumn{3}{c}{\textbf{Test}} \\
\cmidrule(lr){3-5} \cmidrule(lr){6-8}
\textbf{Model} & \textbf{Dataset} & Mean $\pm$ SD & Median & $p_{95}$ & Mean $\pm$ SD & Median & $p_{95}$ \\
\midrule
\multirow{4}{*}{LLaMA-3.1-8B} & TriviaQA & 8.7 $\pm$ 11.4 & 5 & 28 & 8.7 $\pm$ 11.7 & 5 & 28 \\
 & SQuAD & 16.8 $\pm$ 15.9 & 13 & 48 & 17.7 $\pm$ 17.2 & 13 & 53 \\
 & NQ & 18.5 $\pm$ 18.3 & 13 & 58 & 19.3 $\pm$ 18.7 & 14 & 60 \\
 & BioASQ & 18.8 $\pm$ 20.1 & 12 & 62 & 18.5 $\pm$ 20.4 & 11 & 63 \\
\midrule
\multirow{4}{*}{Mistral-Nemo-12B} & TriviaQA & 6.0 $\pm$ 5.2 & 4 & 17 & 5.5 $\pm$ 4.5 & 4 & 15 \\
 & SQuAD & 12.2 $\pm$ 8.1 & 10 & 27 & 12.3 $\pm$ 8.3 & 10 & 28 \\
 & NQ & 13.2 $\pm$ 9.2 & 12 & 29 & 13.8 $\pm$ 9.8 & 12 & 31 \\
 & BioASQ & 16.0 $\pm$ 9.5 & 15 & 32 & 15.8 $\pm$ 9.3 & 15 & 33 \\
\midrule
\multirow{4}{*}{LLaMA-3.3-70B} & TriviaQA & 24.7 $\pm$ 13.6 & 23 & 48 & 24.5 $\pm$ 13.3 & 23 & 46 \\
 & SQuAD & 26.8 $\pm$ 13.4 & 25 & 49 & 27.5 $\pm$ 13.7 & 25 & 52 \\
 & NQ & 26.8 $\pm$ 12.8 & 24 & 48 & 28.1 $\pm$ 13.7 & 25 & 52 \\
 & BioASQ & 34.5 $\pm$ 15.6 & 32 & 60 & 34.6 $\pm$ 15.4 & 33 & 61 \\
\bottomrule
\end{tabular}
\end{table}

\begin{figure}[H]
\centering
\includegraphics[width=\textwidth]{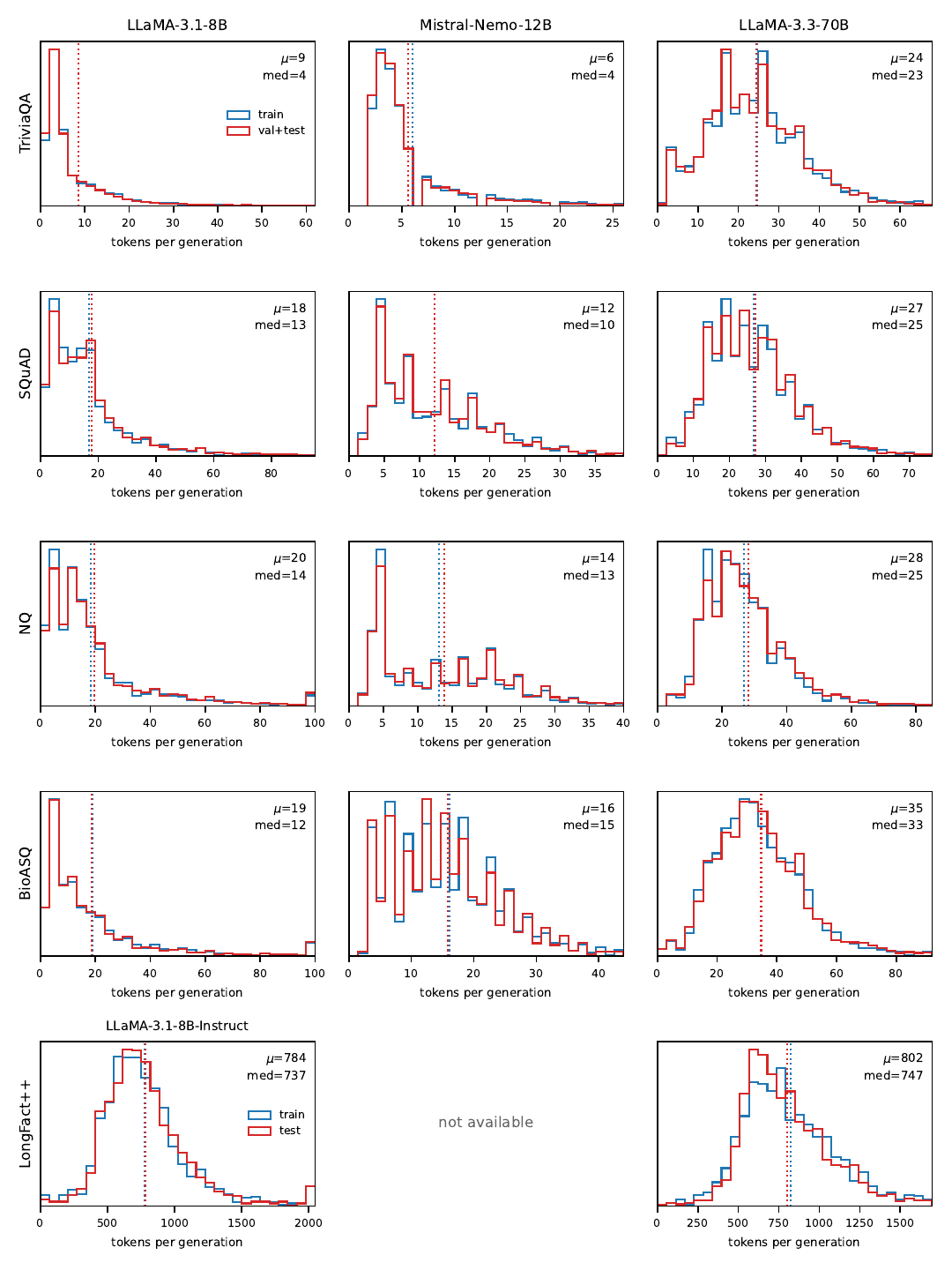}
\caption{Response-length distributions by dataset (rows) and model (columns):
training responses (blue) and evaluation responses (red; validation and test for
QA, test for LongFact++). Histograms have unit area; dotted lines mark means.
Annotations give the mean and median of the evaluation responses. Panel-specific
axes extend to the $99$th percentile. LongFact++ has answers of
LLaMA-3.1-8B-Instruct and LLaMA-3.3-70B-Instruct only.}
\label{fig:tokendist}
\end{figure}

\subsection{LongFact++}
\label{app:length_study}
LongFact++~\citep{obeso2026realtimedetectionhallucinatedentities} extends the
LongFact prompts~\citep{longfact} to more domains and query types and annotates
the entities in model-written answers as Supported, Not Supported or
Insufficient Information, using an LLM with web search. We use its release with
augmented annotations, which holds about $9{,}700$ answers per model from
LLaMA-3.1-8B-Instruct and LLaMA-3.3-70B-Instruct, one per prompt; its training,
validation and test splits share no prompt.

\paragraph{Labels.}
An answer is labeled hallucinated when at least $10\%$ of its annotated entity
spans are labeled Not Supported. The median answer has $13$ (8B) and $14$ (70B)
annotated spans.

\paragraph{Splits.}
Test consists of the dataset's test answers (about $920$) filled up to
$1{,}500$ with answers from its training split, and validation of its
validation answers (about $870$) filled up to $1{,}000$ in the same way, with a
fixed seed and no answer used twice. Training uses $500$ non-hallucinated and
$500$ hallucinated answers drawn from the remaining training answers, while
validation and test keep the natural class rates
(Table~\ref{tab:length_counts}).

\paragraph{Hidden states.}
Token states are obtained by teacher forcing as described under Prompt 2 in
Appendix~\ref{app:set}; LLaMA-3.3-70B is loaded in $4$-bit as for the QA
benchmarks. States are stored at layers $0, 2, \dots, 32$ for
LLaMA-3.1-8B-Instruct and $0, 4, \dots, 80$ for LLaMA-3.3-70B-Instruct.

\begin{table}[H]
\centering
\small
\setlength{\tabcolsep}{4pt}
\renewcommand{\arraystretch}{1.15}
\caption{LongFact++ answers per split (hallucinated counts in parentheses).}
\label{tab:length_counts}
\begin{tabular}{lrrr}
\toprule
\textbf{Model} & \textbf{Train} & \textbf{Validation} & \textbf{Test} \\
\midrule
LLaMA-3.1-8B-Instruct & 1,000 (500) & 1,000 (667) & 1,500 (1,011) \\
LLaMA-3.3-70B-Instruct & 1,000 (500) & 1,000 (421) & 1,500 (674) \\
\bottomrule
\end{tabular}
\end{table}

\paragraph{Response length.}
Table~\ref{tab:length_tokens} reports answer lengths in LLaMA-3 tokens;
Figure~\ref{fig:tokendist} plots the distributions alongside the QA benchmarks.
Hallucinated and non-hallucinated answers have similar lengths (test medians
$759$ and $688$ tokens for 8B, $749$ and $743$ for 70B). Scoring answers by
length alone gives test AUC $0.586$ (8B) and $0.517$ (70B), and by the number of
annotated spans $0.533$ and $0.515$.

\begin{table}[H]
\centering
\small
\setlength{\tabcolsep}{5pt}
\renewcommand{\arraystretch}{1.15}
\caption{LongFact++ answer lengths in LLaMA-3 tokens.}
\label{tab:length_tokens}
\begin{tabular}{lrrr@{\hskip 12pt}rrr}
\toprule
& \multicolumn{3}{c}{\textbf{Train}} & \multicolumn{3}{c}{\textbf{Test}} \\
\cmidrule(lr){2-4} \cmidrule(lr){5-7}
\textbf{Model} & Mean $\pm$ SD & Median & $p_{95}$ & Mean $\pm$ SD & Median & $p_{95}$ \\
\midrule
LLaMA-3.1-8B-Instruct & 773.7 $\pm$ 327.7 & 726 & 1369 & 784.0 $\pm$ 319.1 & 737 & 1358 \\
LLaMA-3.3-70B-Instruct & 824.2 $\pm$ 311.0 & 774 & 1386 & 802.1 $\pm$ 279.0 & 747 & 1303 \\
\bottomrule
\end{tabular}
\end{table}

\paragraph{Layer selection and paired comparisons.}
Every method is trained at every stored layer with five runs, and the layer
with the highest mean validation AUC is reported
(Table~\ref{tab:layer_length}; curves in
Figure~\ref{fig:layer_sensitivity_lfpp}). Table~\ref{tab:lfpp_paired} gives
$95\%$ intervals for the difference between max pooling and each method,
from a paired bootstrap over test answers ($2{,}000$ resamples) of the
run-averaged scores.

\begin{table}[H]
\centering
\small
\setlength{\tabcolsep}{6pt}
\renewcommand{\arraystretch}{1.2}
\caption{Paired bootstrap $95\%$ intervals of the LongFact++ test AUC
difference, max pooling minus each method. Each method's score is first
averaged over its five runs and the test answers are resampled, so the
intervals refer to the run-averaged detectors and need not contain the
differences of the mean AUCs in Table~\ref{tab:length_study}.}
\label{tab:lfpp_paired}
\begin{tabular}{lcc}
\toprule
\textbf{Method} & LLaMA-3.1-8B-Instruct & LLaMA-3.3-70B-Instruct \\
\midrule
SEP          & $[+0.251, +0.331]$ & $[+0.170, +0.242]$ \\
LLMs Know    & $[+0.131, +0.189]$ & $[+0.080, +0.135]$ \\
SAPLMA       & $[+0.104, +0.159]$ & $[+0.067, +0.122]$ \\
HaMI         & $[+0.046, +0.078]$ & $[+0.028, +0.061]$ \\
HaMI (SP)    & $[+0.045, +0.076]$ & $[+0.028, +0.061]$ \\
Mean         & $[+0.014, +0.034]$ & $[-0.007, +0.009]$ \\
\bottomrule
\end{tabular}
\end{table}

\subsection{Label Validation}
\label{app:label_validation}
Appendix~\ref{app:judge_sensitivity} tests the detectors under a second judge and under retrieval-augmented labels.

\paragraph{Agreement with human annotation.}
\label{app:judge_validation}
Our detectors are trained against labels produced by an LLM judge (Section~\ref{ex:setup}) rather than by human annotators. A probe fitted to such labels could in principle learn the judge's idiosyncrasies rather than factual error, so we quantify how far the judge departs from human factuality judgments.

Annotating our own generations by hand at the required scale was not feasible, so we validated the judge on TruthfulQA~\citep{truthfulqa}, which supplies human reference labels for exactly this distinction. We drew $2{,}000$ samples and applied the short-answer QA labeling protocol, leaving $1{,}853$ after discarding responses initially judged hallucinated but judged factual on re-check. Human-annotated incorrect answers form the positive class. Table~\ref{tab:judge_confusion} reports the confusion matrix and the derived rates.

\begin{table}[H]
\centering
\small
\renewcommand{\arraystretch}{1.2}
\caption{GPT-5 mini agreement with TruthfulQA human labels ($n=1{,}853$).}
\label{tab:judge_confusion}
\begin{tabular}{lcc}
\toprule
& \textbf{Human: Incorrect} & \textbf{Human: Correct} \\
\midrule
Judge: Incorrect & 869 & 83 \\
Judge: Correct   & 4   & 897 \\
\bottomrule
\end{tabular}

\vspace{6pt}

\begin{tabular}{ccccc}
\toprule
\textbf{Accuracy} & \textbf{TPR} & \textbf{TNR} & \textbf{FPR} & \textbf{FNR} \\
\midrule
95.30\% & 99.54\% & 91.53\% & 8.47\% & 0.46\% \\
\bottomrule
\end{tabular}
\end{table}

The judge agrees with human annotators on $95.3\%$ of the retained TruthfulQA samples, and misses almost no human-identified error (FNR $0.46\%$). The residual disagreement is asymmetric: $8.5\%$ of human-correct answers are marked incorrect by the judge, typically answers that are factually right but hedged or incompletely phrased. Our reported AUCs measure agreement with labels produced by the corresponding automatic protocol. The $95.3\%$ agreement is specific to this TruthfulQA validation; LongFact++ does not depend on it, because its labels come from the dataset's web-search annotation. Appendix~\ref{app:judge_sensitivity} replaces GPT-5 mini with a second LLM judge.

\paragraph{Discard statistics.}
\label{app:discard}
The two-stage labeling protocol re-checks only responses initially judged hallucinated. A response is discarded only if the re-check judges it factual; it is not relabeled as factual. Initially factual responses and responses judged hallucinated in both passes are retained. Table~\ref{tab:discard_rate} reports QA discard rates; LongFact++ keeps the dataset labels, so no answer is discarded.

QA discard ratios are lowest on TriviaQA ($2.6$--$3.4\%$), moderate on BioASQ ($6.9$--$12.8\%$) and NQ ($9.6$--$17.5\%$), and highest on SQuAD ($17.8$--$24.7\%$). The SQuAD training-split rate grows from $17.8\%$ (LLaMA-3.1-8B) to $21.2\%$ (Mistral-Nemo-12B) to $23.9\%$ (LLaMA-3.3-70B). All methods use the same filtered splits.

\begin{table}[H]
\centering
\small
\renewcommand{\arraystretch}{1.15}
\caption{Discard rate of automatically labeled samples (\%). The last column covers the validation split as a whole, from which both validation and test are drawn.}
\label{tab:discard_rate}
\begin{tabular}{llcc}
\toprule
\textbf{Model} & \textbf{Dataset} & \textbf{Train} & \textbf{Validation + test} \\
\midrule
\multirow{4}{*}{LLaMA-3.1-8B}
& TriviaQA & 3.25\% & 2.85\% \\

& SQuAD & 17.75\% & 18.02\% \\

& NQ & 9.85\% & 9.62\% \\

& BioASQ & 6.91\% & 7.77\% \\

\midrule
\multirow{4}{*}{Mistral-Nemo-12B}
& TriviaQA & 3.35\% & 3.20\% \\

& SQuAD & 21.15\% & 22.44\% \\

& NQ & 13.28\% & 13.44\% \\

& BioASQ & 10.70\% & 10.50\% \\

\midrule
\multirow{4}{*}{LLaMA-3.3-70B}
& TriviaQA & 3.12\% & 2.64\% \\

& SQuAD & 23.85\% & 24.73\% \\

& NQ & 16.27\% & 17.51\% \\

& BioASQ & 12.78\% & 12.80\% \\

\bottomrule
\end{tabular}
\end{table}

\section{Detailed Detection Results}
\label{app:results}
This appendix extends the results of Sections~\ref{ex:auc}--\ref{ex:margin}: timing, per-dataset margins, run-to-run variability, additional baselines and leave-one-dataset-out transfer.

\subsection{End-to-end time per response}
\label{app:throughput}
For each generator and QA benchmark we draw $50$ test questions at random and
use the prompt and sampling settings of Section~\ref{ex:setup}; for
LongFact++ we draw $50$ test prompts for LLaMA-3.1-8B-Instruct and $20$ for
LLaMA-3.3-70B-Instruct and generate under the chat template until the end of
the answer (at most $1{,}536$ tokens with Hugging Face Transformers, also for the four HaMI (SP) samples, regenerated for timing only; median length $611$ and $634$ tokens).
Each answer is generated once per detector with the same sampling seed, so
that every detector sees the same answer (verified for all responses). The
time runs from the prompt to the hallucination score and includes generation
(one A100; two for the $4$-bit 70B model; batch size one; CUDA-synchronized
wall clock). Each detector obtains its input during generation: max pooling,
mean pooling, HaMI and the fixed-position probes receive each new token's
state at their layer through a forward hook (a running maximum, a running sum,
one score per token, or the last state), so no token states are stored and
only the score leaves the GPU. HaMI (SP) additionally samples four answers,
clusters the five answers with the same gpt-4o-mini entailment calls as in
training, and computes the semantic probability. Detector weights are randomly
initialized with the trained architectures, which does not affect timing.
Table~\ref{tab:throughput_comparison} averages the time per response over
the four benchmarks ($200$ questions per generator); Table~\ref{tab:sp_breakdown} splits the time
of HaMI (SP) into its parts.

\begin{table}[H]
\centering
\small
\setlength{\tabcolsep}{1.4pt}
\renewcommand{\arraystretch}{1.2}
\caption{Parts of the end-to-end time of HaMI (SP) per QA response (s, mean
over responses; TQ: TriviaQA, SQ: SQuAD, BQ: BioASQ). The first row generates the answer and scores it online
with HaMI; the semantic clustering calls an LLM judge through an API.}
\label{tab:sp_breakdown}
\resizebox{\textwidth}{!}{%
\begin{tabular}{@{}l*{12}{c}@{}}
\toprule
& \multicolumn{4}{c}{\textbf{LLaMA-3.1-8B}}
& \multicolumn{4}{c}{\textbf{Mistral-Nemo-12B}}
& \multicolumn{4}{c}{\textbf{LLaMA-3.3-70B}} \\
\cmidrule(lr){2-5} \cmidrule(lr){6-9} \cmidrule(lr){10-13}
\textbf{Part} & TQ & SQ & NQ & BQ & TQ & SQ & NQ & BQ & TQ & SQ & NQ & BQ \\
\midrule
Answer + HaMI scoring & 0.23 & 0.47 & 0.44 & 0.48 & 0.17 & 0.35 & 0.36 & 0.50 & 5.44 & 5.82 & 5.88 & 7.29 \\
Four extra answers & 0.77 & 1.74 & 1.65 & 1.61 & 0.70 & 1.43 & 1.48 & 1.83 & 22.75 & 24.29 & 24.72 & 28.92 \\
Semantic clustering (API) & 7.14 & 15.51 & 12.93 & 11.93 & 4.49 & 9.80 & 8.46 & 10.41 & 8.30 & 12.98 & 9.42 & 12.35 \\
Probability and score & 0.00 & 0.00 & 0.01 & 0.00 & 0.00 & 0.00 & 0.00 & 0.01 & 0.00 & 0.00 & 0.00 & 0.01 \\
\midrule
Total & 8.15 & 17.72 & 15.04 & 14.04 & 5.37 & 11.59 & 10.30 & 12.76 & 36.49 & 43.10 & 40.02 & 48.56 \\
\bottomrule
\end{tabular}}
\end{table}

HaMI (SP) spends most of its time outside the detector: the four extra
answers take $3.4$--$4.2$ times the generation time, and the semantic
clustering takes $82$--$88\%$ of the total on LLaMA-3.1-8B and
Mistral-Nemo-12B and $23$--$30\%$ on LLaMA-3.3-70B, where generation is
slower. Scoring itself takes milliseconds.

\paragraph{Cost of extracting the hidden states.}
Table~\ref{tab:extraction} separates the stages of three ways of obtaining
the max-pooling input. (i) All layers stored, layer selected afterwards:
generation returns every layer's hidden states, the answer-token states of all
layers are gathered on the GPU and copied to CPU memory, and the detector's
layer is selected and scored. (ii) Selected layer stored: a forward hook keeps
only the answer-token states of the detector's layer on the GPU, which are
copied to CPU memory and scored there. (iii) Online: the hook keeps the
running feature-wise maximum, and only the score leaves the GPU. We also time
re-reading an already generated answer, one teacher-forced pass stopped after
the detector's layer, for the case where the answer comes from elsewhere.
Storing all layers adds $1.0$--$1.9\%$ to the time per QA response and
$1.8\%$ / $2.9\%$ on LongFact++ (median per response), moving up to
$0.93$\,GB per long answer off the GPU; storing only the selected layer moves one layer's share of the data
($11$\,MB for a long 70B answer) and adds at most $0.7\%$; online scoring
moves no states and adds at most $0.6\%$. Re-reading a given answer takes
$2$--$9\%$ of its generation time on QA and $0.3\%$ on LongFact++.

\begin{table}[H]
\centering
\small
\setlength{\tabcolsep}{3pt}
\renewcommand{\arraystretch}{1.12}
\caption{Time of each stage of obtaining the max-pooling input and scoring
it, per response (mean over responses; QA: mean over the four benchmarks).
Times are in ms except the generation rows (s); data moved off the GPU is in
MB.}
\label{tab:extraction}
\resizebox{\textwidth}{!}{%
\begin{tabular}{@{}lccccc@{}}
\toprule
& \multicolumn{3}{c}{\textbf{QA}} & \multicolumn{2}{c}{\textbf{LongFact++}} \\
\cmidrule(lr){2-4} \cmidrule(lr){5-6}
\textbf{Stage} & \textbf{8B} & \textbf{12B} & \textbf{70B} & \textbf{8B-Inst.} & \textbf{70B-Inst.} \\
\midrule
\textit{Generation only (s)} & 0.40 & 0.35 & 6.12 & 15.05 & 95.70 \\

\midrule
\multicolumn{6}{@{}l}{\textit{All layers stored, layer selected afterwards}} \\
\quad Generation with all hidden states (s) & 0.40 & 0.35 & 6.15 & 15.10 & 98.34 \\
\quad Gather answer-token states, all layers & 4.5 & 3.6 & 18.5 & 165.2 & 408.5 \\
\quad Copy to CPU memory & 0.7 & 0.7 & 9.1 & 38.8 & 201.8 \\
\quad Select layer, move back & 1.3 & 0.6 & 1.6 & 6.2 & 17.0 \\
\quad Max pooling + classifier & 0.7 & 0.6 & 1.3 & 5.1 & 7.5 \\

\quad Data moved off the GPU (MB) & 4.8 & 5.2 & 44.3 & 182.8 & 927.2 \\
\midrule
\multicolumn{6}{@{}l}{\textit{Selected layer stored}} \\
\quad Copy to CPU memory & 0.3 & 0.1 & 0.8 & 1.1 & 2.1 \\
\quad Max pooling + classifier (CPU) & 7.0 & 6.6 & 16.4 & 43.1 & 268.8 \\

\quad Data moved off the GPU (MB) & 0.14 & 0.13 & 0.55 & 5.54 & 11.45 \\
\midrule
\multicolumn{6}{@{}l}{\textit{Online (no states stored)}} \\
\quad After generation: classifier on running max & 0.10 & 0.23 & 0.32 & 0.14 & 0.78 \\
\midrule
\multicolumn{6}{@{}l}{\textit{Re-reading a given answer (no generation)}} \\
\quad Teacher-forced pass up to the layer & 13.5 & 14.6 & 491.5 & 45.6 & 271.8 \\
\quad Max pooling + classifier & 0.2 & 0.3 & 12.3 & 3.1 & 23.5 \\
\bottomrule
\end{tabular}}
\end{table}

\paragraph{Long-form answers.}
Table~\ref{tab:lfpp_speed} repeats the end-to-end measurement on LongFact++.
With answers of several hundred tokens, every detector that obtains its input
during generation stays within $1\%$ of generation alone, except SinkProbe,
whose eager attention adds $28\%$ / $15\%$. HaMI (SP) takes $5.8$ / $5.2$ times
as long, mostly for the four extra answers.

\begin{table}[H]
\centering
\small
\setlength{\tabcolsep}{4pt}
\renewcommand{\arraystretch}{1.1}
\caption{End-to-end time per LongFact++ response (s, mean over $50$ / $20$
responses) and overhead relative to generation alone (median over responses
of the per-response ratio; for HaMI (SP), ratio of the means). Differences
within $\pm1\%$ are timing noise. HaMI (SP) on LLaMA-3.1-8B-Instruct uses $49$ responses, as one clustering API call failed.}
\label{tab:lfpp_speed}
\begin{tabular}{@{}lcccc@{}}
\toprule
& \multicolumn{2}{c}{\textbf{LLaMA-3.1-8B-Instruct}} & \multicolumn{2}{c}{\textbf{LLaMA-3.3-70B-Instruct}} \\
\cmidrule(lr){2-3} \cmidrule(lr){4-5}
\textbf{Method} & Time & Overhead & Time & Overhead \\
\midrule
Generation only & 15.1 & -- & 95.7 & -- \\
\midrule
SEP & 15.0 & 0.0\% & 95.5 & $-$0.3\% \\
LLMs Know & 15.0 & $-$0.2\% & 94.7 & $-$0.9\% \\
SAPLMA & 14.9 & $-$0.2\% & 96.5 & $+$0.5\% \\
\midrule
HaMI & 15.1 & $+$0.9\% & 97.0 & $+$0.5\% \\
HaMI (SP) & 87.6 & 5.8$\times$ & 497.1 & 5.2$\times$ \\
\midrule
Mean & 15.1 & $+$0.4\% & 94.9 & $-$0.7\% \\
\textbf{Max} & 15.1 & $+$0.4\% & 95.6 & $-$0.4\% \\
\midrule
ABMIL & 14.9 & $-$0.2\% & 95.0 & $-$0.8\% \\
Gated ABMIL & 14.9 & $-$0.3\% & 96.1 & $+$0.1\% \\
LSE & 15.0 & $-$0.2\% & 96.5 & $-$0.1\% \\
DeepSets & 14.9 & $-$0.2\% & 96.1 & $+$0.2\% \\
Set Transformer & 15.0 & $-$0.1\% & 95.9 & $-$0.1\% \\
Learned TopK & 14.9 & $-$0.2\% & 96.0 & $+$0.3\% \\
Transformer & 15.0 & $-$0.2\% & 96.1 & $+$0.3\% \\
\midrule
HaluNet & 15.2 & $+$0.6\% & 96.8 & $+$0.6\% \\
SFP & 15.0 & $+$0.2\% & 96.0 & $-$0.3\% \\
SinkProbe & 19.1 & $+$27.7\% & 110.4 & $+$15.4\% \\
HIDE & 15.1 & $+$0.5\% & 95.9 & $+$0.3\% \\
\bottomrule
\end{tabular}
\end{table}

\subsection{Per-Dataset Margins}
\label{app:margin_full}
Table~\ref{tab:margin_comparison} averages over the four QA benchmarks;
the per-dataset values follow. Appendix~\ref{app:validation} checks the analysis behind these margins on trained detectors.

\begin{table}[H]
\centering
\small
\setlength{\tabcolsep}{4pt}
\caption{Mean bag-wise margins by dataset. \textcolor{red}{Red}/\textcolor{blue}{blue}: higher/lower within each block.}
\label{tab:margin_full}
\renewcommand{\arraystretch}{1.2}

\resizebox{\textwidth}{!}{
\begin{tabular}{lcccccccccccc}
\toprule
& \multicolumn{4}{c}{\textbf{LLaMA-3.1-8B}}
& \multicolumn{4}{c}{\textbf{Mistral-Nemo-12B}}
& \multicolumn{4}{c}{\textbf{LLaMA-3.3-70B}} \\
\cmidrule(lr){2-5} \cmidrule(lr){6-9} \cmidrule(lr){10-13}
& TriviaQA & SQuAD & NQ & BioASQ
& TriviaQA & SQuAD & NQ & BioASQ
& TriviaQA & SQuAD & NQ & BioASQ \\
\midrule

HaMI (Original) 
& \textcolor{blue}{6.17} & \textcolor{blue}{2.09} & \textcolor{blue}{3.16} & \textcolor{blue}{4.28} 
& \textcolor{blue}{5.97} & \textcolor{blue}{3.66} & \textcolor{blue}{3.42} & \textcolor{blue}{4.33} 
& \textcolor{blue}{5.12} & \textcolor{blue}{3.67} & \textcolor{blue}{3.56} & \textcolor{blue}{4.06}  
 \\

\textbf{HaMI (SP)} 
& \textcolor{red}{9.02} & \textcolor{red}{5.10} & \textcolor{red}{5.58} & \textcolor{red}{7.51} 
& \textcolor{red}{7.08} & \textcolor{red}{5.48} & \textcolor{red}{5.25} & \textcolor{red}{5.61} 
& \textcolor{red}{6.10} & \textcolor{red}{5.47} & \textcolor{red}{5.23} & \textcolor{red}{5.85} 
\\

\midrule

Mean Pooling 
& \textcolor{blue}{2.026} & \textcolor{blue}{0.935} & \textcolor{blue}{0.870} & \textcolor{blue}{1.286} 
& \textcolor{blue}{2.994} & \textcolor{blue}{1.197} & \textcolor{blue}{1.046} & \textcolor{blue}{1.614} 
& \textcolor{blue}{3.333} & \textcolor{blue}{1.222} & \textcolor{blue}{1.427} & \textcolor{blue}{1.596} 
\\

\textbf{Max Pooling} 
& \textcolor{red}{2.299} & \textcolor{red}{1.073} & \textcolor{red}{1.052} & \textcolor{red}{1.498} 
& \textcolor{red}{3.404} & \textcolor{red}{1.420} & \textcolor{red}{1.231} & \textcolor{red}{1.924} 
& \textcolor{red}{3.722} & \textcolor{red}{1.366} & \textcolor{red}{1.640} & \textcolor{red}{1.824} 
\\

\bottomrule
\end{tabular}
}

\end{table}

\subsection{Variability across Runs}
Table~\ref{tab:std} reports standard errors of detection AUC over five runs on the four QA benchmarks. Table~\ref{tab:std_margin} reports margin standard deviations over the same five runs.
\label{app:std}
\begin{table}[H]
\centering
\small
\setlength{\tabcolsep}{4pt}
\renewcommand{\arraystretch}{1.2}

\caption{Test AUC standard errors over five runs on the four QA benchmarks.}
\label{tab:std}
\resizebox{\textwidth}{!}{
\begin{tabular}{lcccccccccccc}
\toprule
& \multicolumn{4}{c}{\textbf{LLaMA-3.1-8B}}
& \multicolumn{4}{c}{\textbf{Mistral-Nemo-12B}}
& \multicolumn{4}{c}{\textbf{LLaMA-3.3-70B}} \\
\cmidrule(lr){2-5} \cmidrule(lr){6-9} \cmidrule(lr){10-13}
& TriviaQA & SQuAD & NQ & BioASQ
& TriviaQA & SQuAD & NQ & BioASQ
& TriviaQA & SQuAD & NQ & BioASQ \\
\midrule
SEP
& 0.003 & 0.002 & 0.002 & 0.005
& 0.004 & 0.004 & 0.004 & 0.005
& 0.005 & 0.001 & 0.002 & 0.006
\\

LLMs Know
& 0.003 & 0.003 & 0.004 & 0.003
& 0.004 & 0.005 & 0.002 & 0.003
& 0.002 & 0.001 & 0.002 & 0.003
\\

SAPLMA
& 0.005 & 0.003 & 0.002 & 0.002
& 0.002 & 0.003 & 0.003 & 0.003
& 0.002 & 0.001 & 0.002 & 0.003
\\

\midrule

HaMI (Original)
& 0.003 & 0.003 & 0.002 & 0.001
& 0.001 & 0.004 & 0.002 & 0.000
& 0.001 & 0.001 & 0.002 & 0.001
\\

HaMI (SP)
& 0.001 & 0.000 & 0.001 & 0.001
& 0.001 & 0.002 & 0.001 & 0.001
& 0.002 & 0.003 & 0.001 & 0.001
\\

\midrule

Mean Pooling
& 0.000 & 0.000 & 0.000 & 0.000
& 0.000 & 0.000 & 0.001 & 0.001
& 0.000 & 0.000 & 0.001 & 0.000
\\

\textbf{Max Pooling}
& 0.001 & 0.001 & 0.000 & 0.000
& 0.001 & 0.001 & 0.000 & 0.000
& 0.000 & 0.000 & 0.001 & 0.000
\\
\bottomrule
\end{tabular}
}

\vspace{2pt}
\end{table}

\begin{table}[H]
\centering
\small
\setlength{\tabcolsep}{4pt}
\renewcommand{\arraystretch}{1.2}
\caption{Margin standard deviations over five runs.}
\label{tab:std_margin}
\resizebox{\textwidth}{!}{

\begin{tabular}{lcccccccccccc}
\toprule
& \multicolumn{4}{c}{\textbf{LLaMA-3.1-8B}}
& \multicolumn{4}{c}{\textbf{Mistral-Nemo-12B}}
& \multicolumn{4}{c}{\textbf{LLaMA-3.3-70B}} \\
\cmidrule(lr){2-5} \cmidrule(lr){6-9} 
\cmidrule(lr){10-13}
& TriviaQA & SQuAD & NQ & BioASQ
& TriviaQA & SQuAD & NQ & BioASQ
& TriviaQA & SQuAD & NQ & BioASQ \\
\midrule

HaMI (Original) 
& 0.128 & 0.354 & 0.120 & 0.247 
& 0.215 & 0.270 & 0.104 & 0.420 
& 0.126 & 0.181 & 0.116 & 0.344 \\

HaMI (SP) 
& 0.287 & 0.681 & 0.294 & 0.343 
& 0.646 & 0.200 & 0.442 & 0.407 
& 0.191 & 0.138 & 0.305 & 0.298 \\

\midrule

Mean Pooling 
& 0.023 & 0.013 & 0.011 & 0.046 
& 0.020 & 0.037 & 0.019 & 0.029 
& 0.081 & 0.017 & 0.014 & 0.043 
 \\

\textbf{Max Pooling} 
& 0.014 & 0.031 & 0.056 & 0.028 
& 0.062 & 0.026 & 0.022 & 0.028 
& 0.095 & 0.021 & 0.025 & 0.021 
 \\

\bottomrule
\end{tabular}
}

\vspace{2pt}
\end{table}

\begin{table}[H]
\centering
\small
\setlength{\tabcolsep}{6pt}
\renewcommand{\arraystretch}{1.2}
\caption{Test AUC standard errors over five runs on LongFact++
(Section~\ref{ex:length}).}
\label{tab:std_longfact}
\begin{tabular}{lcc}
\toprule
\textbf{Method} & LLaMA-3.1-8B-Instruct & LLaMA-3.3-70B-Instruct \\
\midrule
SEP          & 0.000 & 0.000 \\
LLMs Know    & 0.000 & 0.000 \\
SAPLMA       & 0.002 & 0.005 \\
\midrule
HaMI         & 0.007 & 0.019 \\
HaMI (SP)    & 0.007 & 0.022 \\
Mean         & 0.001 & 0.001 \\
\textbf{Max} & 0.001 & 0.001 \\
\bottomrule
\end{tabular}
\end{table}

Table~\ref{tab:std_longfact} gives the same quantity for LongFact++. The
pooling detectors stay at $0.001$, while HaMI and HaMI (SP) reach $0.019$ and
$0.022$ on LLaMA-3.3-70B-Instruct.

\paragraph{Question-level uncertainty and operating metrics.}
The standard errors above reflect training variance only: the test set is
fixed. To also account for the sampling of test questions and for the
correlation between the five responses to a question, we resample test
questions with all of their responses ($1{,}000$ cluster-bootstrap resamples)
on the four QA benchmarks. These computations use a re-run with the protocol
of Table~\ref{tab:performance_comparison}, whose AUCs agree with that table to
within about $0.01$. The bootstrap standard error of AUC is $0.007$--$0.016$,
against at most $0.006$ across training runs (Table~\ref{tab:std}). Table~\ref{tab:paired_boot}
gives the paired, question-level bootstrap intervals for the AUC difference
between max pooling and each method, and Table~\ref{tab:operating} averages
ranking and threshold-based metrics over the twelve settings; the threshold is
the one that maximizes F1 on validation.

\begin{table}[H]
\centering
\small
\setlength{\tabcolsep}{6pt}
\renewcommand{\arraystretch}{1.15}
\caption{Paired question-level bootstrap $95\%$ intervals of the AUC difference,
max pooling minus each method, over the twelve QA settings: number of settings
in which the interval lies above zero, contains zero, or lies below zero.}
\label{tab:paired_boot}
\begin{tabular}{lccc}
\toprule
Compared with & Max higher & Not significant & Max lower \\
\midrule
Mean & 11 & 1 & 0 \\
HaMI & 11 & 1 & 0 \\
HaMI (SP) & 3 & 8 & 1 \\
SAPLMA & 11 & 1 & 0 \\
LLMs Know & 10 & 2 & 0 \\
SEP & 12 & 0 & 0 \\
\bottomrule
\end{tabular}
\end{table}

\begin{table}[H]
\centering
\small
\setlength{\tabcolsep}{4pt}
\renewcommand{\arraystretch}{1.15}
\caption{Ranking and operating metrics on the QA test sets, averaged over the
twelve settings (mean of five runs). pAUC: standardized partial AUC up to
$10\%$ FPR; F1: test F1 at the validation-selected threshold; P@10\%:
precision among the $10\%$ highest-scored responses.}
\label{tab:operating}
\begin{tabular}{lcccccc}
\toprule
Method & AUC & AUPRC & pAUC & TPR@FPR $10\%$ & F1 & P@10\% \\
\midrule
SEP & 0.802 & 0.707 & 0.637 & 0.480 & 0.671 & 0.803 \\
LLMs Know & 0.854 & 0.789 & 0.697 & 0.607 & 0.734 & 0.873 \\
SAPLMA & 0.856 & 0.786 & 0.696 & 0.604 & 0.734 & 0.870 \\
HaMI & 0.858 & 0.787 & 0.702 & 0.620 & 0.734 & 0.874 \\
HaMI (SP) & 0.885 & 0.827 & 0.738 & 0.685 & 0.763 & 0.905 \\
Mean & 0.870 & 0.808 & 0.719 & 0.640 & 0.739 & 0.896 \\
\textbf{Max} & 0.888 & 0.835 & 0.743 & 0.680 & 0.761 & 0.912 \\
\bottomrule
\end{tabular}
\end{table}

\subsection{Additional Baselines}
\label{app:added_baselines}
We compare max pooling with two further groups of detectors that, like it,
score a single generation from the generator's hidden states.

\paragraph{Set aggregators.} Seven aggregators are applied directly to the
token states, with widths chosen so that each has the same number of
parameters as the max-pooling detector ($256d+513$) within $3.4\%$: attention
MIL pooling and its gated variant~\citep{deepMIL}, log-sum-exp pooling of
token scores with a learned temperature~\citep{lsepool}, DeepSets~\citep{deepsets},
Set Transformer (two set-attention blocks and attention pooling;
\citealp{settransformer}), learned top-$k$ pooling~\citep{gpool}, and a
two-layer Transformer encoder with a classification token. They share the
training recipe of max pooling.

\paragraph{Recent detectors.} HaluNet~\citep{halunet} fuses a one-dimensional
CNN over the token states with the mean token log-likelihood and entropy. The
simple factuality probe~\citep{sfp} (SFP) scores claims in its original form;
at the response level it reduces to an $\ell_1$-regularized logistic regression
on the last-token states, which we apply to three adjacent stored layers.
SinkProbe~\citep{sinkprobe} is a logistic regression on the top-$k$
attention-sink scores of every layer and head. HIDE~\citep{hide} is
training-free: it scores a response by the HSIC between the middle-layer
states of keyword tokens in the prompt and in the answer.

\paragraph{Protocol and results.} The aggregators, HaluNet and SFP are trained
at every stored layer with three runs; each method's validation-best layer is
then evaluated with five runs (SFP is deterministic). SinkProbe uses all layers
and HIDE the middle layer, as in their original definitions, with $k$ and the
kernel width selected on validation. Table~\ref{tab:added_auc} reports the
test AUC of every method. Max pooling matches or exceeds every
additional detector in ten of the twelve settings; the exceptions are
Mistral-Nemo-12B/TriviaQA and LLaMA-3.3-70B/BioASQ, where HaluNet is $0.6$ and
$1.2$ points higher; on the latter most methods lie within $0.880$--$0.896$. HIDE stays near chance
($0.49$--$0.63$ AUC).

\begin{table}[H]
\centering
\small
\setlength{\tabcolsep}{1.4pt}
\renewcommand{\arraystretch}{1.15}
\caption{Test AUC of max pooling and the additional detectors (mean of five
runs at each method's validation-selected layer). Best and second-best are
\textcolor{red}{red} and \textcolor{blue}{blue}.}
\label{tab:added_auc}
\resizebox{\textwidth}{!}{%
\begin{tabular}{@{}l*{12}{c}@{}}
\toprule
& \multicolumn{4}{c}{\textbf{LLaMA-3.1-8B}}
& \multicolumn{4}{c}{\textbf{Mistral-Nemo-12B}}
& \multicolumn{4}{c}{\textbf{LLaMA-3.3-70B}} \\
\cmidrule(lr){2-5} \cmidrule(lr){6-9} \cmidrule(lr){10-13}
Method & TQ & SQ & NQ & BQ & TQ & SQ & NQ & BQ & TQ & SQ & NQ & BQ \\
\midrule
\textbf{Max} & \textcolor{red}{0.924} & \textcolor{red}{0.827} & \textcolor{red}{0.815} & \textcolor{red}{0.899} & \textcolor{blue}{0.936} & \textcolor{red}{0.871} & \textcolor{red}{0.867} & \textcolor{red}{0.913} & \textcolor{red}{0.937} & \textcolor{red}{0.880} & \textcolor{red}{0.897} & 0.884 \\
\midrule
\multicolumn{13}{@{}l}{\textit{Aggregators}} \\
ABMIL & 0.909 & 0.803 & 0.797 & 0.882 & 0.934 & 0.842 & 0.849 & 0.881 & 0.907 & 0.854 & 0.867 & 0.887 \\
Gated ABMIL & 0.909 & \textcolor{blue}{0.813} & \textcolor{blue}{0.807} & 0.884 & 0.933 & 0.847 & 0.853 & 0.883 & 0.904 & 0.866 & 0.882 & \textcolor{blue}{0.892} \\
LSE & 0.908 & \textcolor{blue}{0.813} & 0.797 & 0.885 & 0.932 & \textcolor{blue}{0.856} & \textcolor{blue}{0.855} & \textcolor{blue}{0.908} & \textcolor{blue}{0.928} & \textcolor{blue}{0.867} & \textcolor{blue}{0.888} & 0.890 \\
DeepSets & 0.895 & 0.793 & 0.775 & 0.868 & 0.924 & 0.844 & 0.838 & 0.887 & 0.885 & 0.848 & 0.861 & 0.884 \\
Set Transformer & 0.894 & 0.805 & 0.784 & 0.882 & 0.924 & 0.849 & 0.845 & 0.890 & 0.912 & 0.856 & 0.865 & 0.888 \\
Learned TopK & 0.909 & 0.781 & 0.770 & 0.878 & 0.901 & 0.827 & 0.825 & 0.887 & 0.918 & 0.842 & 0.863 & 0.883 \\
Transformer & 0.901 & 0.808 & 0.790 & \textcolor{blue}{0.890} & 0.930 & 0.854 & 0.852 & 0.903 & 0.910 & 0.857 & 0.871 & \textcolor{blue}{0.892} \\
\midrule
\multicolumn{13}{@{}l}{\textit{Recent detectors}} \\
HaluNet & \textcolor{blue}{0.913} & 0.811 & 0.792 & \textcolor{blue}{0.890} & \textcolor{red}{0.942} & 0.855 & 0.847 & 0.903 & 0.917 & 0.856 & 0.874 & \textcolor{red}{0.896} \\
SFP & 0.883 & 0.772 & 0.755 & 0.864 & 0.910 & 0.829 & 0.838 & 0.872 & 0.927 & 0.859 & 0.868 & 0.887 \\
SinkProbe & 0.866 & 0.775 & 0.776 & 0.855 & 0.914 & 0.838 & 0.827 & 0.879 & 0.907 & 0.852 & 0.866 & 0.888 \\
HIDE & 0.531 & 0.546 & 0.539 & 0.626 & 0.525 & 0.525 & 0.527 & 0.577 & 0.564 & 0.492 & 0.576 & 0.528 \\
\bottomrule
\end{tabular}}
\end{table}

On LongFact++ (Table~\ref{tab:added_lfpp}), with the protocol of
Section~\ref{ex:length}, the set aggregators lie within $3.2$ points of max
pooling on LLaMA-3.1-8B-Instruct and within $2.5$ points on
LLaMA-3.3-70B-Instruct. LSE is the only method above max pooling on both models,
by $0.1$ and $1.3$ points, and gated ABMIL is $0.2$ points above it on
LLaMA-3.3-70B-Instruct. HaluNet and SinkProbe come within $1.8$ points of max
pooling, SFP stays near the fixed-position probes, and HIDE is at chance.

\begin{table}[H]
\centering
\small
\setlength{\tabcolsep}{6pt}
\renewcommand{\arraystretch}{1.15}
\caption{LongFact++ test AUC of max pooling, mean pooling and the additional
detectors (mean of five runs; SFP, SinkProbe and HIDE are deterministic). Best
and second-best are \textcolor{red}{red} and \textcolor{blue}{blue}.}
\label{tab:added_lfpp}
\begin{tabular}{lcc}
\toprule
\textbf{Method} & LLaMA-3.1-8B-Instruct & LLaMA-3.3-70B-Instruct \\
\midrule
\textbf{Max} & \textcolor{blue}{0.783} & 0.739 \\
Mean & 0.759 & 0.739 \\
\midrule
ABMIL & 0.769 & 0.732 \\
Gated ABMIL & 0.761 & \textcolor{blue}{0.741} \\
LSE & \textcolor{red}{0.784} & \textcolor{red}{0.752} \\
DeepSets & 0.756 & 0.714 \\
Set Transformer & 0.773 & 0.730 \\
Learned TopK & 0.751 & 0.722 \\
Transformer & 0.772 & 0.739 \\
\midrule
HaluNet & 0.770 & 0.738 \\
SFP & 0.635 & 0.632 \\
SinkProbe & 0.765 & 0.734 \\
HIDE & 0.500 & 0.534 \\
\bottomrule
\end{tabular}
\end{table}

\paragraph{End-to-end time.} Table~\ref{tab:added_speed} reports the
end-to-end time per QA response of the additional detectors, measured with
max and mean pooling in the same runs as Table~\ref{tab:throughput_comparison}, and
Table~\ref{tab:added_parts} splits it into generation and the work after
generation. The set aggregators keep the answer-token states of their layer
on the GPU during generation and score the stored matrix at the end. HaluNet
also keeps the next-token logits for the log-likelihood and entropy, HIDE
selects keywords with a sentence encoder and computes the HSIC, and SinkProbe
must generate with eager attention, because SDPA does not return the attention
maps from which its sink scores are accumulated. Relative to generation alone,
the median per-response overhead over the twelve QA settings is $0.5\%$ for
max pooling, $0.4\%$ for mean pooling, $0.4\%$ for SFP, at most $0.7\%$ for
the set aggregators, $0.9\%$ for HaluNet, $8.1\%$ for HIDE and $30.8\%$ for
SinkProbe. On LongFact++, the work after generation grows with the answer
length for the detectors that store states (up to $1.3$\,s for HaluNet and
$0.9$\,s for HIDE on LLaMA-3.3-70B), but stays below $1$\,ms for online max
pooling.

\begin{table}[H]
\centering
\small
\setlength{\tabcolsep}{1.2pt}
\renewcommand{\arraystretch}{1.15}
\caption{End-to-end time per QA response (s) of max pooling, mean pooling
and the additional detectors, including generation (mean over $50$ responses;
same runs as Table~\ref{tab:throughput_comparison}).}
\label{tab:added_speed}
\resizebox{\textwidth}{!}{%
\begin{tabular}{@{}l*{12}{c}@{}}
\toprule
& \multicolumn{4}{c}{\textbf{LLaMA-3.1-8B}}
& \multicolumn{4}{c}{\textbf{Mistral-Nemo-12B}}
& \multicolumn{4}{c}{\textbf{LLaMA-3.3-70B}} \\
\cmidrule(lr){2-5} \cmidrule(lr){6-9} \cmidrule(lr){10-13}
\textbf{Method} & TQ & SQ & NQ & BQ & TQ & SQ & NQ & BQ & TQ & SQ & NQ & BQ \\
\midrule
Generation only & 0.23 & 0.46 & 0.43 & 0.47 & 0.17 & 0.35 & 0.37 & 0.50 & 5.45 & 5.85 & 5.87 & 7.30 \\
\midrule
\textbf{Max} & 0.23 & 0.47 & 0.43 & 0.49 & 0.17 & 0.38 & 0.37 & 0.51 & 5.46 & 5.84 & 5.90 & 7.28 \\
Mean & 0.23 & 0.46 & 0.43 & 0.48 & 0.17 & 0.35 & 0.37 & 0.50 & 5.45 & 5.82 & 5.90 & 7.28 \\
\midrule
ABMIL & 0.23 & 0.47 & 0.43 & 0.47 & 0.17 & 0.35 & 0.36 & 0.50 & 5.45 & 5.81 & 5.86 & 7.29 \\
Gated ABMIL & 0.23 & 0.48 & 0.45 & 0.47 & 0.17 & 0.35 & 0.36 & 0.50 & 5.45 & 5.82 & 5.87 & 7.28 \\
LSE & 0.23 & 0.47 & 0.43 & 0.47 & 0.17 & 0.35 & 0.38 & 0.50 & 5.47 & 5.83 & 5.90 & 7.29 \\
DeepSets & 0.23 & 0.45 & 0.43 & 0.47 & 0.17 & 0.35 & 0.36 & 0.50 & 5.45 & 5.82 & 5.87 & 7.32 \\
Set Transformer & 0.23 & 0.47 & 0.44 & 0.47 & 0.17 & 0.36 & 0.38 & 0.51 & 5.46 & 5.83 & 5.89 & 7.30 \\
Learned TopK & 0.23 & 0.47 & 0.44 & 0.47 & 0.17 & 0.36 & 0.36 & 0.50 & 5.45 & 5.82 & 5.87 & 7.32 \\
Transformer & 0.23 & 0.47 & 0.45 & 0.49 & 0.17 & 0.36 & 0.37 & 0.51 & 5.45 & 5.82 & 5.87 & 7.33 \\
\midrule
HaluNet & 0.24 & 0.52 & 0.49 & 0.49 & 0.17 & 0.36 & 0.38 & 0.54 & 5.46 & 5.89 & 6.03 & 7.53 \\
SFP & 0.23 & 0.47 & 0.43 & 0.47 & 0.17 & 0.35 & 0.36 & 0.50 & 5.45 & 5.82 & 5.87 & 7.29 \\
SinkProbe & 0.30 & 0.60 & 0.58 & 0.63 & 0.23 & 0.47 & 0.49 & 0.66 & 5.66 & 6.03 & 6.08 & 7.57 \\
HIDE & 0.26 & 0.51 & 0.50 & 0.51 & 0.19 & 0.38 & 0.41 & 0.56 & 5.51 & 5.93 & 5.93 & 7.43 \\
\bottomrule
\end{tabular}}
\end{table}

\begin{table}[H]
\centering
\small
\setlength{\tabcolsep}{2pt}
\renewcommand{\arraystretch}{1.12}
\caption{Parts of the end-to-end time per response: generation, including the
hooks that collect each detector's input (Gen., s), and the work after
generation until the score (Post, ms). Mean over responses; QA: mean over the
four benchmarks.}
\label{tab:added_parts}
\resizebox{\textwidth}{!}{%
\begin{tabular}{@{}l*{10}{c}@{}}
\toprule
& \multicolumn{2}{c}{\textbf{QA 8B}} & \multicolumn{2}{c}{\textbf{QA 12B}} & \multicolumn{2}{c}{\textbf{QA 70B}}
& \multicolumn{2}{c}{\textbf{LongFact++ 8B-Inst.}} & \multicolumn{2}{c}{\textbf{LongFact++ 70B-Inst.}} \\
\cmidrule(lr){2-3} \cmidrule(lr){4-5} \cmidrule(lr){6-7} \cmidrule(lr){8-9} \cmidrule(lr){10-11}
\textbf{Method} & Gen. & Post & Gen. & Post & Gen. & Post & Gen. & Post & Gen. & Post \\
\midrule
Generation only & 0.40 & 0.0 & 0.35 & 0.0 & 6.12 & 0.0 & 15.05 & 0.0 & 95.70 & 0.0 \\
\midrule
\textbf{Max} & 0.41 & 0.1 & 0.36 & 0.2 & 6.12 & 0.3 & 15.06 & 0.1 & 95.58 & 0.8 \\
Mean & 0.40 & 0.1 & 0.35 & 0.1 & 6.11 & 0.1 & 15.09 & 0.2 & 94.88 & 0.2 \\
\midrule
ABMIL & 0.40 & 0.8 & 0.34 & 1.0 & 6.10 & 1.6 & 14.89 & 7.6 & 94.96 & 6.7 \\
Gated ABMIL & 0.40 & 2.4 & 0.34 & 1.0 & 6.10 & 2.2 & 14.90 & 11.9 & 96.06 & 23.6 \\
LSE & 0.40 & 4.1 & 0.35 & 3.0 & 6.12 & 5.5 & 14.99 & 21.7 & 96.50 & 35.5 \\
DeepSets & 0.40 & 0.5 & 0.35 & 1.1 & 6.11 & 2.0 & 14.89 & 5.4 & 96.03 & 31.9 \\
Set Transformer & 0.40 & 7.4 & 0.35 & 5.5 & 6.11 & 12.8 & 14.94 & 24.3 & 95.79 & 82.3 \\
Learned TopK & 0.40 & 4.1 & 0.35 & 2.9 & 6.11 & 9.0 & 14.90 & 22.4 & 95.83 & 201 \\
Transformer & 0.40 & 14.6 & 0.34 & 7.5 & 6.10 & 18.2 & 14.92 & 39.2 & 95.90 & 247 \\
\midrule
HaluNet & 0.41 & 24.5 & 0.35 & 19.7 & 6.12 & 114 & 14.95 & 260 & 95.43 & 1{,}328 \\
SFP & 0.40 & 0.2 & 0.34 & 0.2 & 6.11 & 0.4 & 15.02 & 0.3 & 96.00 & 1.6 \\
SinkProbe & 0.52 & 3.5 & 0.46 & 4.3 & 6.32 & 9.9 & 19.14 & 4.3 & 110.37 & 16.1 \\
HIDE & 0.40 & 45.6 & 0.35 & 38.9 & 6.12 & 80.3 & 14.88 & 191 & 95.04 & 884 \\
\bottomrule
\end{tabular}}
\end{table}

\subsection{Domain Generalization}
\label{app:lodo}
To test whether a detector trained on some benchmarks transfers to an unseen
one, we hold out each QA benchmark in turn (leave-one-dataset-out). Detectors
are trained on the other three benchmarks, drawing the same number of training questions from each (one answer per
question, as in-domain) so that the total equals the mean size of the three
sources' training splits;
the checkpoint is chosen on the union of the source validation splits, and
each method's layer is the one with the highest mean validation AUC in the
sources' own layer sweeps. No data from the held-out benchmark is used for
training, checkpoint selection, layer selection or hyperparameters. The test
split of the held-out benchmark is the same as in Table~\ref{tab:performance_comparison}.
The bottom rows of Table~\ref{tab:lodo} give the number of questions in each
split: $2{,}799$--$3{,}588$ training questions ($933$--$1{,}196$ per source),
$2{,}235$--$2{,}660$ validation questions and $583$--$1{,}097$ test questions,
with validation and test questions keeping all their retained responses.

Table~\ref{tab:lodo} reports the results. Max pooling exceeds mean pooling and
HaMI in all twelve settings and is the best of the four methods in five.
HaMI (SP) is best in the other seven, including all four LLaMA-3.1-8B
targets, and degrades least: averaged over the twelve settings, AUC falls by
$1.9$ points relative to in-domain training (Table~\ref{tab:performance_comparison}) for HaMI (SP), $3.6$ for max
pooling, $4.0$ for mean pooling and $4.6$ for HaMI. The largest drop for max
pooling is on LLaMA-3.1-8B/BioASQ ($12.0$ points). HaMI (SP) uses semantic
probabilities computed on the target responses themselves, which do not
depend on the training distribution.

\begin{table}[H]
\centering
\small
\setlength{\tabcolsep}{1.4pt}
\renewcommand{\arraystretch}{1.2}
\caption{Leave-one-dataset-out AUC (mean of five runs): each column is the
held-out benchmark, with detectors trained on the other three.
\textcolor{red}{Red}/\textcolor{blue}{blue}: best/second-best.}
\label{tab:lodo}
\begin{tabular}{@{}l*{12}{c}@{}}
\toprule
& \multicolumn{4}{c}{\textbf{LLaMA-3.1-8B}}
& \multicolumn{4}{c}{\textbf{Mistral-Nemo-12B}}
& \multicolumn{4}{c}{\textbf{LLaMA-3.3-70B}} \\
\cmidrule(lr){2-5} \cmidrule(lr){6-9} \cmidrule(lr){10-13}
Method & TQ & SQ & NQ & BQ & TQ & SQ & NQ & BQ & TQ & SQ & NQ & BQ \\
\midrule
\textbf{Max} & \textcolor{blue}{0.885} & \textcolor{blue}{0.793} & \textcolor{blue}{0.798} & \textcolor{blue}{0.779} & \textcolor{red}{0.900} & \textcolor{blue}{0.847} & \textcolor{red}{0.851} & \textcolor{red}{0.868} & \textcolor{red}{0.910} & \textcolor{red}{0.858} & \textcolor{blue}{0.867} & \textcolor{blue}{0.858} \\
Mean & 0.860 & 0.782 & 0.775 & 0.742 & 0.875 & 0.836 & 0.828 & \textcolor{blue}{0.860} & 0.878 & 0.832 & 0.848 & 0.834 \\
HaMI & 0.826 & 0.758 & 0.753 & 0.689 & 0.868 & 0.829 & 0.818 & 0.834 & 0.856 & 0.824 & 0.845 & 0.826 \\
HaMI (SP) & \textcolor{red}{0.907} & \textcolor{red}{0.833} & \textcolor{red}{0.827} & \textcolor{red}{0.862} & \textcolor{blue}{0.898} & \textcolor{red}{0.858} & \textcolor{blue}{0.844} & \textcolor{blue}{0.860} & \textcolor{blue}{0.902} & \textcolor{blue}{0.838} & \textcolor{red}{0.869} & \textcolor{red}{0.877} \\
\midrule
\multicolumn{13}{@{}l}{\textit{Number of questions}} \\
Train & 3{,}012 & 3{,}204 & 3{,}099 & 3{,}588 & 2{,}892 & 3{,}129 & 3{,}024 & 3{,}495 & 2{,}799 & 3{,}075 & 2{,}973 & 3{,}423 \\
Validation & 2{,}358 & 2{,}376 & 2{,}368 & 2{,}660 & 2{,}278 & 2{,}340 & 2{,}306 & 2{,}586 & 2{,}235 & 2{,}310 & 2{,}291 & 2{,}557 \\
Test & 1{,}097 & 1{,}055 & 1{,}087 & 599 & 1{,}090 & 1{,}018 & 1{,}061 & 584 & 1{,}092 & 1{,}004 & 1{,}023 & 583 \\
\bottomrule
\end{tabular}
\end{table}

\section{Ablations and Sensitivity Analyses}
\label{app:ablations}
This appendix complements Section~\ref{ex:ablations} with sensitivity to the feature dimension, the layer and the correctness judge.
\subsection{Impact of Feature Dimension \texorpdfstring{$D$}{D}}
\label{app:dim}
Table~\ref{tab:dim_mean_vs_max} compares mean and max pooling on
LLaMA-3.1-8B/TriviaQA. Both methods use the same dimension grid
$\{1,2,4,8,16,32,64,128,256,512,1024\}$, with all other hyperparameters
held fixed. Both approach a plateau around $D=256$, the dimension used in
the main experiments. 

\begin{table}[H]
\centering
\renewcommand{\arraystretch}{1.2}
\caption{Mean- and max-pooling AUC by feature dimension (LLaMA-3.1-8B, TriviaQA), with shared training settings (separate runs from Table~\ref{tab:performance_comparison}, so $D=256$ can differ by $0.001$).}
\label{tab:dim_mean_vs_max}
\begin{tabular}{lcc}
\toprule
Dimension ($D$) & Mean Pooling & Max Pooling \\
\midrule
1    & 0.770 & 0.833 \\
2    & 0.782 & 0.866 \\
4    & 0.851 & 0.885 \\
8    & 0.871 & 0.912 \\
16   & 0.881 & 0.917 \\
32   & 0.887 & 0.920 \\
64   & 0.892 & 0.922 \\
128  & 0.895 & 0.923 \\
256  & 0.896 & 0.924 \\
512  & 0.897 & 0.925 \\
1024 & 0.896 & 0.925 \\
\bottomrule
\end{tabular}
\end{table}

\paragraph{Dimension at every layer.}
We also trained max-pooling detectors for $D\in\{1,4,16,64,256,1024\}$ at every
second layer (every fourth for LLaMA-3.3-70B) on the four QA benchmarks, with
seeds 42--46, $100$ epochs and the checkpoint with the highest validation AUC
(Figures~\ref{fig:dimlayer_8b}--\ref{fig:dimlayer_70b}). At every layer, AUC
rises from $D=1$ to $D=16$ and then levels off: at the layer with the highest
validation AUC, every $D\ge64$ is within $0.006$ of the best dimension, and
$D=256$ within $0.002$. The shape across layers does not depend on $D$, and
for $D\ge16$ the selected layer is the same at every dimension in $11$ of the
$12$ settings. These are separate runs from
Table~\ref{tab:performance_comparison}, so individual values can differ from it
slightly.

\begin{figure}[H]
\centering
\includegraphics[width=\textwidth,height=0.37\textheight,keepaspectratio]{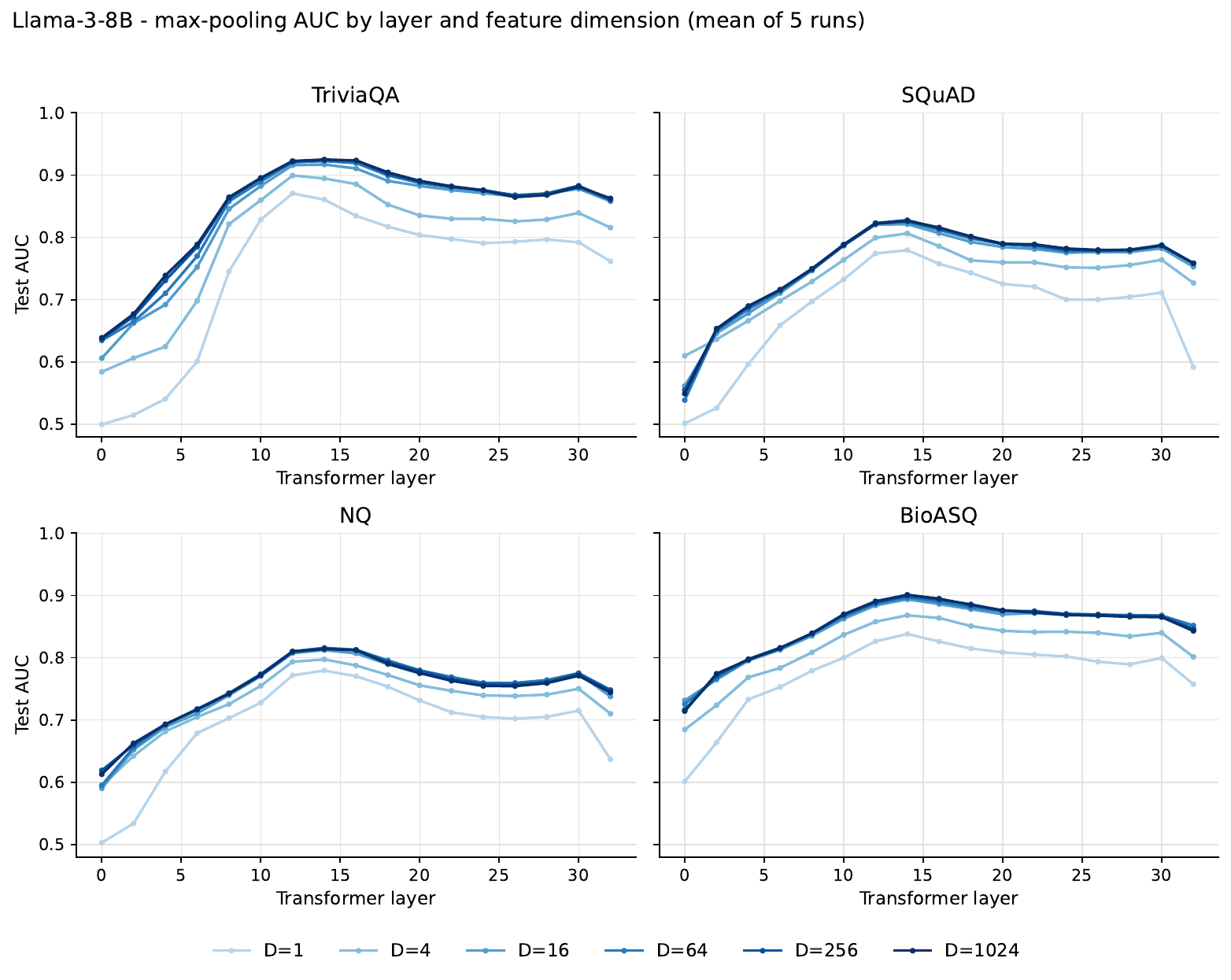}
\caption{LLaMA-3.1-8B: max-pooling test AUC by layer for each feature dimension $D$ (17 layers, stride two; mean of five runs).}
\label{fig:dimlayer_8b}
\end{figure}

\begin{figure}[H]
\centering
\includegraphics[width=\textwidth,height=0.37\textheight,keepaspectratio]{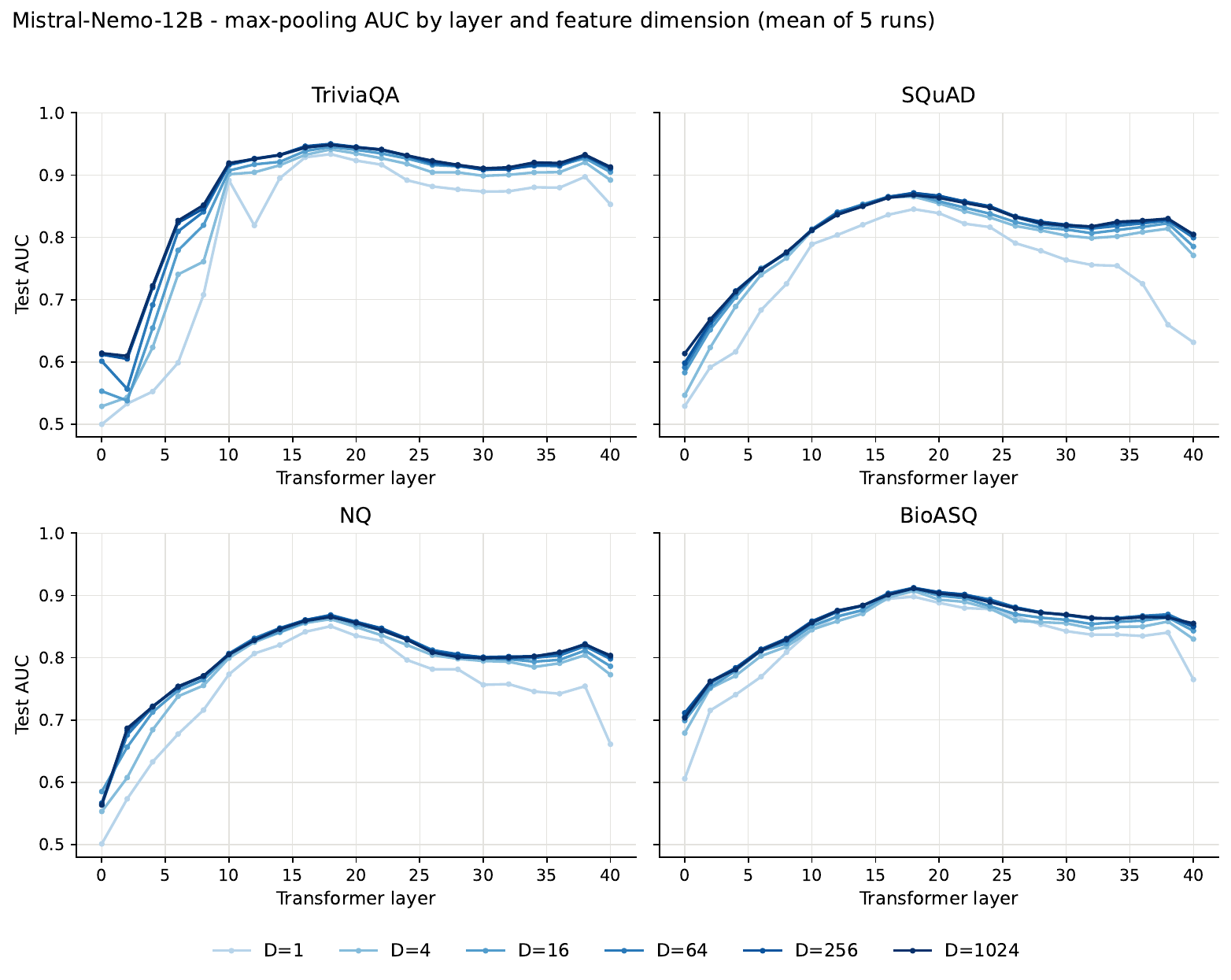}
\caption{Mistral-Nemo-12B: max-pooling test AUC by layer for each feature dimension $D$ (21 layers, stride two; mean of five runs).}
\label{fig:dimlayer_12b}
\end{figure}

\begin{figure}[H]
\centering
\includegraphics[width=\textwidth,height=0.37\textheight,keepaspectratio]{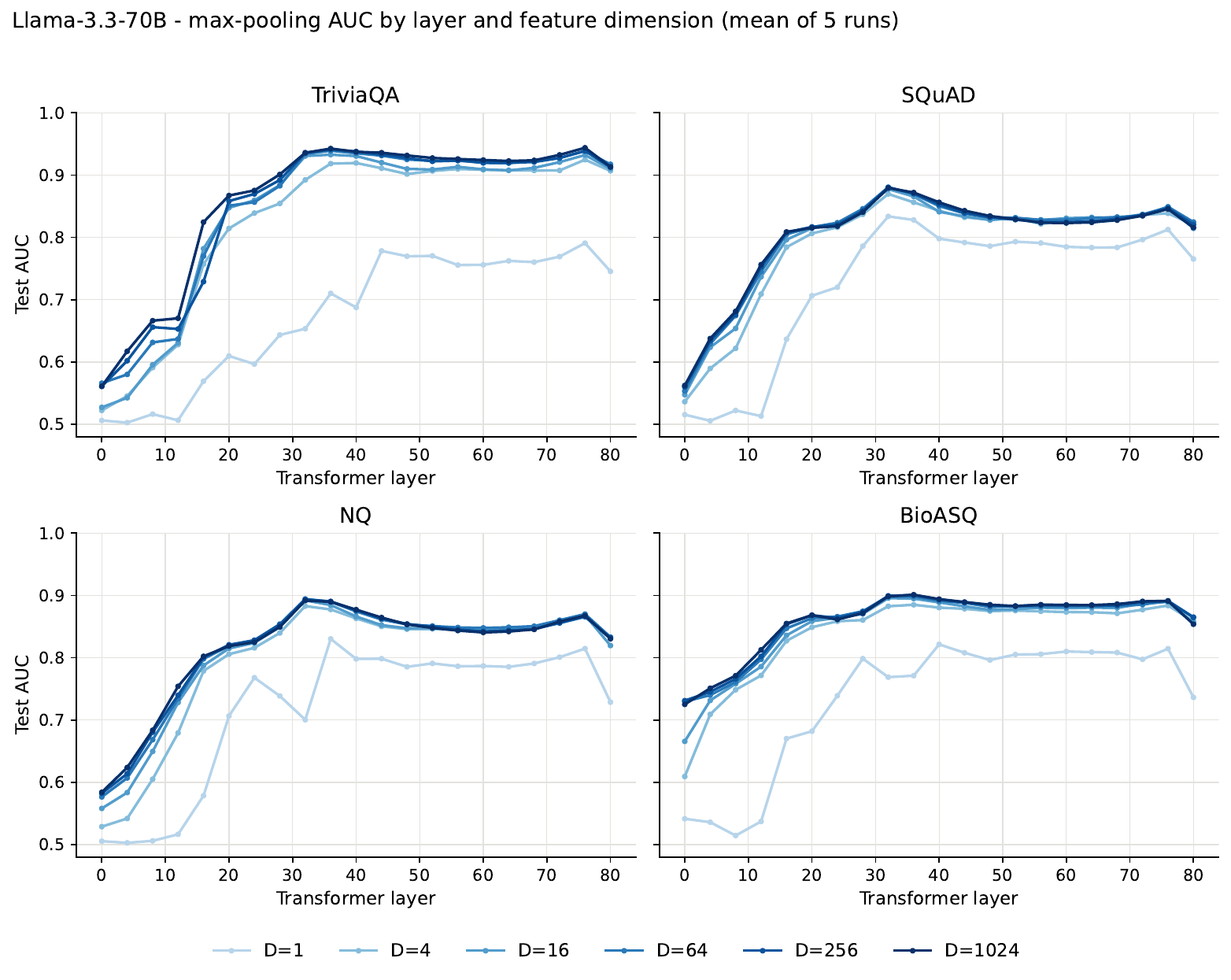}
\caption{LLaMA-3.3-70B: max-pooling test AUC by layer for each feature dimension $D$ (21 layers, stride four; mean of five runs).}
\label{fig:dimlayer_70b}
\end{figure}

\subsection{Layer-Wise Sensitivity}
\label{app:layer_sensitivity}
Layers are selected by validation AUC (Appendix~\ref{app:layer_selection}).
Figures~\ref{fig:layer_sensitivity_8b}--\ref{fig:layer_sensitivity_70b}
show test AUC across the searched layers: $17$ for LLaMA-3.1-8B, $21$ for
Mistral-Nemo-12B and $41$ for LLaMA-3.3-70B, at a stride of two. Every point
in the QA panels of this subsection is a mean of three runs with its standard
error, not the five used for the reported numbers elsewhere; the LongFact++
curves use five runs.

On the short-answer benchmarks, the methods generally
trace an inverted-U, peaking at roughly $40$--$45\%$ of the network's
depth ($12$--$16$ of $32$, $16$--$20$ of $40$, $32$--$38$ of $80$); the peak
location tracks relative rather than absolute depth, and it is set by the model
and the task. The curves are flat near the
peak: moving one stride step away from the selected layer costs a median of
$0.5$ AUC points for the pooling detectors, and stays under $1.3$ points in
nine cases out of ten. Figure~\ref{fig:layer_sensitivity_lfpp} shows the LongFact++ curves;
Table~\ref{tab:layer_length} gives its validation-selected layers.

\begin{figure}[H]
\centering
\includegraphics[width=\textwidth,height=0.37\textheight,keepaspectratio]{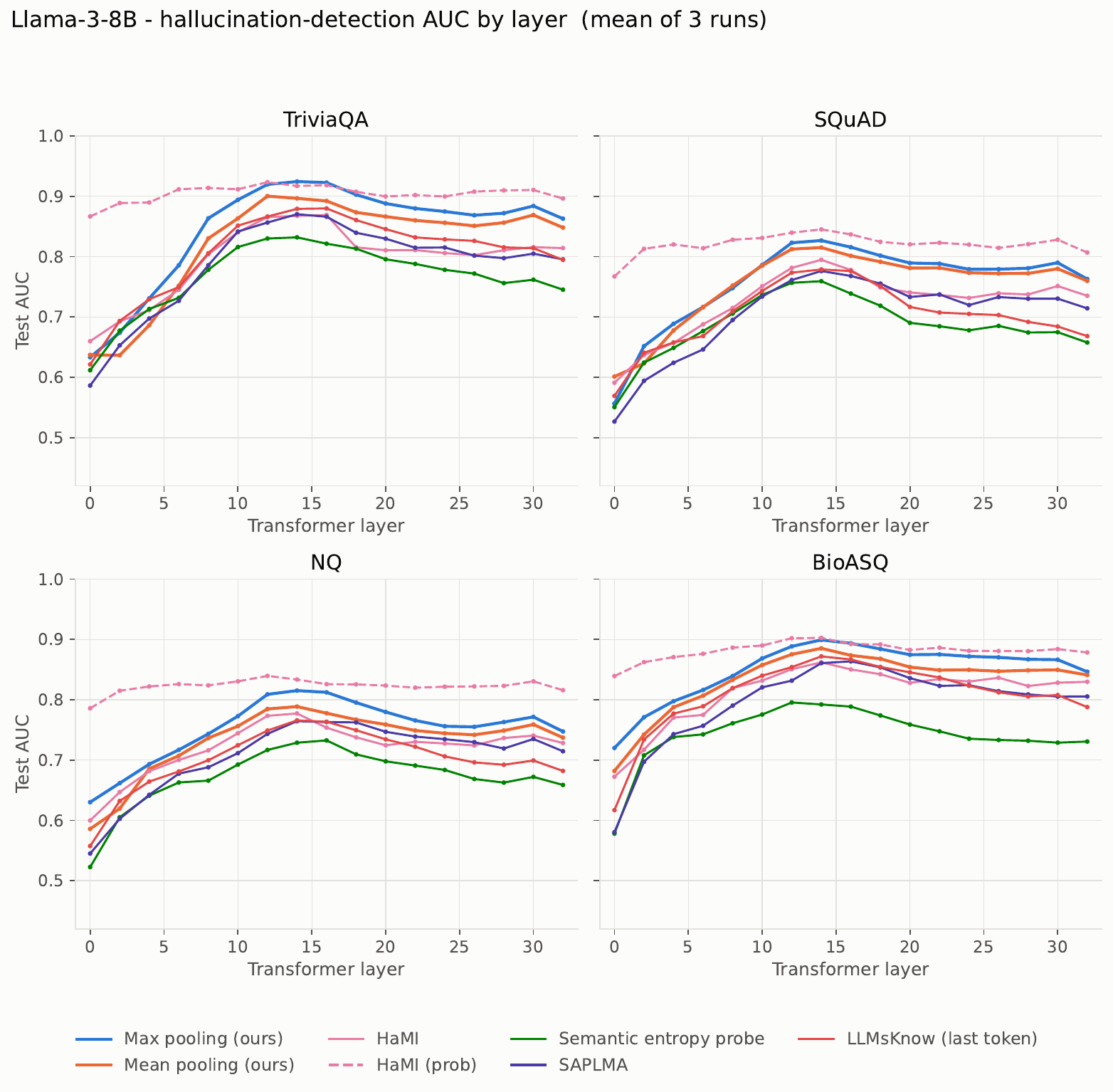}
\caption{LLaMA-3.1-8B layer sensitivity ($17$ layers, stride two).
Curves show mean test AUC over three runs.}
\label{fig:layer_sensitivity_8b}
\end{figure}

\begin{figure}[H]
\centering
\includegraphics[width=\textwidth,height=0.37\textheight,keepaspectratio]{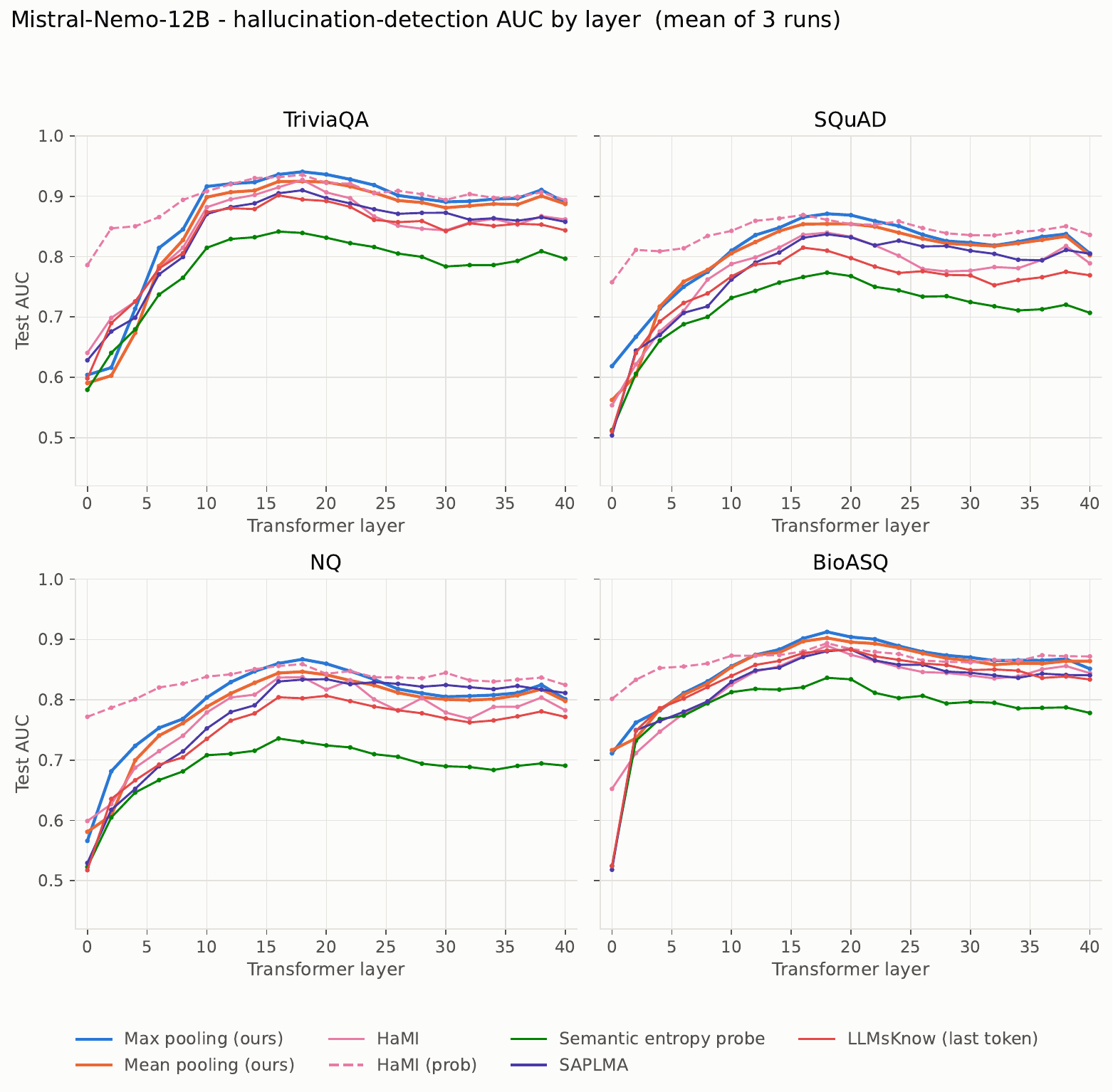}
\caption{Mistral-Nemo-12B layer sensitivity ($21$ layers, stride two).
Conventions as in Figure~\ref{fig:layer_sensitivity_8b}.}
\label{fig:layer_sensitivity_12b}
\end{figure}

\begin{figure}[H]
\centering
\includegraphics[width=\textwidth,height=0.37\textheight,keepaspectratio]{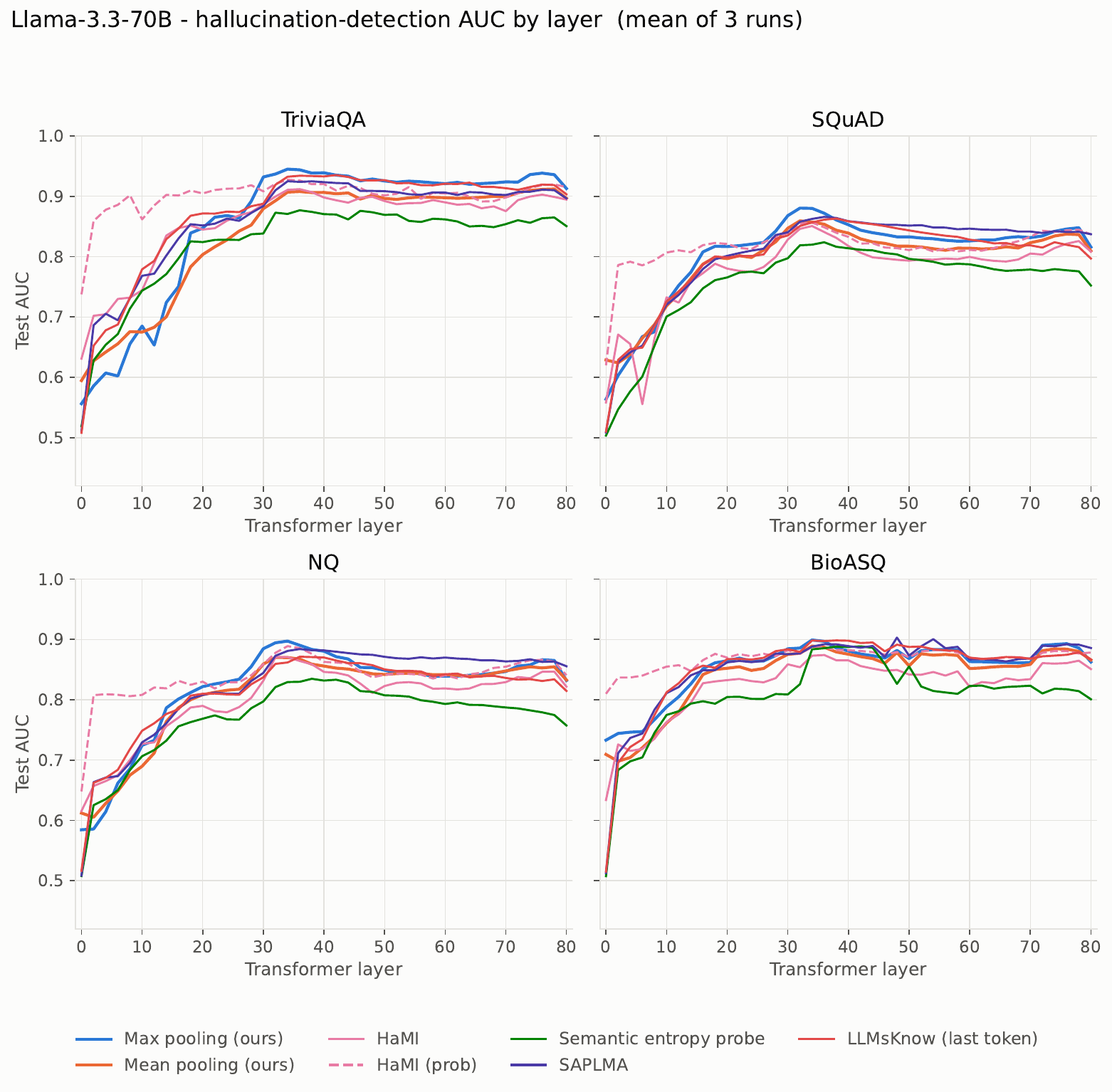}
\caption{LLaMA-3.3-70B layer sensitivity ($41$ layers, stride two).
Conventions as in Figure~\ref{fig:layer_sensitivity_8b}.}
\label{fig:layer_sensitivity_70b}
\end{figure}

\begin{figure}[H]
\centering
\includegraphics[width=\textwidth,height=0.37\textheight,keepaspectratio]{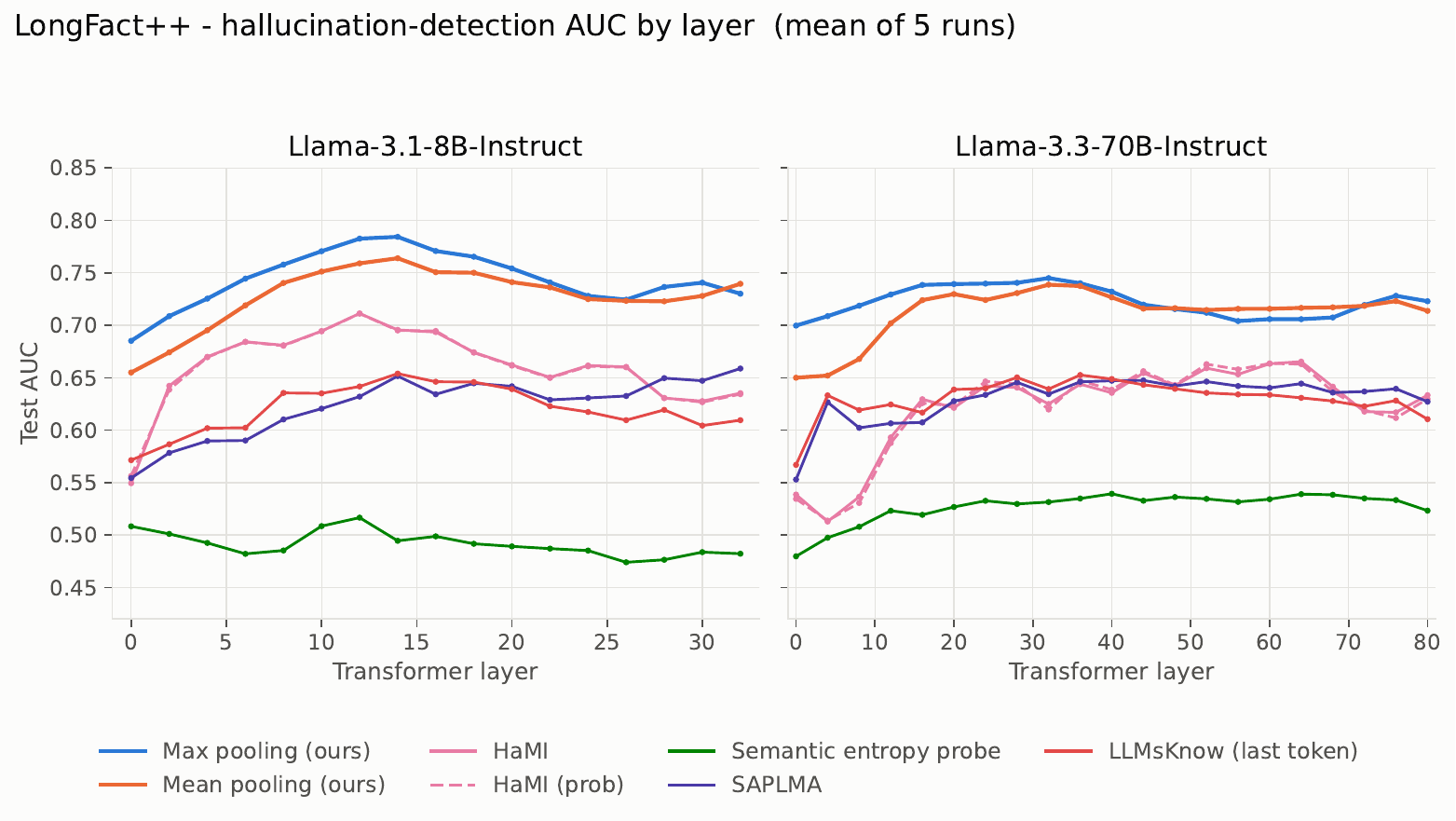}
\caption{LongFact++ layer sensitivity for LLaMA-3.1-8B-Instruct ($17$ layers,
stride two) and LLaMA-3.3-70B-Instruct ($21$ layers, stride four). Curves show
mean test AUC over five runs.}
\label{fig:layer_sensitivity_lfpp}
\end{figure}

\subsection{Sensitivity to the Correctness Judge}
\label{app:judge_sensitivity}
QA labels come from GPT-5 mini (Appendix~\ref{app:label_validation}). To test whether the detectors depend on this judge, we relabeled test responses twice with a judge from a different model family, Claude Sonnet 5.5: first from the answer alone, as GPT-5 mini does, and then with web search for the responses it cannot settle that way.

\paragraph{Setup.} For each generator and QA dataset we drew $300$ test responses at random ($3{,}600$ in total). The judge saw the question, the reference answers and the response, and applied the same correctness criterion as GPT-5 mini. The detectors are those of Table~\ref{tab:performance_comparison}, trained on GPT-5 mini labels and not retrained; only the test labels change. We report the mean test AUC over five runs on the same responses under each set of labels.

\paragraph{Without web search.} The two judges agree on $94.1\%$ of responses (mean Cohen's $\kappa=0.86$; Table~\ref{tab:judge_agreement}). Under the Sonnet labels, every method loses $0.013$--$0.024$ AUC on average (Table~\ref{tab:judge_sensitivity}). Outside BioASQ the mean change per method is at most $0.014$. On BioASQ, Sonnet marks fewer responses as hallucinated ($24$--$42\%$ vs.\ $34$--$49\%$), mostly list answers that are partly correct, and every method loses $0.04$--$0.06$. Max pooling keeps the highest mean AUC ($0.879$, against $0.877$ for HaMI (SP) and $0.869$ for mean pooling).

\paragraph{With web search.} Some disagreements concern facts that neither judge can settle from its own knowledge. The Sonnet judge therefore re-checked, with one web search per response, every response whose answer does not contain a reference answer verbatim or on which the two judges disagree ($2{,}013$ of the $3{,}600$). This changes $2.9\%$ of the Sonnet labels and brings them closer to GPT-5 mini (mean $\kappa=0.89$, range $0.79$--$0.99$). Under these labels every method loses $0.012$--$0.024$ AUC relative to GPT-5 mini (Table~\ref{tab:judge_web}), and max pooling again has the highest mean AUC ($0.883$, against $0.876$ for HaMI (SP) and $0.870$ for mean pooling).

\paragraph{Summary.} Mean AUC of max pooling is $0.897$ under GPT-5 mini, $0.879$ under Sonnet without web search and $0.883$ with it; it has the highest mean AUC of all methods under each of the three label sets.

\begin{table}[H]
\centering
\small
\renewcommand{\arraystretch}{1.15}
\caption{Agreement between GPT-5 mini and Claude Sonnet 5.5 on $300$ test responses per setting: Cohen's $\kappa$ (agreement, \%), and hallucination rate under GPT-5 mini / Sonnet.}
\label{tab:judge_agreement}
\begin{tabular}{lcccc}
\toprule
Model & TriviaQA & SQuAD & NQ & BioASQ \\
\midrule
LLaMA-3.1-8B     & $0.99$ ($99.7$) & $0.78$ ($91.3$) & $0.89$ ($94.7$) & $0.81$ ($90.7$) \\
Mistral-Nemo-12B & $0.99$ ($99.7$) & $0.89$ ($94.3$) & $0.91$ ($95.3$) & $0.84$ ($92.3$) \\
LLaMA-3.3-70B    & $0.81$ ($97.7$) & $0.82$ ($91.0$) & $0.86$ ($94.7$) & $0.71$ ($88.0$) \\
\midrule
\multicolumn{5}{l}{Hallucination rate (\%), GPT-5 mini / Sonnet} \\
LLaMA-3.1-8B     & $26$ / $26$ & $71$ / $74$ & $60$ / $56$ & $49$ / $42$ \\
Mistral-Nemo-12B & $16$ / $16$ & $57$ / $54$ & $50$ / $45$ & $41$ / $33$ \\
LLaMA-3.3-70B    & $7$ / $6$ & $49$ / $48$ & $29$ / $24$ & $34$ / $24$ \\
\bottomrule
\end{tabular}
\end{table}

\begin{table}[H]
\centering
\small
\renewcommand{\arraystretch}{1.15}
\caption{Without web search: test AUC on the same responses under GPT-5 mini and Claude Sonnet 5.5 labels (mean over five runs, then over the $12$ model--dataset settings). Detectors are trained on GPT-5 mini labels. ``Sonnet by model'' averages the four QA datasets per generator; ``Best'' counts settings in which the method has the highest AUC under Sonnet labels.}
\label{tab:judge_sensitivity}
\resizebox{\linewidth}{!}{%
\begin{tabular}{lccccccc}
\toprule
& \multicolumn{3}{c}{All settings} & \multicolumn{3}{c}{Sonnet by model} & \\
\cmidrule(lr){2-4}\cmidrule(lr){5-7}
Method & GPT-5 mini & Sonnet & $\Delta$ & 8B & 12B & 70B & Best \\
\midrule
\textbf{Max} & $\mathbf{0.897}$ & $\mathbf{0.879}$ & $-0.018$ & $0.863$ & $\mathbf{0.895}$ & $0.879$ & $5$ \\
Mean & $0.882$ & $0.869$ & $-0.013$ & $0.846$ & $0.883$ & $0.878$ & $2$ \\
HaMI         & $0.867$ & $0.852$ & $-0.016$ & $0.827$ & $0.860$ & $0.868$ & $0$ \\
HaMI (SP)    & $0.892$ & $0.877$ & $-0.015$ & $\mathbf{0.867}$ & $0.884$ & $\mathbf{0.880}$ & $3$ \\
SAPLMA       & $0.857$ & $0.840$ & $-0.017$ & $0.803$ & $0.846$ & $0.871$ & $1$ \\
LLMs Know    & $0.858$ & $0.841$ & $-0.017$ & $0.808$ & $0.841$ & $0.873$ & $1$ \\
SEP          & $0.811$ & $0.787$ & $-0.024$ & $0.759$ & $0.785$ & $0.817$ & $0$ \\
\bottomrule
\end{tabular}}
\end{table}

\begin{table}[H]
\centering
\small
\setlength{\tabcolsep}{5pt}
\renewcommand{\arraystretch}{1.15}
\caption{With web search: test AUC under GPT-5 mini labels and under web-verified Sonnet labels
(same responses and detectors as Table~\ref{tab:judge_sensitivity}). ``Best''
counts the settings in which the method has the highest AUC under the
web-verified labels.}
\label{tab:judge_web}
\begin{tabular}{lcccccc}
\toprule
& \multicolumn{2}{c}{All settings} & \multicolumn{3}{c}{Web-verified by model} & \\
\cmidrule(lr){2-3}\cmidrule(lr){4-6}
Method & GPT-5 mini & Web-verified & 8B & 12B & 70B & Best \\
\midrule
\textbf{Max} & $\mathbf{0.897}$ & $\mathbf{0.883}$ & $0.871$ & $\mathbf{0.893}$ & $\mathbf{0.886}$ & $6$ \\
Mean & $0.882$ & $0.870$ & $0.854$ & $0.882$ & $0.873$ & $1$ \\
HaMI & $0.867$ & $0.853$ & $0.835$ & $0.860$ & $0.864$ & $0$ \\
HaMI (SP) & $0.892$ & $0.876$ & $\mathbf{0.875}$ & $0.878$ & $0.874$ & $3$ \\
SAPLMA & $0.857$ & $0.843$ & $0.810$ & $0.843$ & $0.874$ & $1$ \\
LLMs Know & $0.858$ & $0.842$ & $0.814$ & $0.843$ & $0.870$ & $1$ \\
SEP & $0.811$ & $0.787$ & $0.758$ & $0.784$ & $0.819$ & $0$ \\
\bottomrule
\end{tabular}
\end{table}

\section{Empirical Check of \texorpdfstring{Lemma~\ref{thm:global} and Theorem~\ref{thm:sparse}}{Lemma 2 and Theorem 3}}
\label{app:validation}
This section checks, on trained detectors, the statements of
Section~\ref{sec:margin_max_pooling}. Theorem~\ref{theorem:hami} is examined
separately in Appendix~\ref{app:assumption_check}.

\paragraph{Bag-level margin dynamics (Lemmas~\ref{thm:global} and~\ref{lem:margin_dynamics}).}
We train max- and mean-pooling detectors on LLaMA-3.1-8B (TriviaQA, NQ,
BioASQ) with the settings of Appendix~\ref{app:set}. At epochs
$0,1,5,20,50,100$, on four random batches of $128$ training bags, we compare
each bag's first-order gain
$\eta\|\nabla_\theta z_\mathbf{B}\|^2/(1+\exp(m_\mathbf{B}))$
with the actual change of its margin after its own gradient step
($\eta=10^{-3}$, evaluation mode). Over the $144$ measurements, the Spearman
correlation between predicted and actual gains is at least $0.999$, and the
median relative error is at most $0.06\%$, so the first-order expansion is
accurate at practical step sizes.

\paragraph{Activation sparsity (Definition~\ref{ass:sparse-mil}).}
For ReLU features, $\partial z_\mathbf{B}/\partial W_j$ involves
$\mathbf{h}_{\mathbf{B},i}$ only through $\mathbf{1}\{u_{\mathbf{B},i,j}>0\}$, so an
inactive token contributes neither to the pooled value nor to the gradient.
We measure the quantities of Definition~\ref{ass:sparse-mil} on the trained
max-pooling detectors at the selected layers (five runs, all test bags). The
implementation has a first-layer bias, which we absorb by appending a constant
$1$ to each hidden state, as in the proof. Table~\ref{tab:validation_sparse}
reports the share of token--feature pairs that are inactive, the median over
active features of the active fraction $\alpha_{\mathbf{B},j}$ and of
$\kappa_{\mathbf{B},j}$, on the four QA benchmarks and on LongFact++ (the
LongFact++ detectors of Section~\ref{ex:length}, retrained with the same
settings and seeds, which reproduces their test AUC). On QA, $20$--$38\%$ of
the pairs are inactive and the median $\alpha_{\mathbf{B},j}$ is
$0.71$--$1.00$, so activation is sparse but not extreme on these short
answers. On the long LongFact++ answers (median $737$--$747$ tokens), $47$--$51\%$
of the pairs are inactive and the median $\alpha_{\mathbf{B},j}$ falls to
$0.46$--$0.54$. In both cases $\kappa_{\mathbf{B},j}$ stays close to one
($1.00$--$1.15$).

\paragraph{Gradient norm ratio (Theorem~\ref{thm:sparse}).}
At the same parameters, we compute
$\|\nabla_\theta z_\mathbf{B}^{\max}\|^2/\|\nabla_\theta z_\mathbf{B}^{\mathrm{mean}}\|^2$
in closed form for every test bag. The ratio exceeds one for every bag; on QA
its median is $2.85$--$6.27$ for hallucinated bags and $2.93$--$6.42$ over all
bags, so factual bags behave alike, as the label-free statement of the theorem
predicts, and on LongFact++ it rises to $15.9$--$17.8$ ($15.4$--$17.1$ over all bags). The bound of
Theorem~\ref{thm:sparse}, evaluated with each bag's measured
$\alpha_{\mathbf{B},j}$, $\kappa_{\mathbf{B},j}$ and $a_{\mathbf{B},j}$,
holds for every bag of either label in every run, and its median exceeds one
in every setting: $1.20$--$1.73$ on QA and $2.12$--$2.49$ on LongFact++, where
activation is sparser. The ratio exceeds the bound by a factor of at least
$1.13$, and is positively rank correlated with it over hallucinated bags
(Spearman $0.23$--$0.71$). The ratio also grows with model size on QA, in line
with the longer responses of LLaMA-3.3-70B.

\begin{table}[H]
\centering
\small
\setlength{\tabcolsep}{2.4pt}
\renewcommand{\arraystretch}{1.15}
\caption{Validation of Definition~\ref{ass:sparse-mil} and
Theorem~\ref{thm:sparse} on trained max-pooling detectors (test bags, mean of
five runs). $T$: median response length; Inactive: share of token--feature
pairs with zero activation; $\alpha$, $\kappa$: median over active features
(then over bags) of $\alpha_{\mathbf{B},j}$ and $\kappa_{\mathbf{B},j}$; Ratio:
median $\|\nabla z^{\max}\|^2/\|\nabla z^{\mathrm{mean}}\|^2$ over hallucinated
(Pos.) and all bags; Bound: median of the bound of
Theorem~\ref{thm:sparse}; Holds: share of bags satisfying it (lowest over the
five runs); $r_s$: Spearman correlation between the ratio and the bound over
hallucinated bags.}
\label{tab:validation_sparse}
\begin{tabular}{llrrrrrrrrr}
\toprule
\textbf{Model} & \textbf{Dataset} & $T$ & Inactive & $\alpha$ & $\kappa$ & Ratio (Pos.) & Ratio (all) & Bound & Holds & $r_s$ \\
\midrule
\multirow{4}{*}{LLaMA-3.1-8B} & TriviaQA & 5 & 23\% & 1.00 & 1.00 & 2.95 & 3.12 & 1.44 & 100\% & 0.37 \\
 & SQuAD & 13 & 22\% & 0.85 & 1.05 & 3.22 & 3.24 & 1.26 & 100\% & 0.58 \\
 & NQ & 14 & 20\% & 0.88 & 1.04 & 3.23 & 3.25 & 1.20 & 100\% & 0.30 \\
 & BioASQ & 11 & 20\% & 0.89 & 1.07 & 2.95 & 2.93 & 1.23 & 100\% & 0.49 \\
\midrule
\multirow{4}{*}{Mistral-Nemo-12B} & TriviaQA & 4 & 35\% & 0.86 & 1.00 & 2.85 & 3.04 & 1.50 & 100\% & 0.26 \\
 & SQuAD & 10 & 32\% & 0.75 & 1.06 & 3.71 & 3.71 & 1.50 & 100\% & 0.39 \\
 & NQ & 12 & 35\% & 0.71 & 1.06 & 4.23 & 4.16 & 1.58 & 100\% & 0.52 \\
 & BioASQ & 15 & 33\% & 0.77 & 1.07 & 4.19 & 4.22 & 1.61 & 100\% & 0.53 \\
\midrule
\multirow{4}{*}{LLaMA-3.3-70B} & TriviaQA & 23 & 38\% & 1.00 & 1.06 & 5.56 & 6.34 & 1.48 & 100\% & 0.62 \\
 & SQuAD & 25 & 36\% & 0.74 & 1.07 & 6.08 & 6.33 & 1.73 & 100\% & 0.71 \\
 & NQ & 25 & 38\% & 0.76 & 1.08 & 6.04 & 5.93 & 1.36 & 100\% & 0.59 \\
 & BioASQ & 33 & 36\% & 0.79 & 1.10 & 6.27 & 6.42 & 1.48 & 100\% & 0.66 \\
\midrule
LLaMA-3.1-8B-Inst. & LongFact++ & 737 & 51\% & 0.46 & 1.14 & 15.92 & 15.38 & 2.49 & 100\% & 0.38 \\
LLaMA-3.3-70B-Inst. & LongFact++ & 747 & 47\% & 0.54 & 1.15 & 17.81 & 17.14 & 2.12 & 100\% & 0.23 \\
\bottomrule
\end{tabular}
\end{table}

\section{Empirical Check of \texorpdfstring{Theorem~\ref{theorem:hami}}{Theorem 1}}
\label{app:assumption_check}
This appendix checks the condition of Theorem~\ref{theorem:hami} on trained HaMI detectors, using the local sensitivity; Appendix~\ref{app:integrated_checks} repeats the check with the integrated sensitivity of Appendix~\ref{app:active_non_invariant}.
\subsection{Summary of Expected Values}
\label{app:expected_values_summary}

Tables~\ref{tab:exp_llama8b_train}--\ref{tab:exp_llama70b_val} summarize semantic probability, local sensitivity, and their joint product for each class. The joint product determines $\gamma$ in Equation~\eqref{eq:gamma_definition}; the required quantity is the negative-to-positive class ratio, not the negative-class fraction. Each table is computed on the bags its own row reports, which is a subset of
the split counted in Table~\ref{tab:class_balance}; the condition is checked
against the ratio of the same subset. Table~\ref{tab:class_balance} supplies the class ratios for all three generators, so the condition can be checked directly: on every model--dataset pair among the four short-answer benchmarks the measured ratio exceeds the corresponding $1/\gamma$ threshold, the tightest case being LLaMA-3.1-8B on SQuAD ($0.52$ against $0.39$). Appendix~\ref{app:integrated_checks} repeats the analysis with integrated sensitivities.

\begin{table}[H]
\centering
\caption{Empirical expected values and sensitivity ratios for \textbf{LLaMA-3.1-8B} (\textbf{Train}).}
\label{tab:exp_llama8b_train}
\resizebox{1.0\textwidth}{!}{%
\begin{tabular}{l cccccc cc}
\toprule
& \multicolumn{2}{c}{\textbf{Semantic Prob.} $\mathbb{E}[P_{\mathrm{sem}}^{\mathbf{B}}]$} 
& \multicolumn{2}{c}{\textbf{Sensitivity} $\mathbb{E}[\bar{C}_{\mathbf{B}}]$} 
& \multicolumn{2}{c}{\textbf{Joint Product} $\mathbb{E}[P_{\mathrm{sem}}^{\mathbf{B}} \bar{C}_{\mathbf{B}}]$}
& \multicolumn{2}{c}{\textbf{Ratios}} \\
\cmidrule(lr){2-3} \cmidrule(lr){4-5} \cmidrule(lr){6-7} \cmidrule(lr){8-9}
\textbf{Dataset} & Pos ($\mathbb{E}_{pos}$) & Neg ($\mathbb{E}_{neg}$) & Pos ($\mathbb{E}_{pos}$) & Neg ($\mathbb{E}_{neg}$) & Pos ($\mathbb{E}_{pos}$) & Neg ($\mathbb{E}_{neg}$) & $\gamma$ & $1/\gamma$ \\
\midrule
TriviaQA & 0.430 & \textbf{0.914} & 13.4 & \textbf{18.8} & 6.49 & \textbf{17.5} & \textbf{2.70} & \textbf{0.37} \\
SQuAD    & 0.352 & \textbf{0.730} & 13.4 & \textbf{18.6} & 5.42 & \textbf{14.8} & \textbf{2.73} & \textbf{0.37} \\
NQ       & 0.382 & \textbf{0.772} & 13.0 & \textbf{16.8} & 5.52 & \textbf{13.7} & \textbf{2.48} & \textbf{0.40} \\
BioASQ   & 0.392 & \textbf{0.756} & 12.0 & \textbf{16.0} & 5.06 & \textbf{12.8} & \textbf{2.53} & \textbf{0.40} \\
\bottomrule
\end{tabular}
}
\end{table}

\begin{table}[H]
\centering
\caption{Empirical expected values and sensitivity ratios for \textbf{LLaMA-3.1-8B} (\textbf{Test}).}
\label{tab:exp_llama8b_val}
\resizebox{1.0\textwidth}{!}{%
\begin{tabular}{l cccccc cc}
\toprule
& \multicolumn{2}{c}{\textbf{Semantic Prob.} $\mathbb{E}[P_{\mathrm{sem}}^{\mathbf{B}}]$} 
& \multicolumn{2}{c}{\textbf{Sensitivity} $\mathbb{E}[\bar{C}_{\mathbf{B}}]$} 
& \multicolumn{2}{c}{\textbf{Joint Product} $\mathbb{E}[P_{\mathrm{sem}}^{\mathbf{B}} \bar{C}_{\mathbf{B}}]$}
& \multicolumn{2}{c}{\textbf{Ratios}} \\
\cmidrule(lr){2-3} \cmidrule(lr){4-5} \cmidrule(lr){6-7} \cmidrule(lr){8-9}
\textbf{Dataset} & Pos ($\mathbb{E}_{pos}$) & Neg ($\mathbb{E}_{neg}$) & Pos ($\mathbb{E}_{pos}$) & Neg ($\mathbb{E}_{neg}$) & Pos ($\mathbb{E}_{pos}$) & Neg ($\mathbb{E}_{neg}$) & $\gamma$ & $1/\gamma$ \\
\midrule
TriviaQA & 0.451 & \textbf{0.926} & 13.8 & \textbf{18.9} & 7.09 & \textbf{17.8} & \textbf{2.51} & \textbf{0.40} \\
SQuAD    & 0.360 & \textbf{0.733} & 13.5 & \textbf{18.5} & 5.67 & \textbf{14.6} & \textbf{2.57} & \textbf{0.39} \\
NQ       & 0.383 & \textbf{0.732} & 12.9 & \textbf{16.4} & 5.53 & \textbf{12.8} & \textbf{2.31} & \textbf{0.43} \\
BioASQ   & 0.380 & \textbf{0.764} & 12.1 & \textbf{16.1} & 4.98 & \textbf{13.0} & \textbf{2.61} & \textbf{0.38} \\
\bottomrule
\end{tabular}
}
\end{table}

\begin{table}[H]
\centering
\caption{Empirical expected values and sensitivity ratios for \textbf{Mistral-12B} (\textbf{Train}).}
\label{tab:exp_mistral12b_train}
\resizebox{1.0\textwidth}{!}{%
\begin{tabular}{l cccccc cc}
\toprule
& \multicolumn{2}{c}{\textbf{Semantic Prob.} $\mathbb{E}[P_{\mathrm{sem}}^{\mathbf{B}}]$} 
& \multicolumn{2}{c}{\textbf{Sensitivity} $\mathbb{E}[\bar{C}_{\mathbf{B}}]$} 
& \multicolumn{2}{c}{\textbf{Joint Product} $\mathbb{E}[P_{\mathrm{sem}}^{\mathbf{B}} \bar{C}_{\mathbf{B}}]$}
& \multicolumn{2}{c}{\textbf{Ratios}} \\
\cmidrule(lr){2-3} \cmidrule(lr){4-5} \cmidrule(lr){6-7} \cmidrule(lr){8-9}
\textbf{Dataset} & Pos ($\mathbb{E}_{pos}$) & Neg ($\mathbb{E}_{neg}$) & Pos ($\mathbb{E}_{pos}$) & Neg ($\mathbb{E}_{neg}$) & Pos ($\mathbb{E}_{pos}$) & Neg ($\mathbb{E}_{neg}$) & $\gamma$ & $1/\gamma$ \\
\midrule
TriviaQA & 0.581 & \textbf{0.960} & 15.1 & \textbf{20.8} & 9.40 & \textbf{20.2} & \textbf{2.15} & \textbf{0.47} \\
SQuAD    & 0.457 & \textbf{0.835} & 9.60 & \textbf{15.1} & 5.09 & \textbf{13.3} & \textbf{2.61} & \textbf{0.38} \\
NQ       & 0.535 & \textbf{0.886} & 10.5 & \textbf{15.0} & 6.43 & \textbf{13.8} & \textbf{2.15} & \textbf{0.47} \\
BioASQ   & 0.480 & \textbf{0.881} & 8.54 & \textbf{11.7} & 4.53 & \textbf{10.6} & \textbf{2.34} & \textbf{0.43} \\
\bottomrule
\end{tabular}
}
\end{table}

\begin{table}[H]
\centering
\caption{Empirical expected values and sensitivity ratios for \textbf{Mistral-12B} (\textbf{Test}).}
\label{tab:exp_mistral12b_val}
\resizebox{1.0\textwidth}{!}{%
\begin{tabular}{l cccccc cc}
\toprule
& \multicolumn{2}{c}{\textbf{Semantic Prob.} $\mathbb{E}[P_{\mathrm{sem}}^{\mathbf{B}}]$} 
& \multicolumn{2}{c}{\textbf{Sensitivity} $\mathbb{E}[\bar{C}_{\mathbf{B}}]$} 
& \multicolumn{2}{c}{\textbf{Joint Product} $\mathbb{E}[P_{\mathrm{sem}}^{\mathbf{B}} \bar{C}_{\mathbf{B}}]$}
& \multicolumn{2}{c}{\textbf{Ratios}} \\
\cmidrule(lr){2-3} \cmidrule(lr){4-5} \cmidrule(lr){6-7} \cmidrule(lr){8-9}
\textbf{Dataset} & Pos ($\mathbb{E}_{pos}$) & Neg ($\mathbb{E}_{neg}$) & Pos ($\mathbb{E}_{pos}$) & Neg ($\mathbb{E}_{neg}$) & Pos ($\mathbb{E}_{pos}$) & Neg ($\mathbb{E}_{neg}$) & $\gamma$ & $1/\gamma$ \\
\midrule
TriviaQA & 0.590 & \textbf{0.966} & 15.6 & \textbf{20.8} & 10.0 & \textbf{20.3} & \textbf{2.03} & \textbf{0.49} \\
SQuAD    & 0.502 & \textbf{0.849} & 10.4 & \textbf{15.0} & 6.04 & \textbf{13.3} & \textbf{2.20} & \textbf{0.45} \\
NQ       & 0.525 & \textbf{0.881} & 10.6 & \textbf{14.8} & 6.50 & \textbf{13.5} & \textbf{2.08} & \textbf{0.48} \\
BioASQ   & 0.491 & \textbf{0.898} & 8.48 & \textbf{11.7} & 4.65 & \textbf{10.8} & \textbf{2.32} & \textbf{0.43} \\
\bottomrule
\end{tabular}
}
\end{table}

\begin{table}[H]
\centering
\caption{Empirical expected values and sensitivity ratios for \textbf{LLaMA-3.3-70B} (\textbf{Train}).}
\label{tab:exp_llama70b_train}
\resizebox{1.0\textwidth}{!}{%
\begin{tabular}{l cccccc cc}
\toprule
& \multicolumn{2}{c}{\textbf{Semantic Prob.} $\mathbb{E}[P_{\mathrm{sem}}^{\mathbf{B}}]$} 
& \multicolumn{2}{c}{\textbf{Sensitivity} $\mathbb{E}[\bar{C}_{\mathbf{B}}]$} 
& \multicolumn{2}{c}{\textbf{Joint Product} $\mathbb{E}[P_{\mathrm{sem}}^{\mathbf{B}} \bar{C}_{\mathbf{B}}]$}
& \multicolumn{2}{c}{\textbf{Ratios}} \\
\cmidrule(lr){2-3} \cmidrule(lr){4-5} \cmidrule(lr){6-7} \cmidrule(lr){8-9}
\textbf{Dataset} & Pos ($\mathbb{E}_{pos}$) & Neg ($\mathbb{E}_{neg}$) & Pos ($\mathbb{E}_{pos}$) & Neg ($\mathbb{E}_{neg}$) & Pos ($\mathbb{E}_{pos}$) & Neg ($\mathbb{E}_{neg}$) & $\gamma$ & $1/\gamma$ \\
\midrule
TriviaQA & 0.484 & \textbf{0.944} & 10.0 & \textbf{13.8} & 5.43 & \textbf{13.2} & \textbf{2.43} & \textbf{0.41} \\
SQuAD    & 0.439 & \textbf{0.842} & 8.66 & \textbf{10.8} & 4.14 & \textbf{9.40} & \textbf{2.27} & \textbf{0.44} \\
NQ       & 0.538 & \textbf{0.888} & 8.42 & \textbf{11.1} & 4.97 & \textbf{10.1} & \textbf{2.03} & \textbf{0.49} \\
BioASQ   & 0.459 & \textbf{0.866} & 9.12 & \textbf{11.9} & 4.49 & \textbf{10.6} & \textbf{2.36} & \textbf{0.42} \\
\bottomrule
\end{tabular}
}
\end{table}

\begin{table}[H]
\centering
\caption{Empirical expected values and sensitivity ratios for \textbf{LLaMA-3.3-70B} (\textbf{Test}).}
\label{tab:exp_llama70b_val}
\resizebox{1.0\textwidth}{!}{%
\begin{tabular}{l cccccc cc}
\toprule
& \multicolumn{2}{c}{\textbf{Semantic Prob.} $\mathbb{E}[P_{\mathrm{sem}}^{\mathbf{B}}]$} 
& \multicolumn{2}{c}{\textbf{Sensitivity} $\mathbb{E}[\bar{C}_{\mathbf{B}}]$} 
& \multicolumn{2}{c}{\textbf{Joint Product} $\mathbb{E}[P_{\mathrm{sem}}^{\mathbf{B}} \bar{C}_{\mathbf{B}}]$}
& \multicolumn{2}{c}{\textbf{Ratios}} \\
\cmidrule(lr){2-3} \cmidrule(lr){4-5} \cmidrule(lr){6-7} \cmidrule(lr){8-9}
\textbf{Dataset} & Pos ($\mathbb{E}_{pos}$) & Neg ($\mathbb{E}_{neg}$) & Pos ($\mathbb{E}_{pos}$) & Neg ($\mathbb{E}_{neg}$) & Pos ($\mathbb{E}_{pos}$) & Neg ($\mathbb{E}_{neg}$) & $\gamma$ & $1/\gamma$ \\
\midrule
TriviaQA & 0.554 & \textbf{0.940} & 10.6 & \textbf{13.7} & 6.44 & \textbf{13.1} & \textbf{2.03} & \textbf{0.49} \\
SQuAD    & 0.471 & \textbf{0.843} & 8.79 & \textbf{10.8} & 4.44 & \textbf{9.34} & \textbf{2.10} & \textbf{0.48} \\
NQ       & 0.525 & \textbf{0.894} & 8.27 & \textbf{11.0} & 4.81 & \textbf{10.0} & \textbf{2.08} & \textbf{0.48} \\
BioASQ   & 0.490 & \textbf{0.866} & 9.16 & \textbf{11.8} & 4.85 & \textbf{10.5} & \textbf{2.16} & \textbf{0.46} \\
\bottomrule
\end{tabular}
}
\end{table}

\subsection{Distribution of Semantic Probability}
\label{app:sem}
Figures~\ref{fig:dist_psem_train} (train) and \ref{fig:dist_psem_val} (test) show the empirical distribution of semantic probabilities $P_{\mathrm{sem}}^{\mathbf{B}}$. Across all datasets and models, negative bags exhibit a higher density in high-probability regions than positive bags, confirming the relationship $\mathbb{E}_{neg}[P_{\mathrm{sem}}^{\mathbf{B}}] > \mathbb{E}_{pos}[P_{\mathrm{sem}}^{\mathbf{B}}]$. The expected values for each condition are provided in Appendix \ref{app:expected_values_summary}.
The AUC performance is summarized in Table~\ref{tab:se_sp_auc}. $P_{\mathrm{sem}}^{\mathbf{B}}$ consistently achieves high discriminative power. 

\begin{figure}[H]
    \centering
    \begin{subfigure}{0.32\textwidth}
        \includegraphics[width=\linewidth]{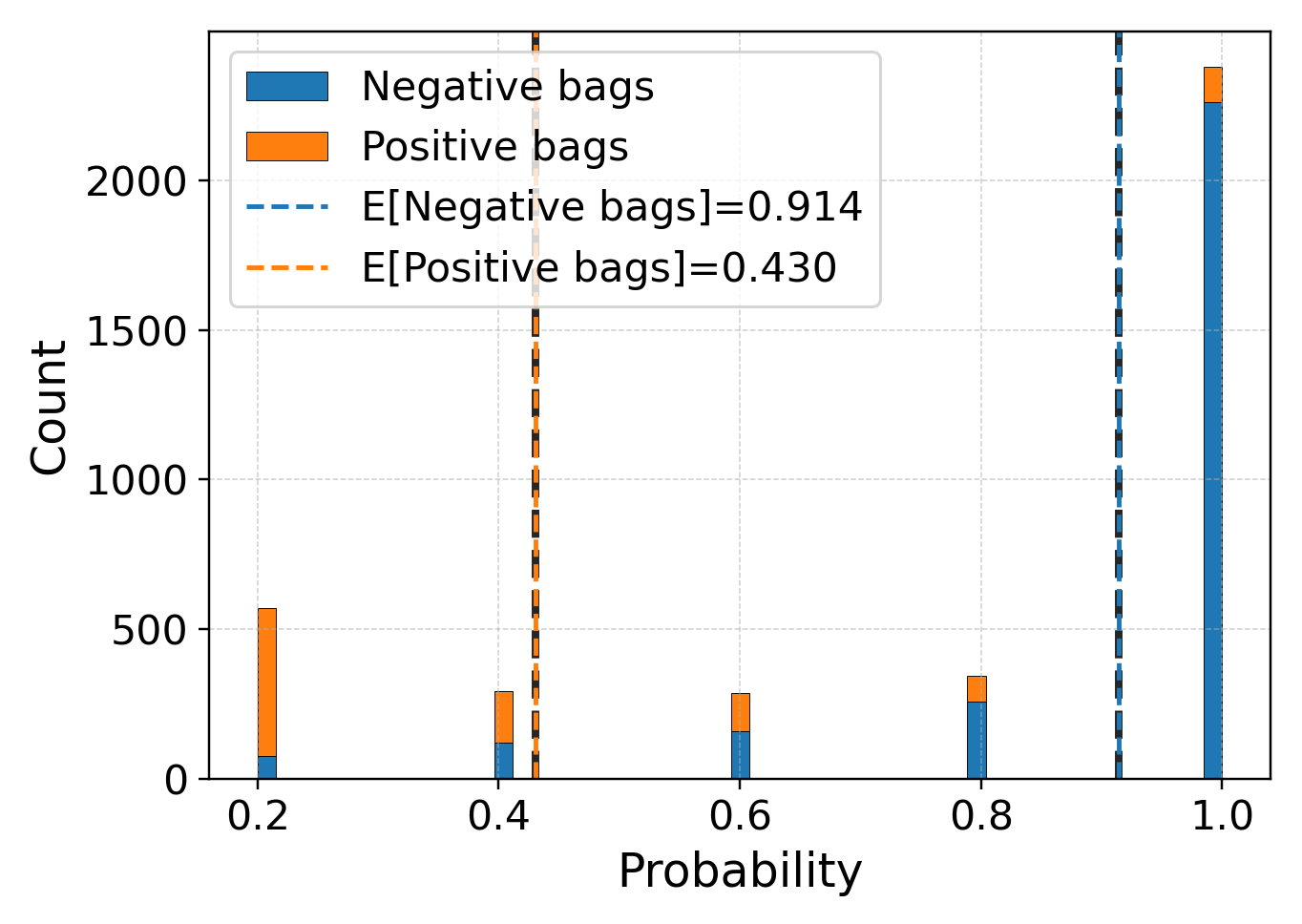}
        \caption{TriviaQA\ (Llama-3.1-8B)}
    \end{subfigure}\hfill
    \begin{subfigure}{0.32\textwidth}
        \includegraphics[width=\linewidth]{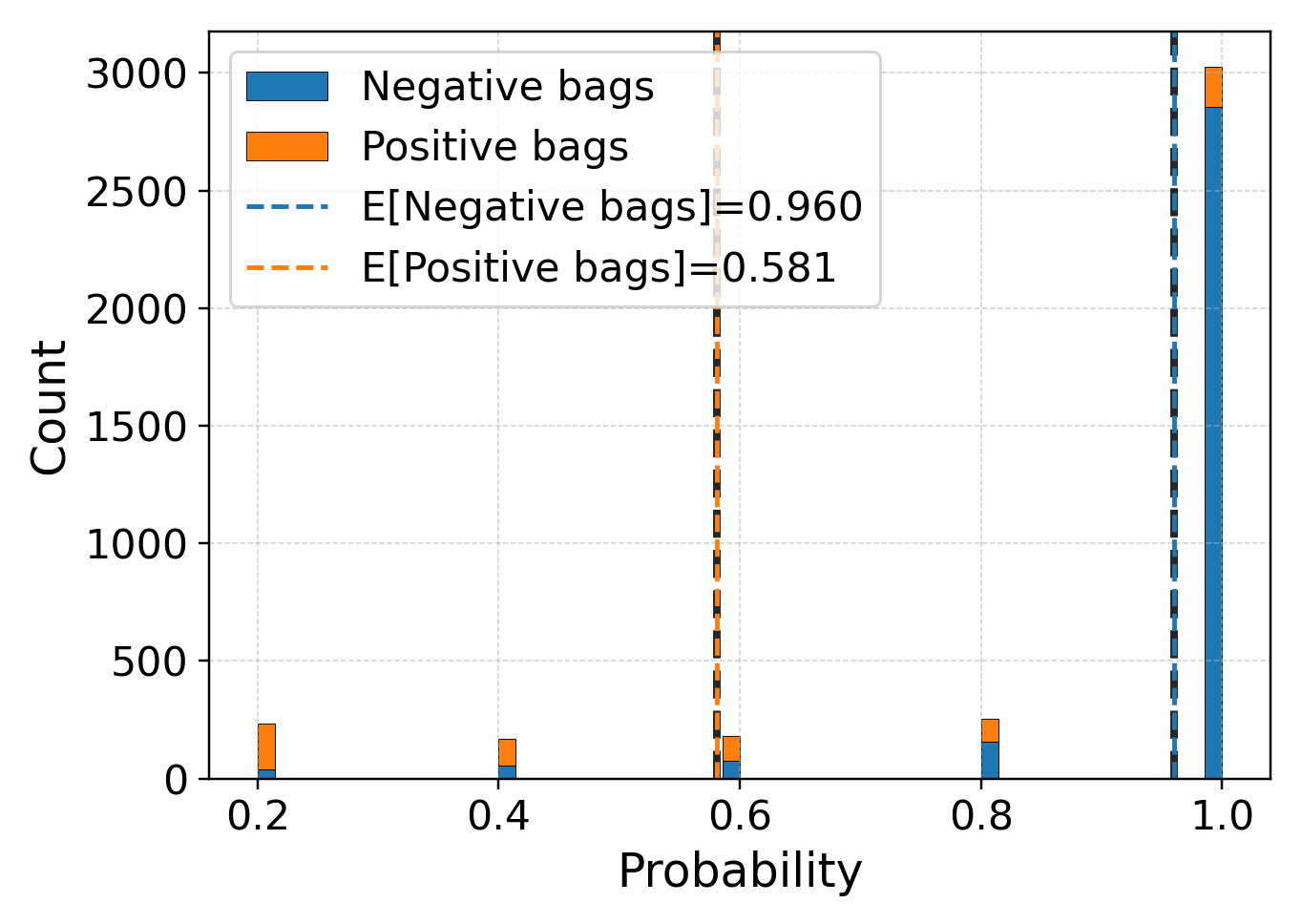}
        \caption{TriviaQA\ (Mistral-12B)}
    \end{subfigure}\hfill
    \begin{subfigure}{0.32\textwidth}
        \includegraphics[width=\linewidth]{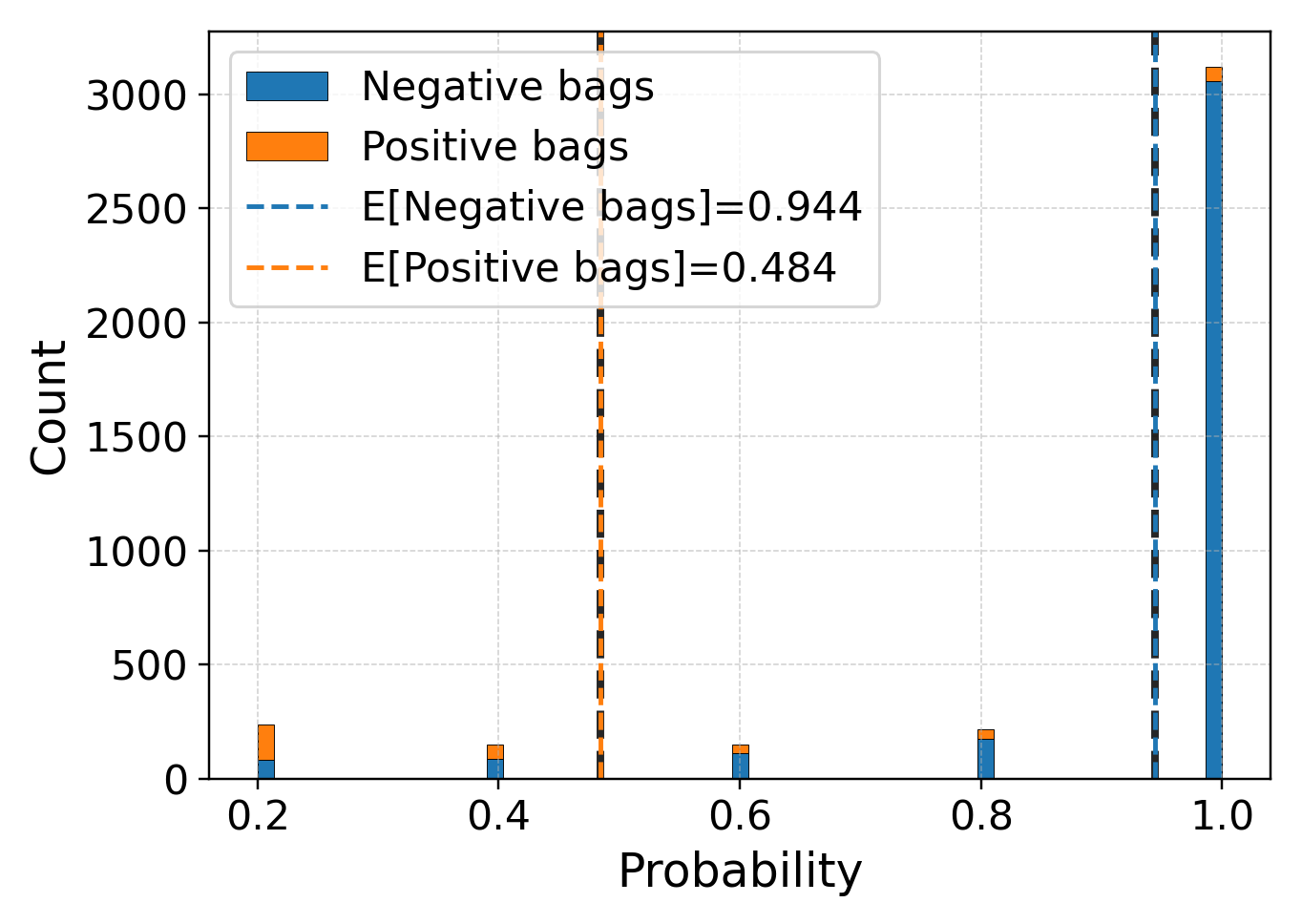}
        \caption{TriviaQA\ (Llama-3.3-70B)}
    \end{subfigure}\hfill
    \vspace{0.3cm}
    \begin{subfigure}{0.32\textwidth}
        \includegraphics[width=\linewidth]{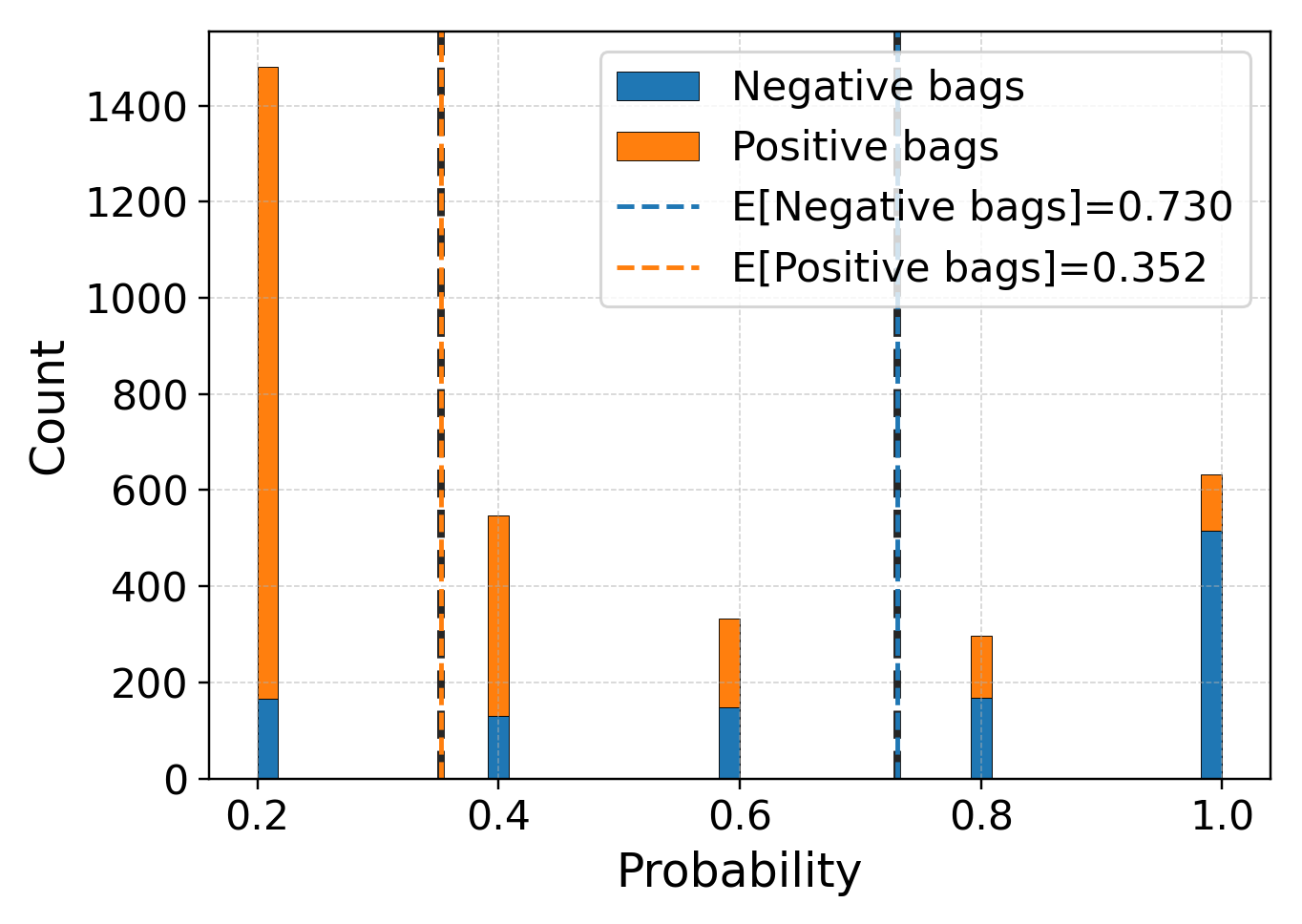}
        \caption{SQuAD\ (Llama-3.1-8B)}
    \end{subfigure}\hfill
    \begin{subfigure}{0.32\textwidth}
        \includegraphics[width=\linewidth]{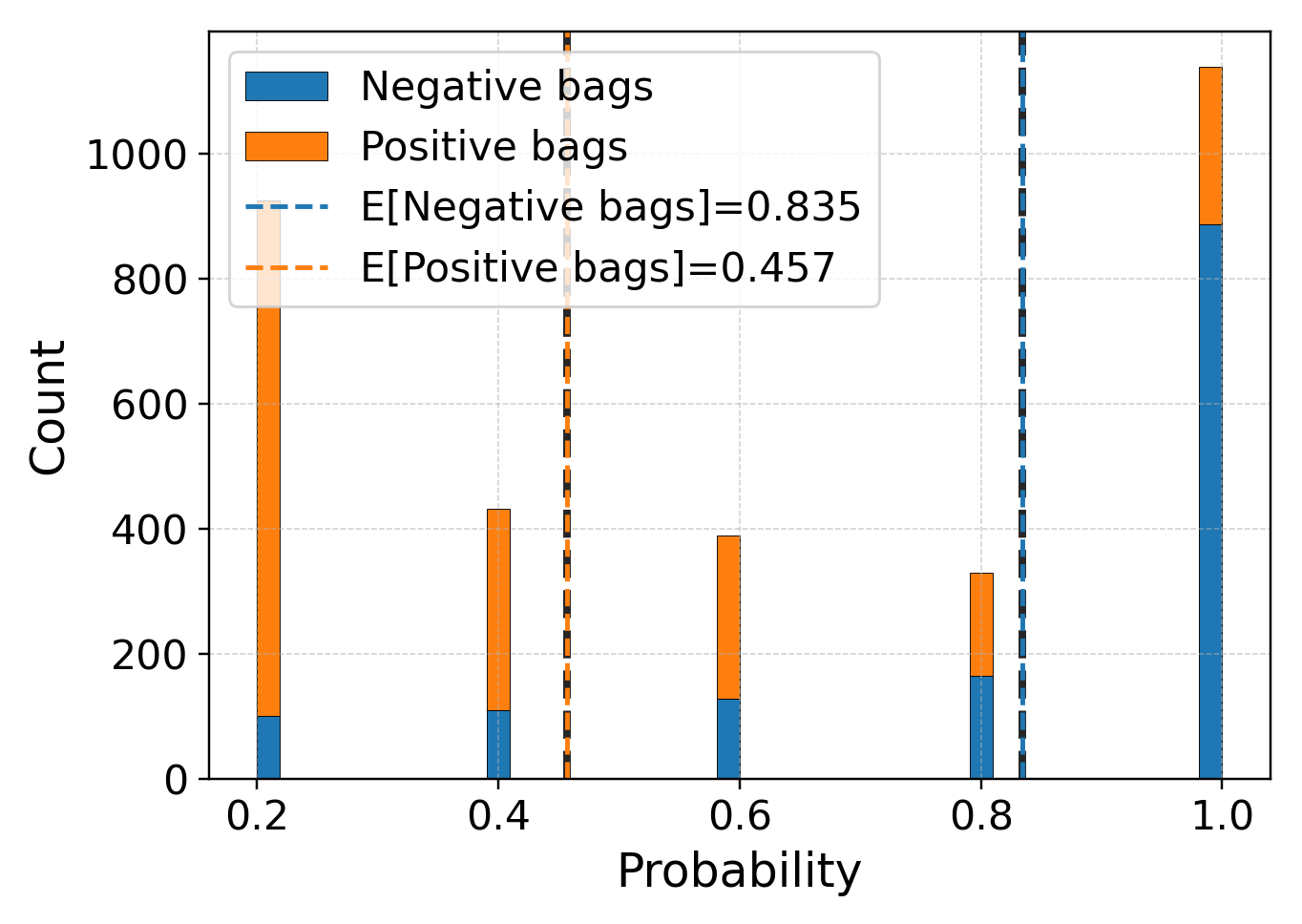}
        \caption{SQuAD\ (Mistral-12B)}
    \end{subfigure}\hfill
    \begin{subfigure}{0.32\textwidth}
        \includegraphics[width=\linewidth]{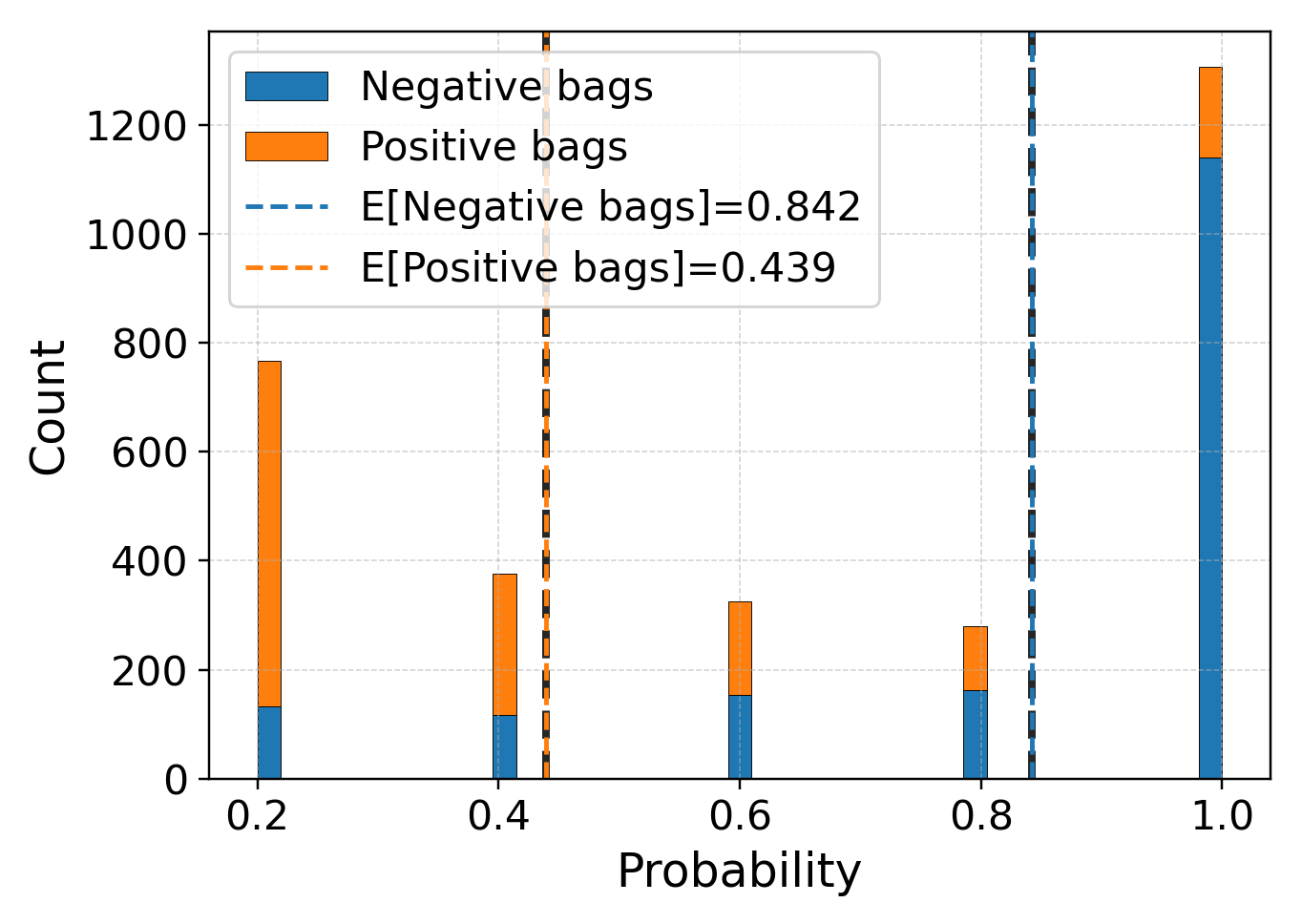}
        \caption{SQuAD\ (Llama-3.3-70B)}
    \end{subfigure}\hfill
    \vspace{0.3cm}
    \begin{subfigure}{0.32\textwidth}
        \includegraphics[width=\linewidth]{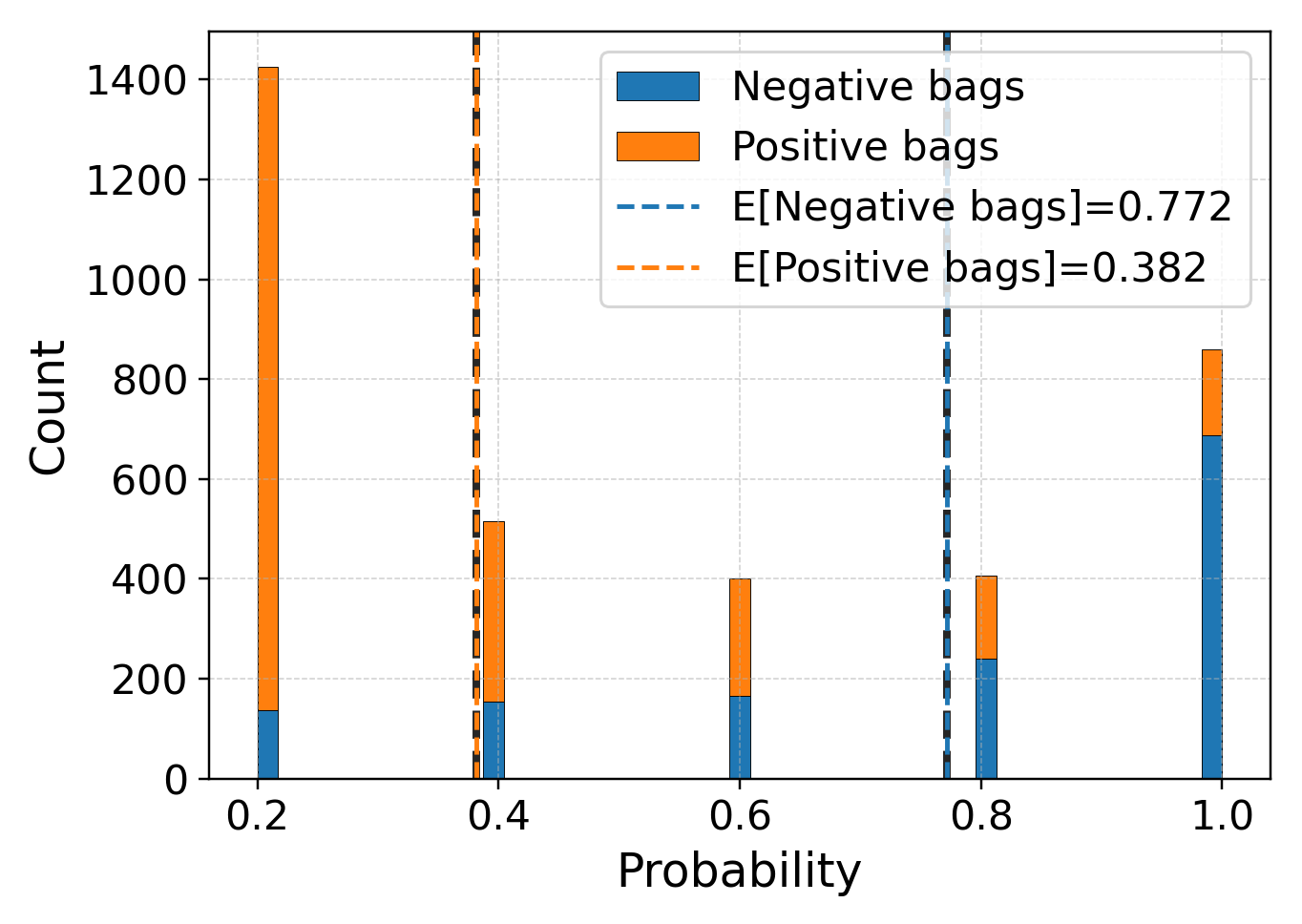}
        \caption{NQ\ (Llama-3.1-8B)}
    \end{subfigure}\hfill
    \begin{subfigure}{0.32\textwidth}
        \includegraphics[width=\linewidth]{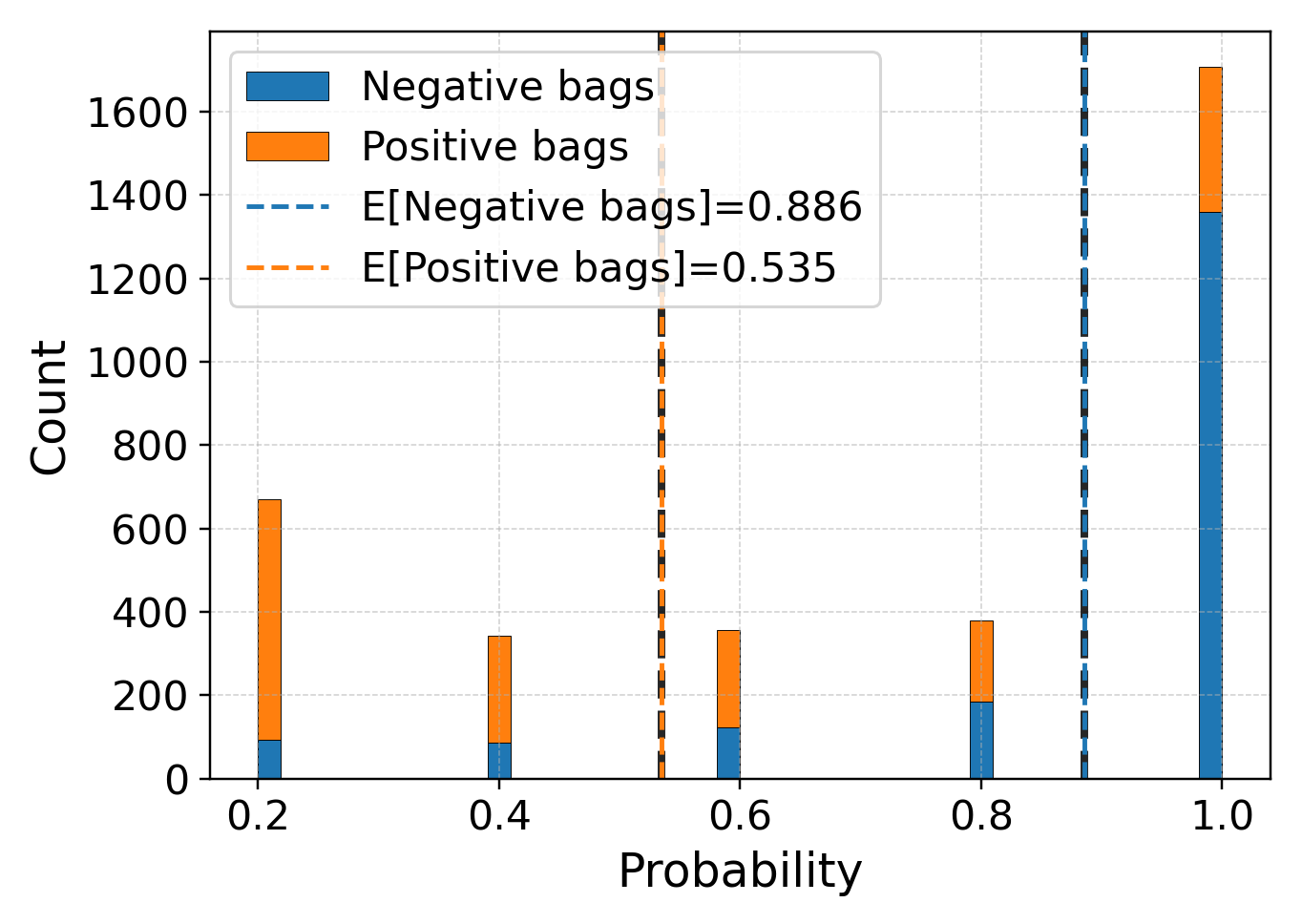}
        \caption{NQ\ (Mistral-12B)}
    \end{subfigure}\hfill
    \begin{subfigure}{0.32\textwidth}
        \includegraphics[width=\linewidth]{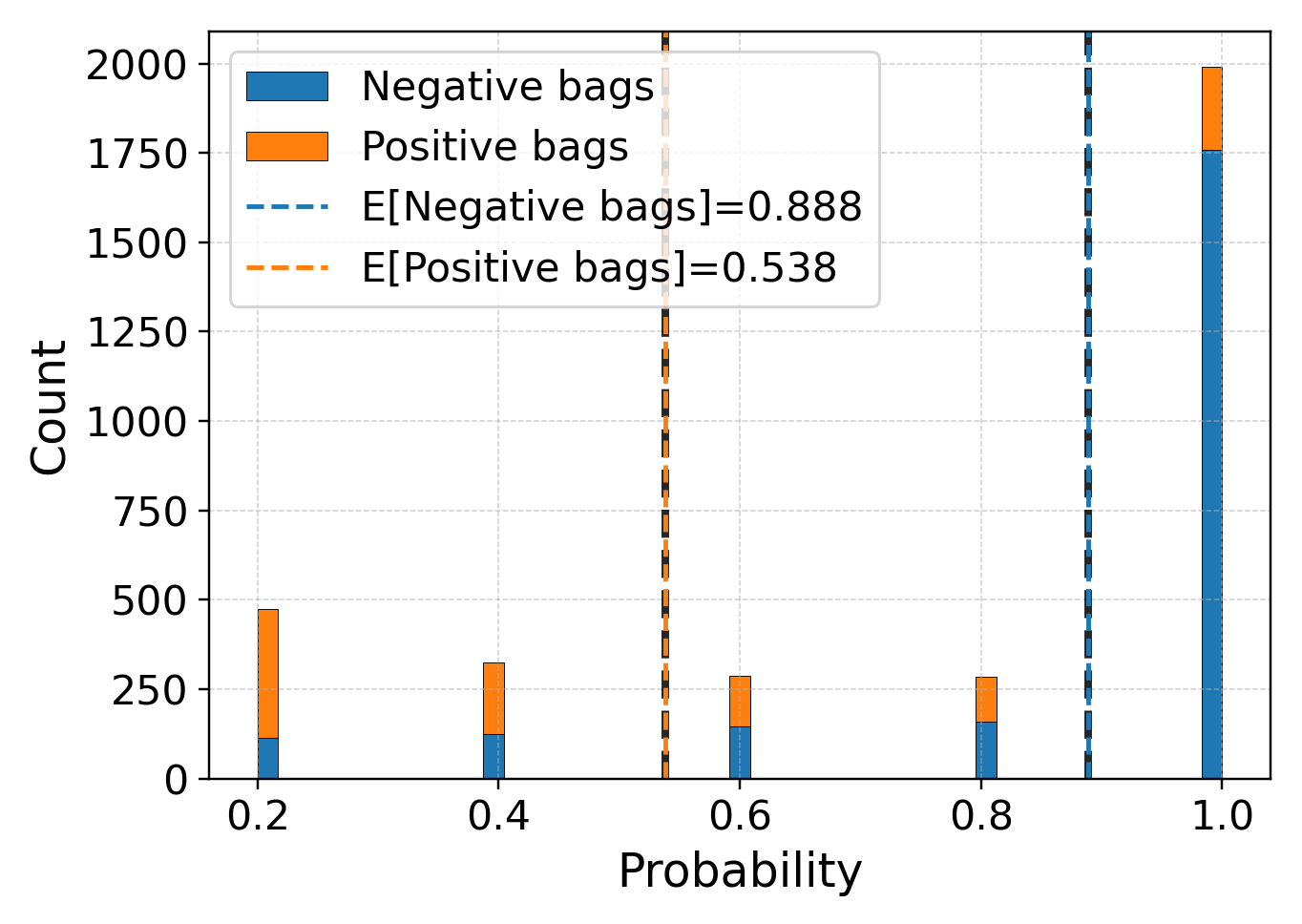}
        \caption{NQ\ (Llama-3.3-70B)}
    \end{subfigure}\hfill
    \vspace{0.3cm}
    \begin{subfigure}{0.32\textwidth}
        \includegraphics[width=\linewidth]{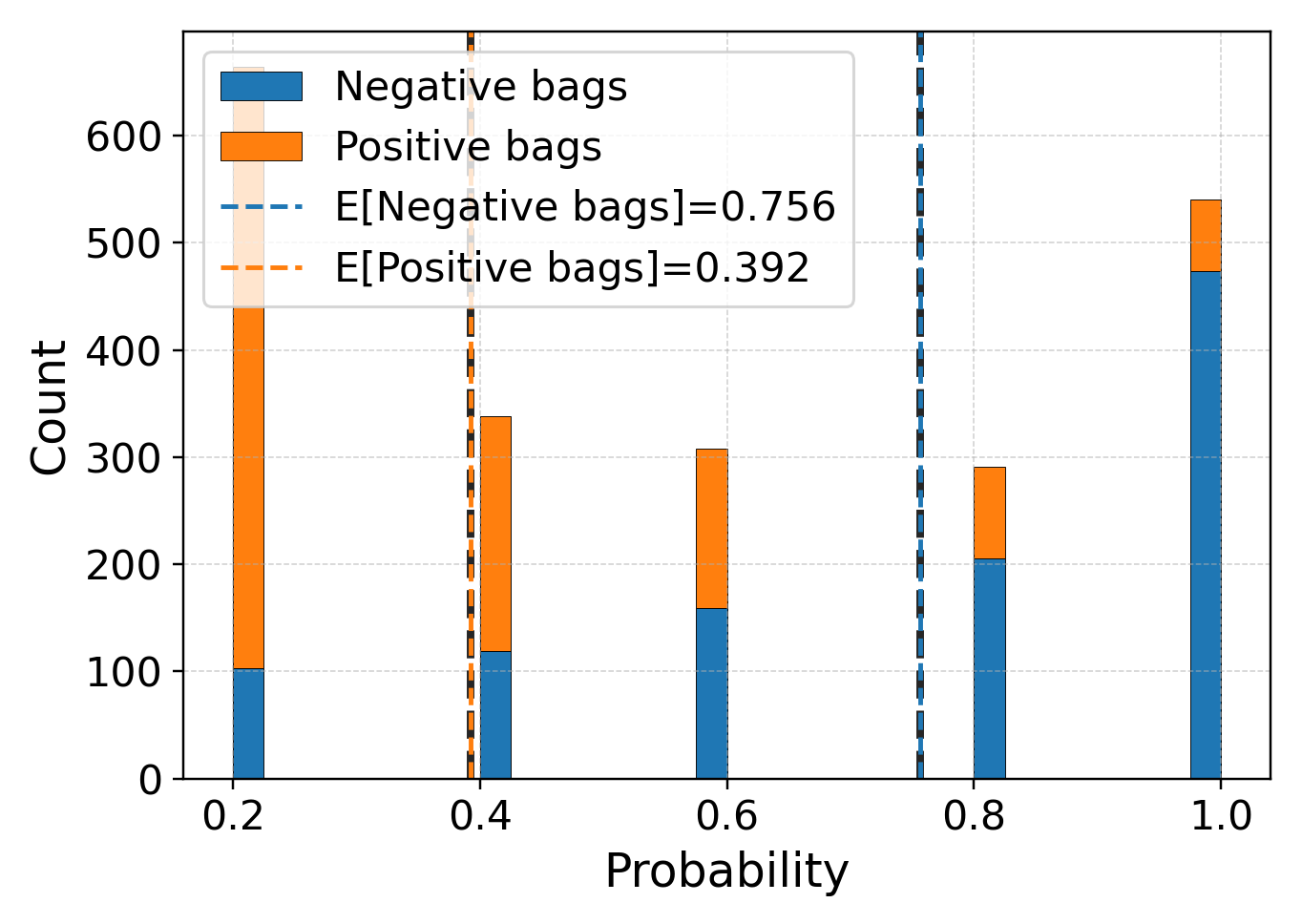}
        \caption{BioASQ\ (Llama-3.1-8B)}
    \end{subfigure}\hfill
    \begin{subfigure}{0.32\textwidth}
        \includegraphics[width=\linewidth]{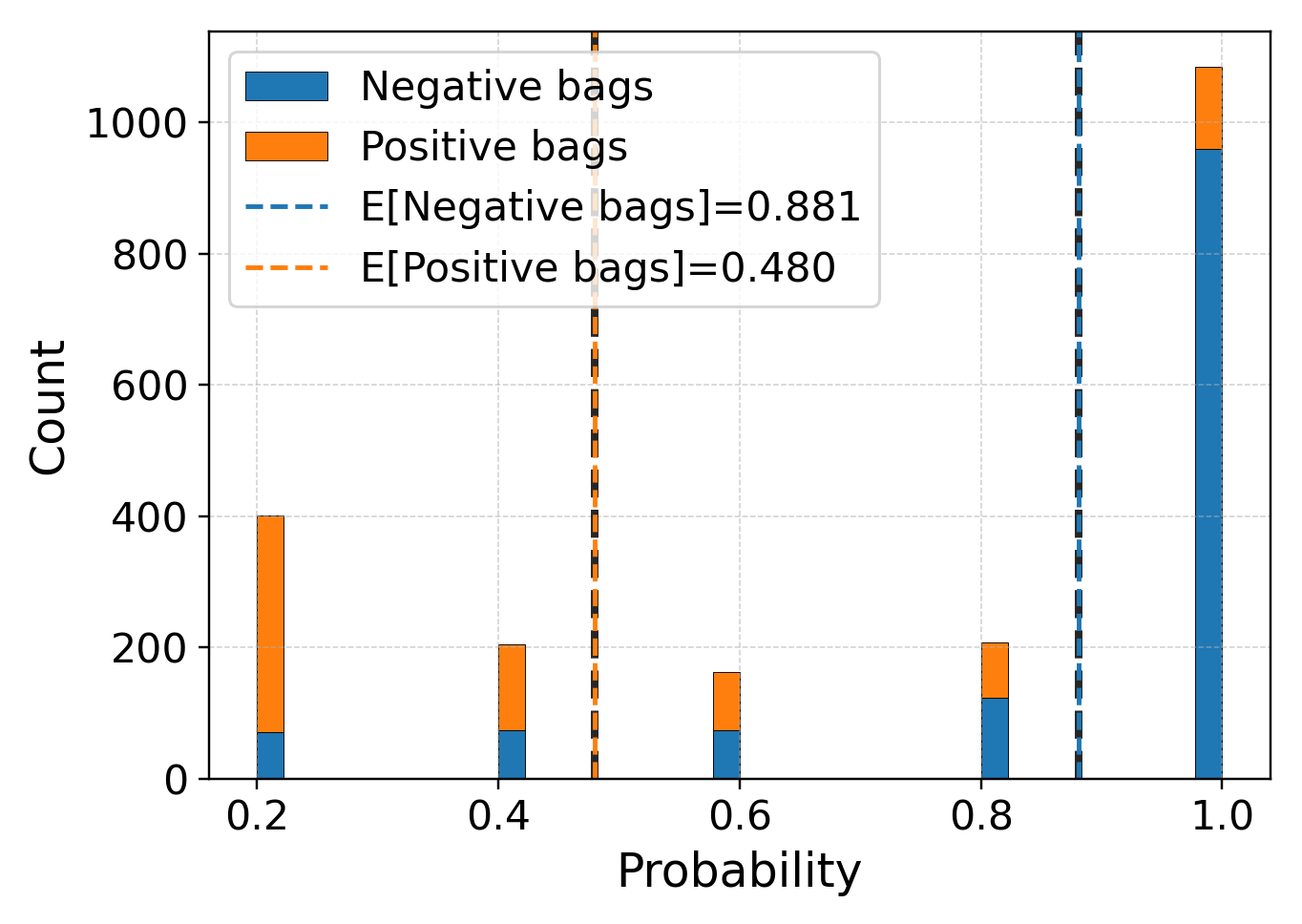}
        \caption{BioASQ\ (Mistral-12B)}
    \end{subfigure}\hfill
    \begin{subfigure}{0.32\textwidth}
        \includegraphics[width=\linewidth]{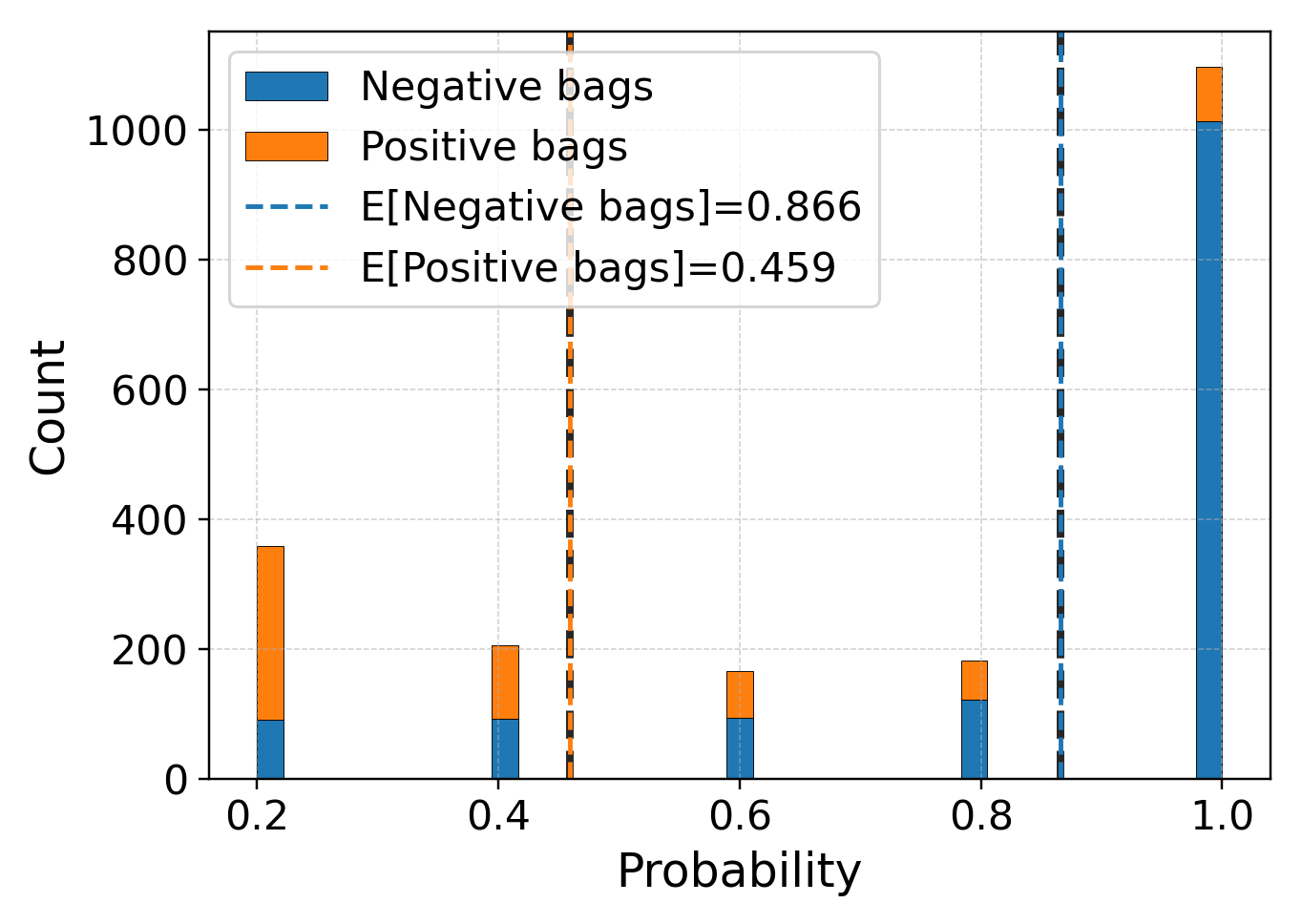}
        \caption{BioASQ\ (Llama-3.3-70B)}
    \end{subfigure}\hfill
    \vspace{0.3cm}
    \caption{Empirical distribution of semantic probabilities $P_{\mathrm{sem}}^{\mathbf{B}}$ in \textbf{train} data.}
    \label{fig:dist_psem_train}
\end{figure}

\begin{figure}[H]
    \centering
    \begin{subfigure}{0.32\textwidth}
        \includegraphics[width=\linewidth]{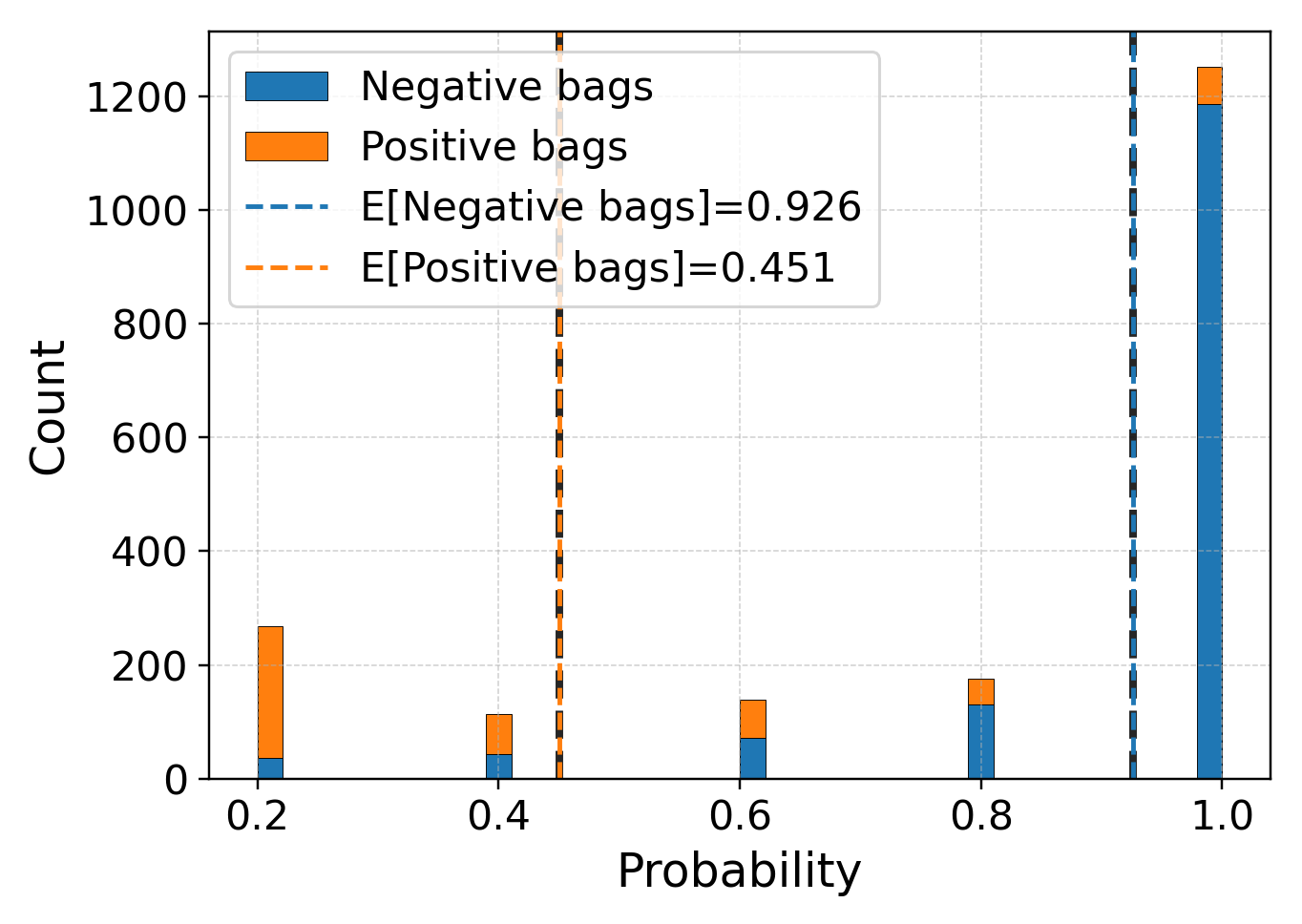}
        \caption{TriviaQA\ (Llama-3.1-8B)}
    \end{subfigure}\hfill
    \begin{subfigure}{0.32\textwidth}
        \includegraphics[width=\linewidth]{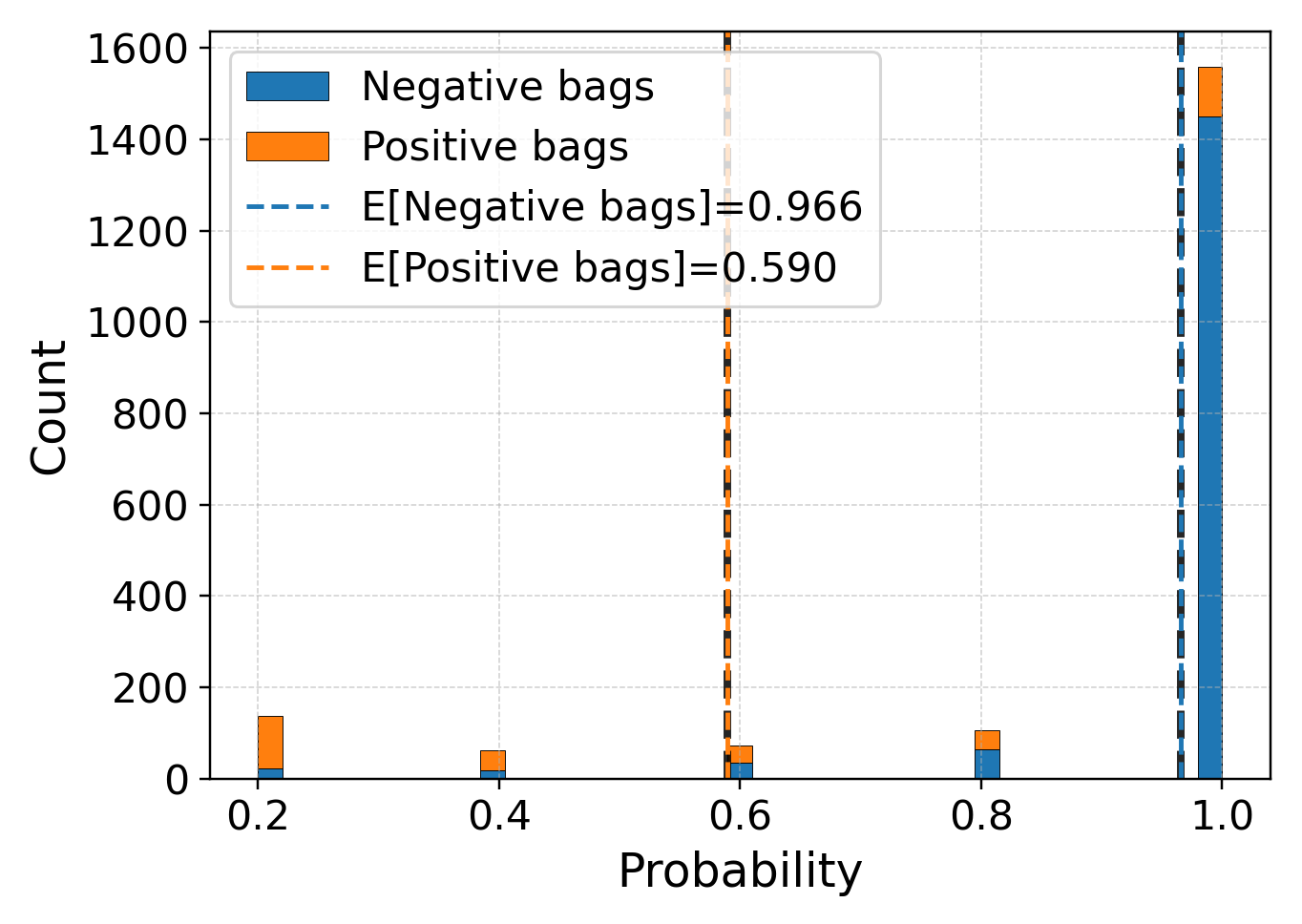}
        \caption{TriviaQA\ (Mistral-12B)}
    \end{subfigure}\hfill
    \begin{subfigure}{0.32\textwidth}
        \includegraphics[width=\linewidth]{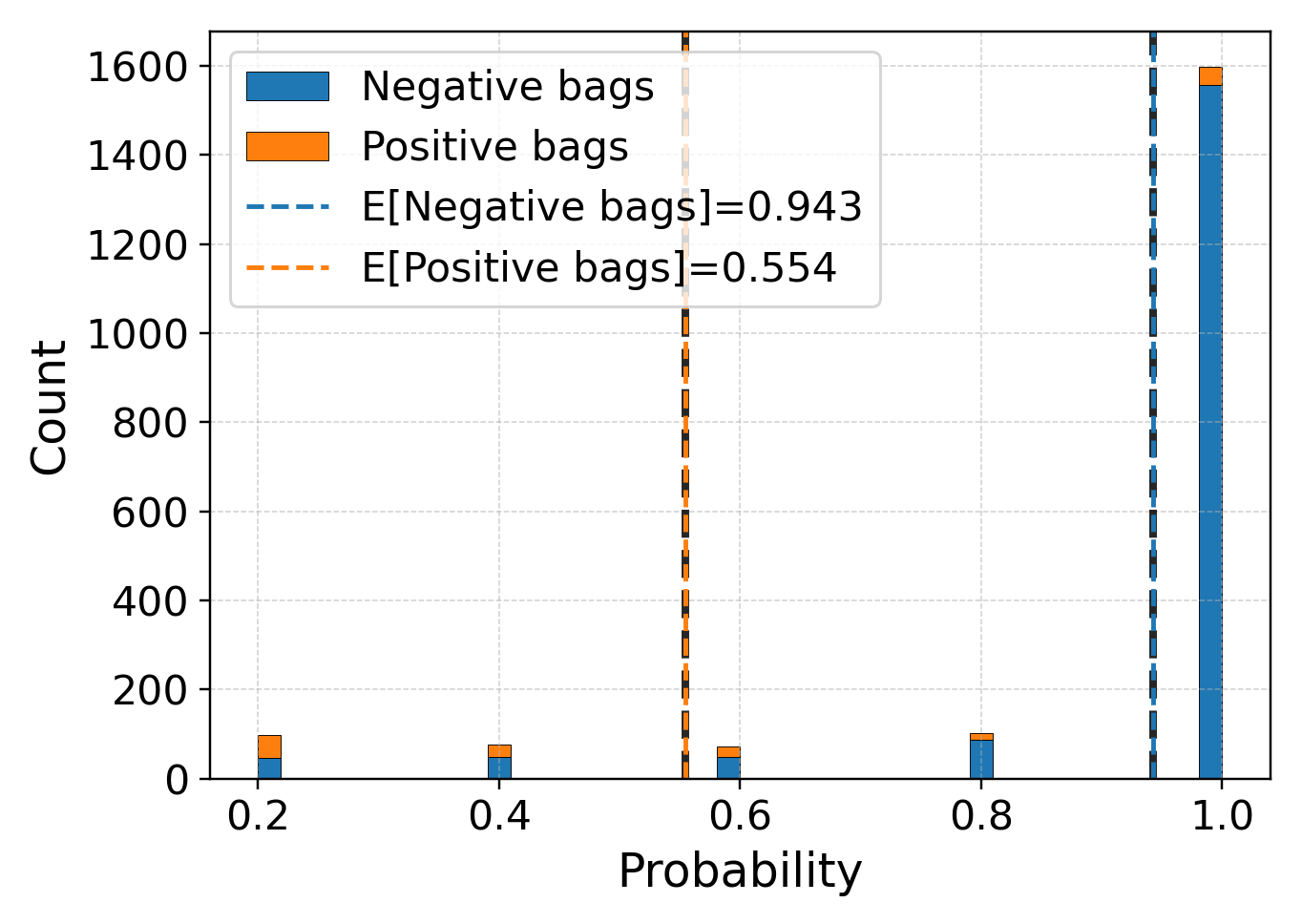}
        \caption{TriviaQA\ (Llama-3.3-70B)}
    \end{subfigure}\hfill
    \vspace{0.3cm}
    \begin{subfigure}{0.32\textwidth}
        \includegraphics[width=\linewidth]{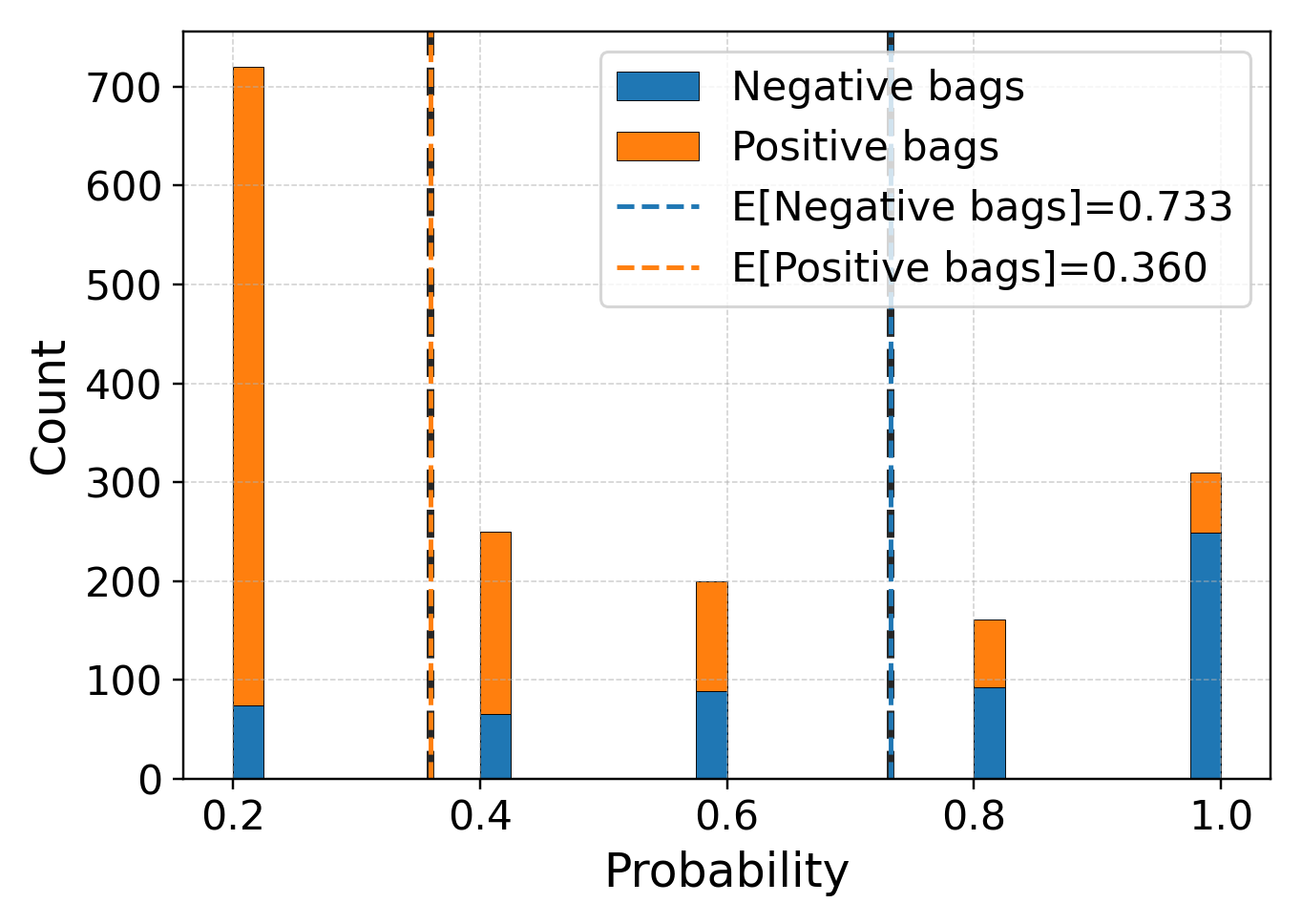}
        \caption{SQuAD\ (Llama-3.1-8B)}
    \end{subfigure}\hfill
    \begin{subfigure}{0.32\textwidth}
        \includegraphics[width=\linewidth]{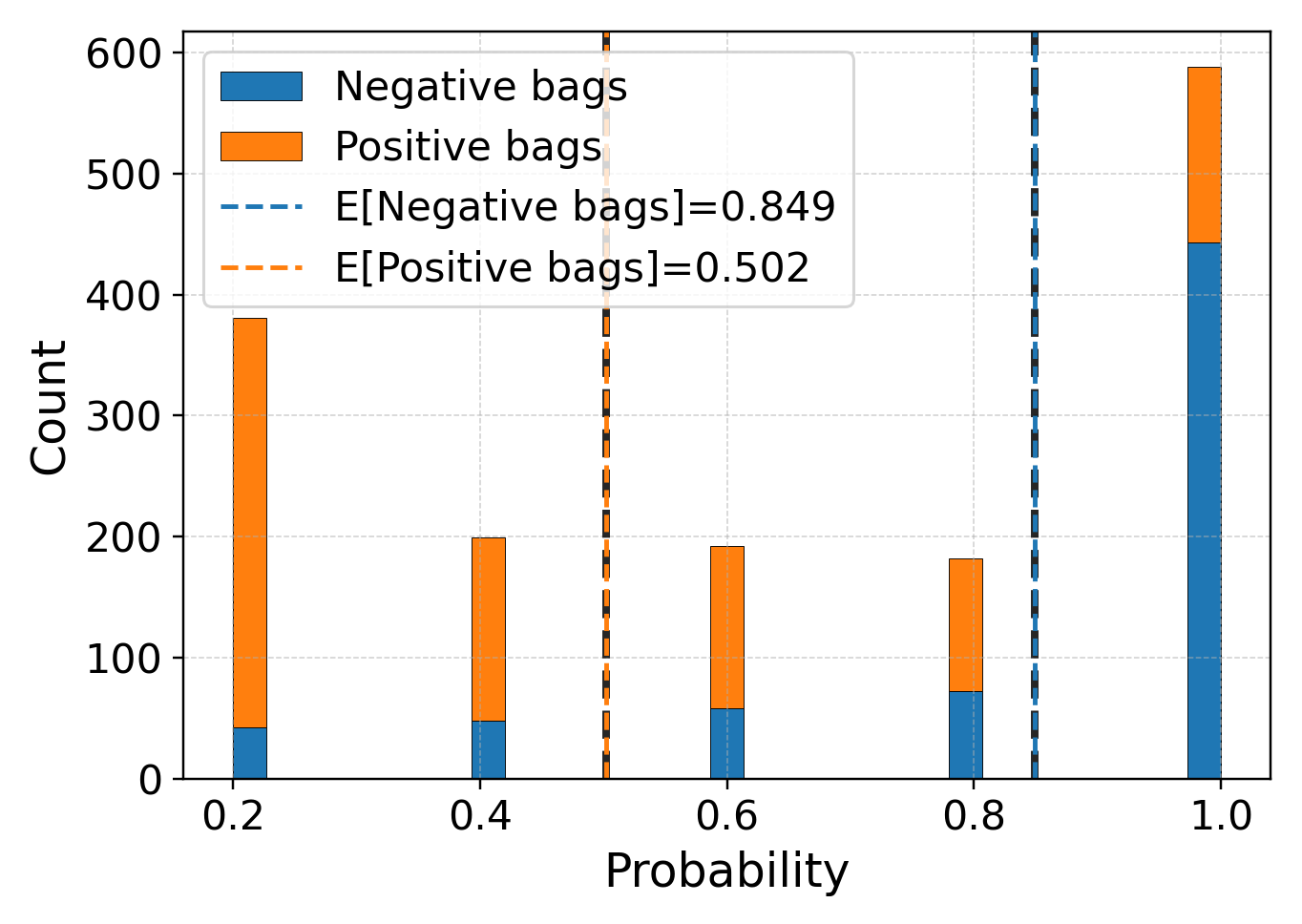}
        \caption{SQuAD\ (Mistral-12B)}
    \end{subfigure}\hfill
    \begin{subfigure}{0.32\textwidth}
        \includegraphics[width=\linewidth]{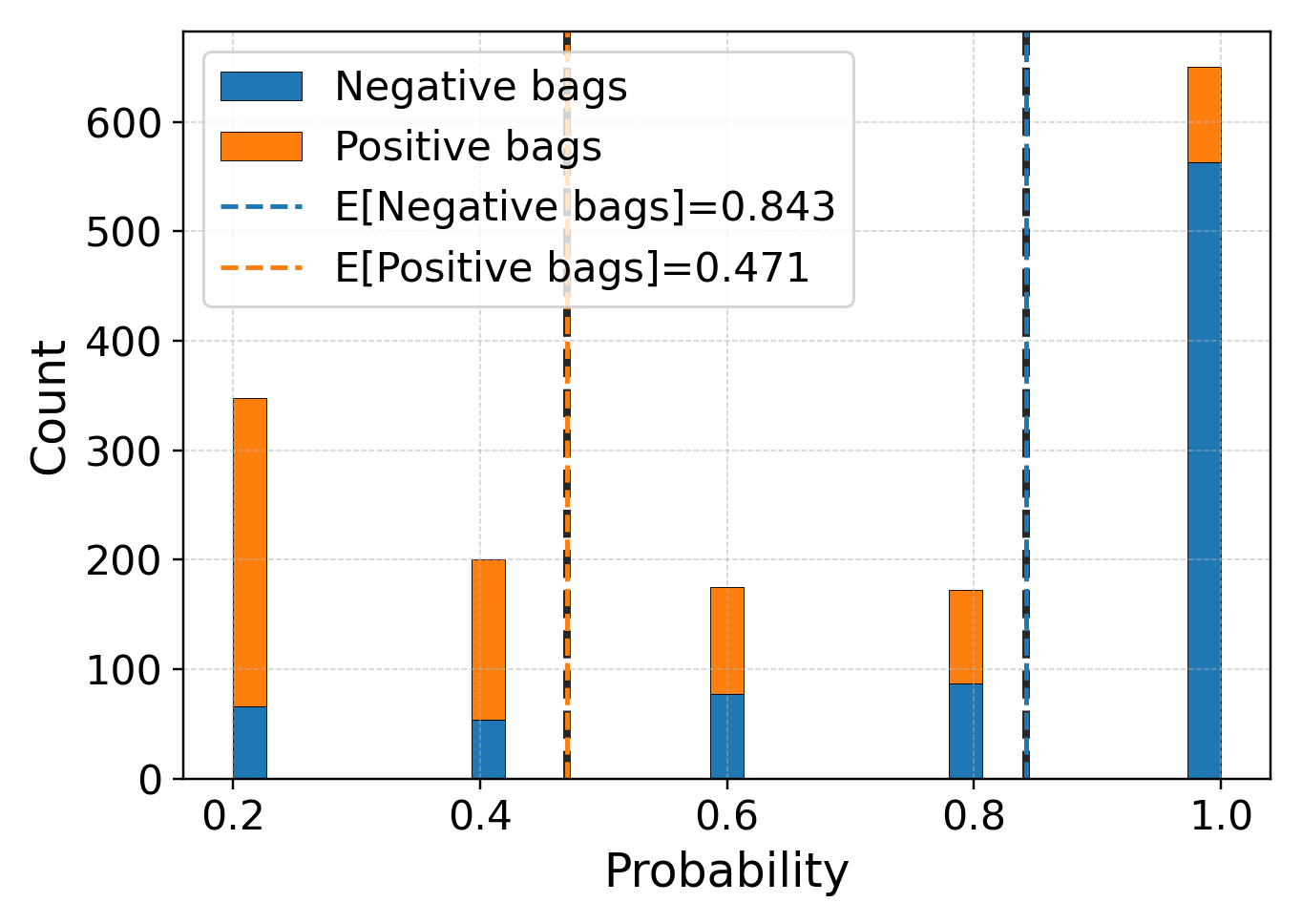}
        \caption{SQuAD\ (Llama-3.3-70B)}
    \end{subfigure}\hfill
    \vspace{0.3cm}
    \begin{subfigure}{0.32\textwidth}
        \includegraphics[width=\linewidth]{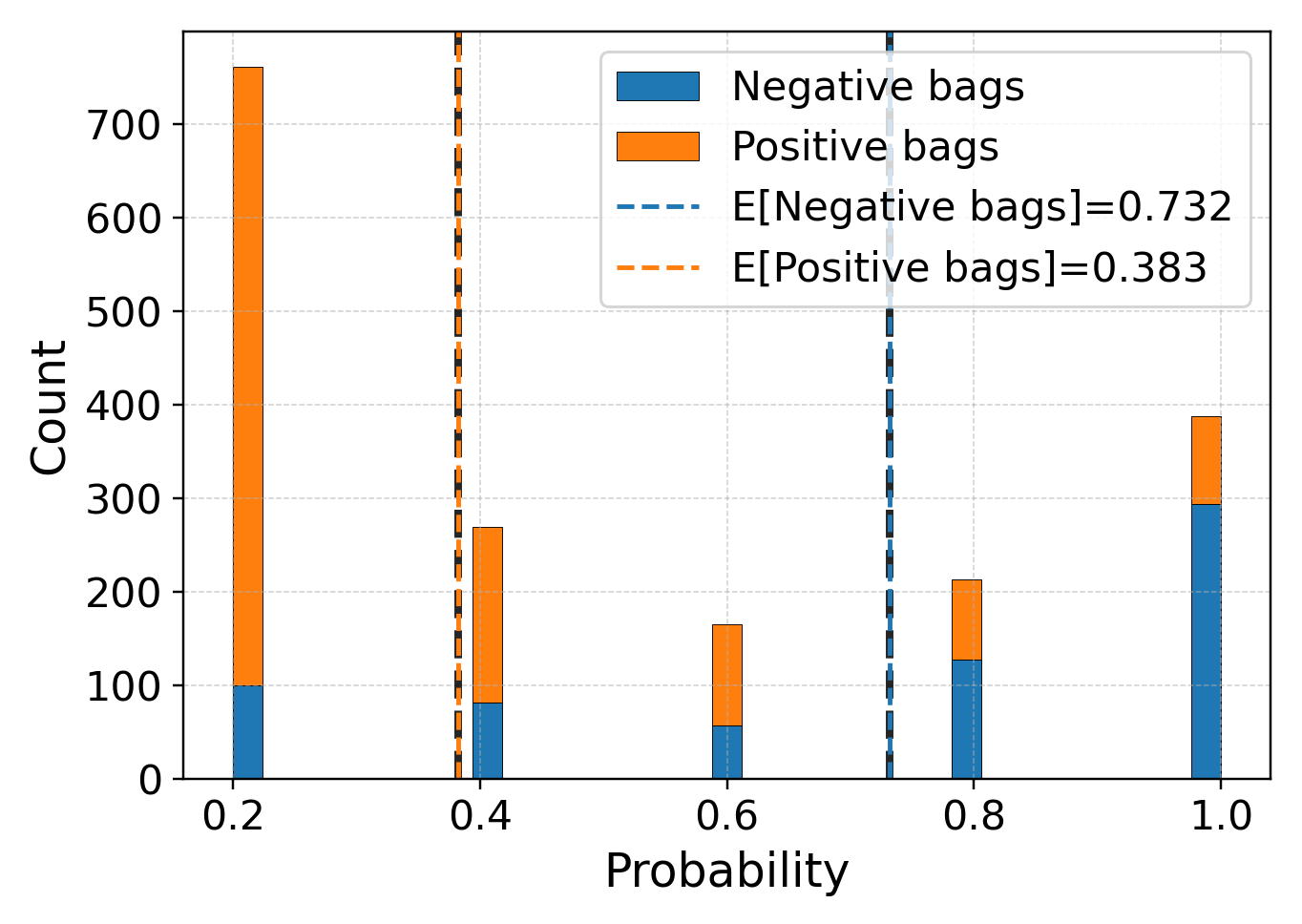}
        \caption{NQ\ (Llama-3.1-8B)}
    \end{subfigure}\hfill
    \begin{subfigure}{0.32\textwidth}
        \includegraphics[width=\linewidth]{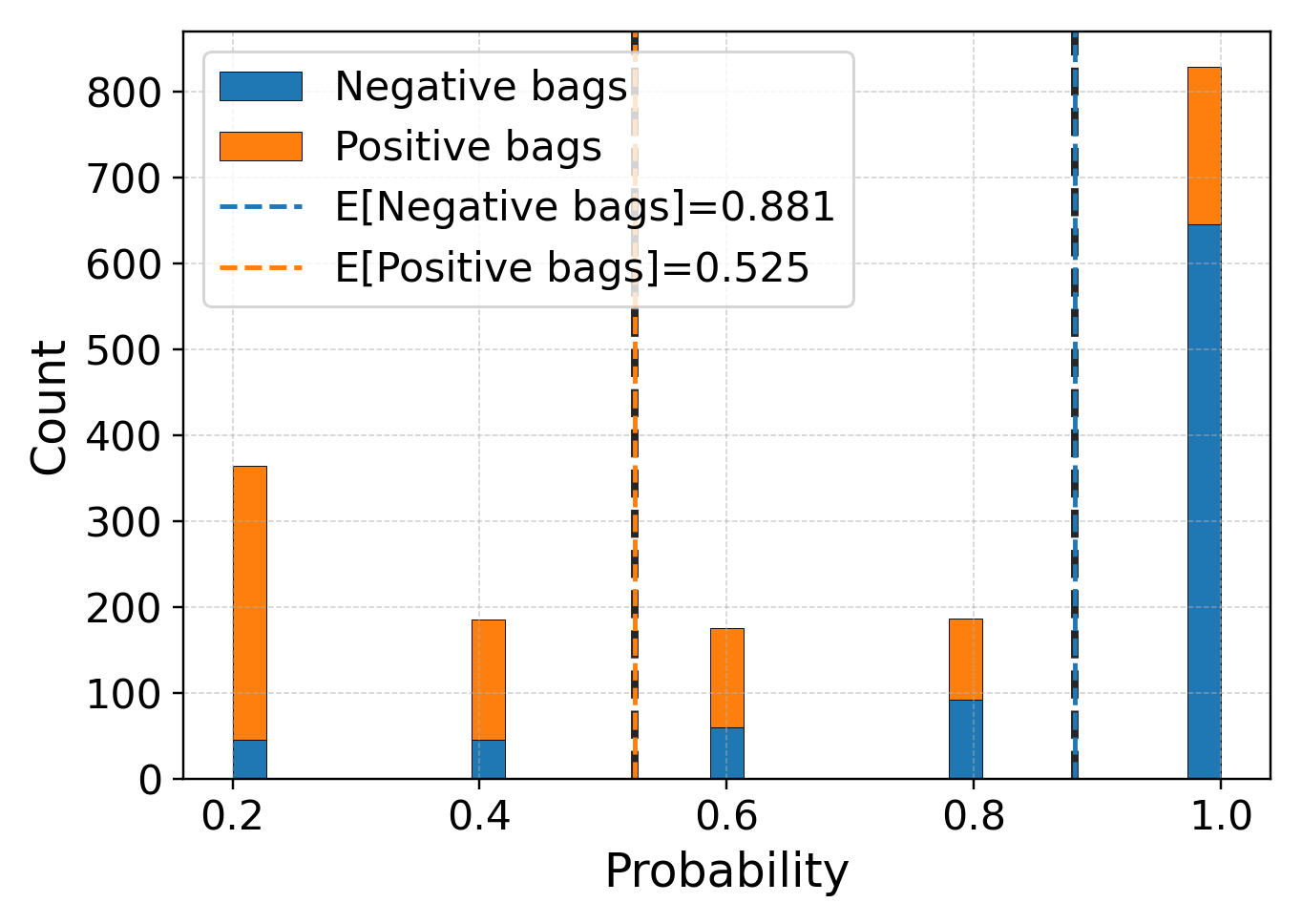}
        \caption{NQ\ (Mistral-12B)}
    \end{subfigure}\hfill
    \begin{subfigure}{0.32\textwidth}
        \includegraphics[width=\linewidth]{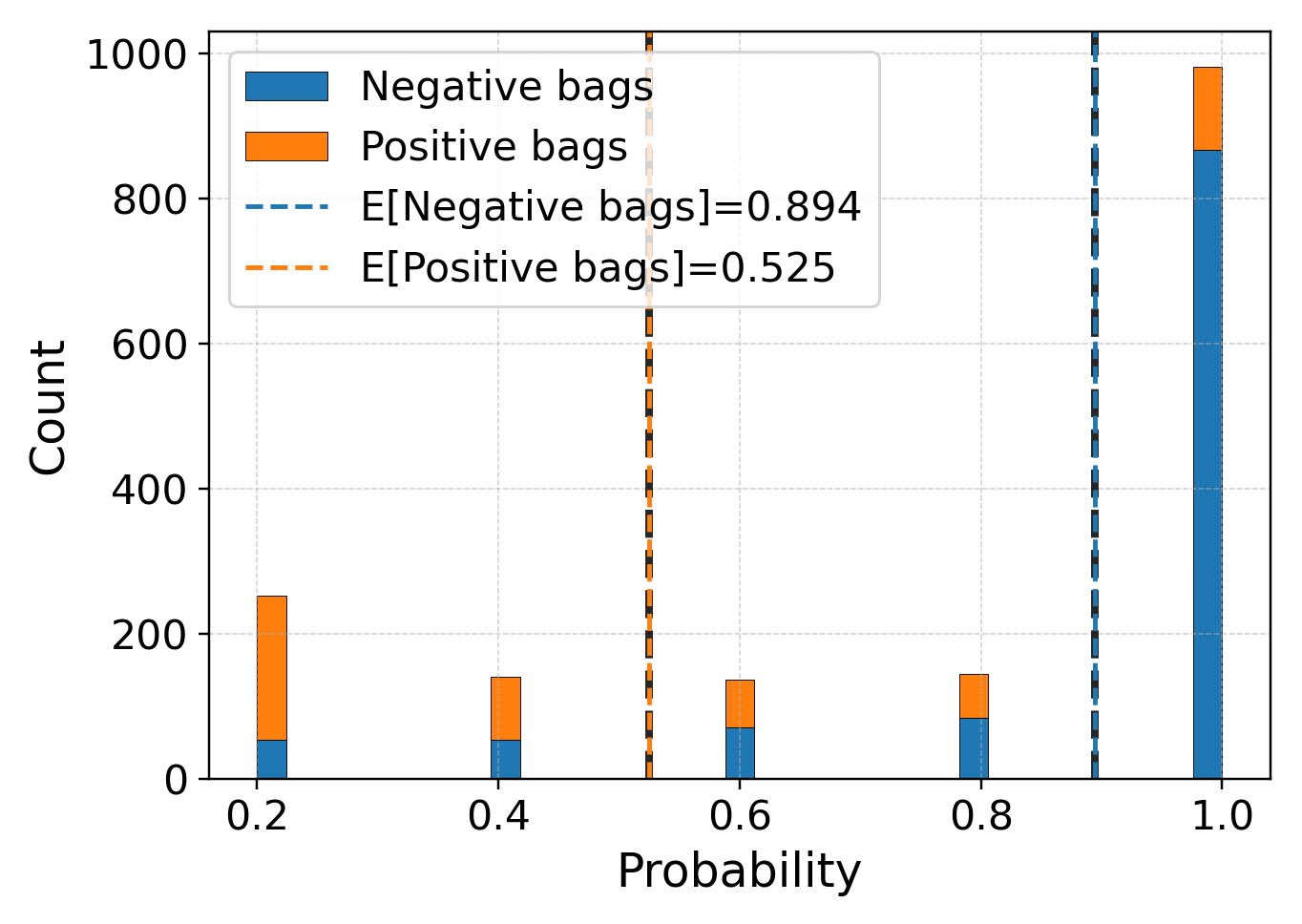}
        \caption{NQ\ (Llama-3.3-70B)}
    \end{subfigure}\hfill
    \vspace{0.3cm}
    \begin{subfigure}{0.32\textwidth}
        \includegraphics[width=\linewidth]{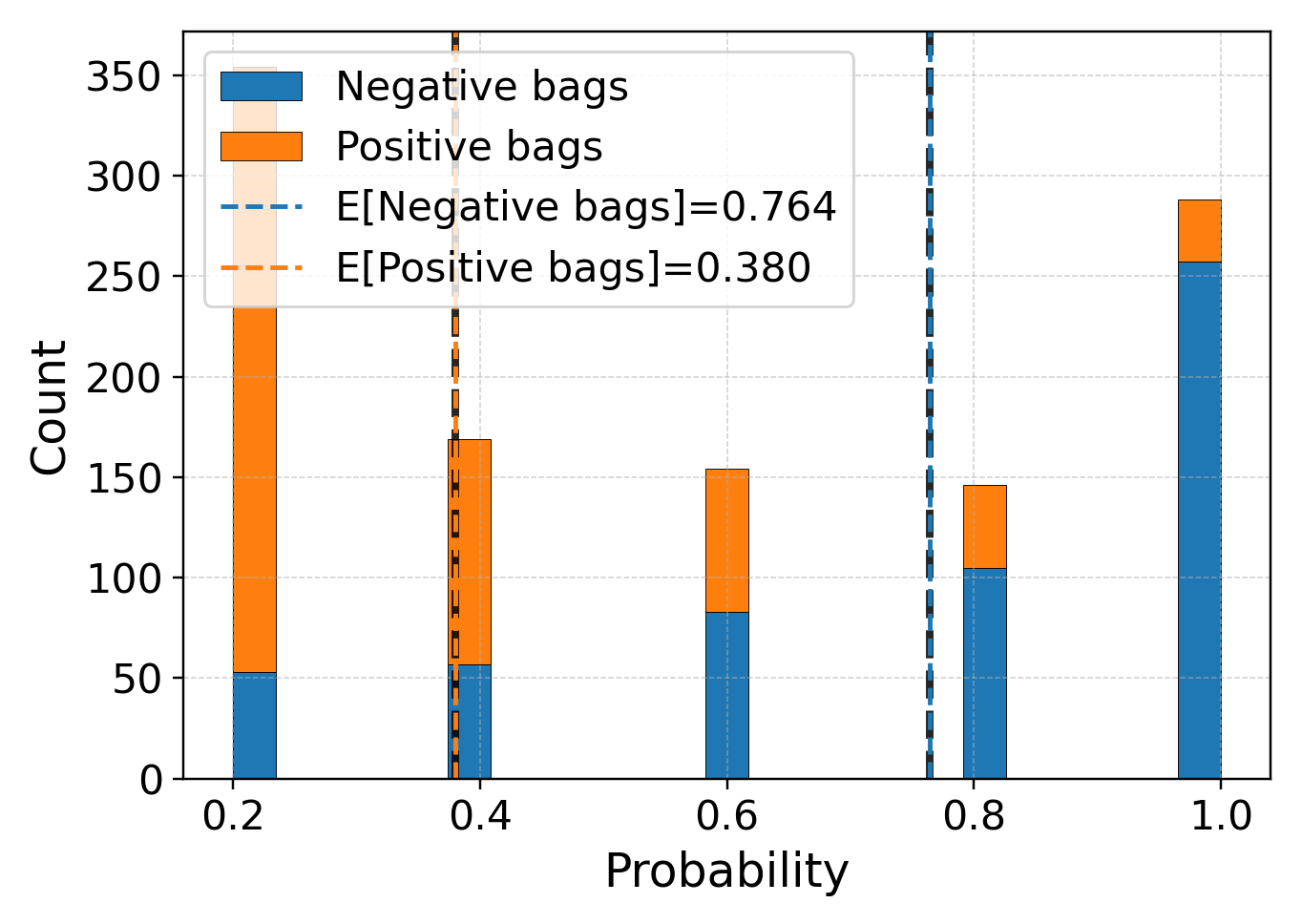}
        \caption{BioASQ\ (Llama-3.1-8B)}
    \end{subfigure}\hfill
    \begin{subfigure}{0.32\textwidth}
        \includegraphics[width=\linewidth]{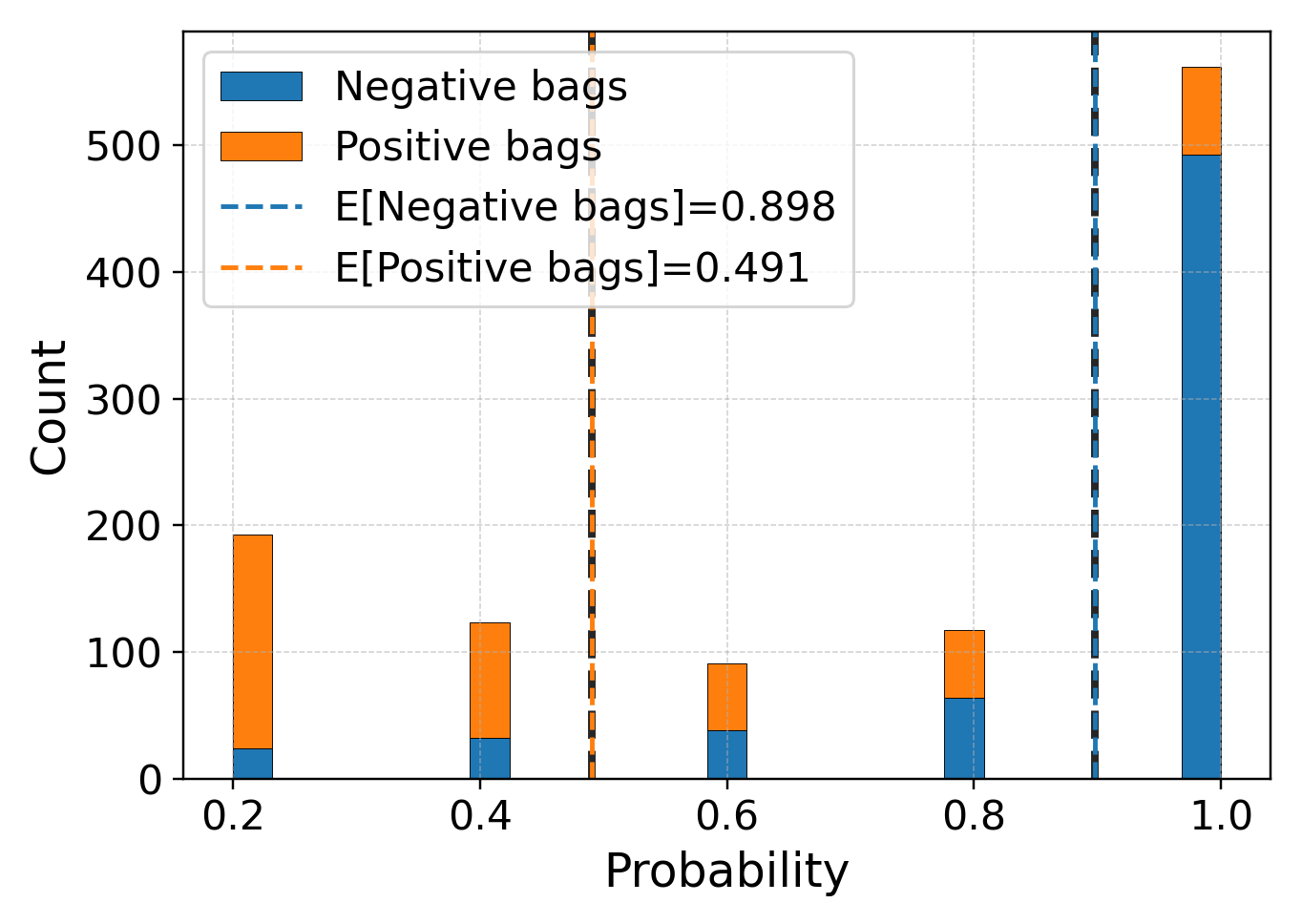}
        \caption{BioASQ\ (Mistral-12B)}
    \end{subfigure}\hfill
    \begin{subfigure}{0.32\textwidth}
        \includegraphics[width=\linewidth]{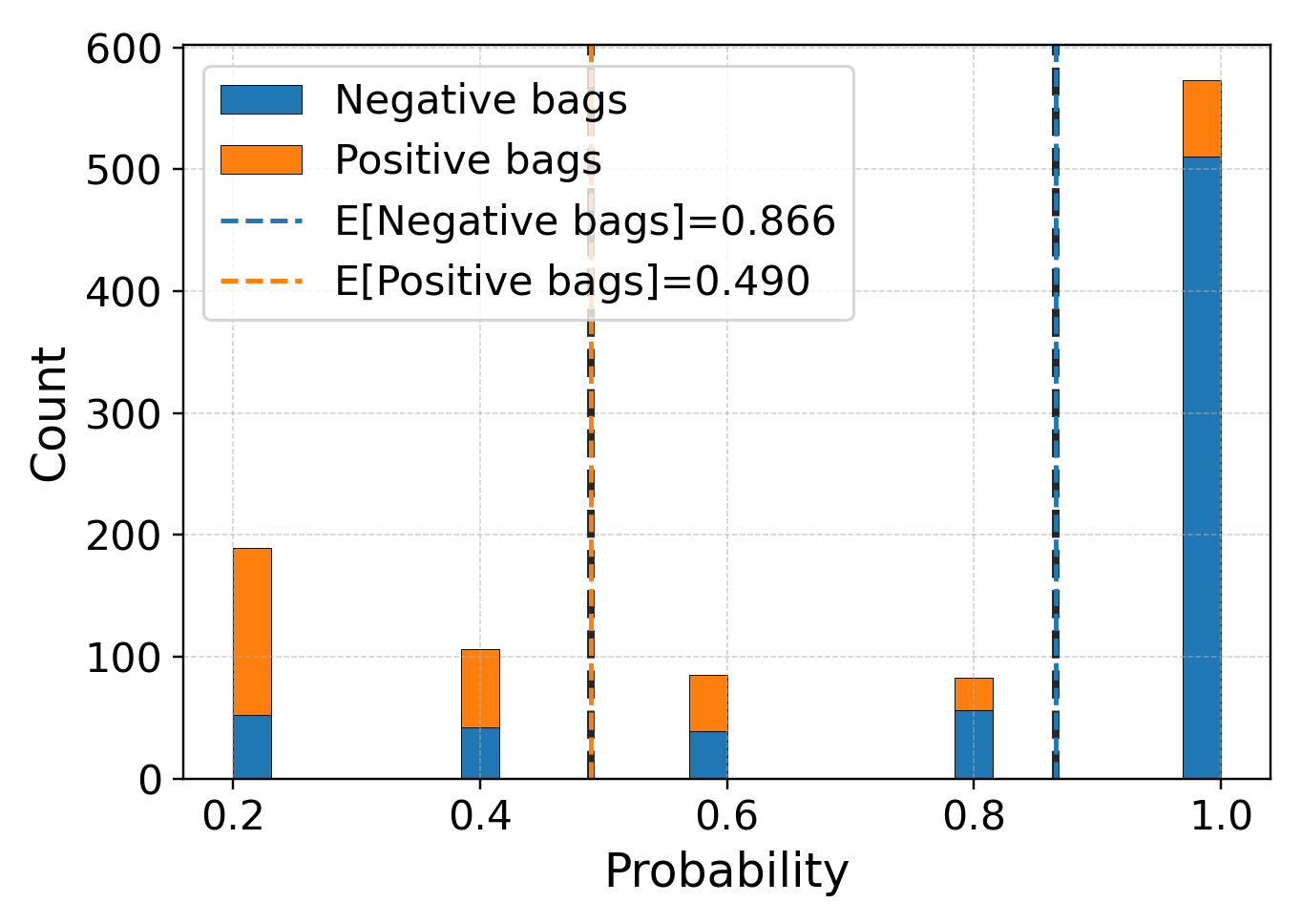}
        \caption{BioASQ\ (Llama-3.3-70B)}
    \end{subfigure}\hfill
    \vspace{0.3cm}
    \caption{Empirical distribution of semantic probabilities $P_{\mathrm{sem}}^{\mathbf{B}}$ in \textbf{test} data.}
    \label{fig:dist_psem_val}
\end{figure}

\begin{table}[H]
\centering
\caption{AUC scores for Semantic Probability as a discriminator across datasets and models.}
\label{tab:se_sp_auc}
\resizebox{0.9\textwidth}{!}{%
\begin{tabular}{lccc}
\toprule
\textbf{Dataset} & \textbf{LLaMA-3.1-8B} & \textbf{Mistral-12B} & \textbf{LLaMA-3.3-70B} \\
\midrule
TriviaQA & 0.883 & 0.812 & 0.832 \\
SQuAD    & 0.812 & 0.802 & 0.817 \\
NQ       & 0.800 & 0.794 & 0.807 \\
BioASQ   & 0.840 & 0.843 & 0.810 \\
\bottomrule
\end{tabular}
}
\end{table}

\subsection{Distribution of \texorpdfstring{$\bar{C}_B$}{C-bar}}
\label{app:sensitivity}
To examine $\mathbb{E}_{neg}[\bar{C}_{\mathbf{B}}] > \mathbb{E}_{pos}[\bar{C}_{\mathbf{B}}]$, we analyzed the sensitivity distributions in Figures~\ref{fig:dist_sensitivity_train}--\ref{fig:dist_sensitivity_val} and the class-wise means in Appendix~\ref{app:expected_values_summary}. Across the evaluated models and four short-answer benchmarks, the distributions for negative bags tend to be shifted toward larger values. The higher empirical means indicate that the trained network exhibits greater sensitivity to input scaling for negative bags on average. This pattern is observed in both data splits.
\begin{figure}[H]
    \centering
    \begin{subfigure}{0.32\textwidth}
        \includegraphics[width=\linewidth]{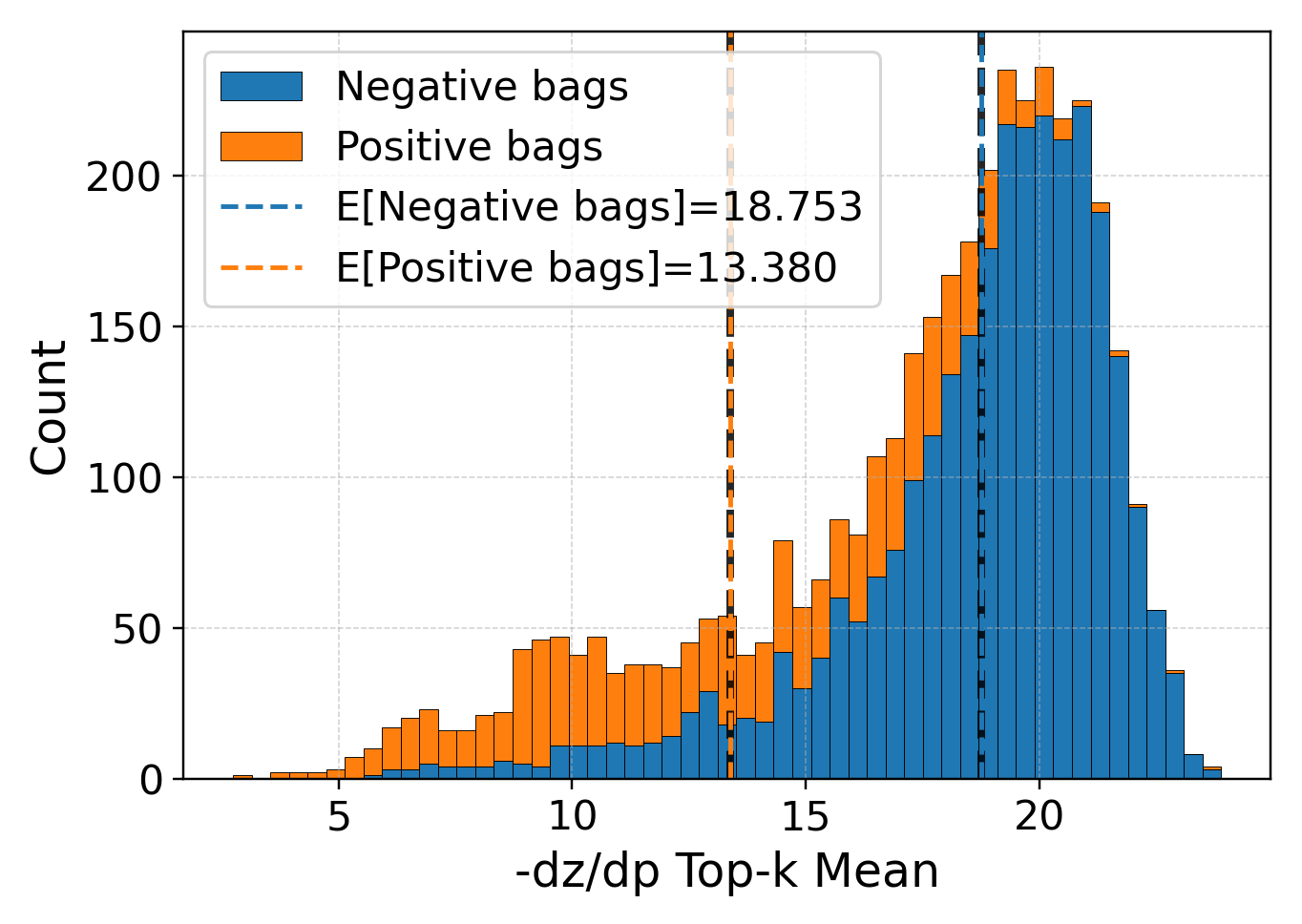}
        \caption{TriviaQA\ (Llama-3.1-8B)}
    \end{subfigure}\hfill
    \begin{subfigure}{0.32\textwidth}
        \includegraphics[width=\linewidth]{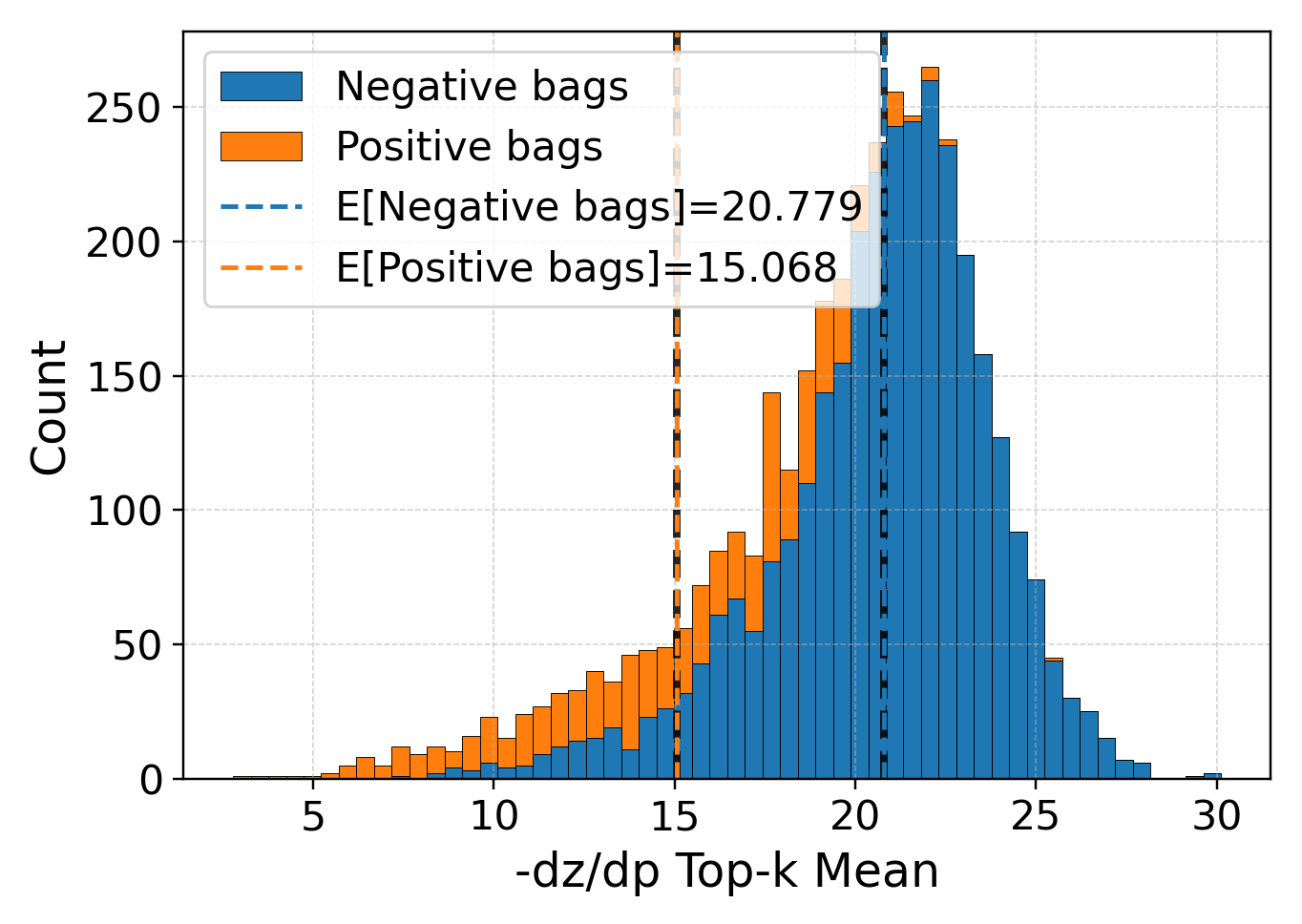}
        \caption{TriviaQA\ (Mistral-12B)}
    \end{subfigure}\hfill
    \begin{subfigure}{0.32\textwidth}
        \includegraphics[width=\linewidth]{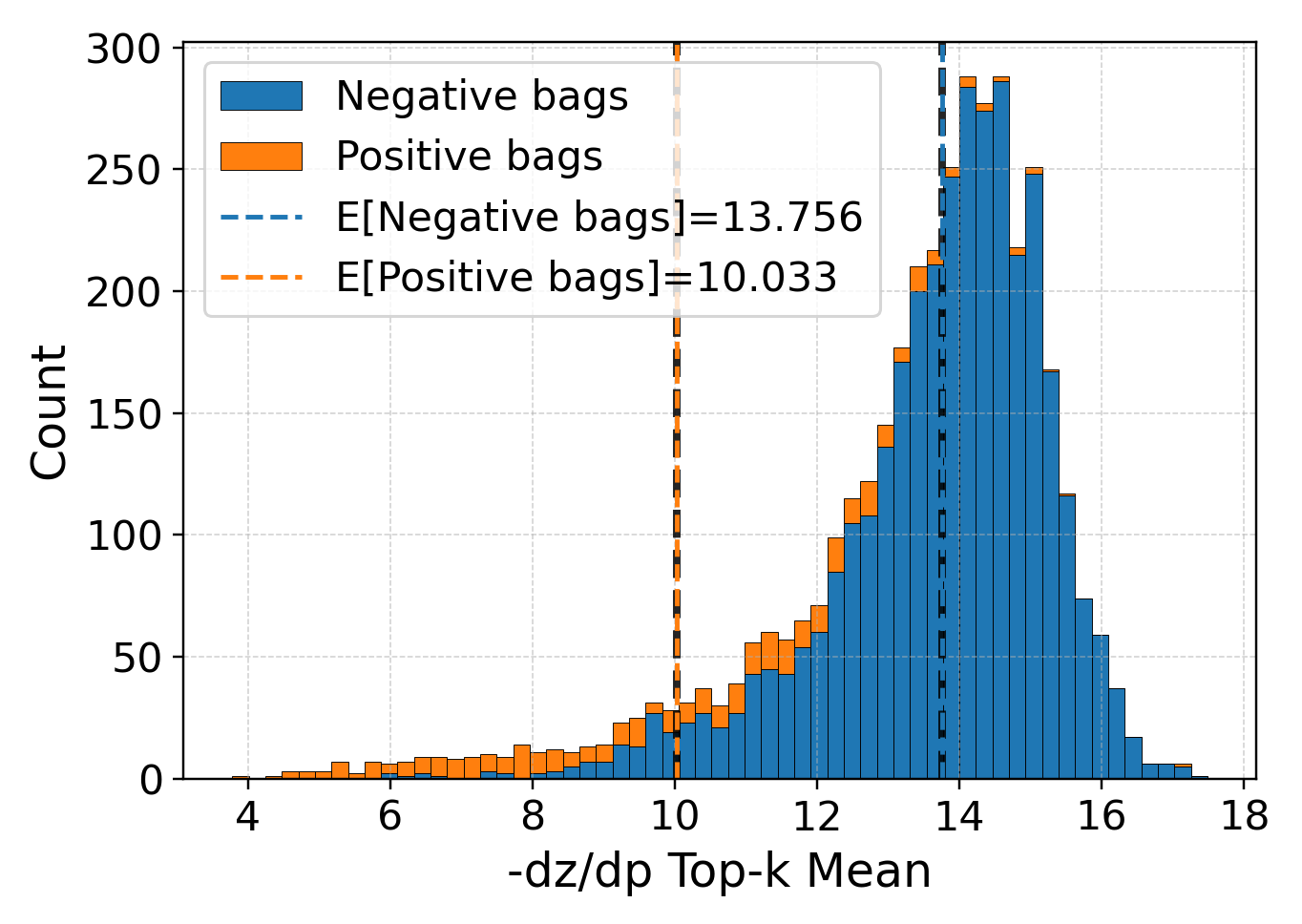}
        \caption{TriviaQA\ (Llama-3.3-70B)}
    \end{subfigure}\hfill
    \vspace{0.3cm}
    \begin{subfigure}{0.32\textwidth}
        \includegraphics[width=\linewidth]{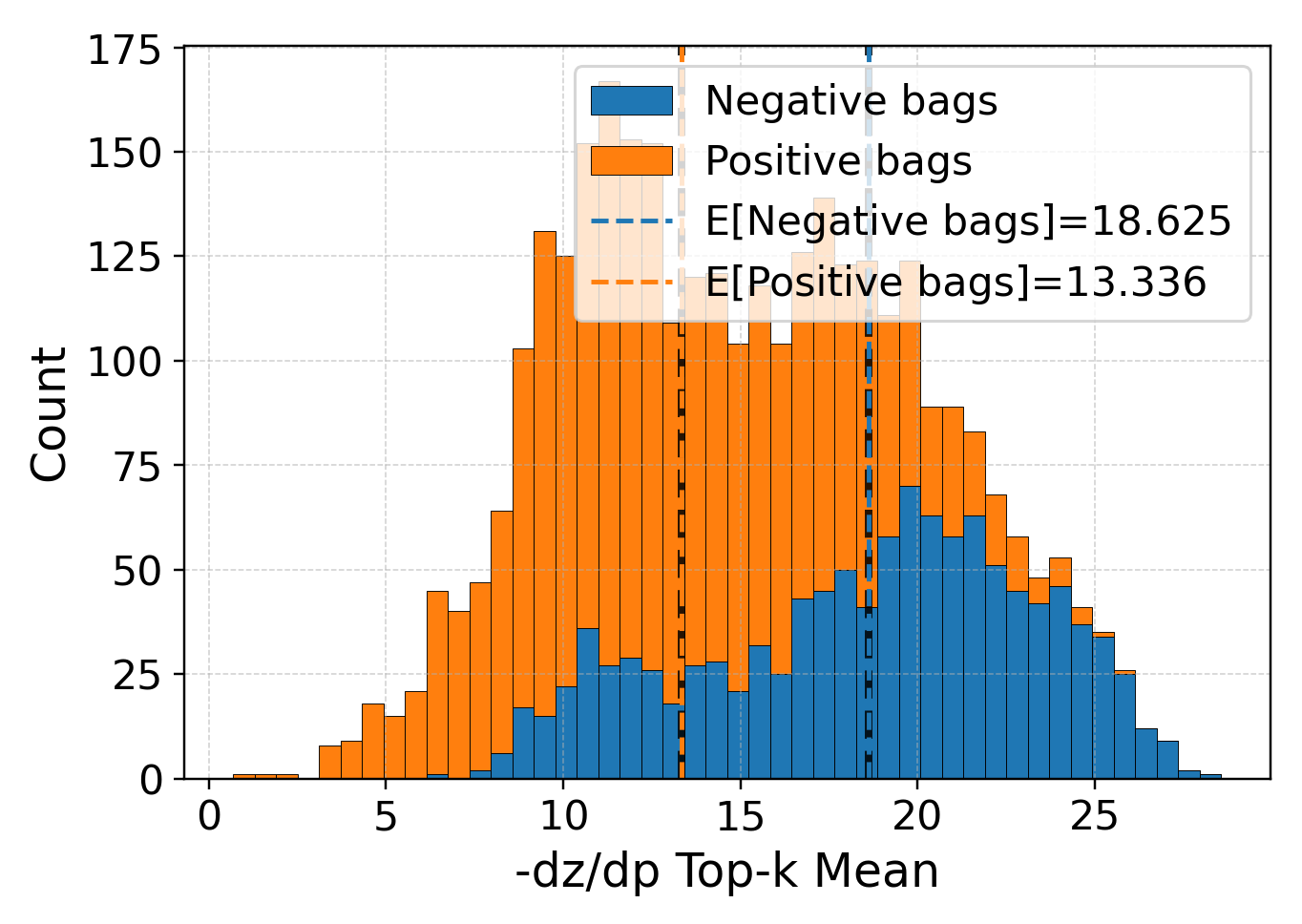}
        \caption{SQuAD\ (Llama-3.1-8B)}
    \end{subfigure}\hfill
    \begin{subfigure}{0.32\textwidth}
        \includegraphics[width=\linewidth]{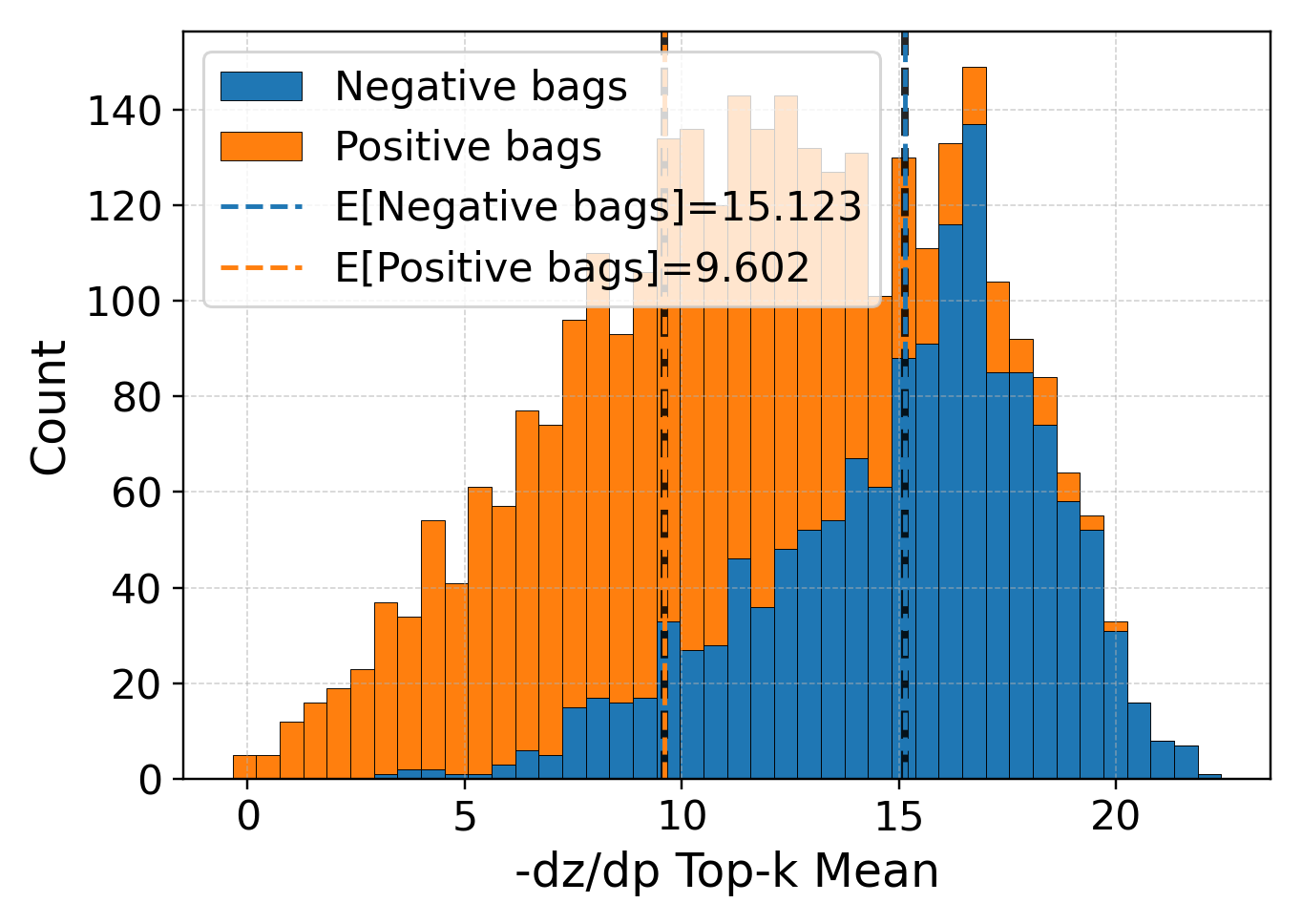}
        \caption{SQuAD\ (Mistral-12B)}
    \end{subfigure}\hfill
    \begin{subfigure}{0.32\textwidth}
        \includegraphics[width=\linewidth]{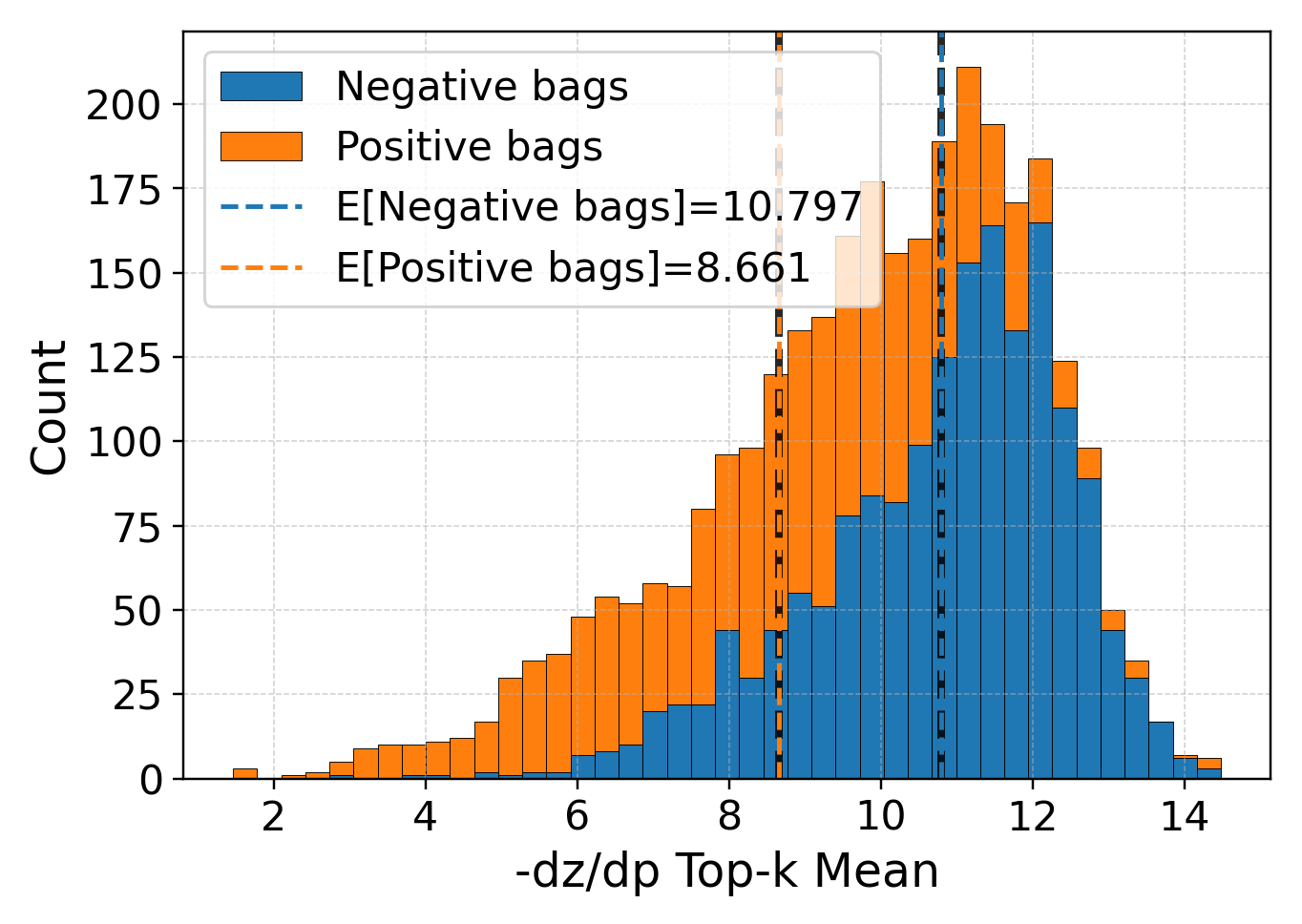}
        \caption{SQuAD\ (Llama-3.3-70B)}
    \end{subfigure}\hfill
    \vspace{0.3cm}
    \begin{subfigure}{0.32\textwidth}
        \includegraphics[width=\linewidth]{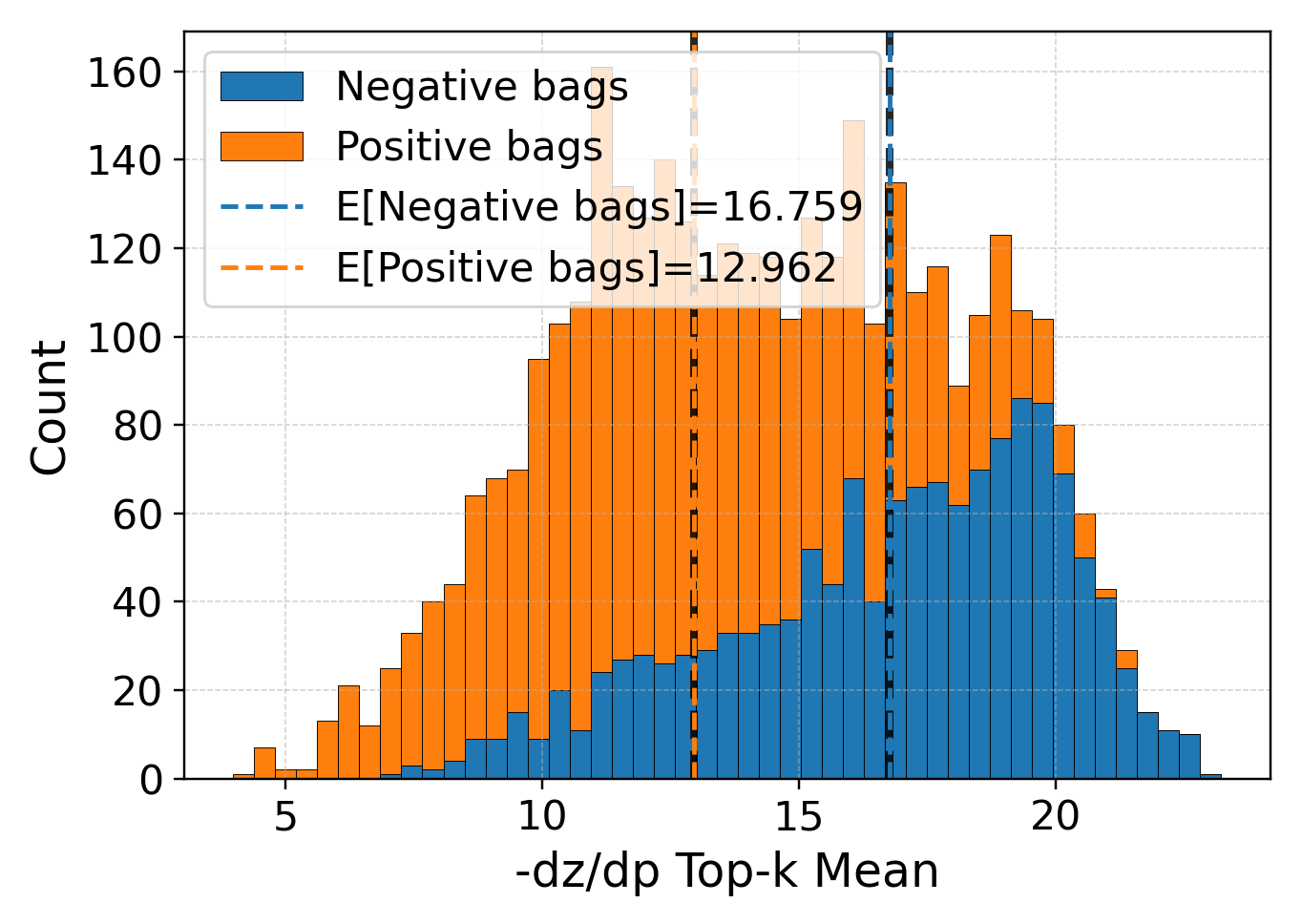}
        \caption{NQ\ (Llama-3.1-8B)}
    \end{subfigure}\hfill
    \begin{subfigure}{0.32\textwidth}
        \includegraphics[width=\linewidth]{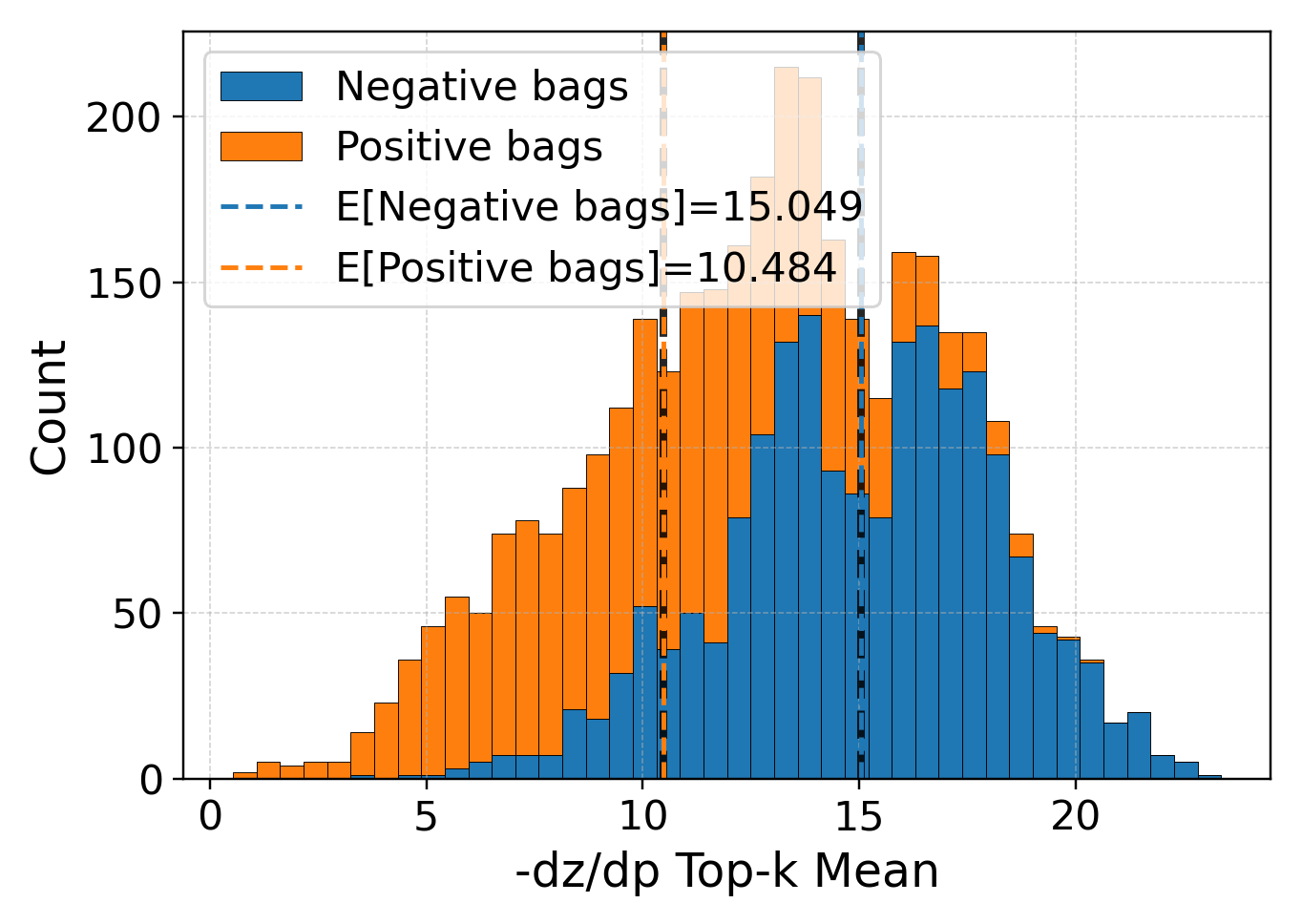}
        \caption{NQ\ (Mistral-12B)}
    \end{subfigure}\hfill
    \begin{subfigure}{0.32\textwidth}
        \includegraphics[width=\linewidth]{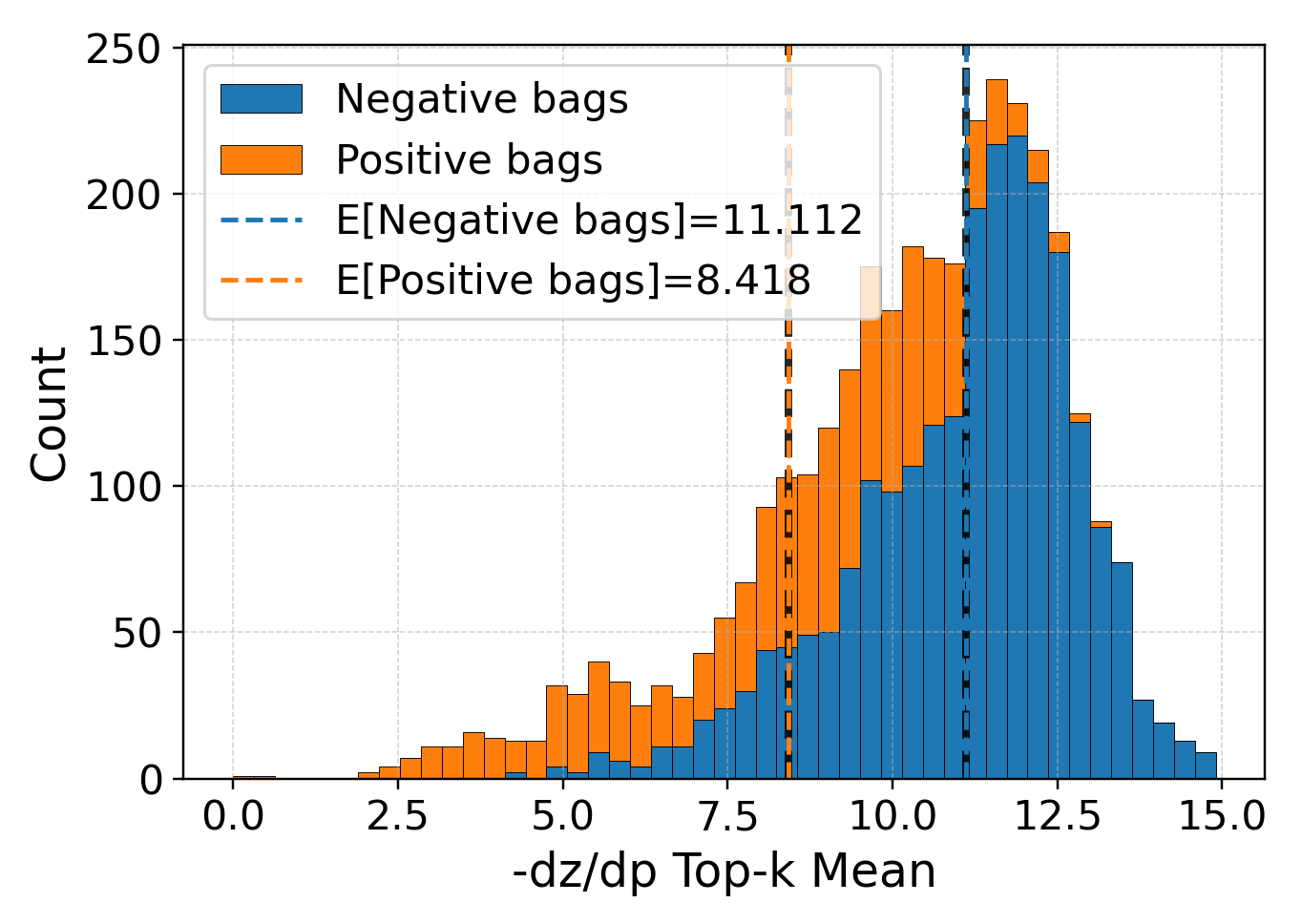}
        \caption{NQ\ (Llama-3.3-70B)}
    \end{subfigure}\hfill
    \vspace{0.3cm}
    \begin{subfigure}{0.32\textwidth}
        \includegraphics[width=\linewidth]{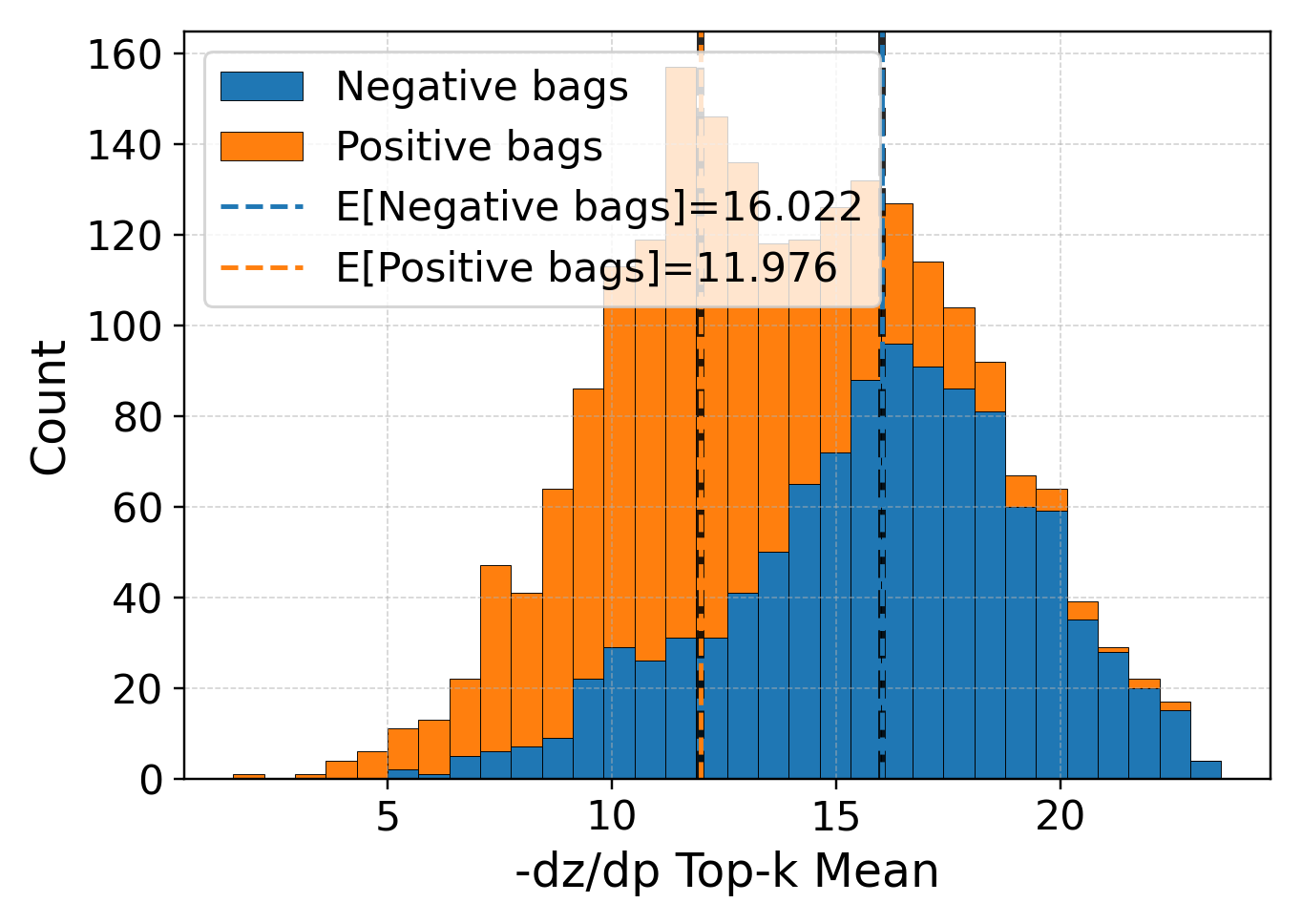}
        \caption{BioASQ\ (Llama-3.1-8B)}
    \end{subfigure}\hfill
    \begin{subfigure}{0.32\textwidth}
        \includegraphics[width=\linewidth]{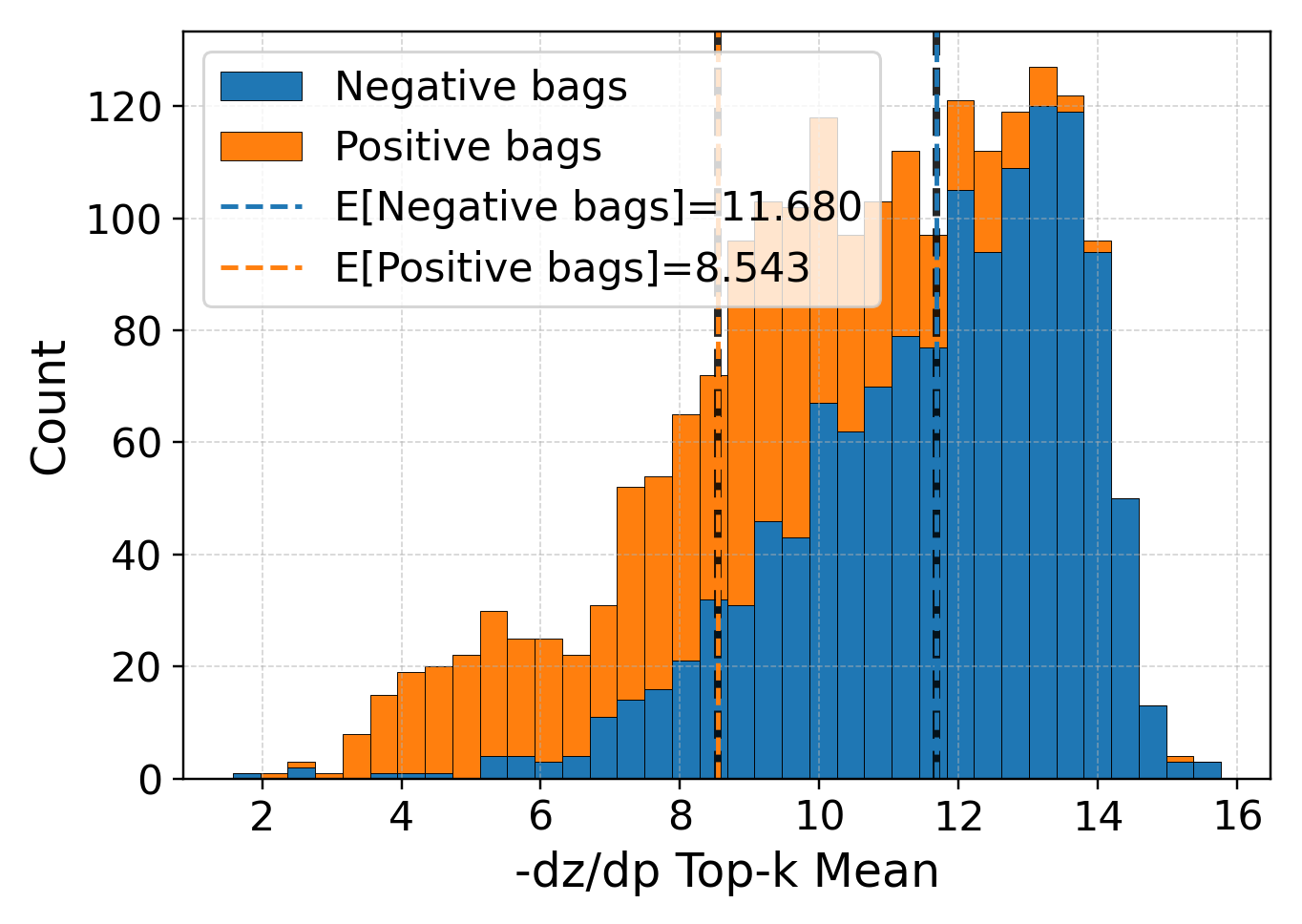}
        \caption{BioASQ\ (Mistral-12B)}
    \end{subfigure}\hfill
    \begin{subfigure}{0.32\textwidth}
        \includegraphics[width=\linewidth]{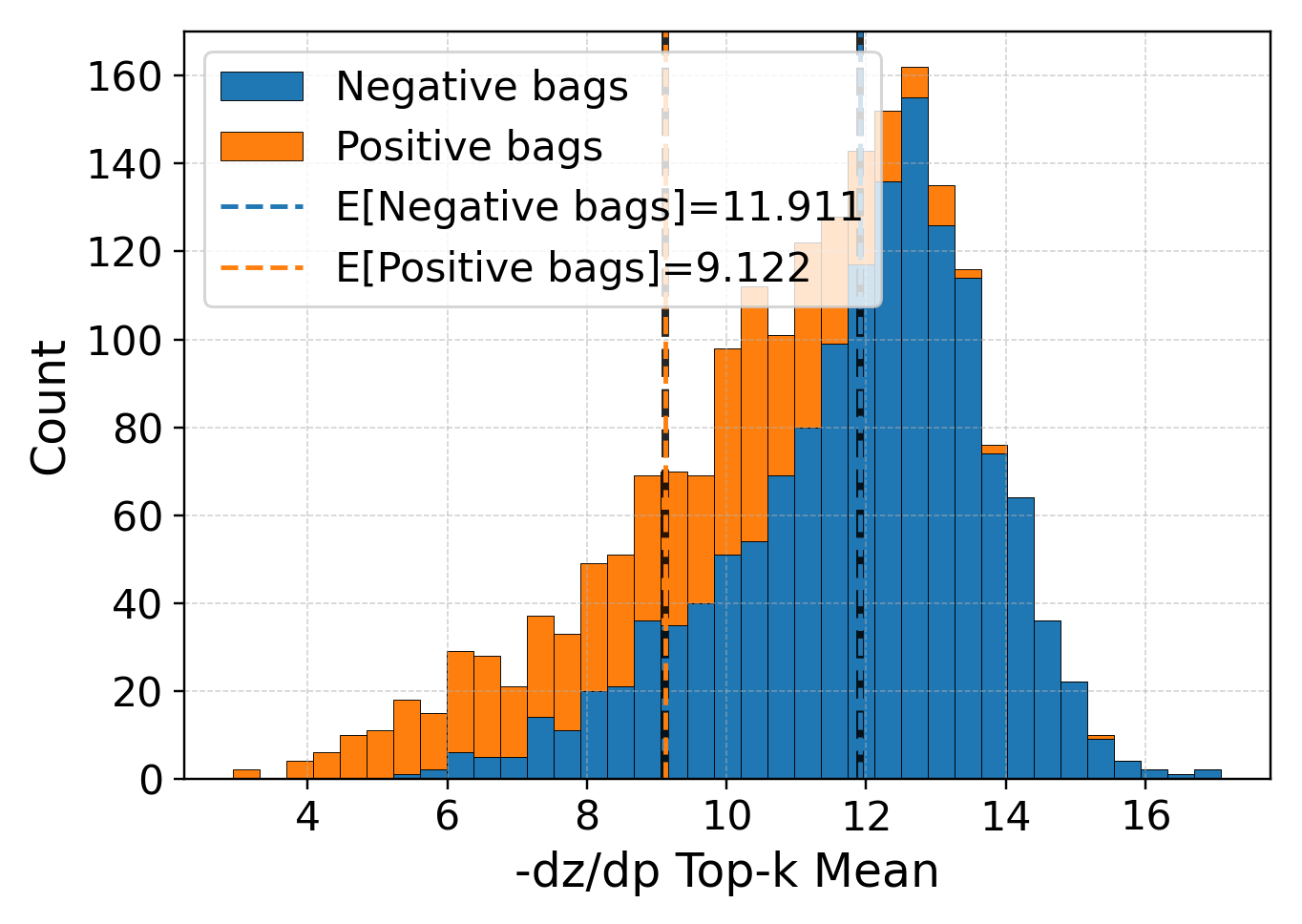}
        \caption{BioASQ\ (Llama-3.3-70B)}
    \end{subfigure}\hfill
    \vspace{0.3cm}
    \caption{Empirical distribution of the bag-level sensitivity to input scaling $\bar{C}_{\mathbf{B}}$ in \textbf{train} data.}
    \label{fig:dist_sensitivity_train}
\end{figure}

\begin{figure}[H]
    \centering
    \begin{subfigure}{0.32\textwidth}
        \includegraphics[width=\linewidth]{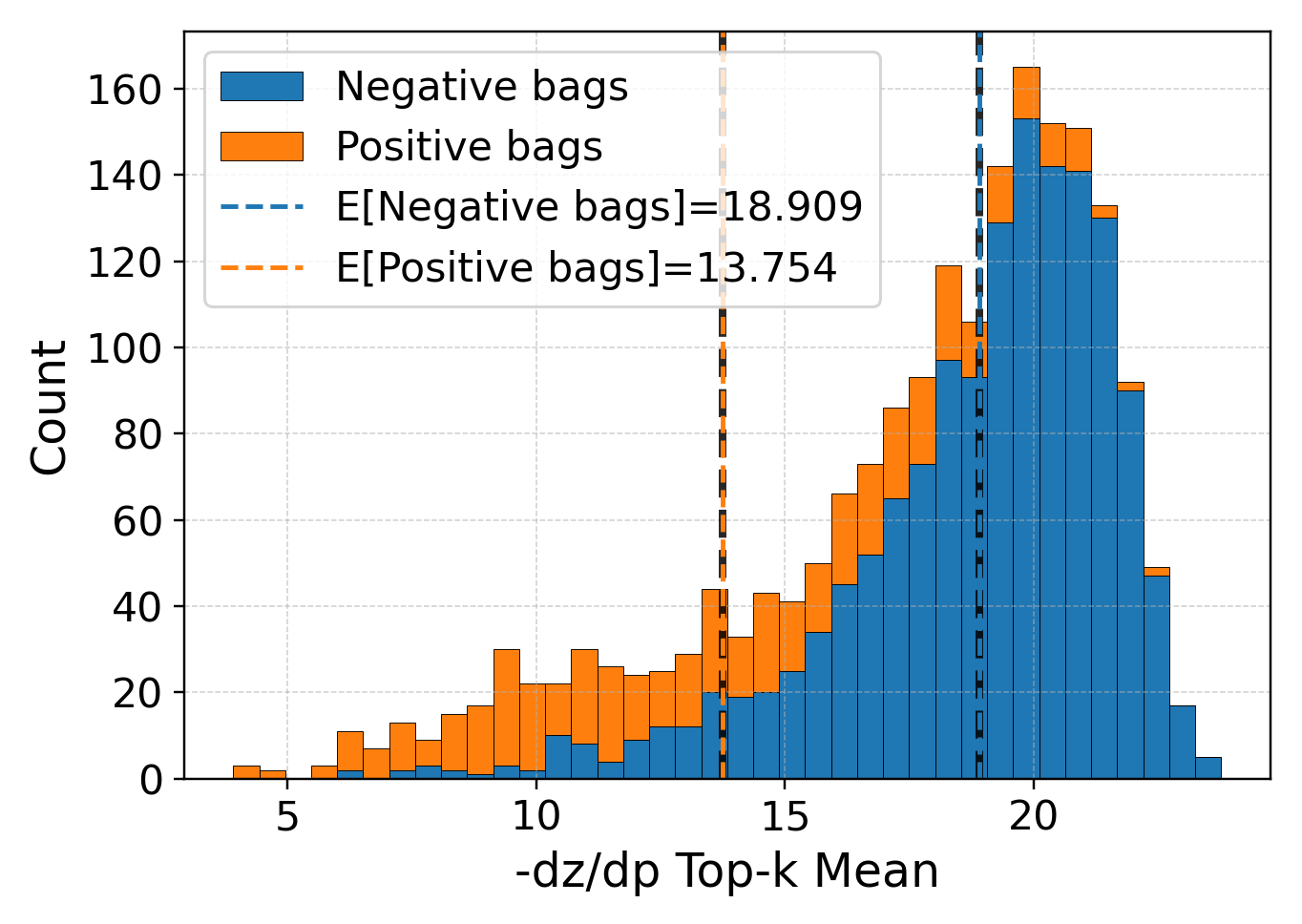}
        \caption{TriviaQA\ (Llama-3.1-8B)}
    \end{subfigure}\hfill
    \begin{subfigure}{0.32\textwidth}
        \includegraphics[width=\linewidth]{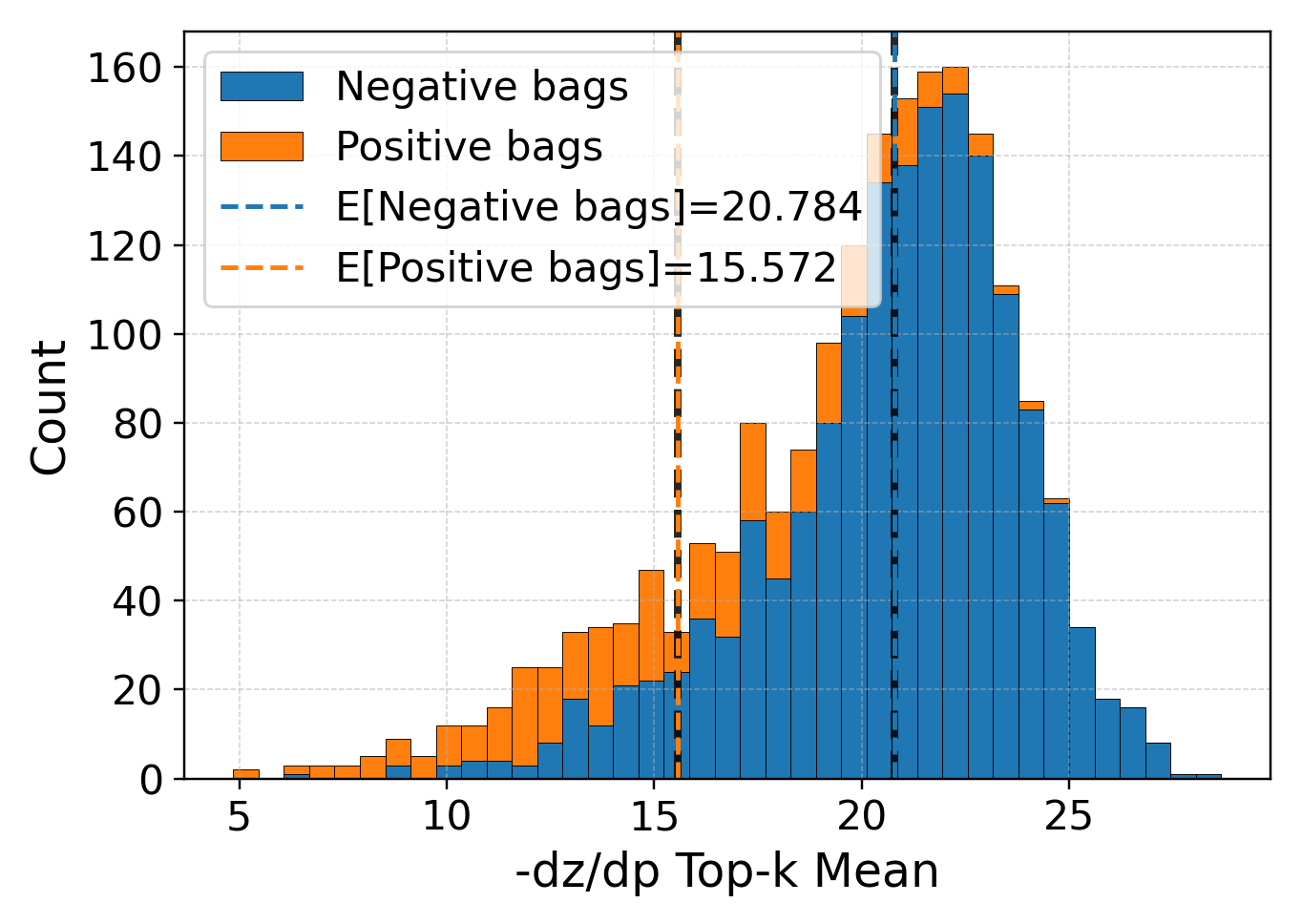}
        \caption{TriviaQA\ (Mistral-12B)}
    \end{subfigure}\hfill
    \begin{subfigure}{0.32\textwidth}
        \includegraphics[width=\linewidth]{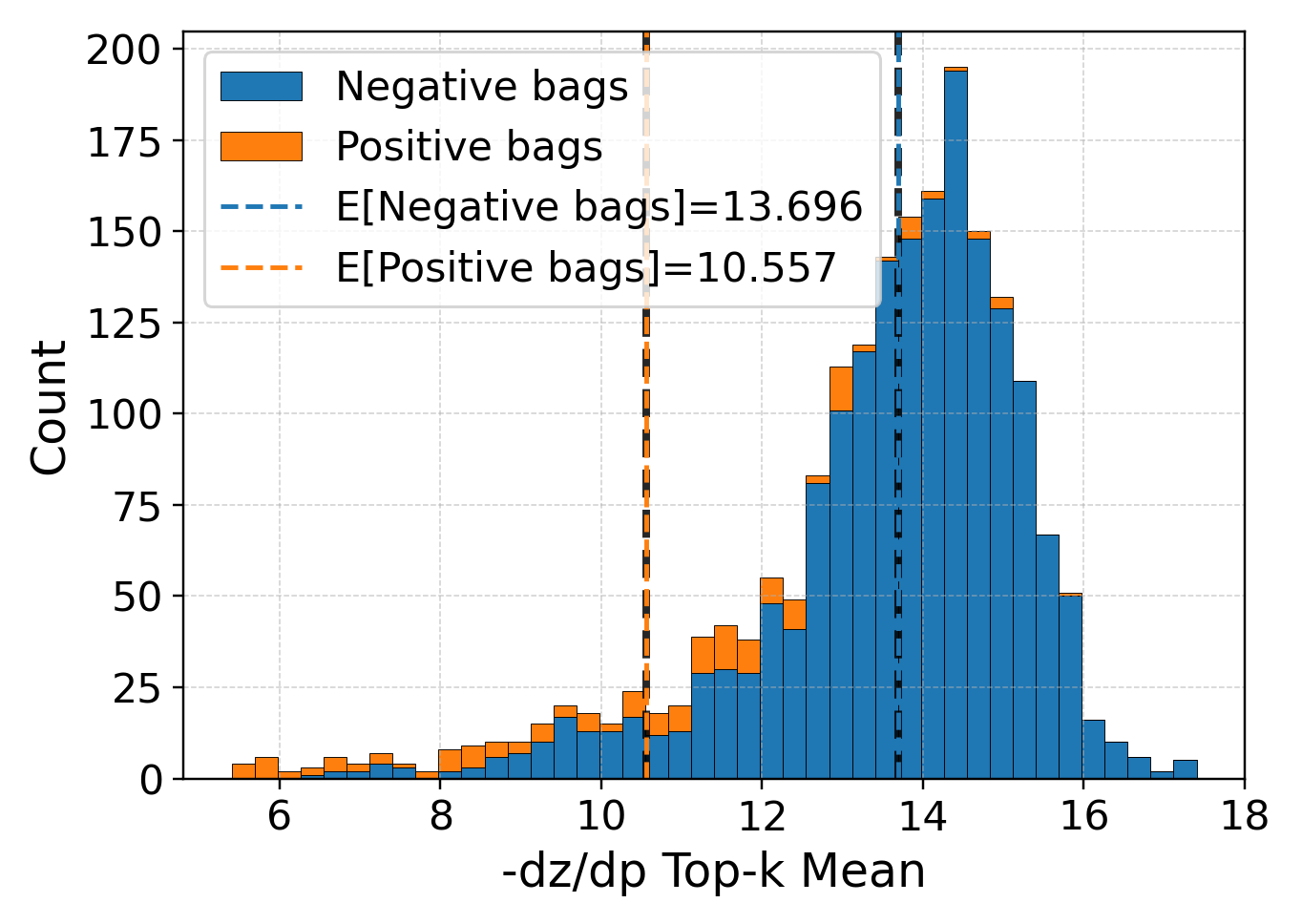}
        \caption{TriviaQA\ (Llama-3.3-70B)}
    \end{subfigure}\hfill
    \vspace{0.3cm}
    \begin{subfigure}{0.32\textwidth}
        \includegraphics[width=\linewidth]{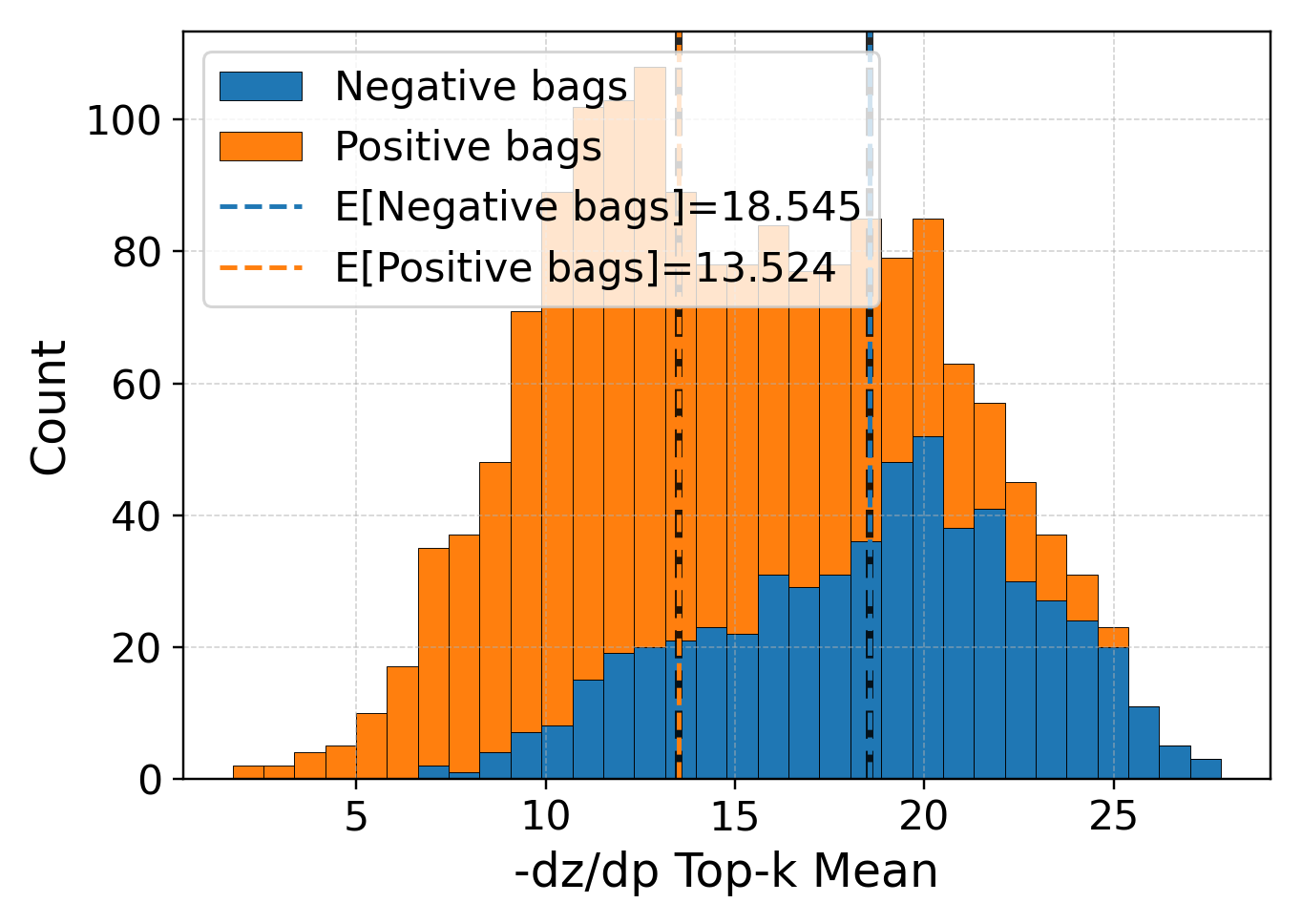}
        \caption{SQuAD\ (Llama-3.1-8B)}
    \end{subfigure}\hfill
    \begin{subfigure}{0.32\textwidth}
        \includegraphics[width=\linewidth]{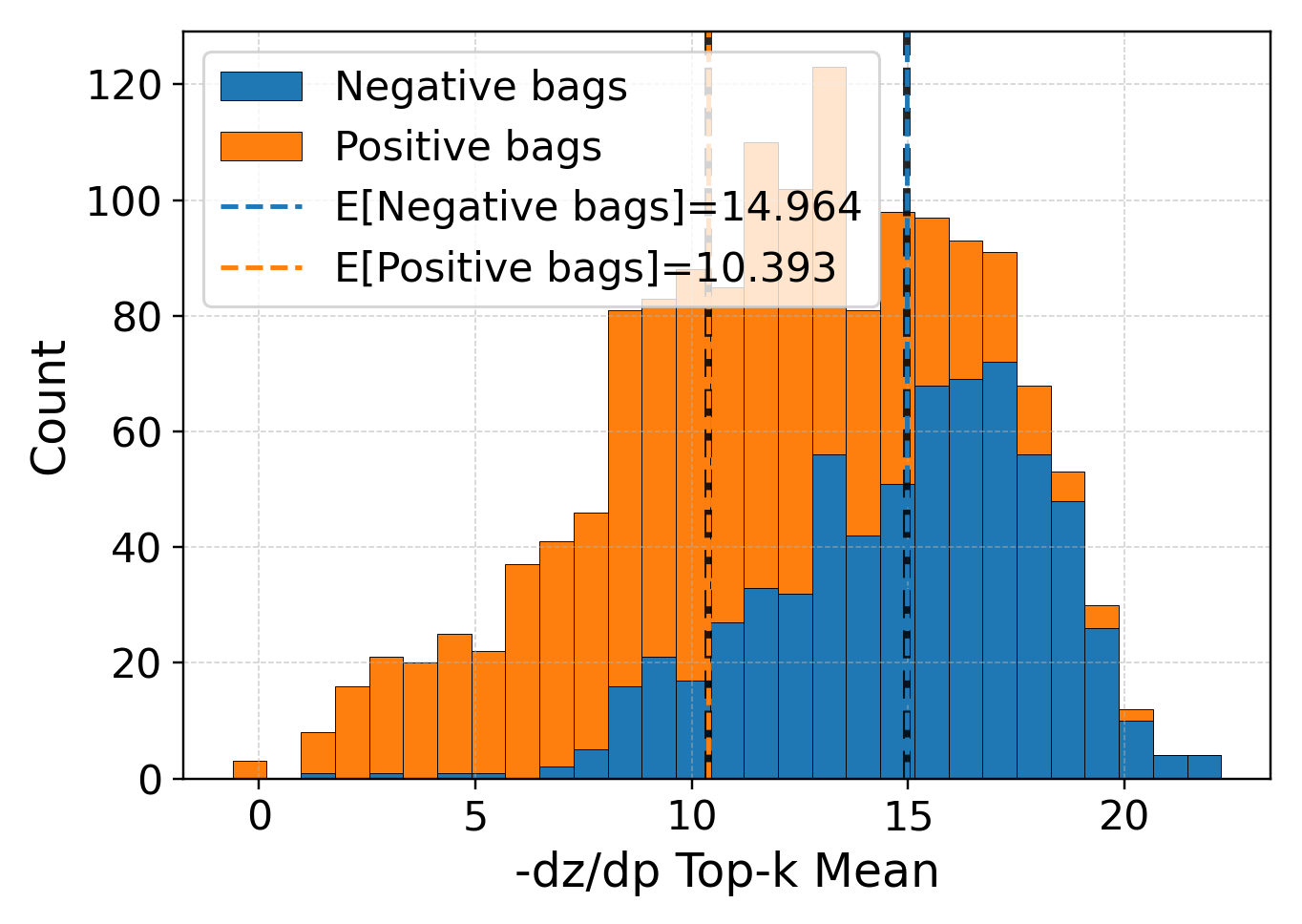}
        \caption{SQuAD\ (Mistral-12B)}
    \end{subfigure}\hfill
    \begin{subfigure}{0.32\textwidth}
        \includegraphics[width=\linewidth]{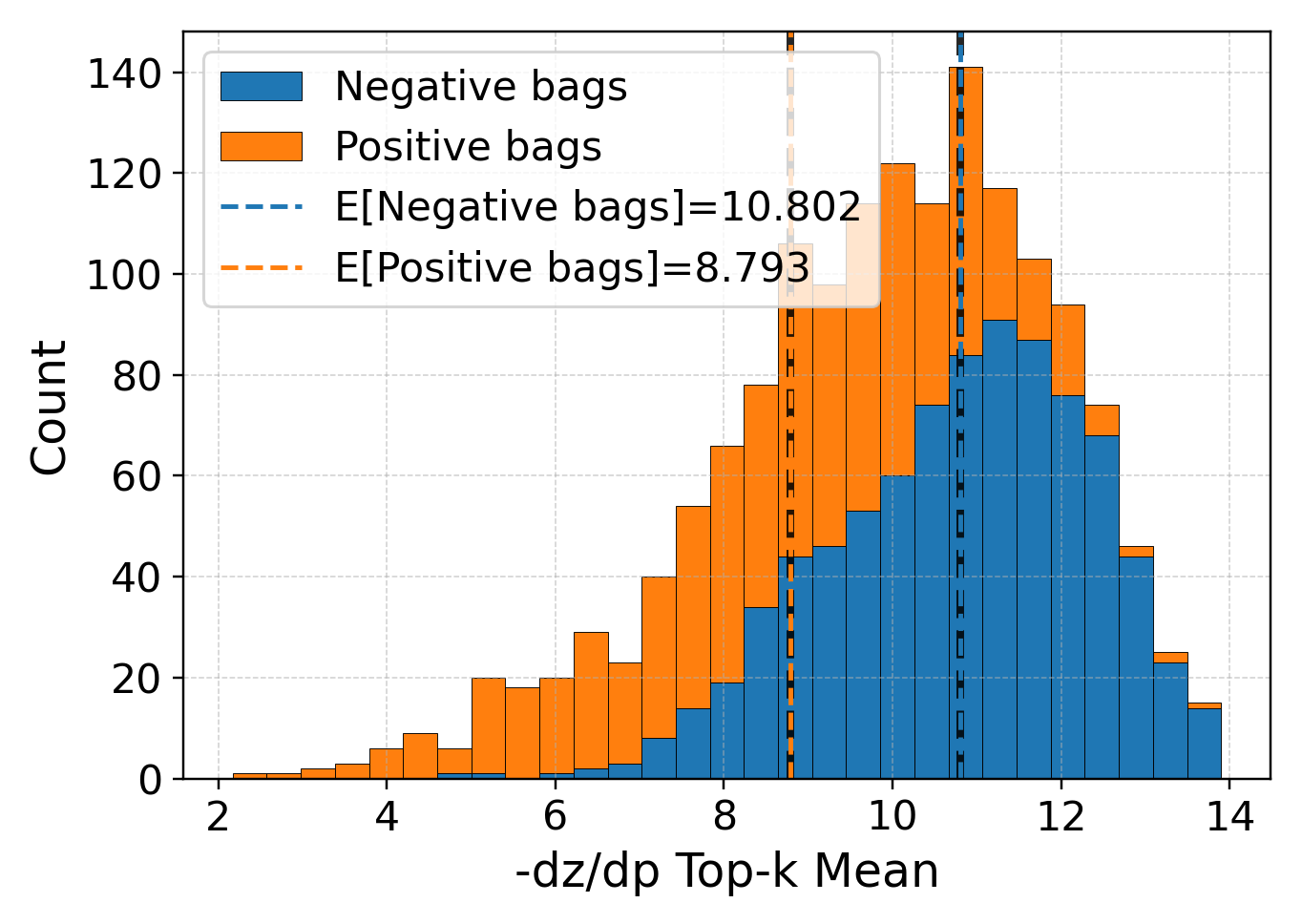}
        \caption{SQuAD\ (Llama-3.3-70B)}
    \end{subfigure}\hfill
    \vspace{0.3cm}
    \begin{subfigure}{0.32\textwidth}
        \includegraphics[width=\linewidth]{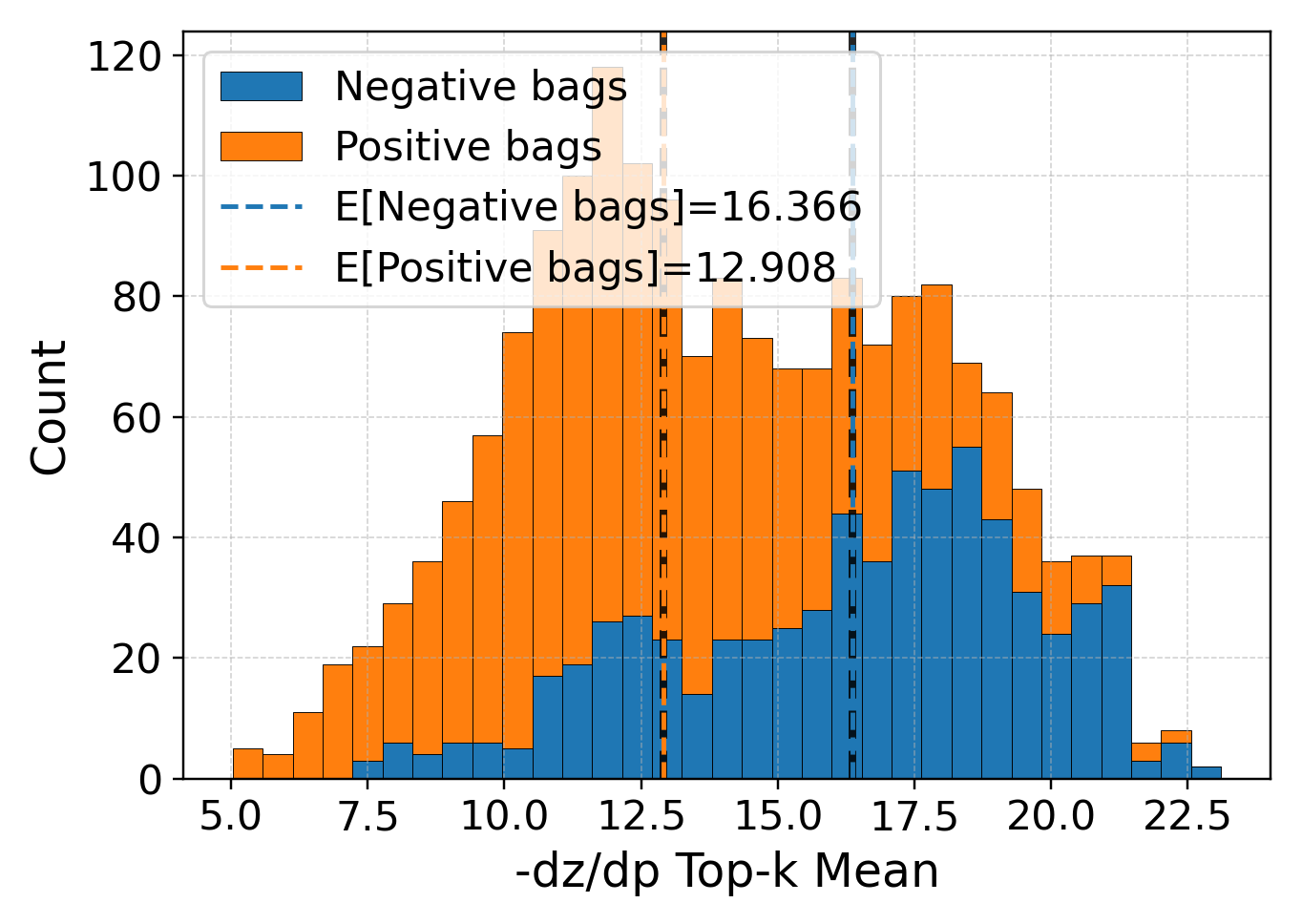}
        \caption{NQ\ (Llama-3.1-8B)}
    \end{subfigure}\hfill
    \begin{subfigure}{0.32\textwidth}
        \includegraphics[width=\linewidth]{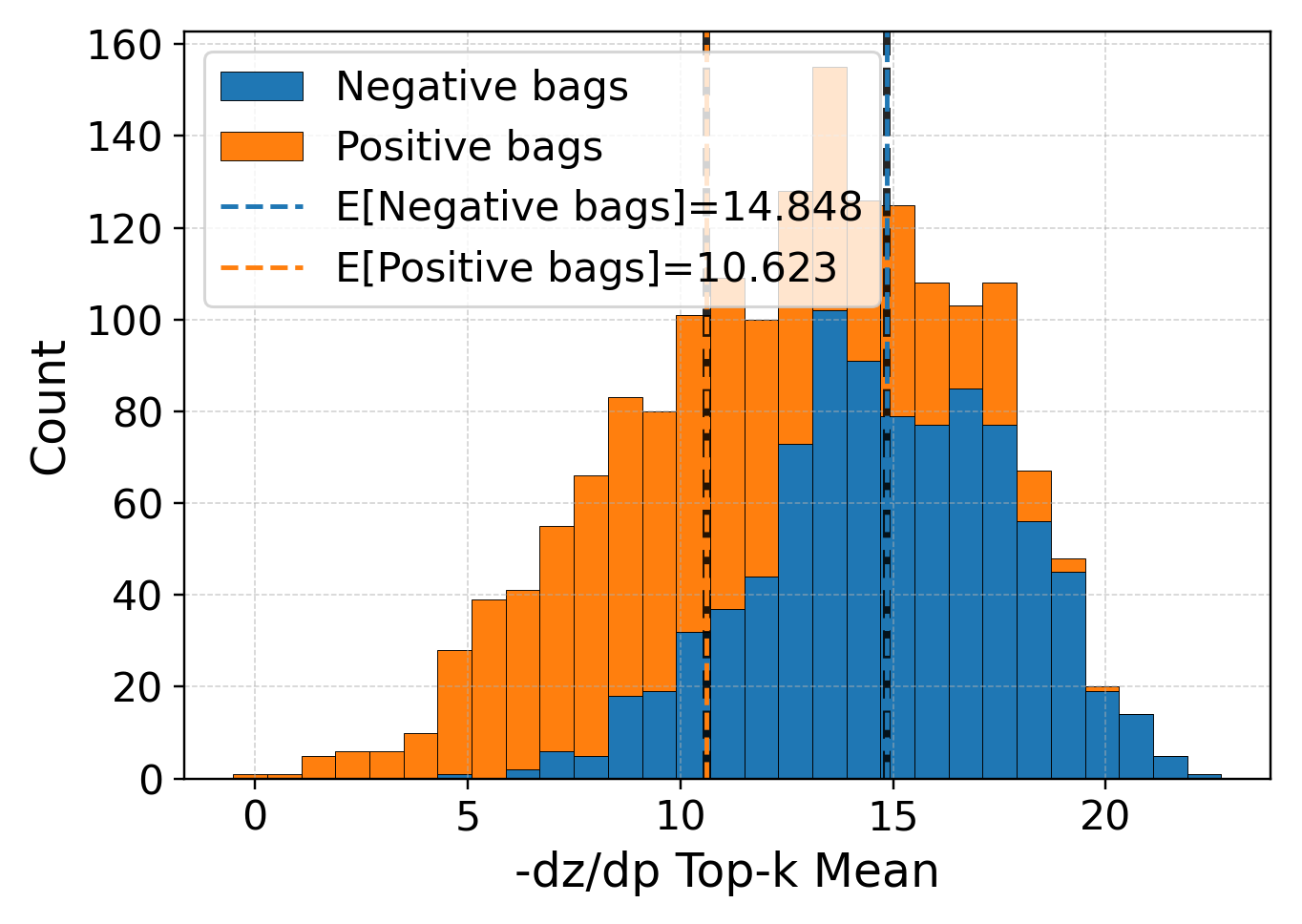}
        \caption{NQ\ (Mistral-12B)}
    \end{subfigure}\hfill
    \begin{subfigure}{0.32\textwidth}
        \includegraphics[width=\linewidth]{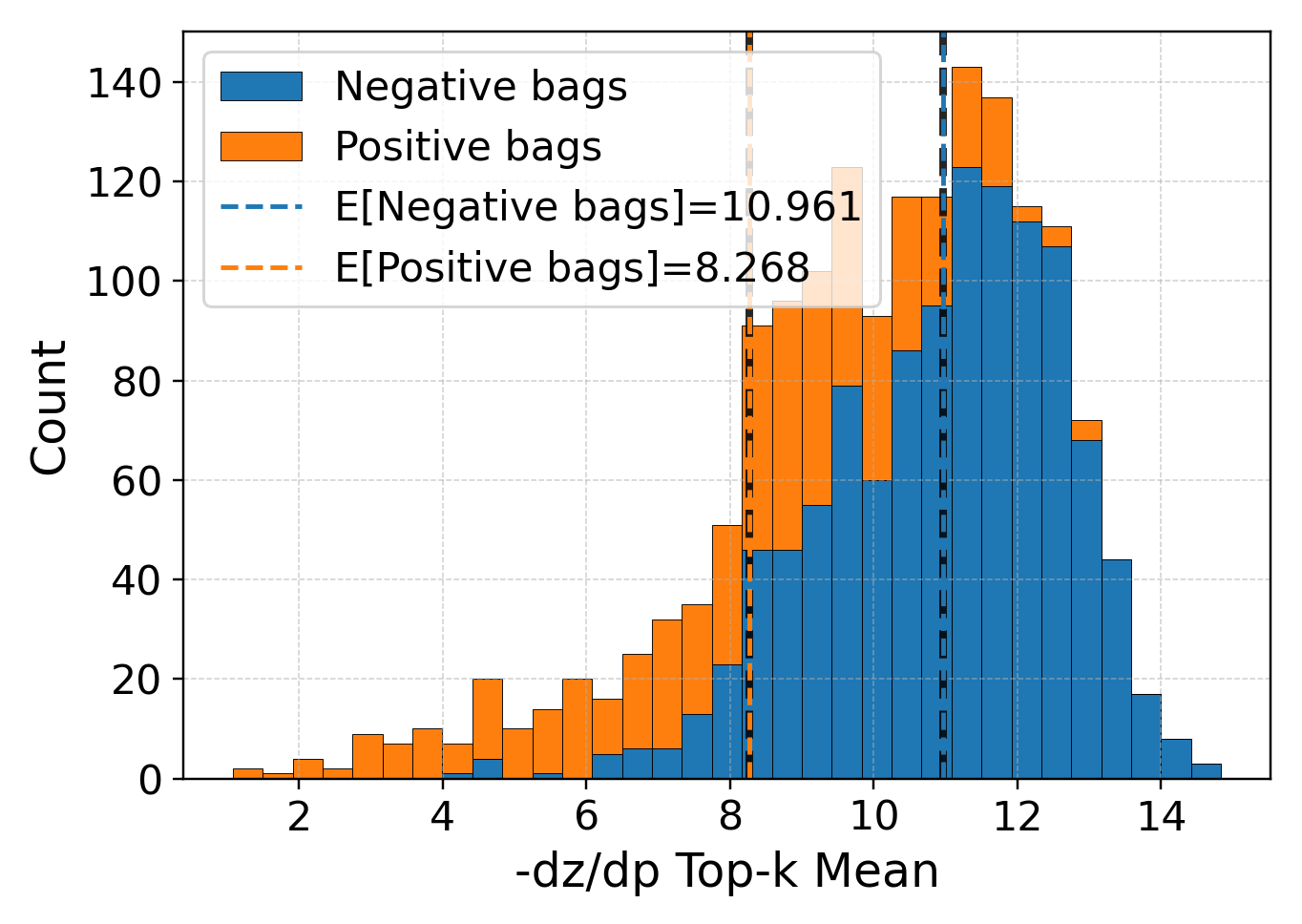}
        \caption{NQ\ (Llama-3.3-70B)}
    \end{subfigure}\hfill
    \vspace{0.3cm}
    \begin{subfigure}{0.32\textwidth}
        \includegraphics[width=\linewidth]{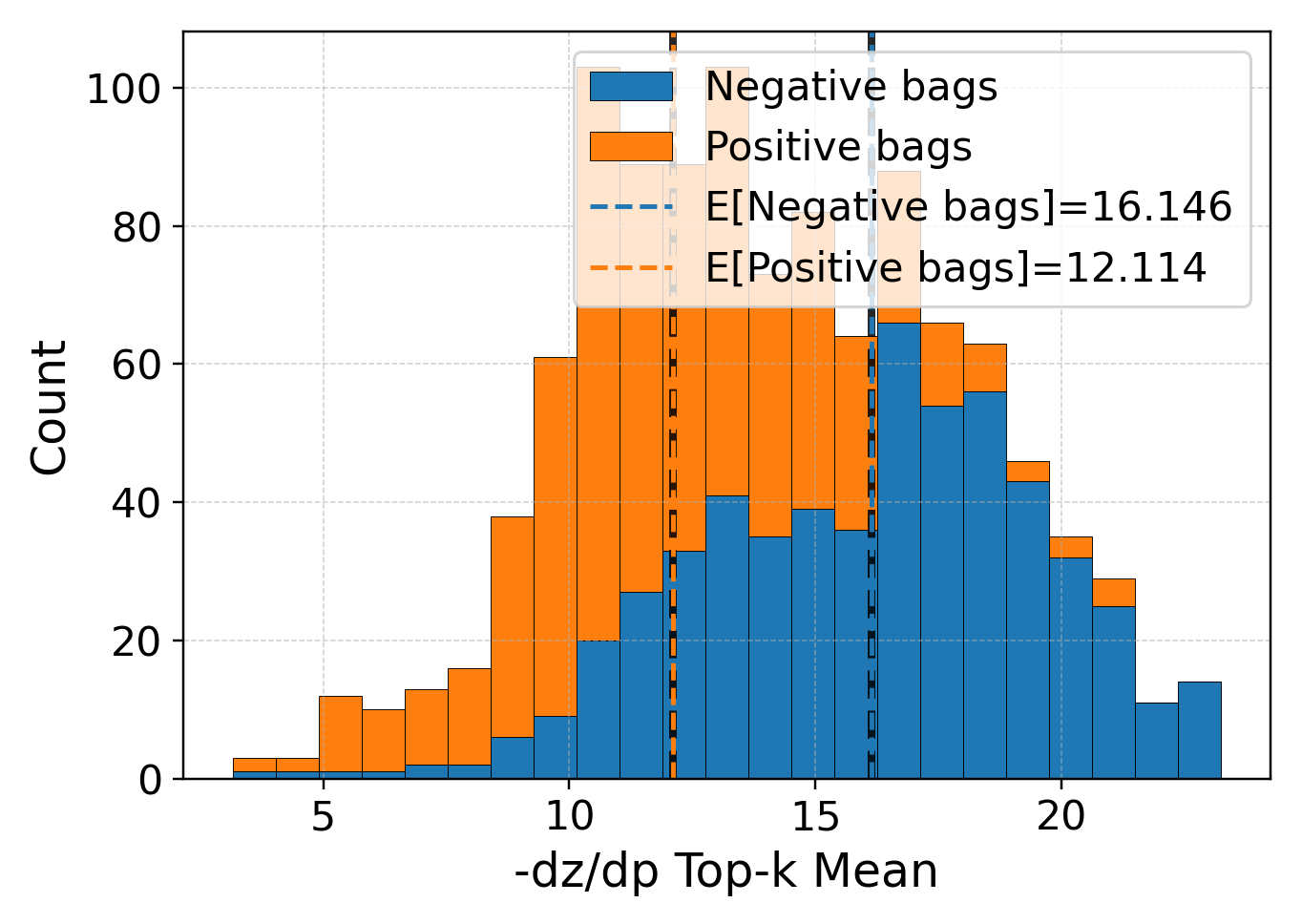}
        \caption{BioASQ\ (Llama-3.1-8B)}
    \end{subfigure}\hfill
    \begin{subfigure}{0.32\textwidth}
        \includegraphics[width=\linewidth]{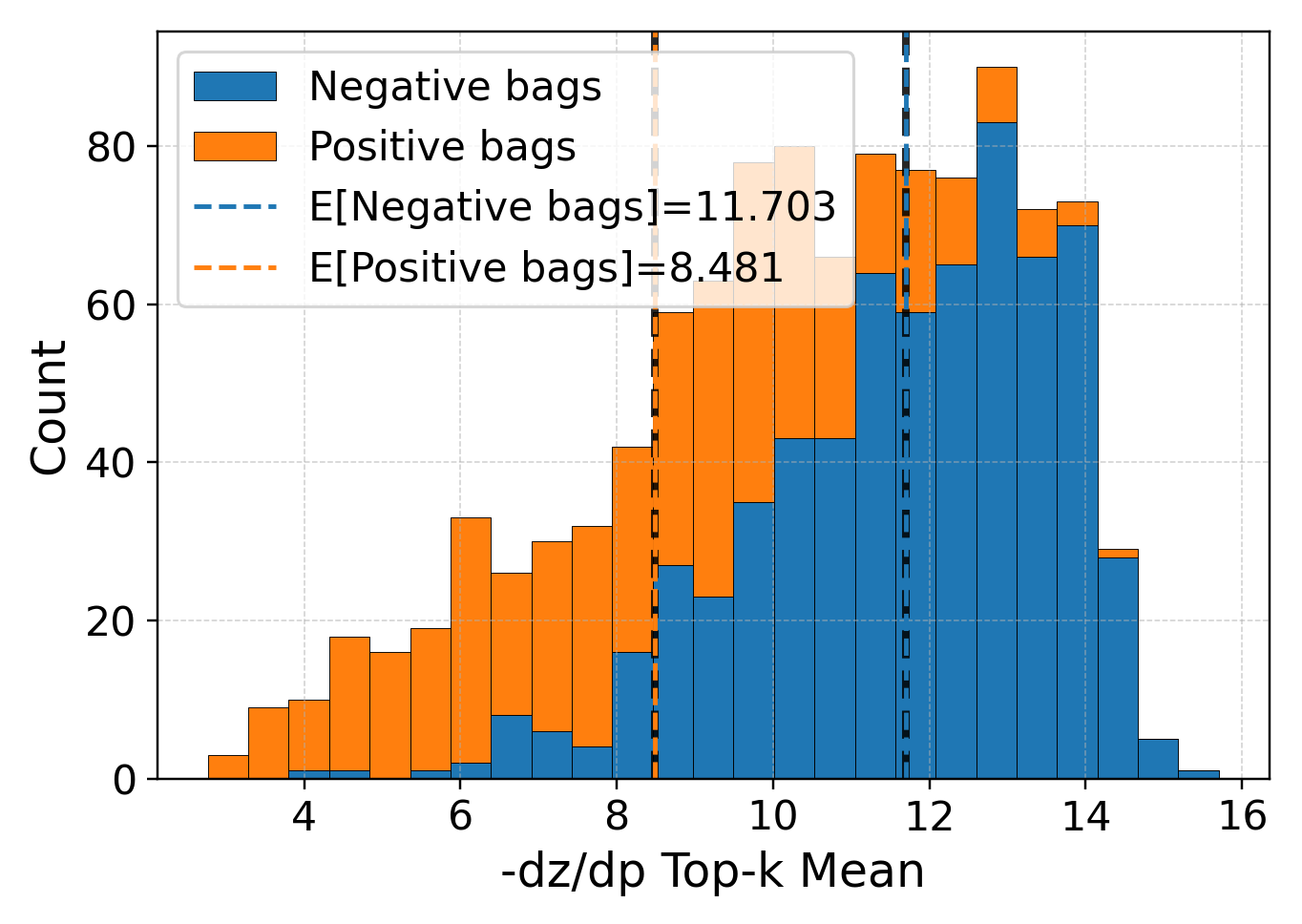}
        \caption{BioASQ\ (Mistral-12B)}
    \end{subfigure}\hfill
    \begin{subfigure}{0.32\textwidth}
        \includegraphics[width=\linewidth]{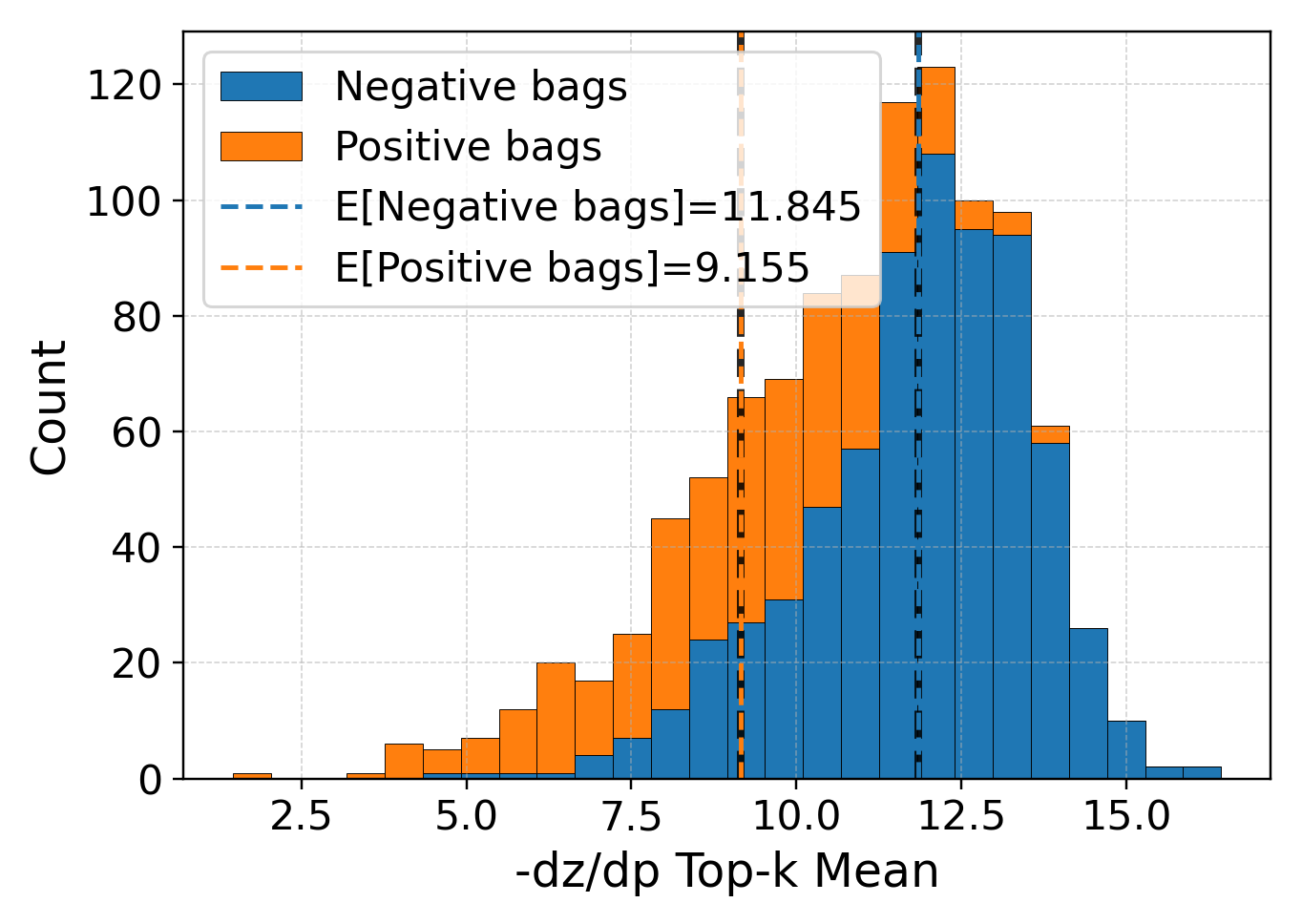}
        \caption{BioASQ\ (Llama-3.3-70B)}
    \end{subfigure}\hfill
    \vspace{0.3cm}
    \caption{Empirical distribution of the bag-level sensitivity to input scaling $\bar{C}_{\mathbf{B}}$ in \textbf{test} data.}
    \label{fig:dist_sensitivity_val}
\end{figure}

\clearpage

\subsection{Distribution of \texorpdfstring{$P_{\mathrm{sem}}^{\mathbf{B}} \bar{C}_{\mathbf{B}}$}{P-sem C-bar}}
We computed the joint product $P_{\mathrm{sem}}^{\mathbf{B}} \bar{C}_{\mathbf{B}}$ for each bag. Figures~\ref{fig:dist_joint_train} and \ref{fig:dist_joint_val} illustrate the empirical distributions of this product for the training and test sets, respectively. Although the distributions of individual bags overlap, the negative class has a higher density at larger values, and the class-wise expected values satisfy the inequality $\mathbb{E}_{neg}[P_{\mathrm{sem}}^{\mathbf{B}} \bar{C}_{\mathbf{B}}] > \mathbb{E}_{pos}[P_{\mathrm{sem}}^{\mathbf{B}} \bar{C}_{\mathbf{B}}]$ across all evaluated datasets and across both data splits. The expected values for each condition are provided in Appendix \ref{app:expected_values_summary}. 

\begin{figure}[H]
    \centering
    \begin{subfigure}{0.32\textwidth}
        \includegraphics[width=\linewidth]{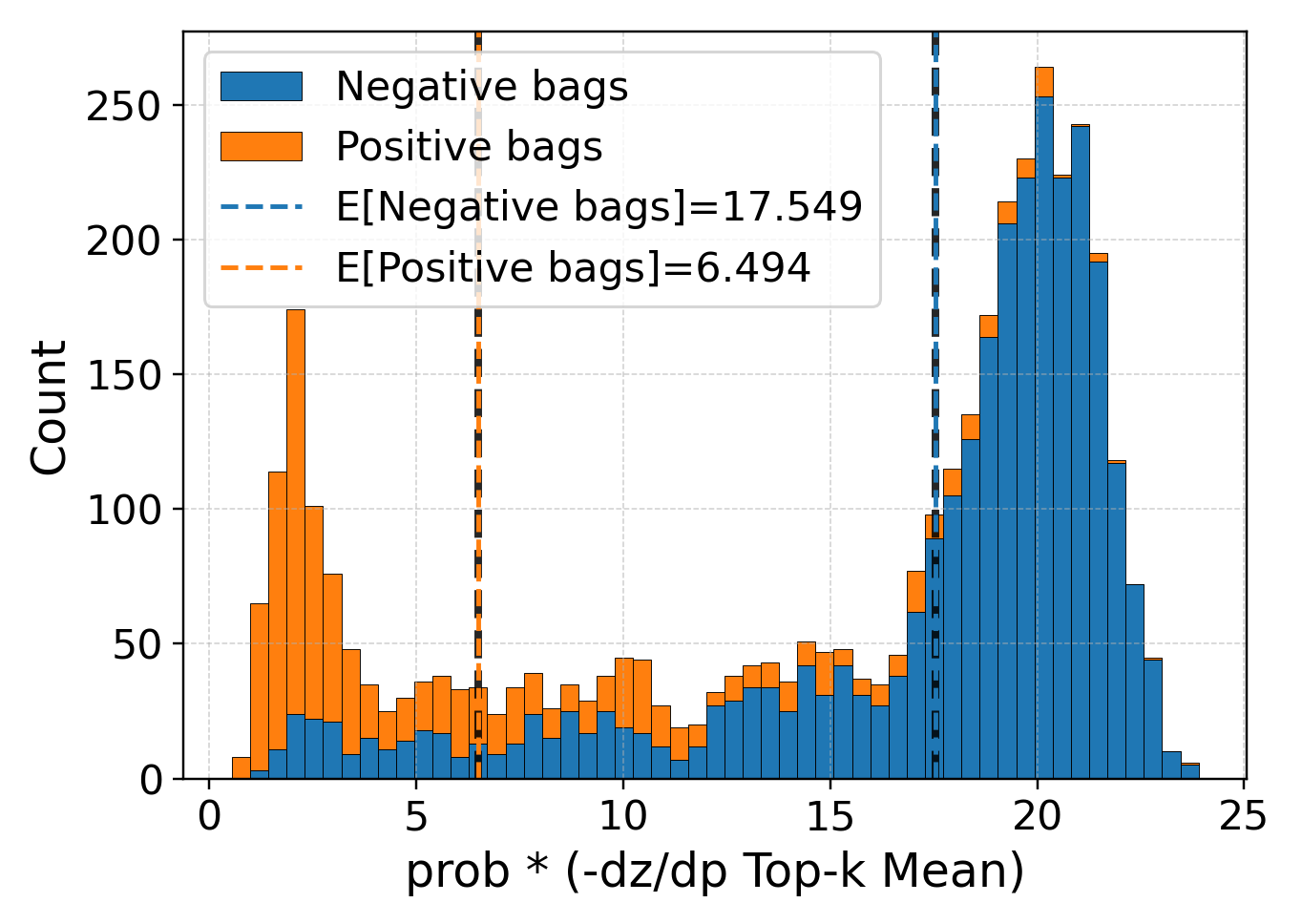}
        \caption{TriviaQA\ (Llama-3.1-8B)}
    \end{subfigure}\hfill
    \begin{subfigure}{0.32\textwidth}
        \includegraphics[width=\linewidth]{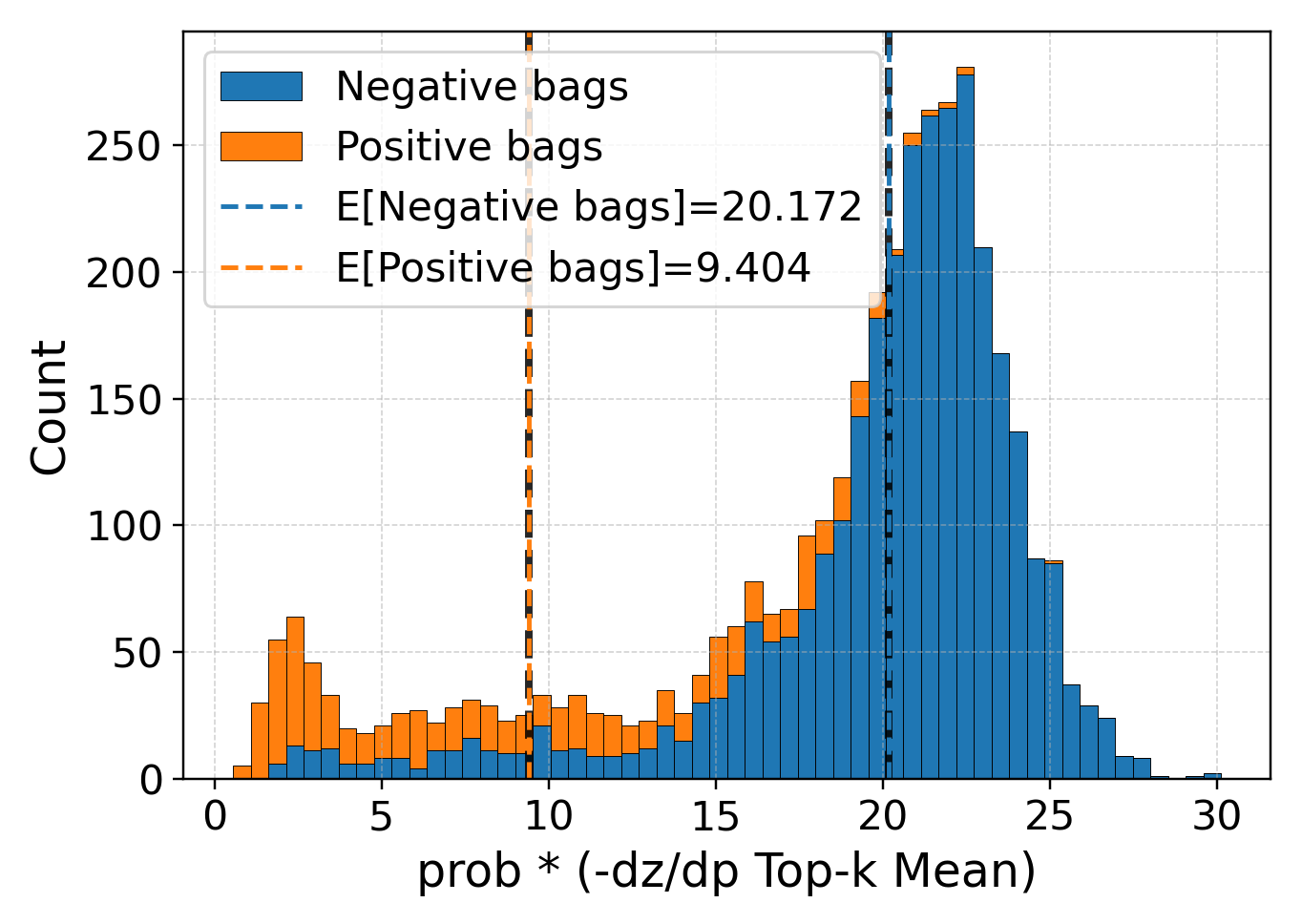}
        \caption{TriviaQA\ (Mistral-12B)}
    \end{subfigure}\hfill
    \begin{subfigure}{0.32\textwidth}
        \includegraphics[width=\linewidth]{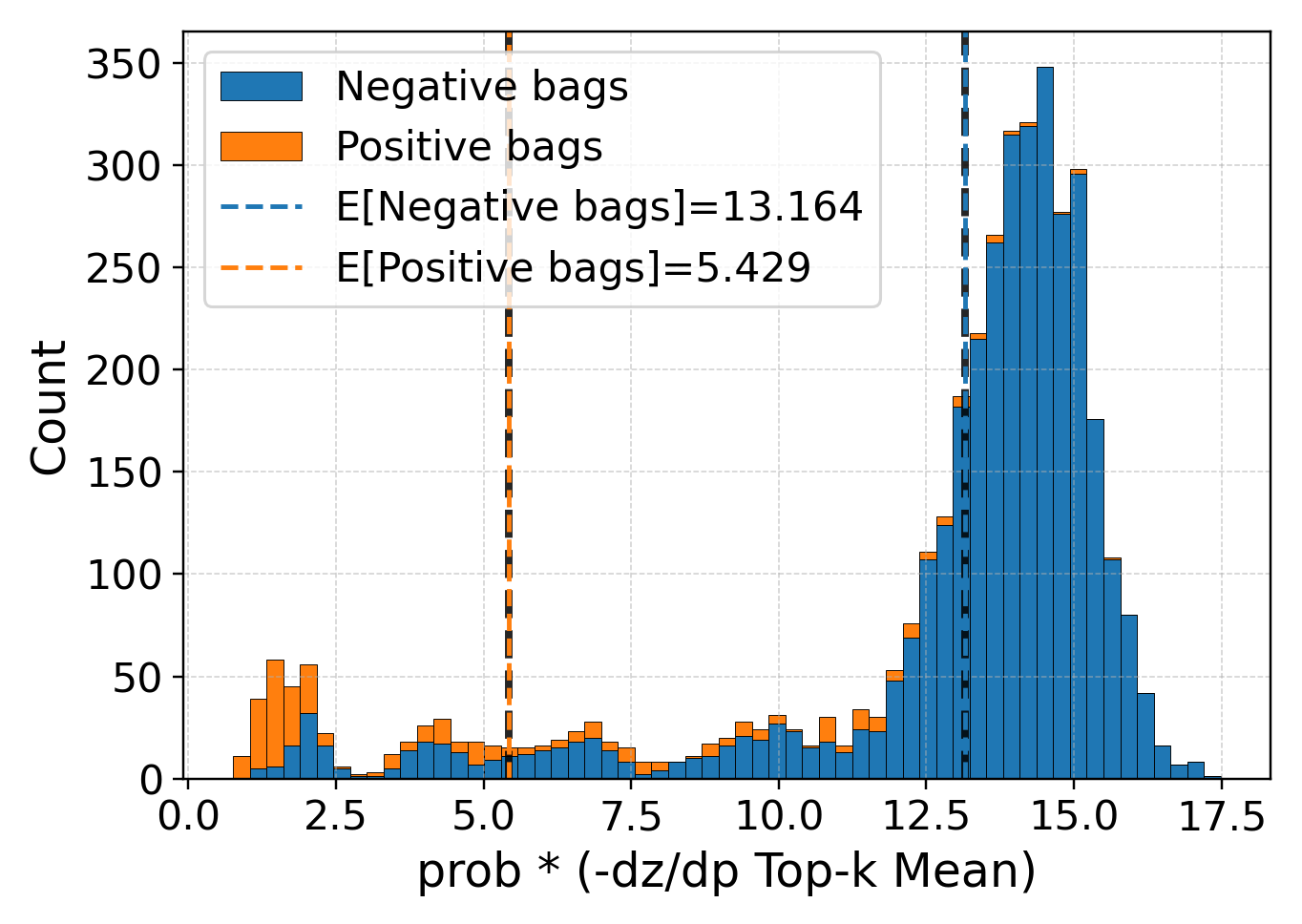}
        \caption{TriviaQA\ (Llama-3.3-70B)}
    \end{subfigure}\hfill
    \vspace{0.3cm}
    \begin{subfigure}{0.32\textwidth}
        \includegraphics[width=\linewidth]{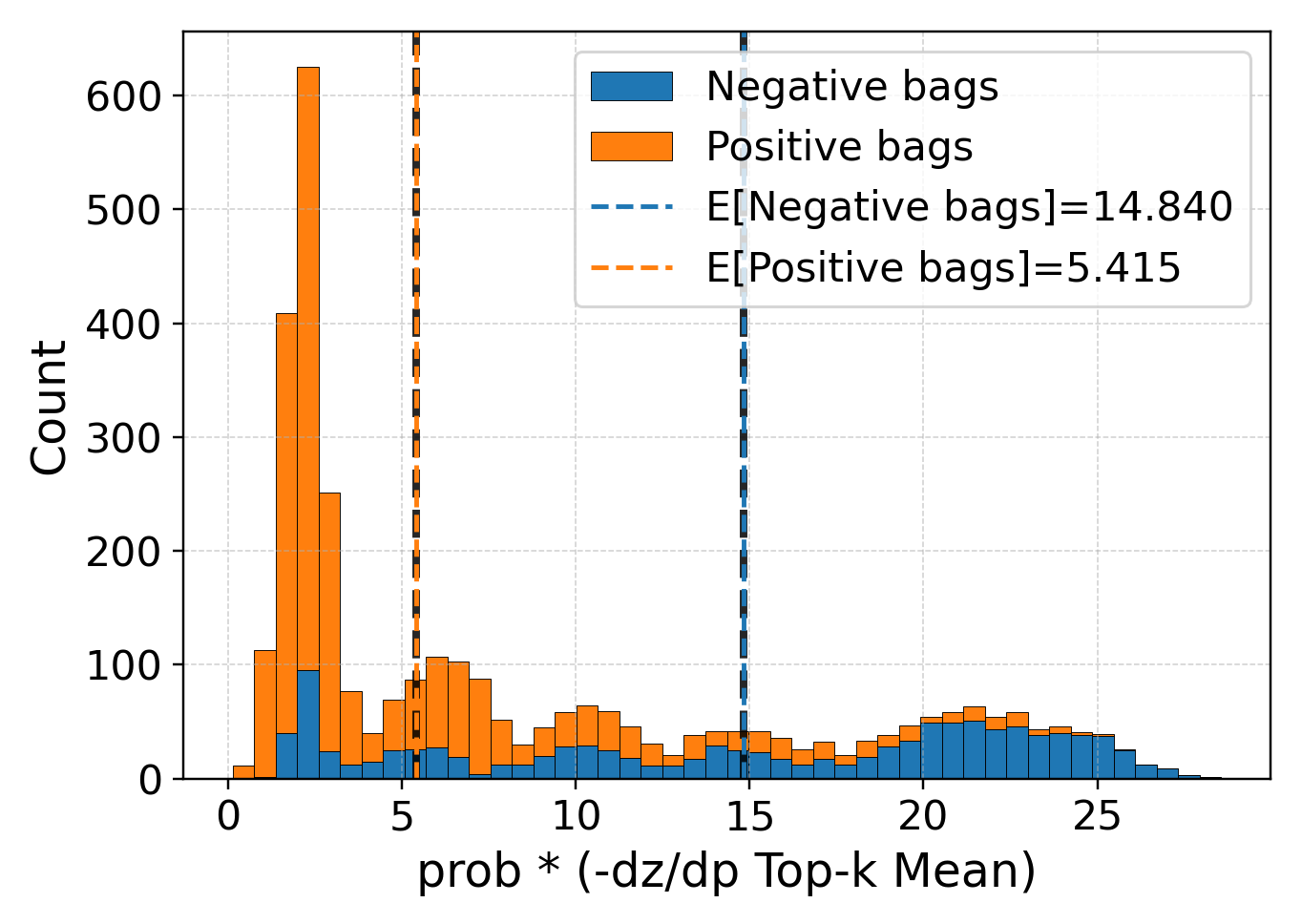}
        \caption{SQuAD\ (Llama-3.1-8B)}
    \end{subfigure}\hfill
    \begin{subfigure}{0.32\textwidth}
        \includegraphics[width=\linewidth]{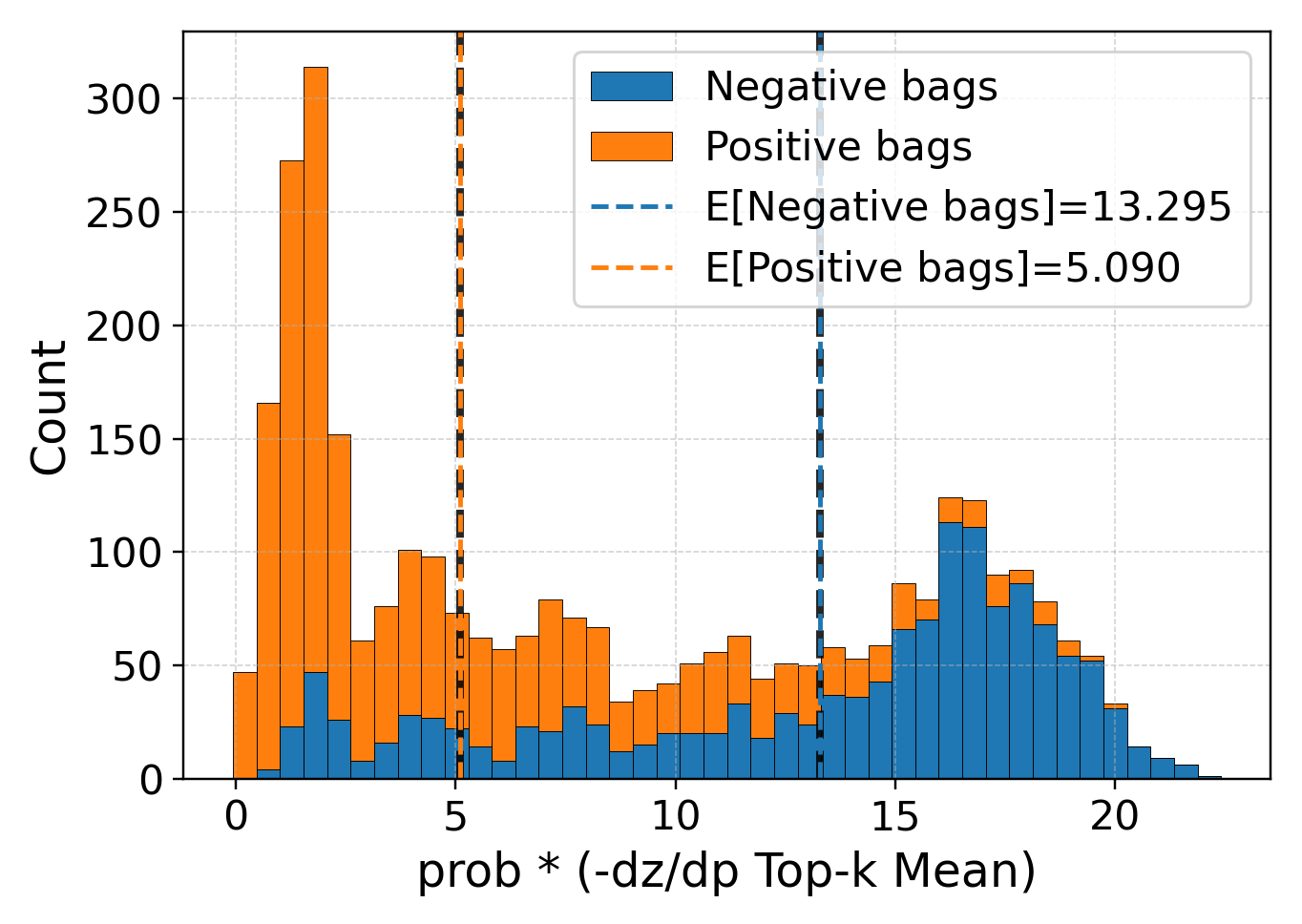}
        \caption{SQuAD\ (Mistral-12B)}
    \end{subfigure}\hfill
    \begin{subfigure}{0.32\textwidth}
        \includegraphics[width=\linewidth]{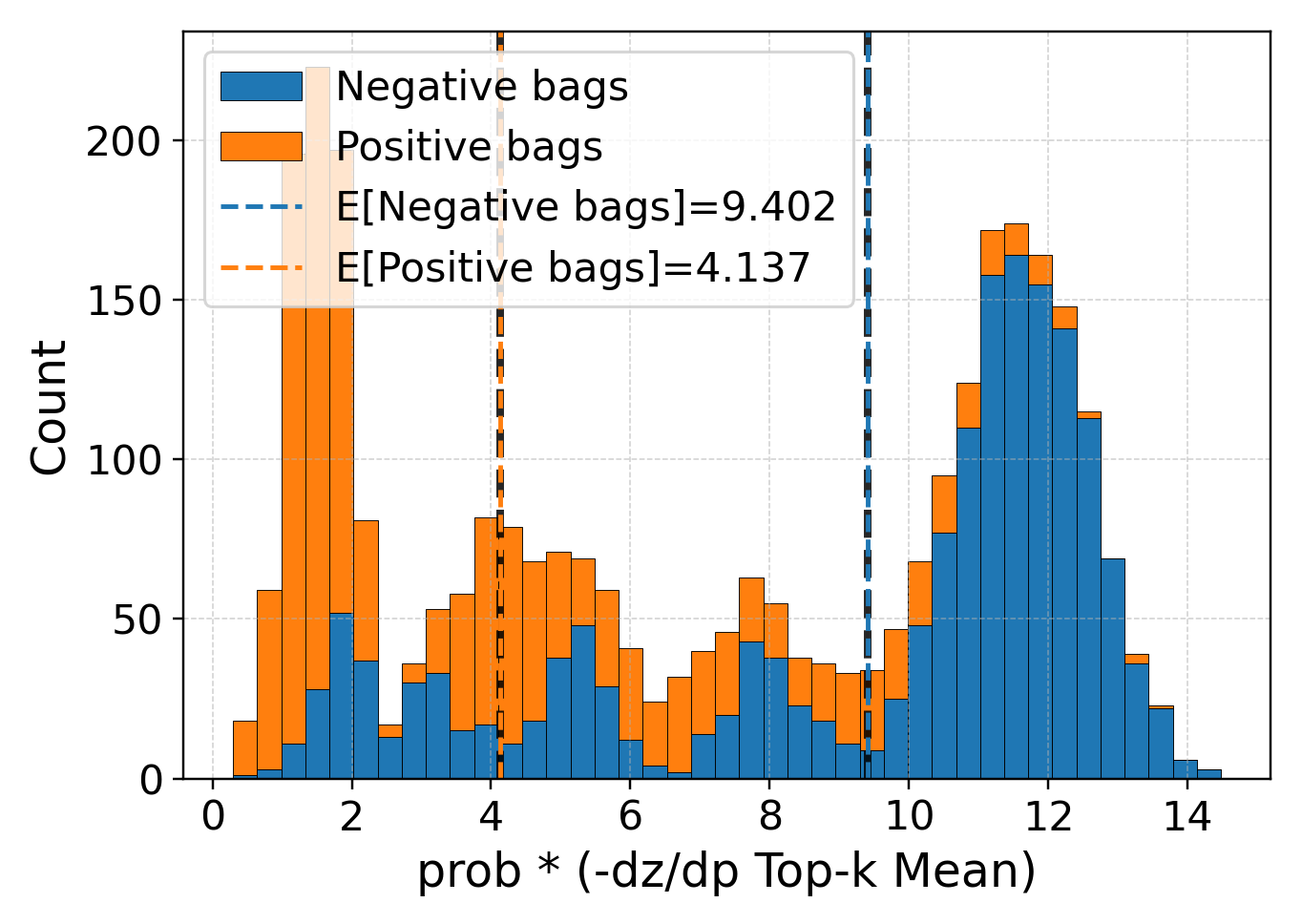}
        \caption{SQuAD\ (Llama-3.3-70B)}
    \end{subfigure}\hfill
    \vspace{0.3cm}
    \begin{subfigure}{0.32\textwidth}
        \includegraphics[width=\linewidth]{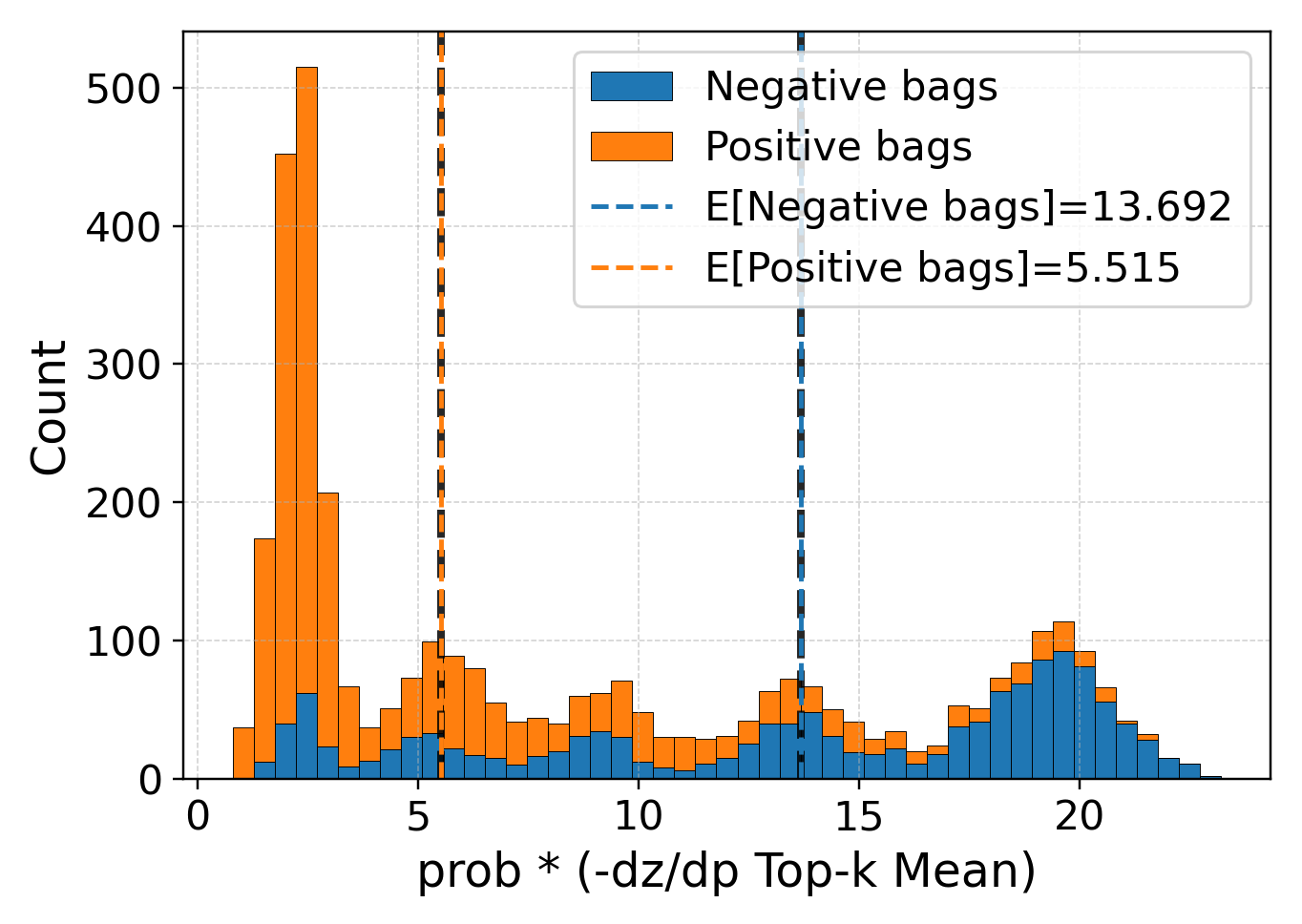}
        \caption{NQ\ (Llama-3.1-8B)}
    \end{subfigure}\hfill
    \begin{subfigure}{0.32\textwidth}
        \includegraphics[width=\linewidth]{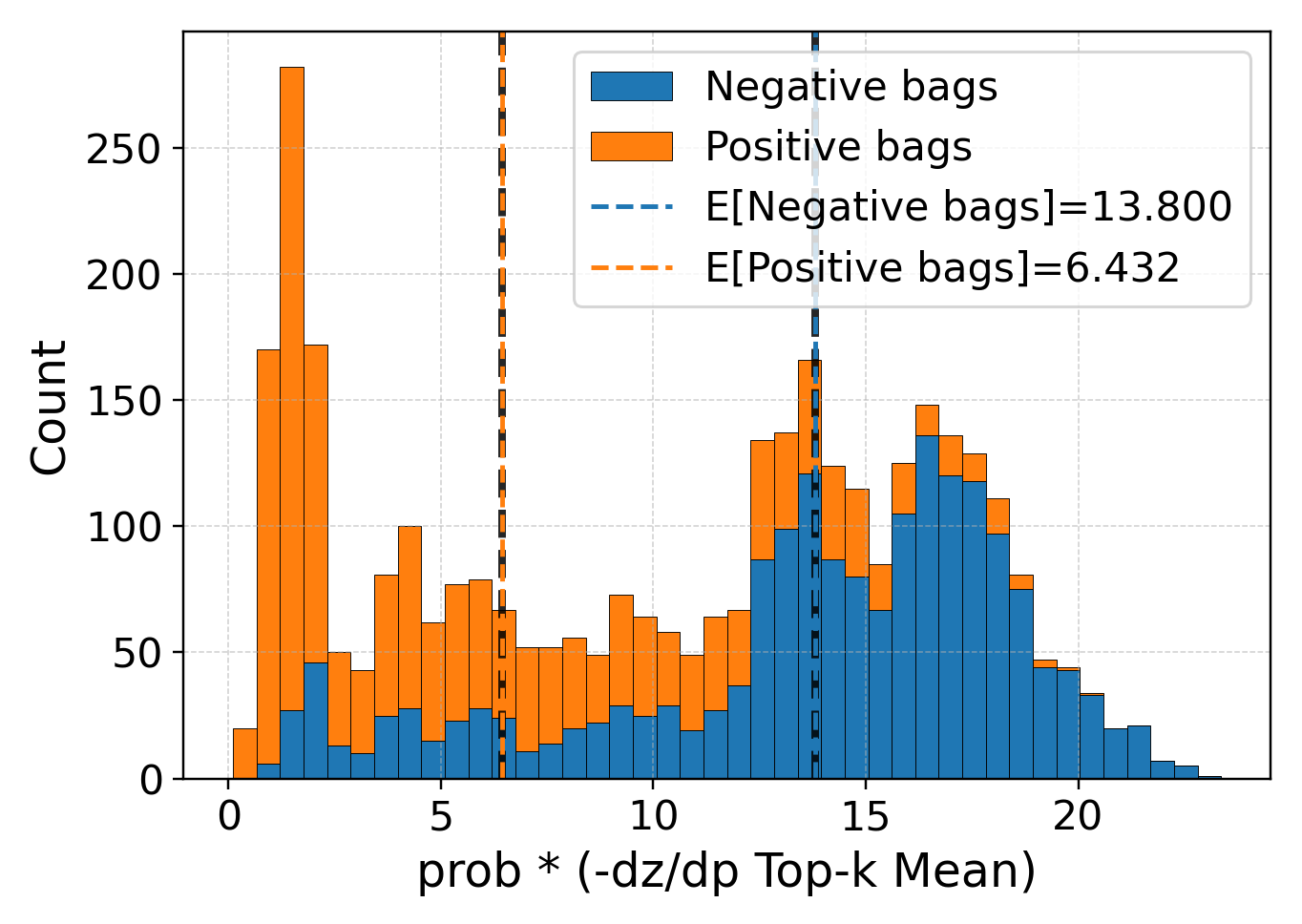}
        \caption{NQ\ (Mistral-12B)}
    \end{subfigure}\hfill
    \begin{subfigure}{0.32\textwidth}
        \includegraphics[width=\linewidth]{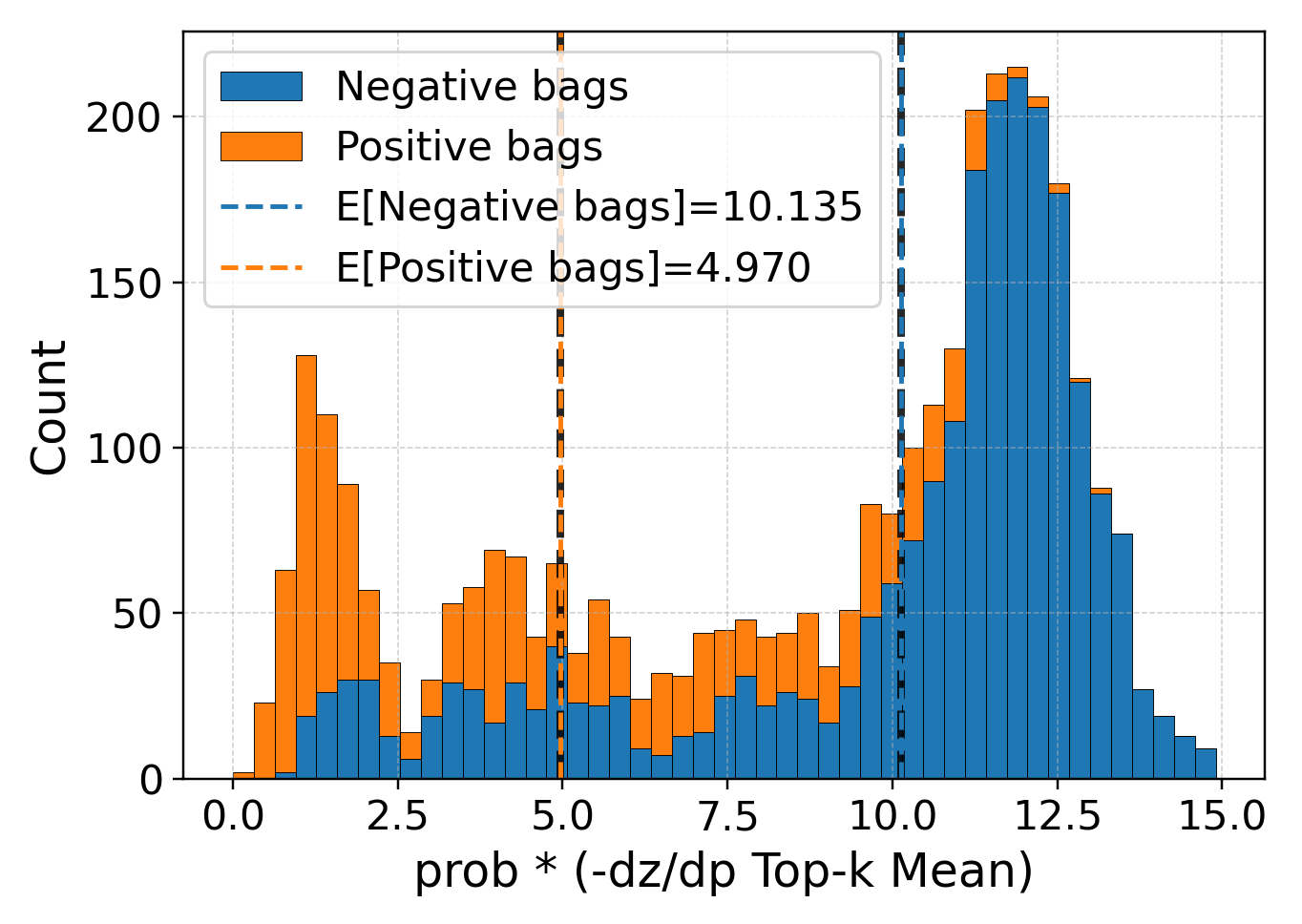}
        \caption{NQ\ (Llama-3.3-70B)}
    \end{subfigure}\hfill
    \vspace{0.3cm}
    \begin{subfigure}{0.32\textwidth}
        \includegraphics[width=\linewidth]{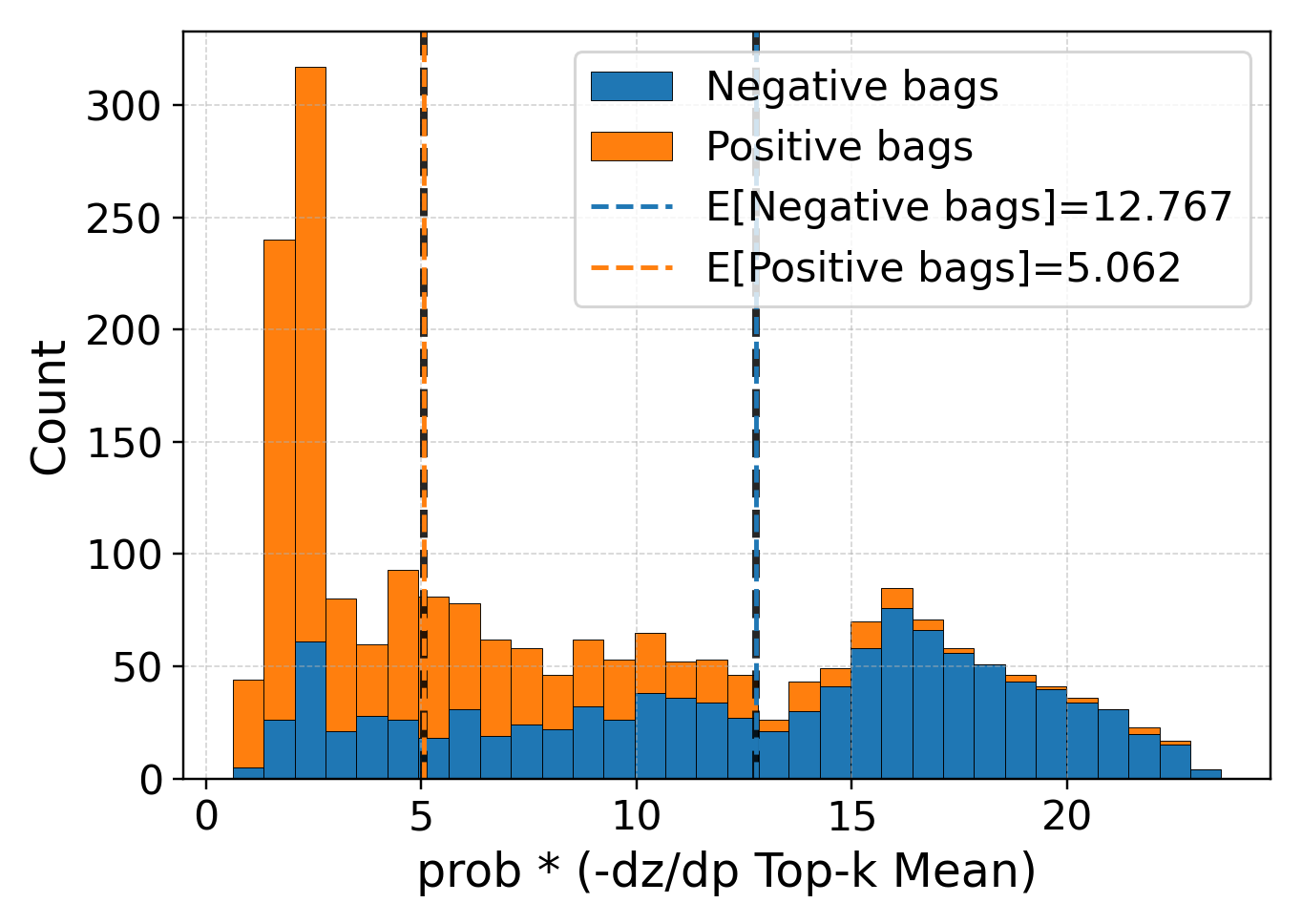}
        \caption{BioASQ\ (Llama-3.1-8B)}
    \end{subfigure}\hfill
    \begin{subfigure}{0.32\textwidth}
        \includegraphics[width=\linewidth]{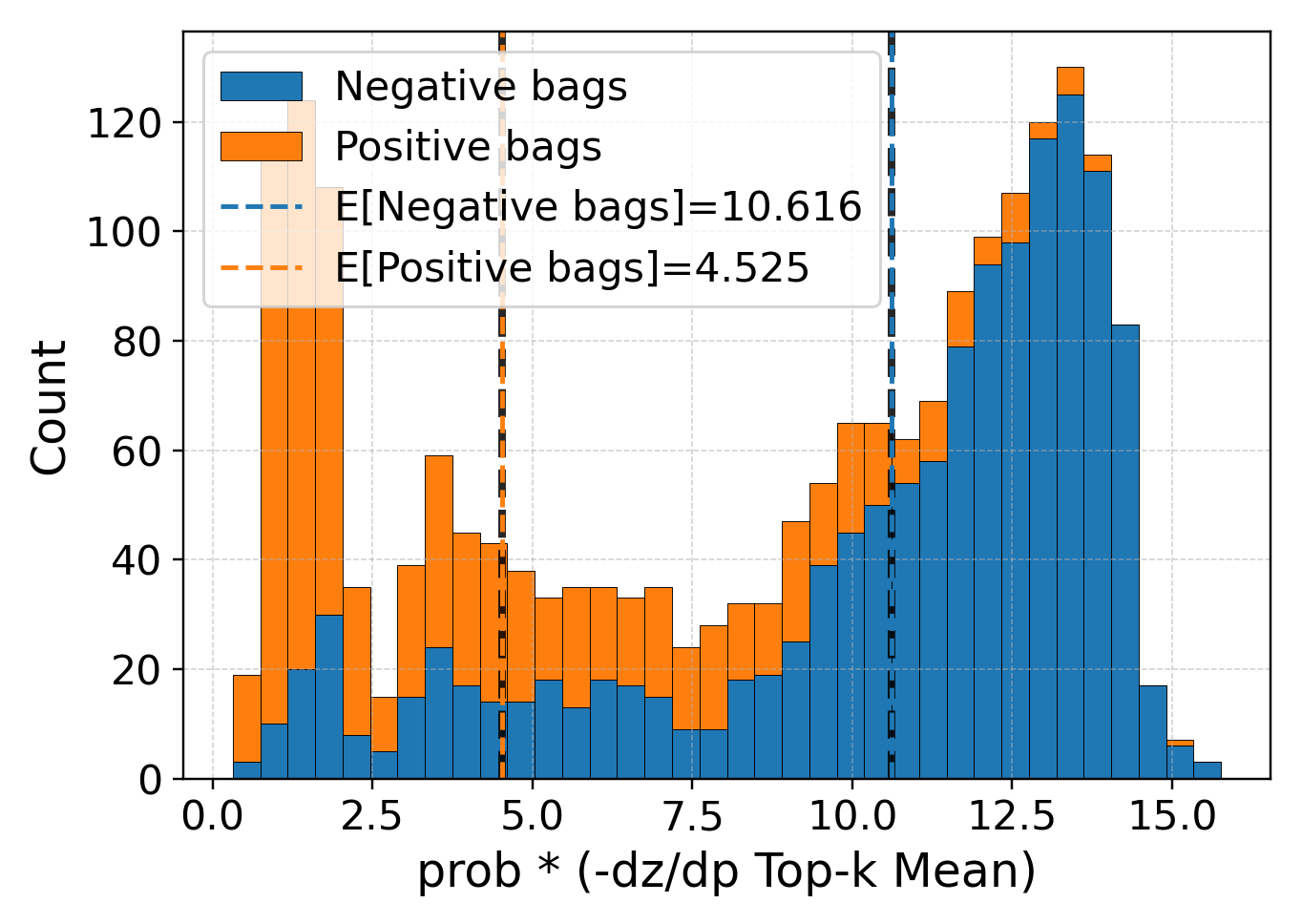}
        \caption{BioASQ\ (Mistral-12B)}
    \end{subfigure}\hfill
    \begin{subfigure}{0.32\textwidth}
        \includegraphics[width=\linewidth]{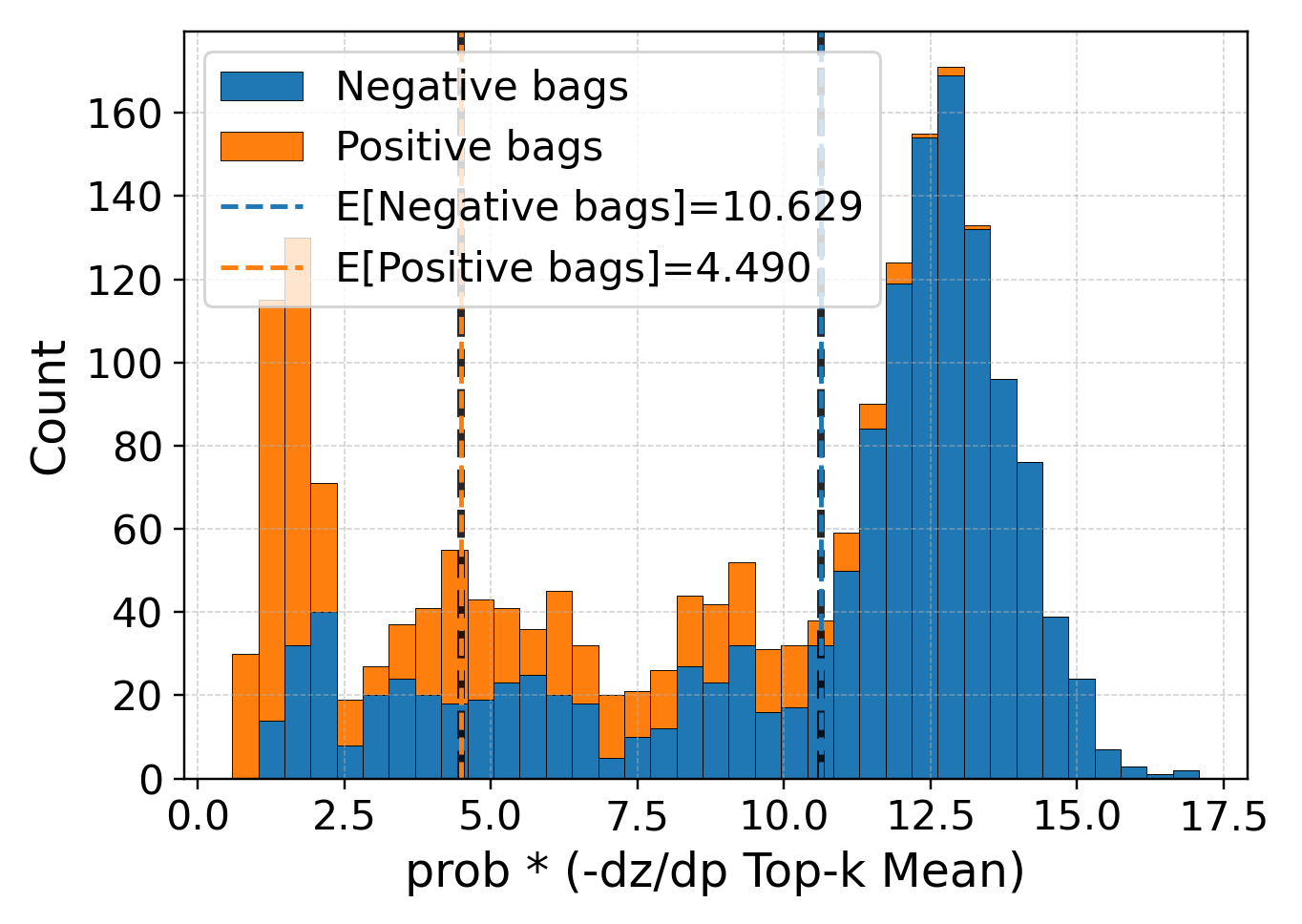}
        \caption{BioASQ\ (Llama-3.3-70B)}
    \end{subfigure}\hfill
    \vspace{0.3cm}
    \caption{Distribution of the joint product $P_{\mathrm{sem}}^{\mathbf{B}} \bar{C}_{\mathbf{B}}$ in \textbf{train} data.}
    \label{fig:dist_joint_train}
\end{figure}

\begin{figure}[H]
    \centering
    \begin{subfigure}{0.32\textwidth}
        \includegraphics[width=\linewidth]{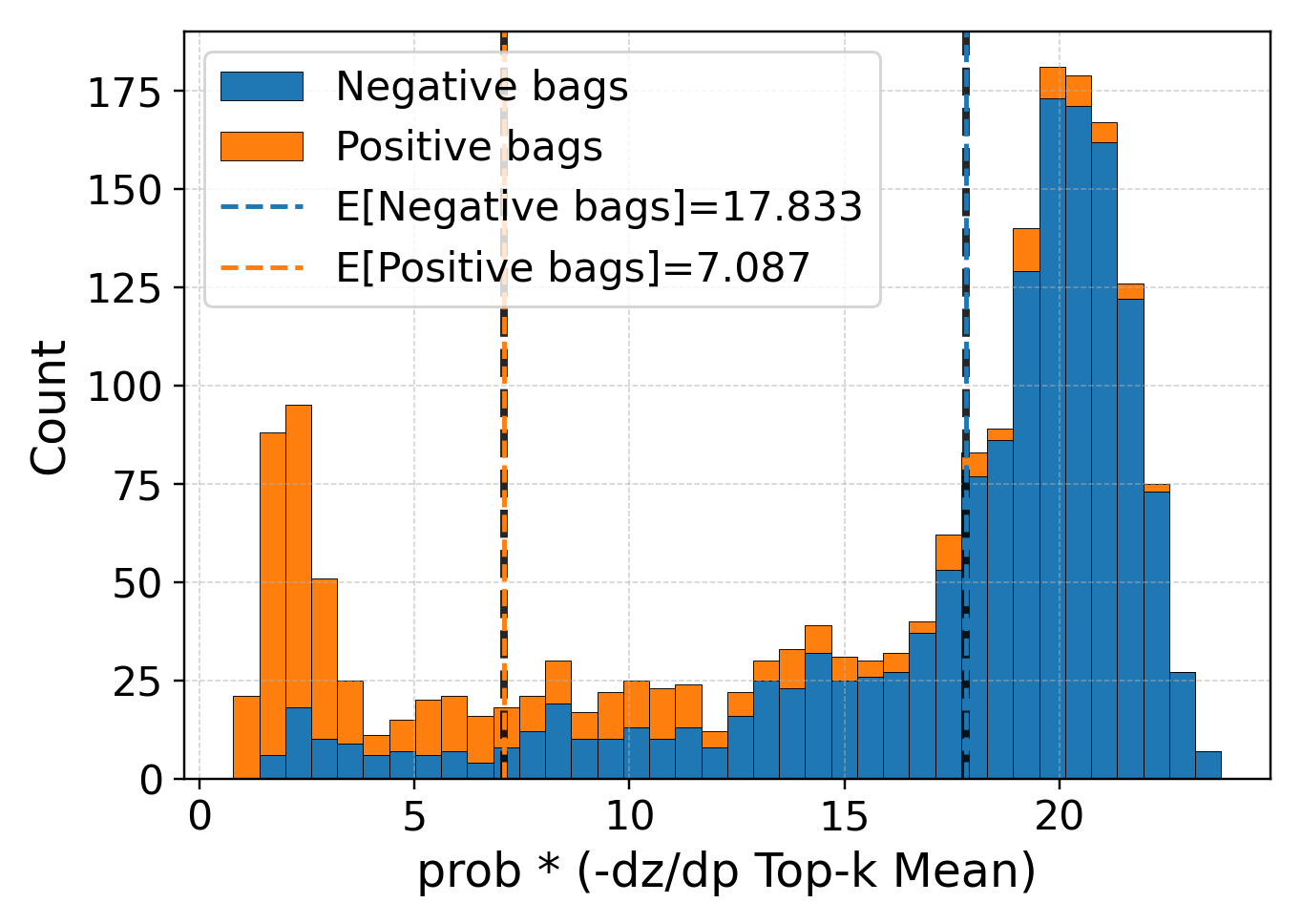}
        \caption{TriviaQA\ (Llama-3.1-8B)}
    \end{subfigure}\hfill
    \begin{subfigure}{0.32\textwidth}
        \includegraphics[width=\linewidth]{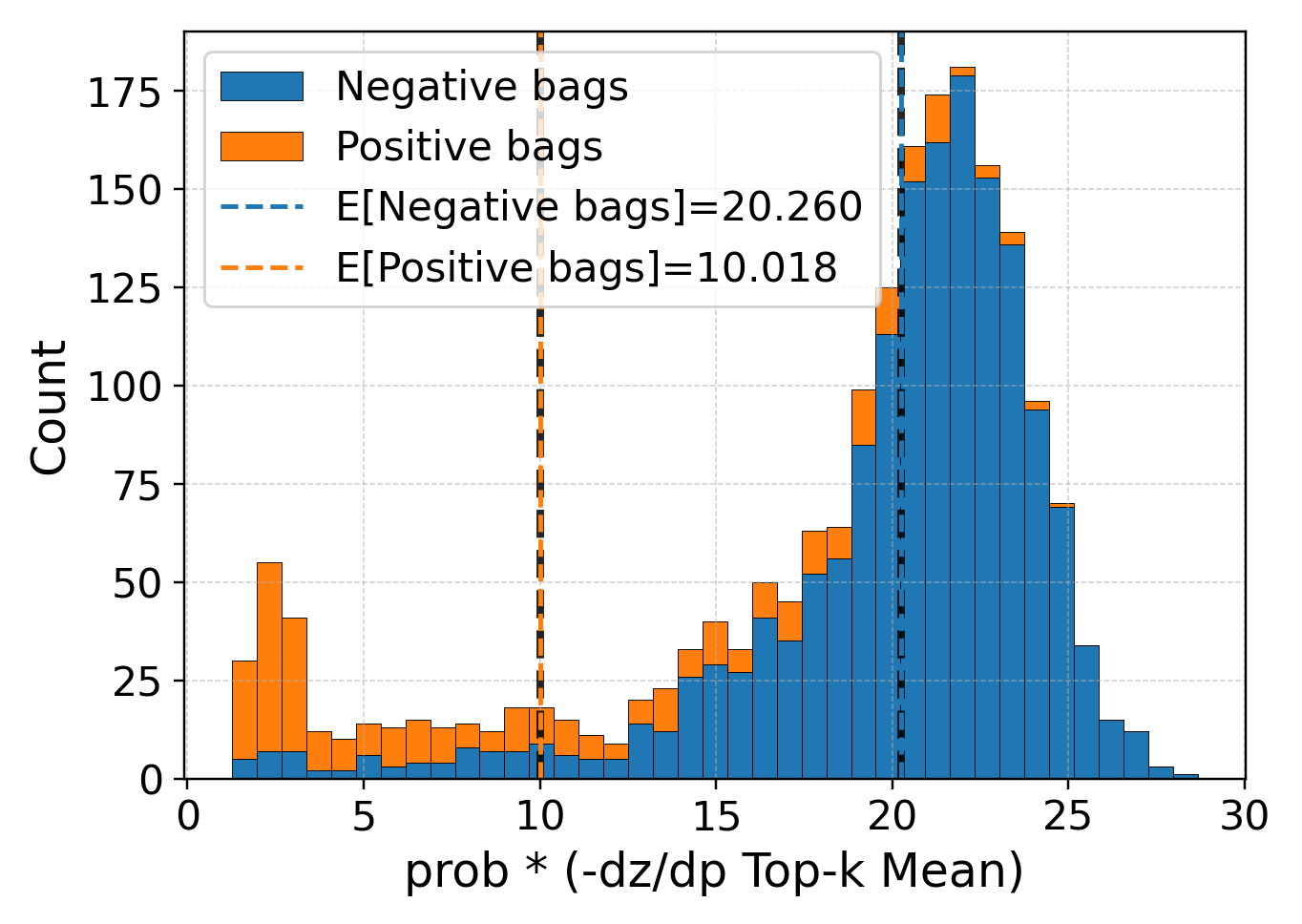}
        \caption{TriviaQA\ (Mistral-12B)}
    \end{subfigure}\hfill
    \begin{subfigure}{0.32\textwidth}
        \includegraphics[width=\linewidth]{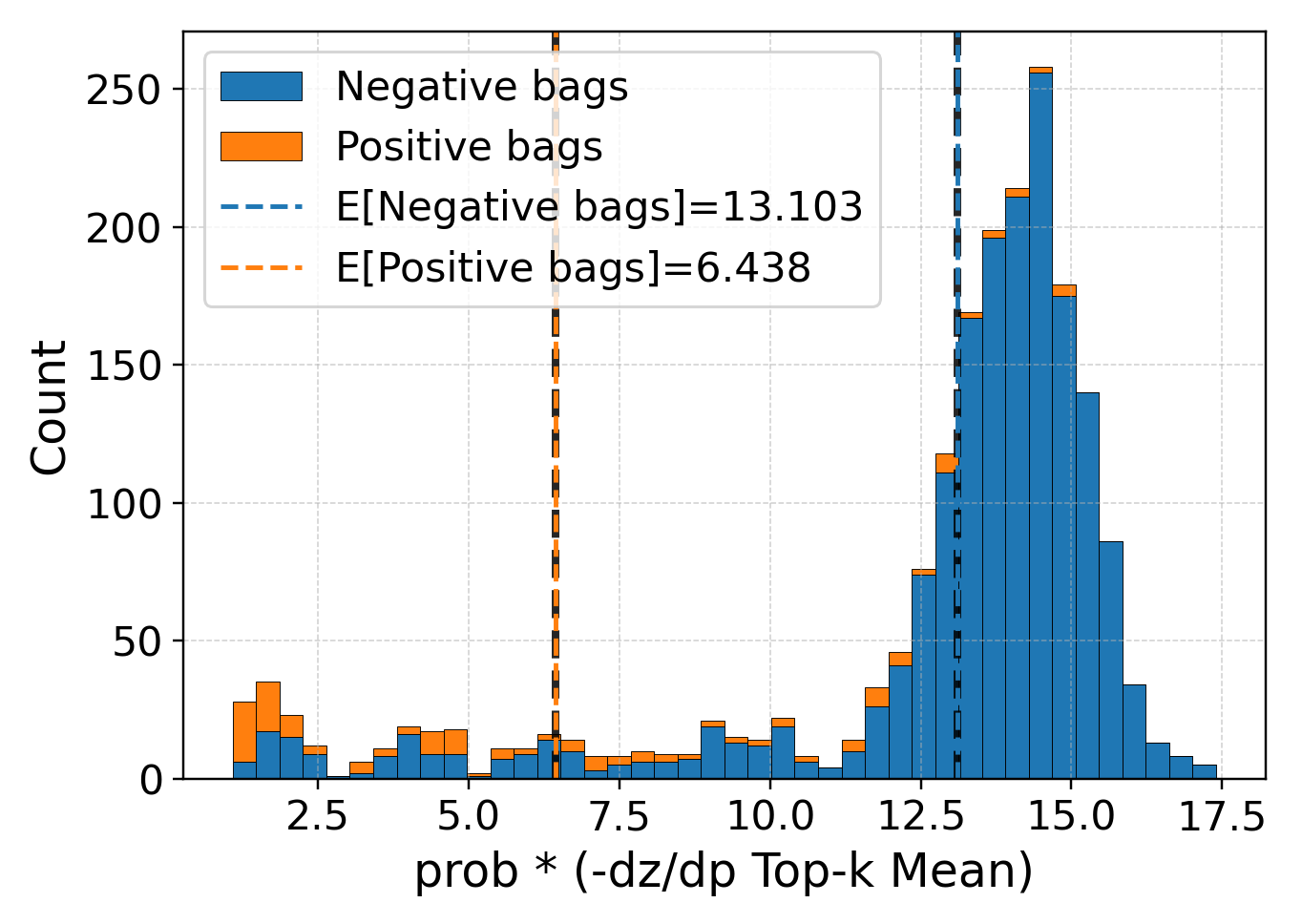}
        \caption{TriviaQA\ (Llama-3.3-70B)}
    \end{subfigure}\hfill
    \vspace{0.3cm}
    \begin{subfigure}{0.32\textwidth}
        \includegraphics[width=\linewidth]{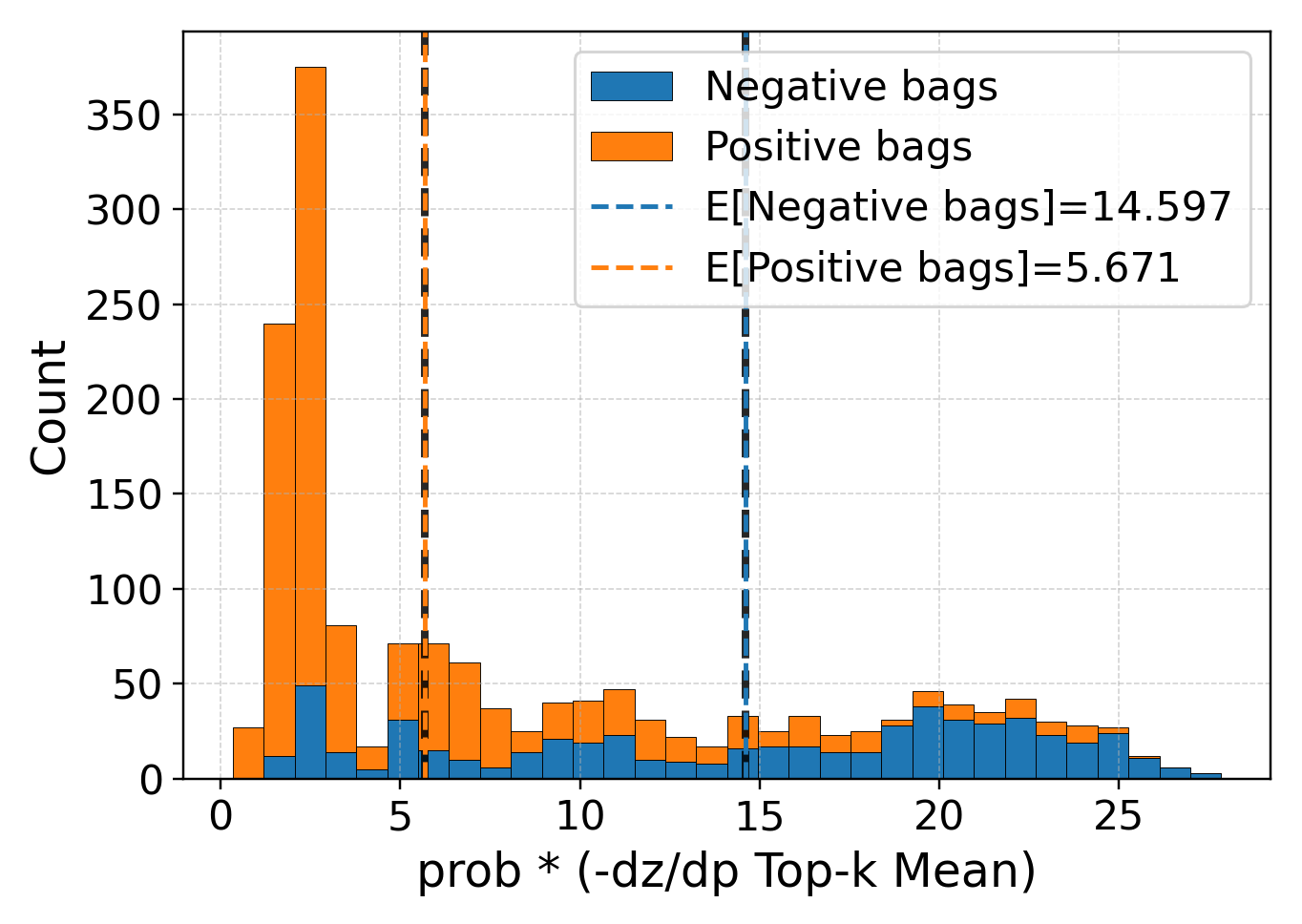}
        \caption{SQuAD\ (Llama-3.1-8B)}
    \end{subfigure}\hfill
    \begin{subfigure}{0.32\textwidth}
        \includegraphics[width=\linewidth]{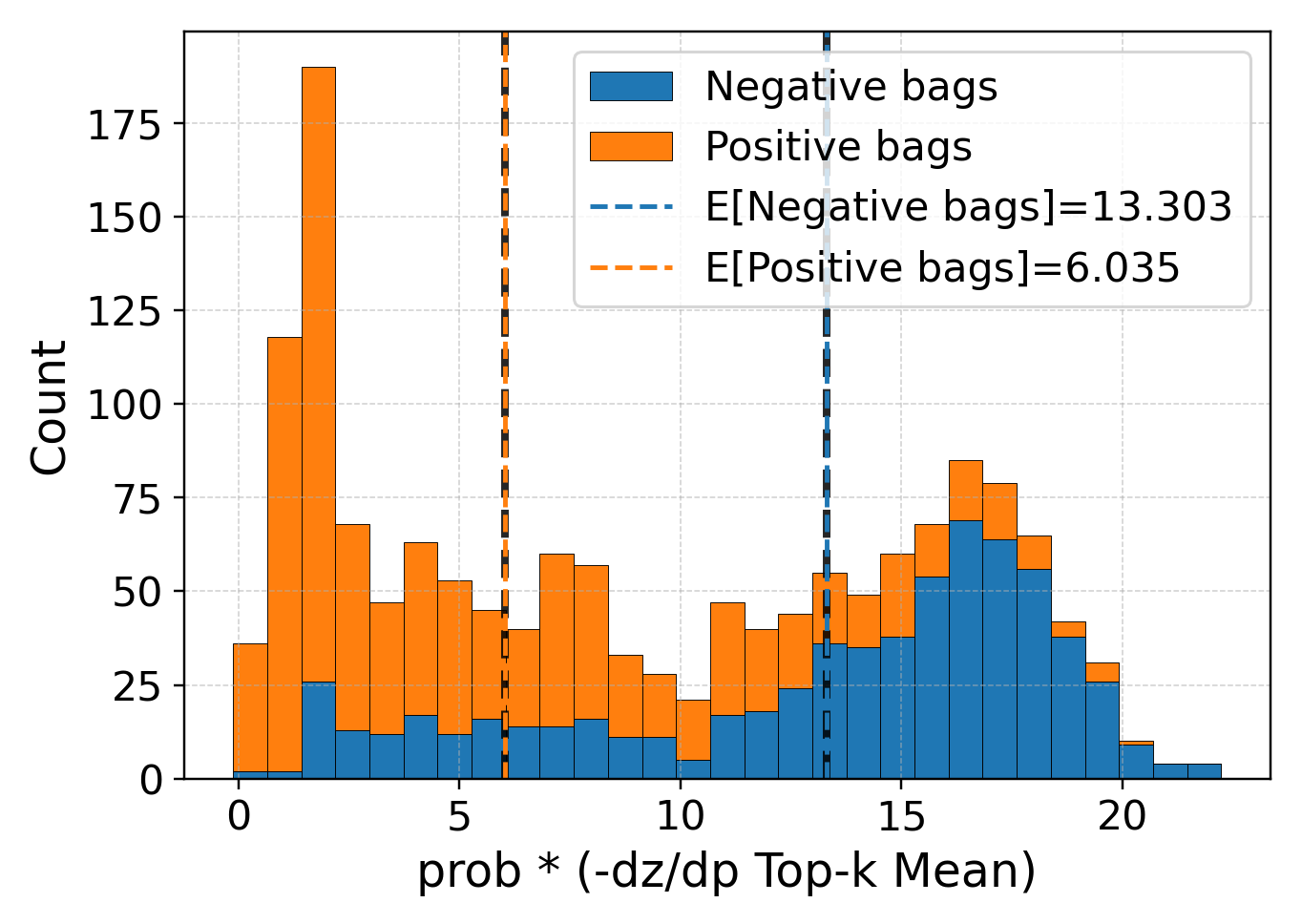}
        \caption{SQuAD\ (Mistral-12B)}
    \end{subfigure}\hfill
    \begin{subfigure}{0.32\textwidth}
        \includegraphics[width=\linewidth]{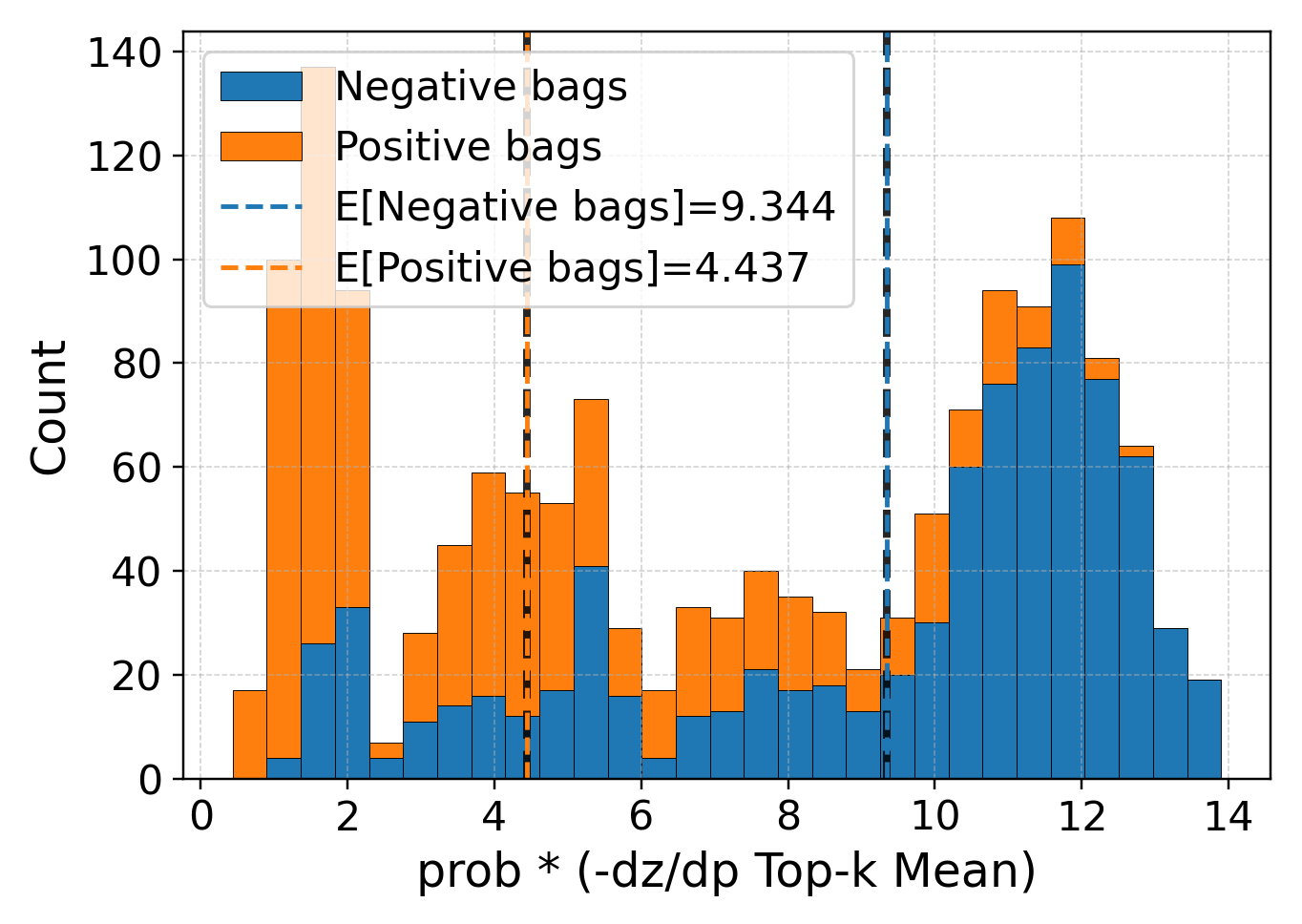}
        \caption{SQuAD\ (Llama-3.3-70B)}
    \end{subfigure}\hfill
    \vspace{0.3cm}
    \begin{subfigure}{0.32\textwidth}
        \includegraphics[width=\linewidth]{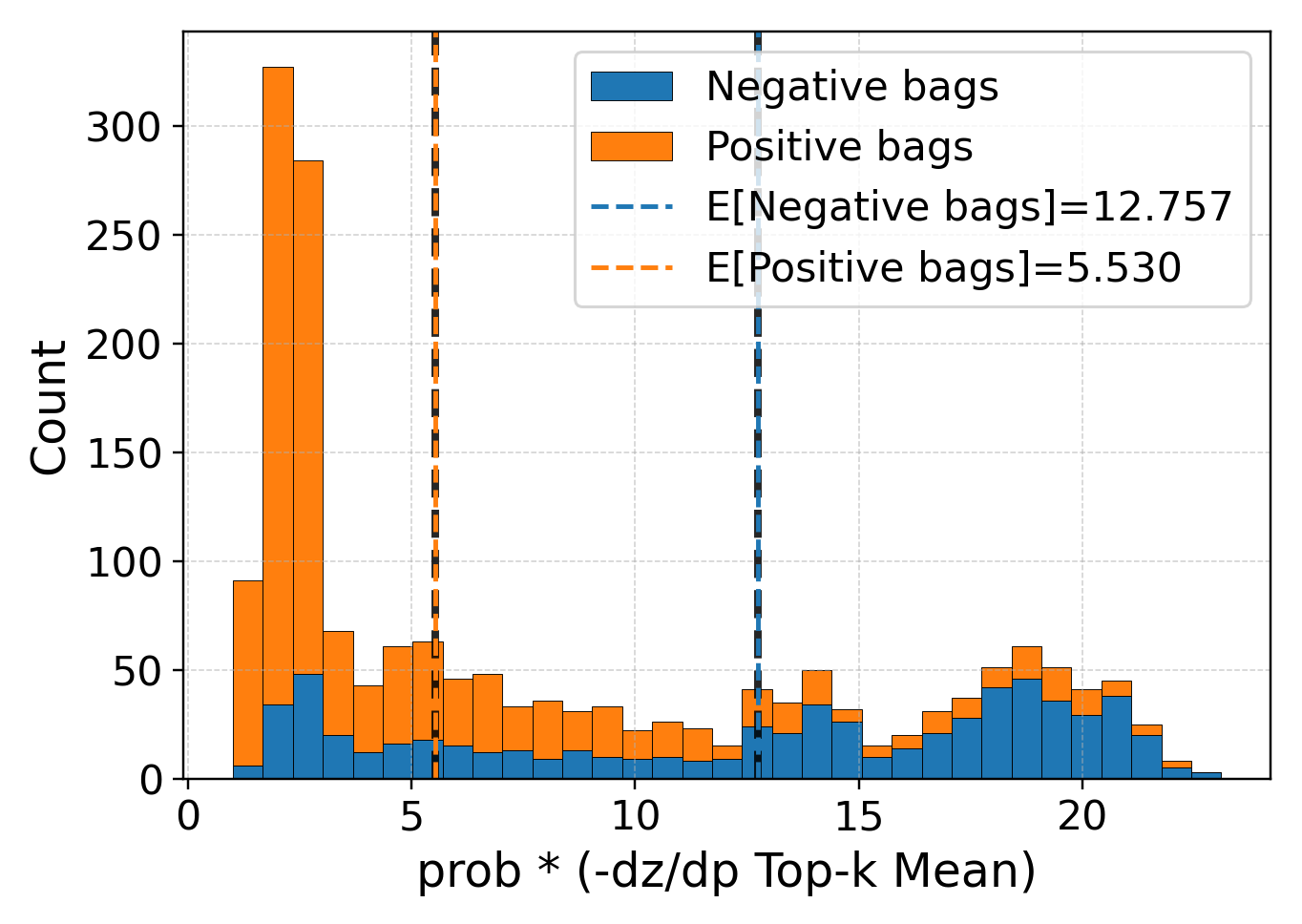}
        \caption{NQ\ (Llama-3.1-8B)}
    \end{subfigure}\hfill
    \begin{subfigure}{0.32\textwidth}
        \includegraphics[width=\linewidth]{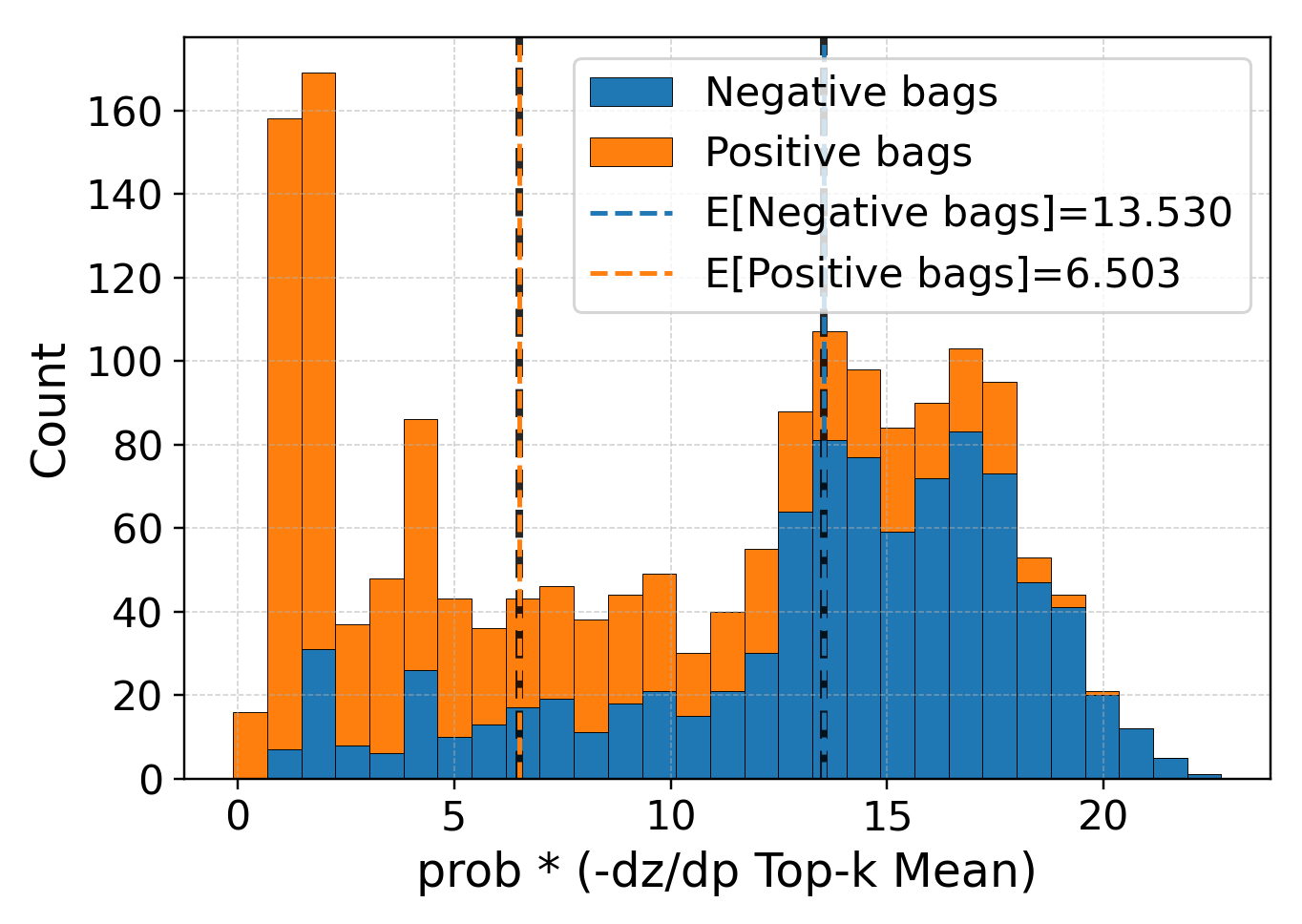}
        \caption{NQ\ (Mistral-12B)}
    \end{subfigure}\hfill
    \begin{subfigure}{0.32\textwidth}
        \includegraphics[width=\linewidth]{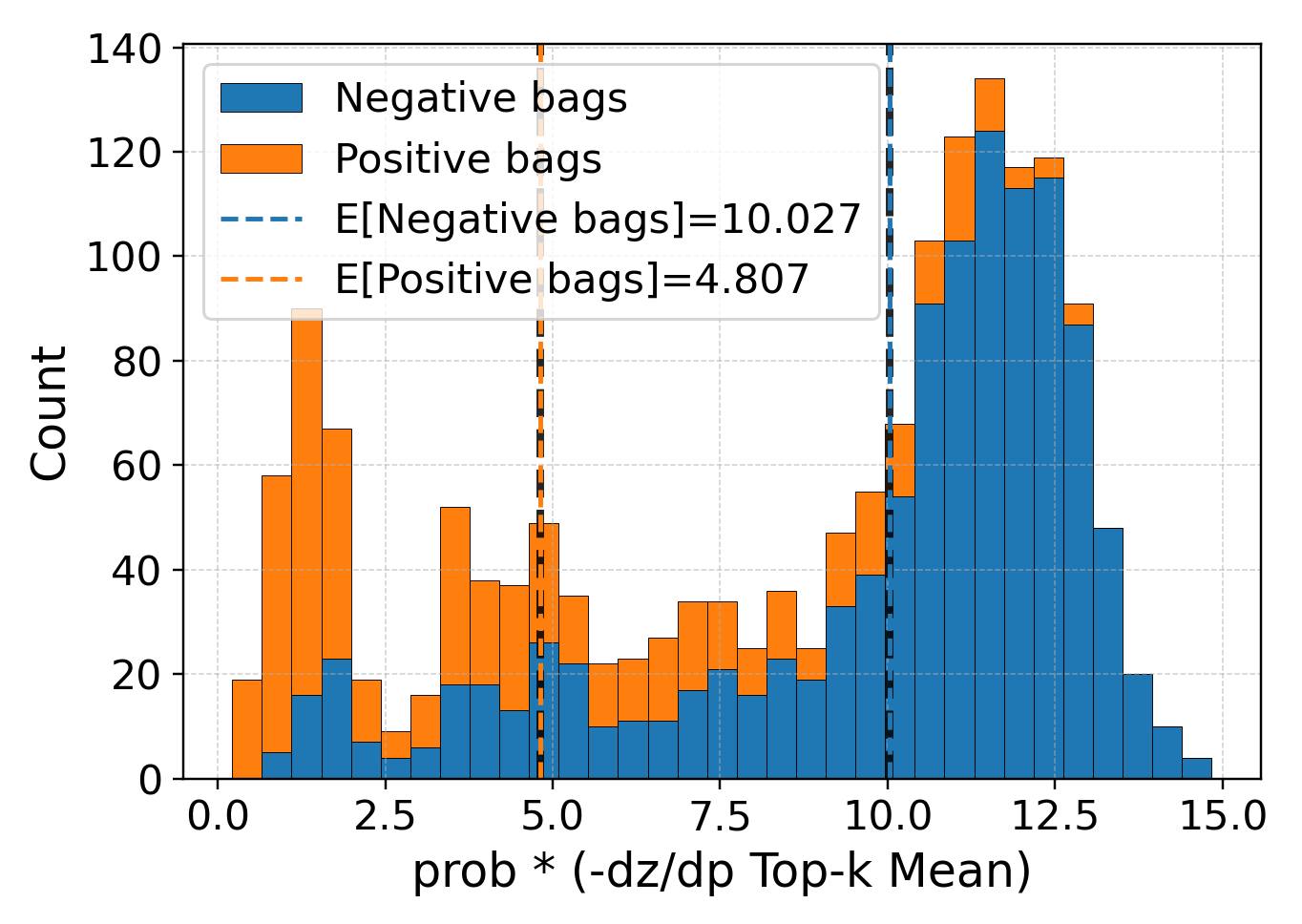}
        \caption{NQ\ (Llama-3.3-70B)}
    \end{subfigure}\hfill
    \vspace{0.3cm}
    \begin{subfigure}{0.32\textwidth}
        \includegraphics[width=\linewidth]{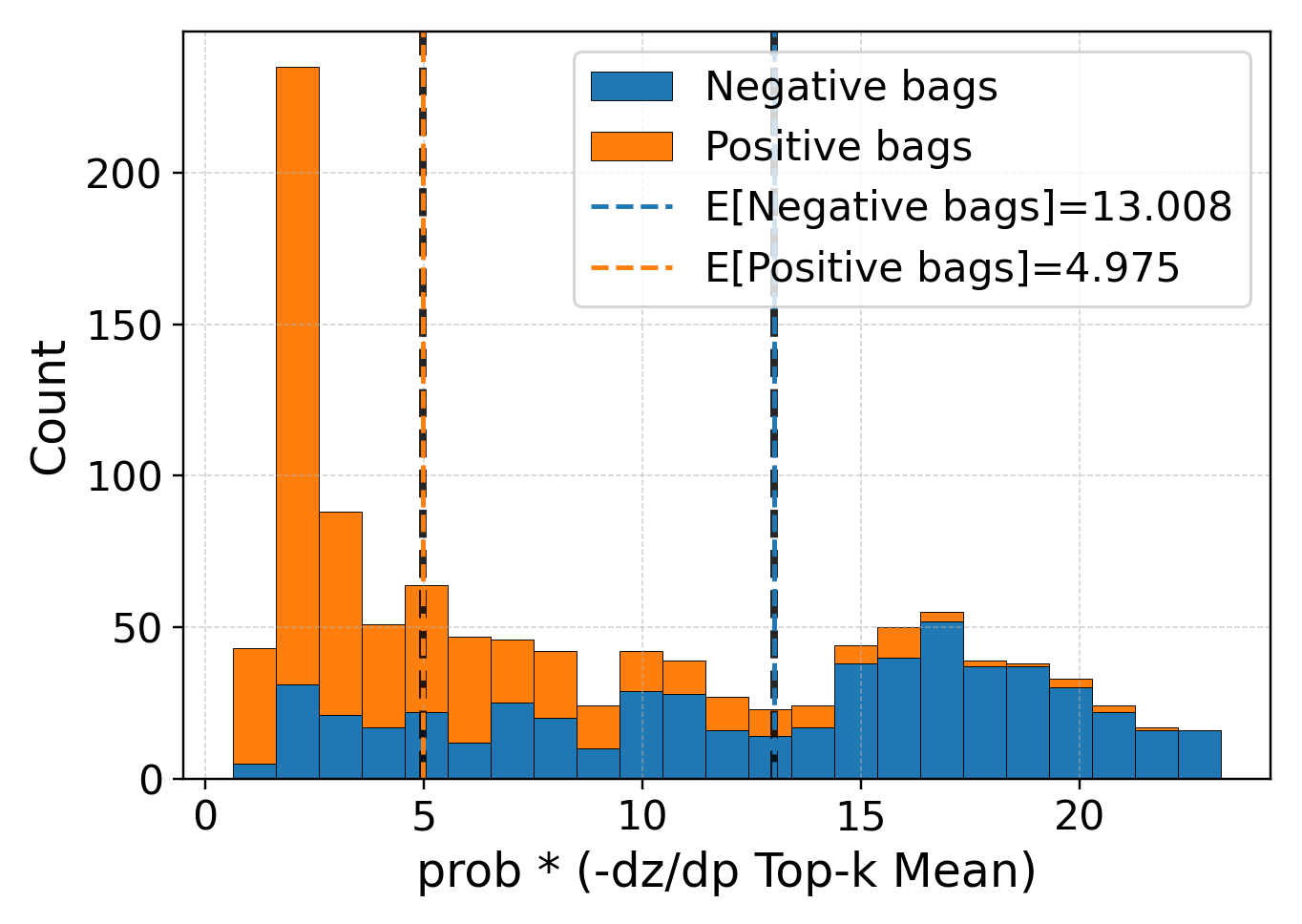}
        \caption{BioASQ\ (Llama-3.1-8B)}
    \end{subfigure}\hfill
    \begin{subfigure}{0.32\textwidth}
        \includegraphics[width=\linewidth]{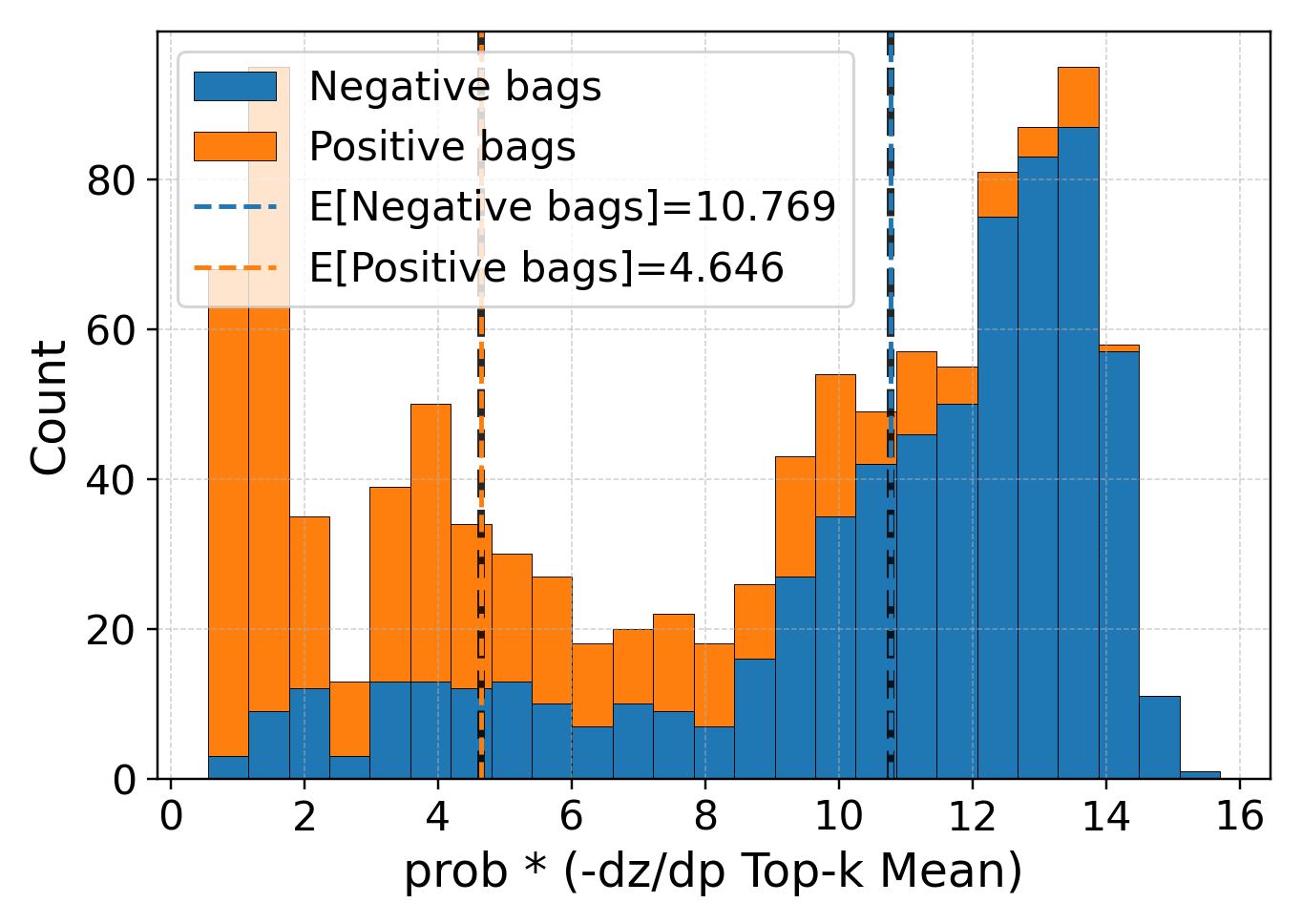}
        \caption{BioASQ\ (Mistral-12B)}
    \end{subfigure}\hfill
    \begin{subfigure}{0.32\textwidth}
        \includegraphics[width=\linewidth]{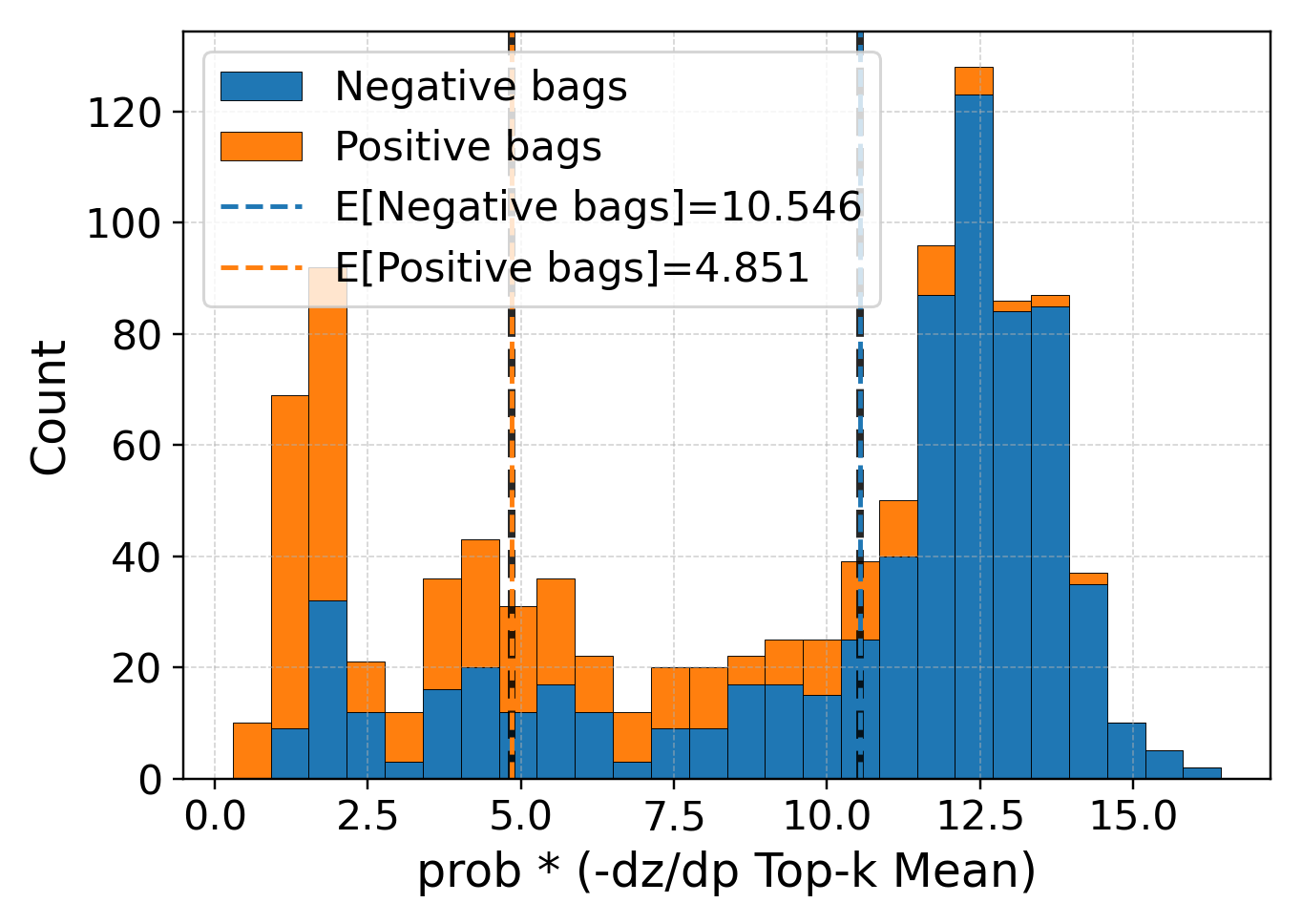}
        \caption{BioASQ\ (Llama-3.3-70B)}
    \end{subfigure}\hfill
    \vspace{0.3cm}
    \caption{Distribution of the joint product $P_{\mathrm{sem}}^{\mathbf{B}} \bar{C}_{\mathbf{B}}$ in \textbf{test} data.}
    \label{fig:dist_joint_val}
\end{figure}
\clearpage
\section{Empirical Check of \texorpdfstring{Theorem~\ref{theorem:hami}}{Theorem 1} with Integrated Sensitivity}
\label{app:integrated_checks}
\subsection{Summary of Expected Values via Path Integration}
\label{sum}
These tables use $\lambda=1$ and $N=1000$ integration steps to estimate $\bar C_{\mathbf B}^{\mathrm{int}}$, as defined in Appendix~\ref{app:active_non_invariant}. As in Appendix~\ref{app:expected_values_summary}, the negative class has the larger joint expectation. The margin condition is satisfied for all three generators on the four short-answer benchmarks, using the class ratios of Table~\ref{tab:class_balance}.

\begin{table}[H]
\centering
\caption{Empirical expected values for \textbf{LLaMA-3.1-8B} (\textbf{Train}) under non-invariant conditions.}
\label{tab:exp_llama8b_train_non_invariant}
\resizebox{1.0\textwidth}{!}{%
\begin{tabular}{l cccccc cc}
\toprule
& \multicolumn{2}{c}{\textbf{Semantic Prob.} $\mathbb{E}[P_{\mathrm{sem}}^{\mathbf{B}}]$}  
& \multicolumn{2}{c}{\textbf{Sensitivity} $\mathbb{E}[\bar{C}_{\mathbf{B}}^{\mathrm{int}}]$}  
& \multicolumn{2}{c}{\textbf{Joint Product} $\mathbb{E}[P_{\mathrm{sem}}^{\mathbf{B}} \bar{C}_{\mathbf{B}}^{\mathrm{int}}]$}
& \multicolumn{2}{c}{\textbf{Ratios}} \\
\cmidrule(lr){2-3} \cmidrule(lr){4-5} \cmidrule(lr){6-7} \cmidrule(lr){8-9}
\textbf{Dataset} & Pos ($\mathbb{E}_{\mathrm{pos}}$) & Neg ($\mathbb{E}_{\mathrm{neg}}$) & Pos ($\mathbb{E}_{\mathrm{pos}}$) & Neg ($\mathbb{E}_{\mathrm{neg}}$) & Pos ($\mathbb{E}_{\mathrm{pos}}$) & Neg ($\mathbb{E}_{\mathrm{neg}}$) & $\gamma$ & $1/\gamma$ \\
\midrule
TriviaQA & 0.430 & \textbf{0.914} & 13.2 & \textbf{19.0} & 6.59 & \textbf{17.8} & \textbf{2.70} & \textbf{0.37} \\
SQuAD    & 0.352 & \textbf{0.730} & 12.6 & \textbf{18.8} & 5.27 & \textbf{15.0} & \textbf{2.85} & \textbf{0.35} \\
NQ       & 0.382 & \textbf{0.772} & 12.8 & \textbf{16.9} & 5.59 & \textbf{13.8} & \textbf{2.47} & \textbf{0.41} \\
BioASQ   & 0.392 & \textbf{0.756} & 11.9 & \textbf{16.3} & 5.11 & \textbf{13.0} & \textbf{2.54} & \textbf{0.39} \\
\bottomrule
\end{tabular}
}
\end{table}

\begin{table}[H]
\centering
\caption{Empirical expected values for \textbf{LLaMA-3.1-8B} (\textbf{Test}) under non-invariant conditions.}
\label{tab:exp_llama8b_val_non_invariant}
\resizebox{1.0\textwidth}{!}{%
\begin{tabular}{l cccccc cc}
\toprule
& \multicolumn{2}{c}{\textbf{Semantic Prob.} $\mathbb{E}[P_{\mathrm{sem}}^{\mathbf{B}}]$}  
& \multicolumn{2}{c}{\textbf{Sensitivity} $\mathbb{E}[\bar{C}_{\mathbf{B}}^{\mathrm{int}}]$}  
& \multicolumn{2}{c}{\textbf{Joint Product} $\mathbb{E}[P_{\mathrm{sem}}^{\mathbf{B}} \bar{C}_{\mathbf{B}}^{\mathrm{int}}]$}
& \multicolumn{2}{c}{\textbf{Ratios}} \\
\cmidrule(lr){2-3} \cmidrule(lr){4-5} \cmidrule(lr){6-7} \cmidrule(lr){8-9}
\textbf{Dataset} & Pos ($\mathbb{E}_{\mathrm{pos}}$) & Neg ($\mathbb{E}_{\mathrm{neg}}$) & Pos ($\mathbb{E}_{\mathrm{pos}}$) & Neg ($\mathbb{E}_{\mathrm{neg}}$) & Pos ($\mathbb{E}_{\mathrm{pos}}$) & Neg ($\mathbb{E}_{\mathrm{neg}}$) & $\gamma$ & $1/\gamma$ \\
\midrule
TriviaQA & 0.451 & \textbf{0.926} & 13.7 & \textbf{19.2} & 7.27 & \textbf{18.1} & \textbf{2.49} & \textbf{0.40} \\
SQuAD    & 0.360 & \textbf{0.733} & 13.1 & \textbf{18.8} & 5.67 & \textbf{14.9} & \textbf{2.63} & \textbf{0.38} \\
NQ       & 0.383 & \textbf{0.732} & 12.8 & \textbf{16.6} & 5.63 & \textbf{13.0} & \textbf{2.31} & \textbf{0.43} \\
BioASQ   & 0.380 & \textbf{0.764} & 12.2 & \textbf{16.6} & 5.10 & \textbf{13.3} & \textbf{2.61} & \textbf{0.38} \\
\bottomrule
\end{tabular}
}
\end{table}

\begin{table}[H]
\centering
\caption{Empirical expected values for \textbf{Mistral-12B} (\textbf{Train}) under non-invariant conditions.}
\label{tab:exp_mistral12b_train_non_invariant}
\resizebox{1.0\textwidth}{!}{%
\begin{tabular}{l cccccc cc}
\toprule
& \multicolumn{2}{c}{\textbf{Semantic Prob.} $\mathbb{E}[P_{\mathrm{sem}}^{\mathbf{B}}]$}  
& \multicolumn{2}{c}{\textbf{Sensitivity} $\mathbb{E}[\bar{C}_{\mathbf{B}}^{\mathrm{int}}]$}  
& \multicolumn{2}{c}{\textbf{Joint Product} $\mathbb{E}[P_{\mathrm{sem}}^{\mathbf{B}} \bar{C}_{\mathbf{B}}^{\mathrm{int}}]$}
& \multicolumn{2}{c}{\textbf{Ratios}} \\
\cmidrule(lr){2-3} \cmidrule(lr){4-5} \cmidrule(lr){6-7} \cmidrule(lr){8-9}
\textbf{Dataset} & Pos ($\mathbb{E}_{\mathrm{pos}}$) & Neg ($\mathbb{E}_{\mathrm{neg}}$) & Pos ($\mathbb{E}_{\mathrm{pos}}$) & Neg ($\mathbb{E}_{\mathrm{neg}}$) & Pos ($\mathbb{E}_{\mathrm{pos}}$) & Neg ($\mathbb{E}_{\mathrm{neg}}$) & $\gamma$ & $1/\gamma$ \\
\midrule
TriviaQA & 0.581 & \textbf{0.960} & 15.0 & \textbf{20.6} & 9.45 & \textbf{20.0} & \textbf{2.12} & \textbf{0.47} \\
SQuAD    & 0.457 & \textbf{0.835} & 9.40 & \textbf{15.1} & 5.08 & \textbf{13.3} & \textbf{2.62} & \textbf{0.38} \\
NQ       & 0.535 & \textbf{0.886} & 10.3 & \textbf{15.2} & 6.44 & \textbf{13.9} & \textbf{2.16} & \textbf{0.46} \\
BioASQ   & 0.480 & \textbf{0.881} & 8.55 & \textbf{11.7} & 4.68 & \textbf{10.6} & \textbf{2.26} & \textbf{0.44} \\
\bottomrule
\end{tabular}
}
\end{table}

\begin{table}[H]
\centering
\caption{Empirical expected values for \textbf{Mistral-12B} (\textbf{Test}) under non-invariant conditions.}
\label{tab:exp_mistral12b_val_non_invariant}
\resizebox{1.0\textwidth}{!}{%
\begin{tabular}{l cccccc cc}
\toprule
& \multicolumn{2}{c}{\textbf{Semantic Prob.} $\mathbb{E}[P_{\mathrm{sem}}^{\mathbf{B}}]$}  
& \multicolumn{2}{c}{\textbf{Sensitivity} $\mathbb{E}[\bar{C}_{\mathbf{B}}^{\mathrm{int}}]$}  
& \multicolumn{2}{c}{\textbf{Joint Product} $\mathbb{E}[P_{\mathrm{sem}}^{\mathbf{B}} \bar{C}_{\mathbf{B}}^{\mathrm{int}}]$}
& \multicolumn{2}{c}{\textbf{Ratios}} \\
\cmidrule(lr){2-3} \cmidrule(lr){4-5} \cmidrule(lr){6-7} \cmidrule(lr){8-9}
\textbf{Dataset} & Pos ($\mathbb{E}_{\mathrm{pos}}$) & Neg ($\mathbb{E}_{\mathrm{neg}}$) & Pos ($\mathbb{E}_{\mathrm{pos}}$) & Neg ($\mathbb{E}_{\mathrm{neg}}$) & Pos ($\mathbb{E}_{\mathrm{pos}}$) & Neg ($\mathbb{E}_{\mathrm{neg}}$) & $\gamma$ & $1/\gamma$ \\
\midrule
TriviaQA & 0.590 & \textbf{0.966} & 15.8 & \textbf{20.8} & 10.4 & \textbf{20.2} & \textbf{1.94} & \textbf{0.51} \\
SQuAD    & 0.502 & \textbf{0.849} & 10.4 & \textbf{15.1} & 6.12 & \textbf{13.4} & \textbf{2.19} & \textbf{0.46} \\
NQ       & 0.525 & \textbf{0.881} & 10.5 & \textbf{15.0} & 6.58 & \textbf{13.7} & \textbf{2.08} & \textbf{0.48} \\
BioASQ   & 0.491 & \textbf{0.898} & 8.63 & \textbf{11.9} & 4.85 & \textbf{10.9} & \textbf{2.25} & \textbf{0.44} \\
\bottomrule
\end{tabular}
}
\end{table}

\begin{table}[H]
\centering
\caption{Empirical expected values for \textbf{LLaMA-3.3-70B} (\textbf{Train}) under non-invariant conditions.}
\label{tab:exp_llama70b_train_non_invariant}
\resizebox{1.0\textwidth}{!}{%
\begin{tabular}{l cccccc cc}
\toprule
& \multicolumn{2}{c}{\textbf{Semantic Prob.} $\mathbb{E}[P_{\mathrm{sem}}^{\mathbf{B}}]$}  
& \multicolumn{2}{c}{\textbf{Sensitivity} $\mathbb{E}[\bar{C}_{\mathbf{B}}^{\mathrm{int}}]$}  
& \multicolumn{2}{c}{\textbf{Joint Product} $\mathbb{E}[P_{\mathrm{sem}}^{\mathbf{B}} \bar{C}_{\mathbf{B}}^{\mathrm{int}}]$}
& \multicolumn{2}{c}{\textbf{Ratios}} \\
\cmidrule(lr){2-3} \cmidrule(lr){4-5} \cmidrule(lr){6-7} \cmidrule(lr){8-9}
\textbf{Dataset} & Pos ($\mathbb{E}_{\mathrm{pos}}$) & Neg ($\mathbb{E}_{\mathrm{neg}}$) & Pos ($\mathbb{E}_{\mathrm{pos}}$) & Neg ($\mathbb{E}_{\mathrm{neg}}$) & Pos ($\mathbb{E}_{\mathrm{pos}}$) & Neg ($\mathbb{E}_{\mathrm{neg}}$) & $\gamma$ & $1/\gamma$ \\
\midrule
TriviaQA & 0.484 & \textbf{0.944} & 9.78 & \textbf{13.6} & 5.50 & \textbf{13.0} & \textbf{2.36} & \textbf{0.42} \\
SQuAD    & 0.439 & \textbf{0.842} & 8.63 & \textbf{10.8} & 4.29 & \textbf{9.35} & \textbf{2.18} & \textbf{0.46} \\
NQ       & 0.538 & \textbf{0.888} & 8.64 & \textbf{11.1} & 5.27 & \textbf{10.1} & \textbf{1.92} & \textbf{0.52} \\
BioASQ   & 0.459 & \textbf{0.866} & 9.04 & \textbf{11.9} & 4.59 & \textbf{10.6} & \textbf{2.31} & \textbf{0.43} \\
\bottomrule
\end{tabular}
}
\end{table}

\begin{table}[H]
\centering
\caption{Empirical expected values for \textbf{LLaMA-3.3-70B} (\textbf{Test}) under non-invariant conditions.}
\label{tab:exp_llama70b_val_non_invariant}
\resizebox{1.0\textwidth}{!}{%
\begin{tabular}{l cccccc cc}
\toprule
& \multicolumn{2}{c}{\textbf{Semantic Prob.} $\mathbb{E}[P_{\mathrm{sem}}^{\mathbf{B}}]$}  
& \multicolumn{2}{c}{\textbf{Sensitivity} $\mathbb{E}[\bar{C}_{\mathbf{B}}^{\mathrm{int}}]$}  
& \multicolumn{2}{c}{\textbf{Joint Product} $\mathbb{E}[P_{\mathrm{sem}}^{\mathbf{B}} \bar{C}_{\mathbf{B}}^{\mathrm{int}}]$}
& \multicolumn{2}{c}{\textbf{Ratios}} \\
\cmidrule(lr){2-3} \cmidrule(lr){4-5} \cmidrule(lr){6-7} \cmidrule(lr){8-9}
\textbf{Dataset} & Pos ($\mathbb{E}_{\mathrm{pos}}$) & Neg ($\mathbb{E}_{\mathrm{neg}}$) & Pos ($\mathbb{E}_{\mathrm{pos}}$) & Neg ($\mathbb{E}_{\mathrm{neg}}$) & Pos ($\mathbb{E}_{\mathrm{pos}}$) & Neg ($\mathbb{E}_{\mathrm{neg}}$) & $\gamma$ & $1/\gamma$ \\
\midrule
TriviaQA & 0.554 & \textbf{0.940} & 10.6 & \textbf{13.6} & 6.65 & \textbf{13.0} & \textbf{1.95} & \textbf{0.51} \\
SQuAD    & 0.471 & \textbf{0.843} & 9.01 & \textbf{10.9} & 4.70 & \textbf{9.42} & \textbf{2.00} & \textbf{0.50} \\
NQ       & 0.525 & \textbf{0.894} & 8.52 & \textbf{11.1} & 5.08 & \textbf{10.1} & \textbf{1.99} & \textbf{0.50} \\
BioASQ   & 0.490 & \textbf{0.866} & 9.45 & \textbf{12.0} & 5.17 & \textbf{10.6} & \textbf{2.05} & \textbf{0.49} \\
\bottomrule
\end{tabular}
}
\end{table}
\subsection{Distribution of Integrated Sensitivity \texorpdfstring{$\bar{C}_{\mathbf{B}}^{\mathrm{int}}$}{C-bar int}}
We analyzed the sensitivity distributions in Figures~\ref{fig:dist_sensitivity_train_non_invariant}--\ref{fig:dist_sensitivity_val_non_invariant}. All values were calculated with $\lambda=1$ and using the integrated cost $\bar{C}_{\mathbf{B}}^{\mathrm{int}}$ via numerical approximation with $N=1000$ steps. Across the evaluated models and four short-answer benchmarks, the distributions for negative bags tend to be shifted toward larger values. The class-wise means in Appendix~\ref{sum} show greater integrated sensitivity for negative bags on average in both data splits.

\begin{figure}[H]
    \centering
    \begin{subfigure}{0.32\textwidth}
        \includegraphics[width=\linewidth]{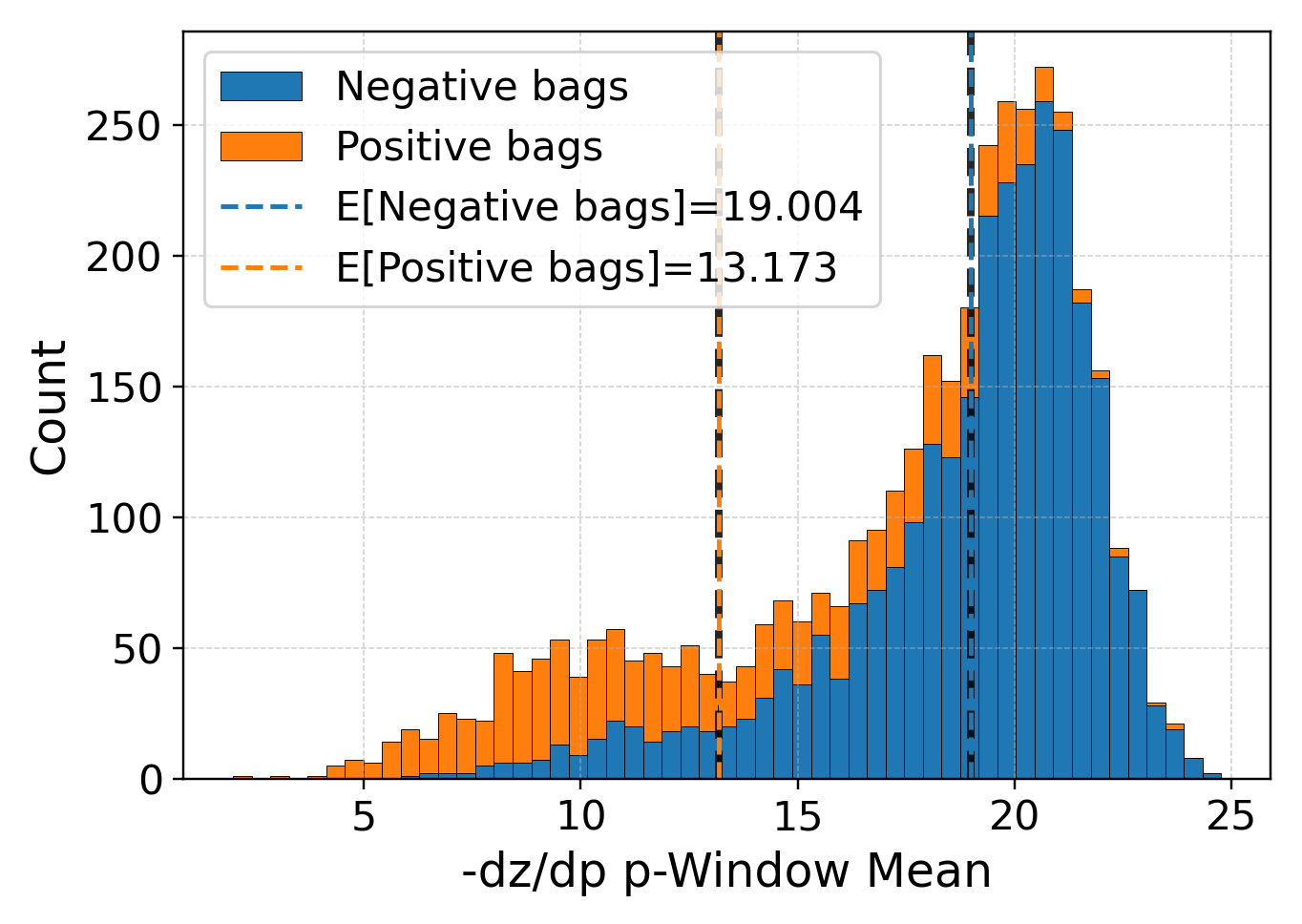}
        \caption{TriviaQA\ (Llama-3.1-8B)}
    \end{subfigure}\hfill
    \begin{subfigure}{0.32\textwidth}
        \includegraphics[width=\linewidth]{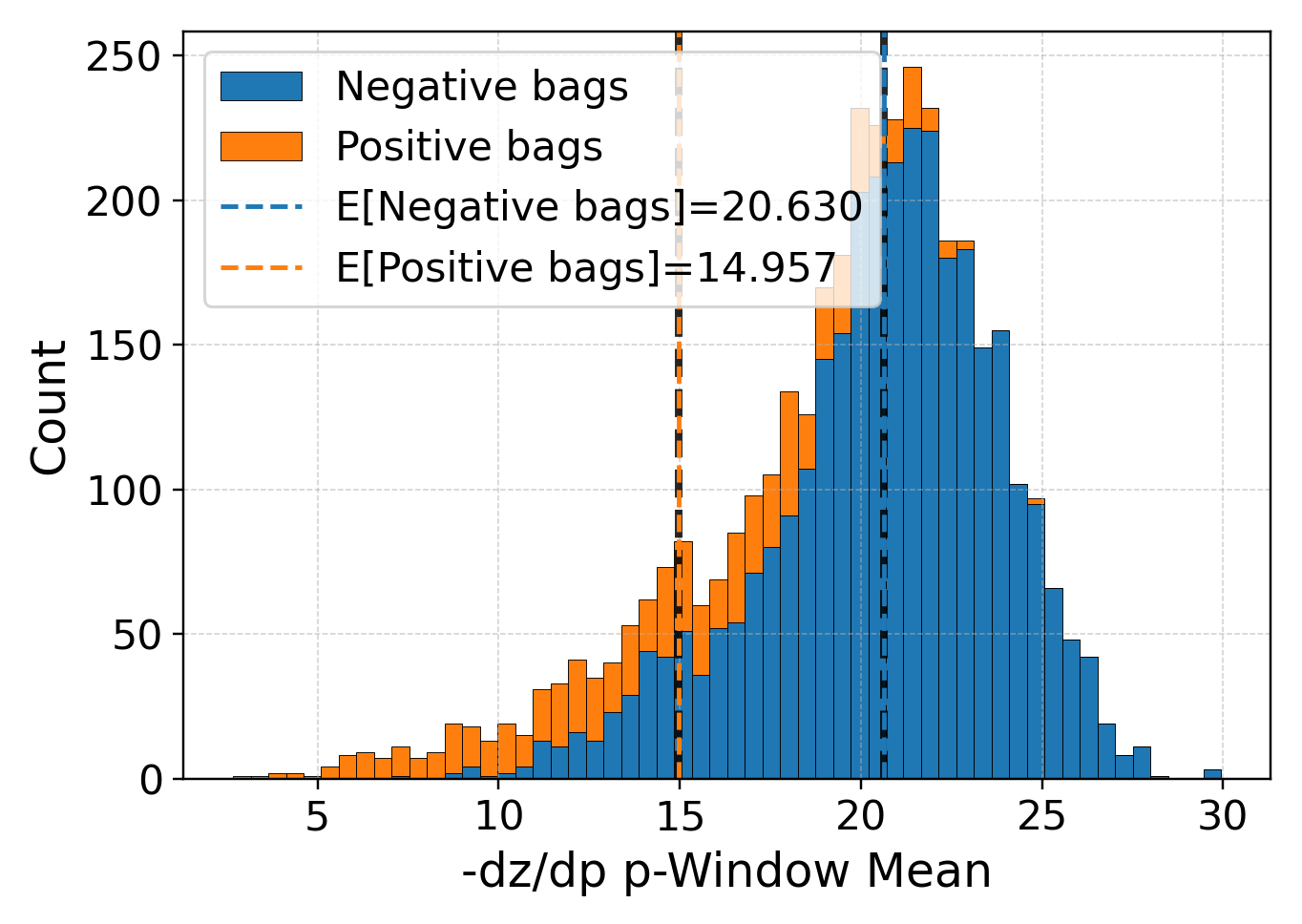}
        \caption{TriviaQA\ (Mistral-12B)}
    \end{subfigure}\hfill
    \begin{subfigure}{0.32\textwidth}
        \includegraphics[width=\linewidth]{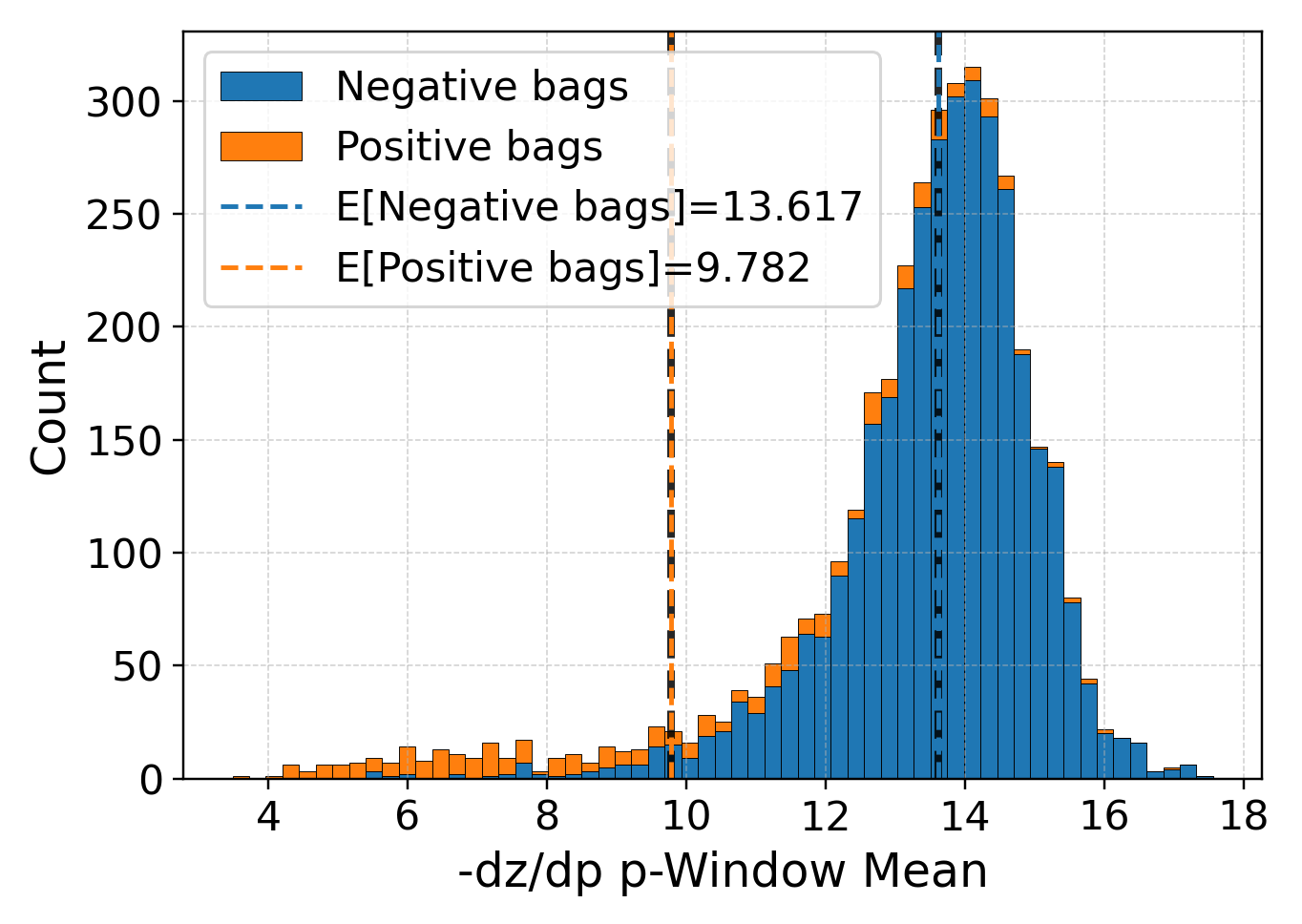}
        \caption{TriviaQA\ (Llama-3.3-70B)}
    \end{subfigure}\hfill
    \vspace{0.3cm}
    \begin{subfigure}{0.32\textwidth}
        \includegraphics[width=\linewidth]{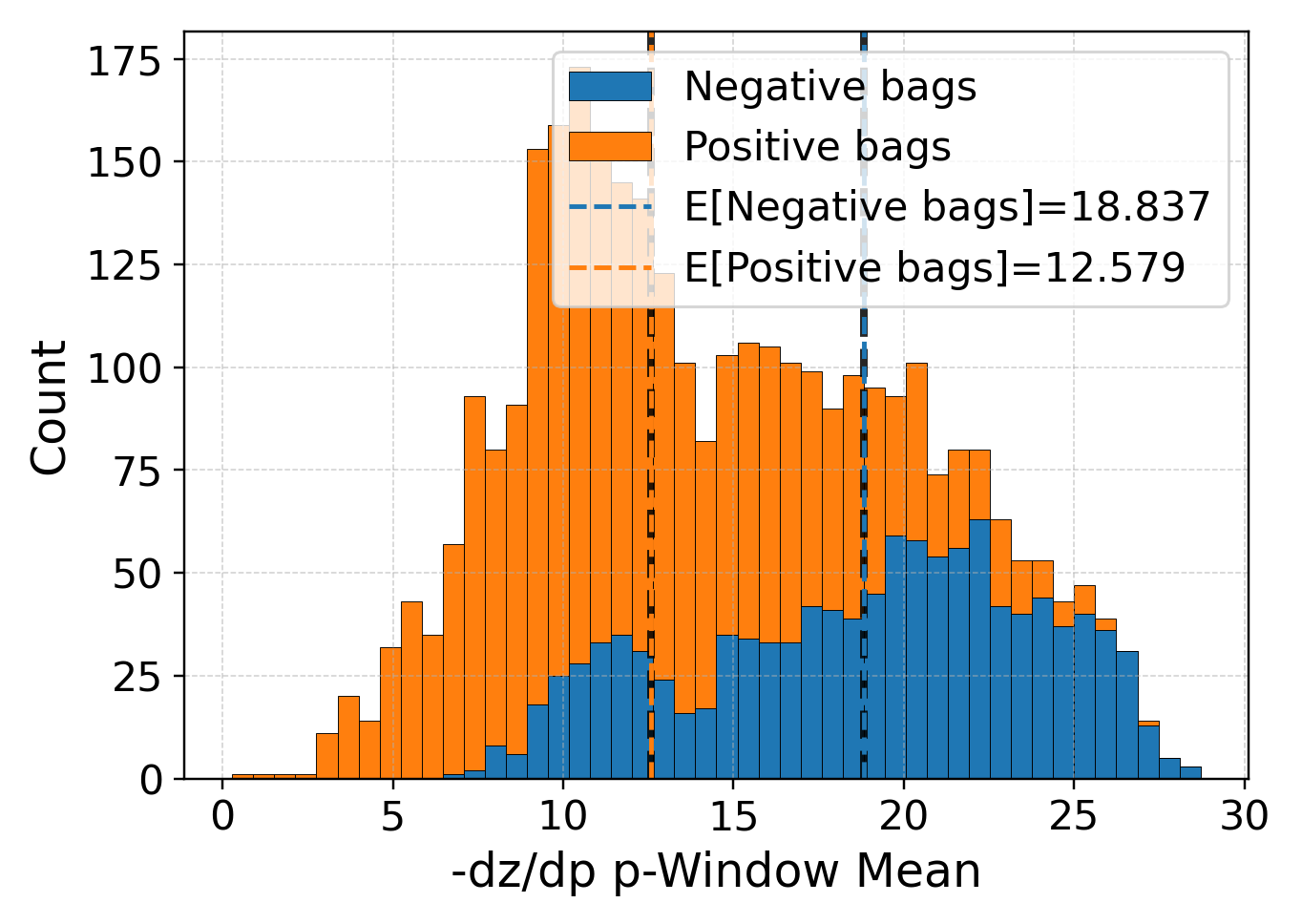}
        \caption{SQuAD\ (Llama-3.1-8B)}
    \end{subfigure}\hfill
    \begin{subfigure}{0.32\textwidth}
        \includegraphics[width=\linewidth]{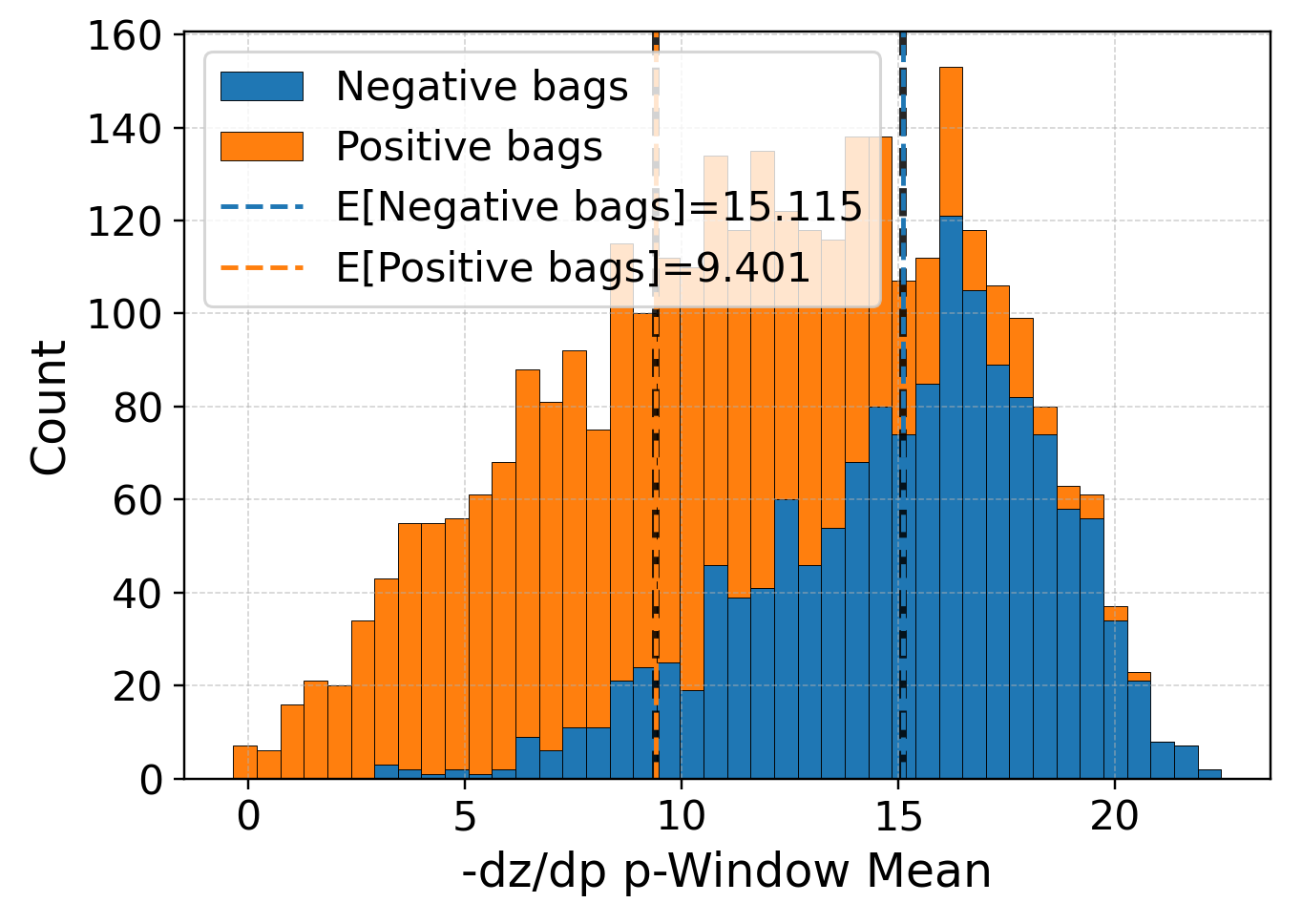}
        \caption{SQuAD\ (Mistral-12B)}
    \end{subfigure}\hfill
    \begin{subfigure}{0.32\textwidth}
        \includegraphics[width=\linewidth]{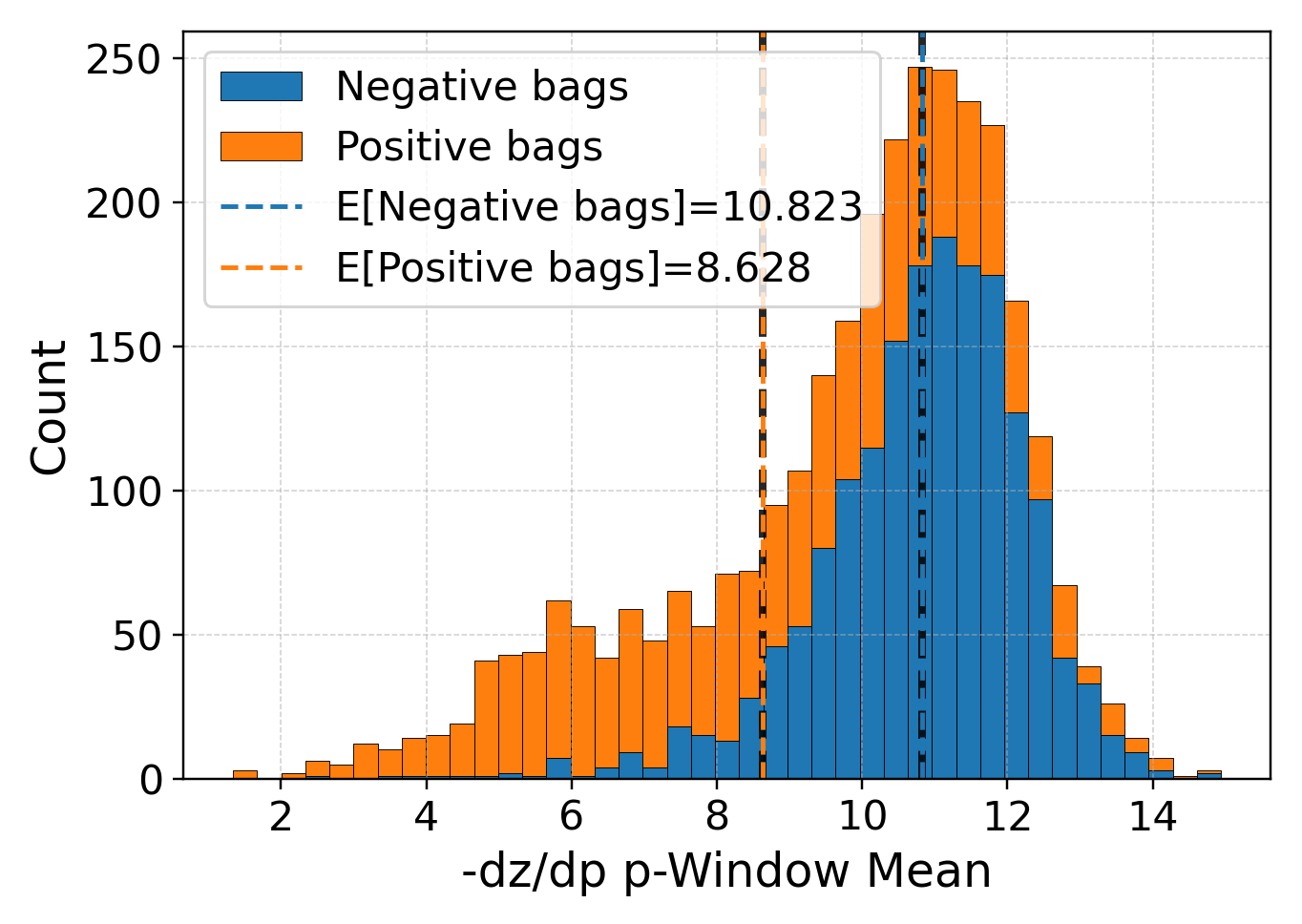}
        \caption{SQuAD\ (Llama-3.3-70B)}
    \end{subfigure}\hfill
    \vspace{0.3cm}
    \begin{subfigure}{0.32\textwidth}
        \includegraphics[width=\linewidth]{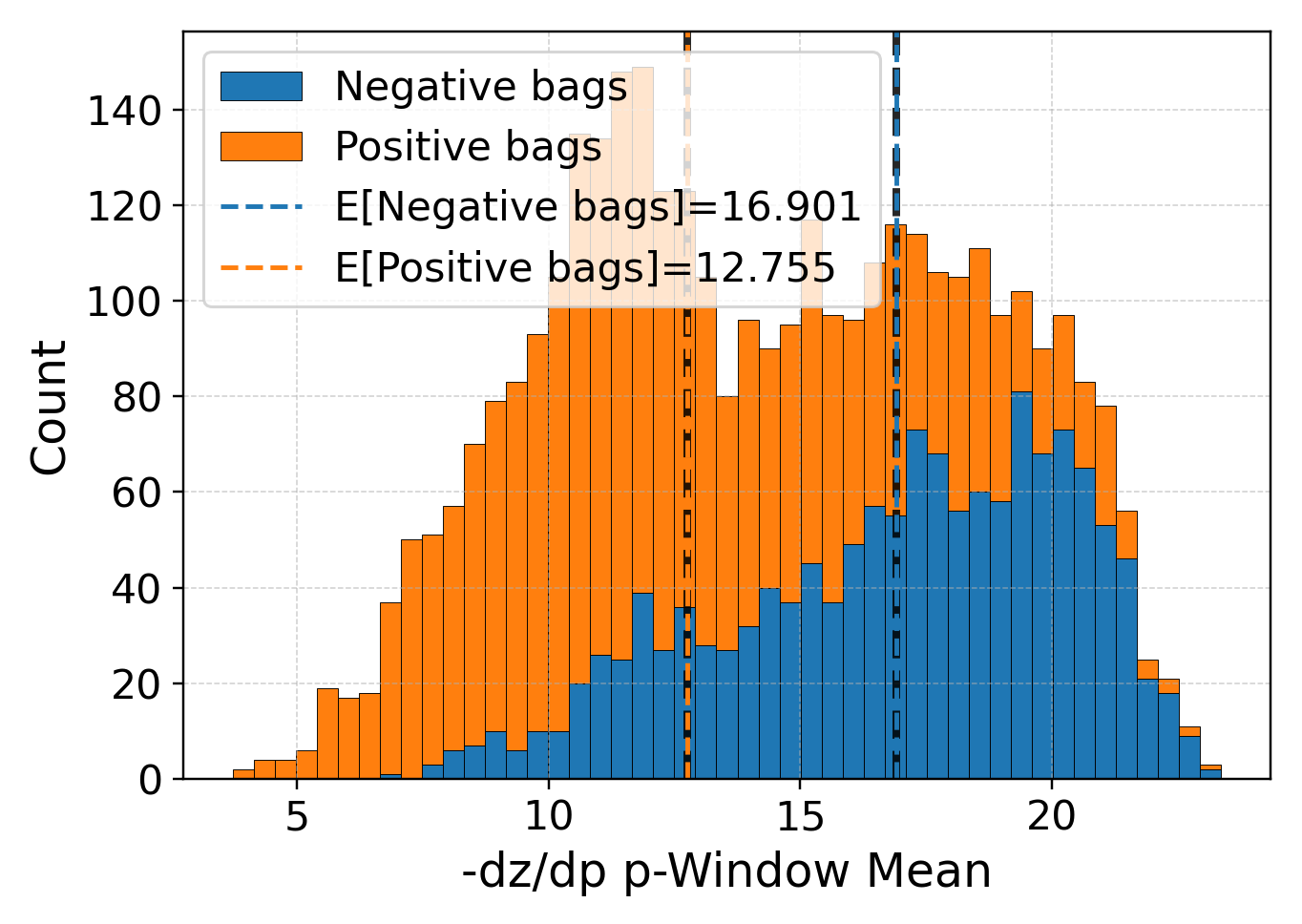}
        \caption{NQ\ (Llama-3.1-8B)}
    \end{subfigure}\hfill
    \begin{subfigure}{0.32\textwidth}
        \includegraphics[width=\linewidth]{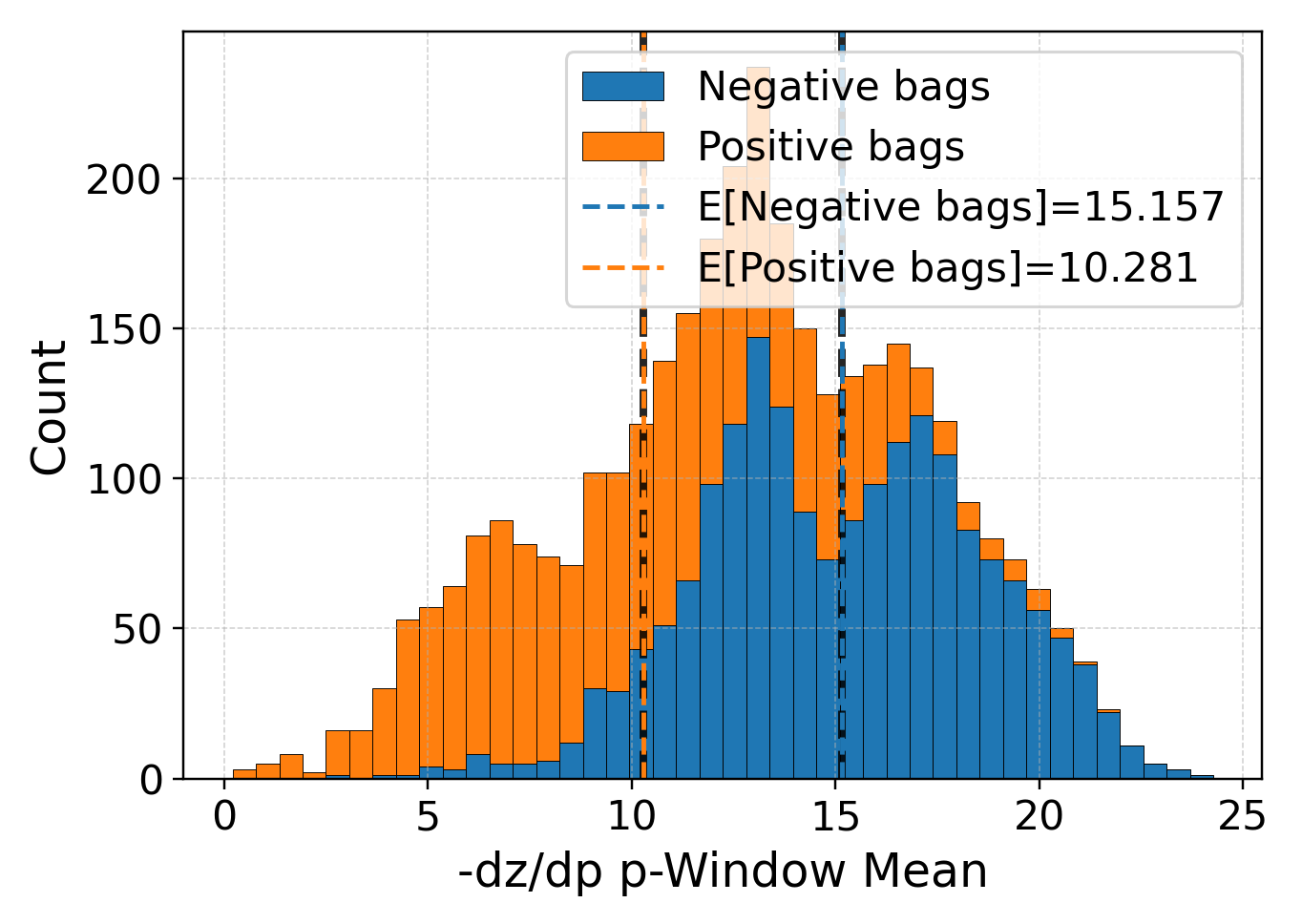}
        \caption{NQ\ (Mistral-12B)}
    \end{subfigure}\hfill
    \begin{subfigure}{0.32\textwidth}
        \includegraphics[width=\linewidth]{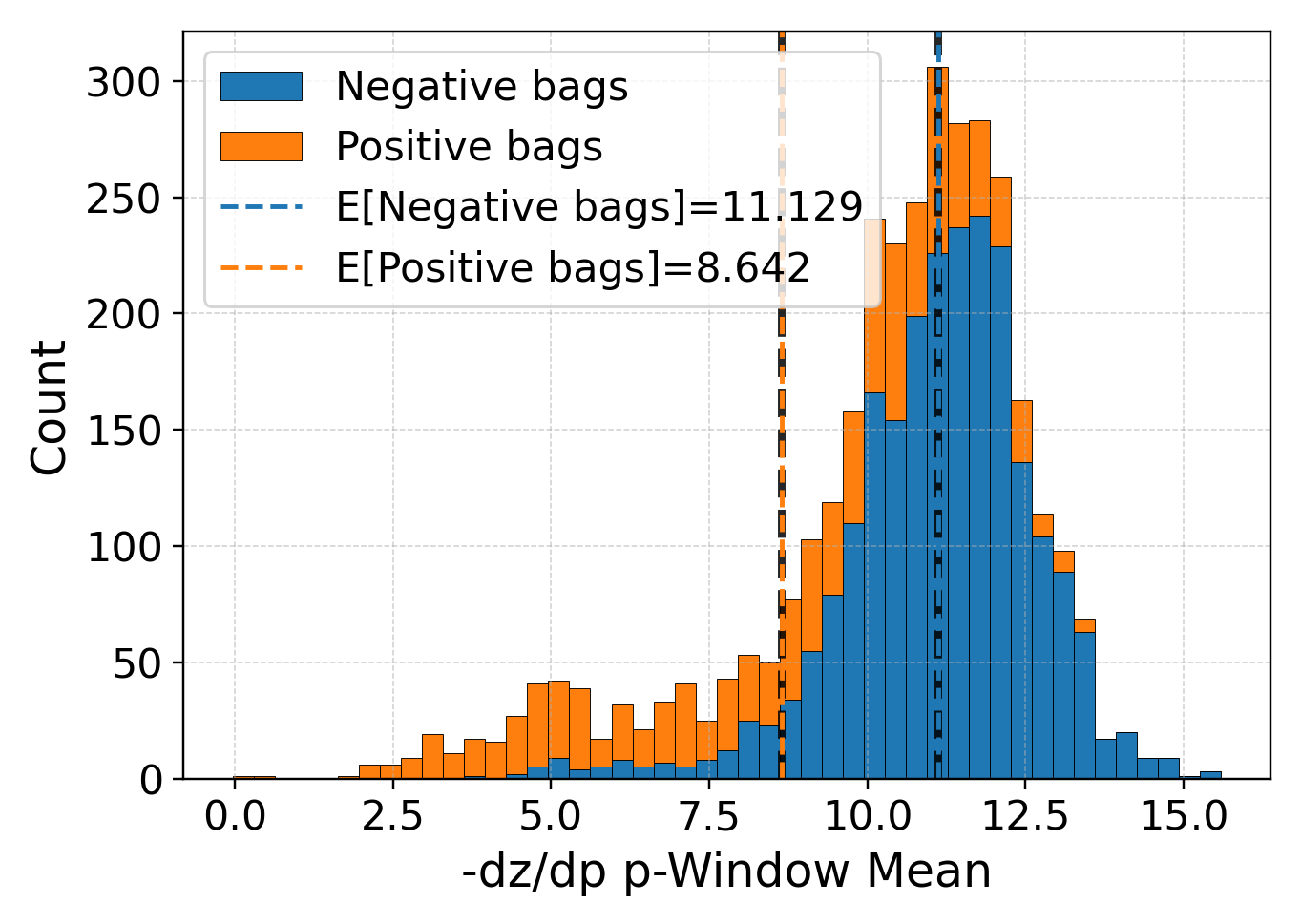}
        \caption{NQ\ (Llama-3.3-70B)}
    \end{subfigure}\hfill
    \vspace{0.3cm}
    \begin{subfigure}{0.32\textwidth}
        \includegraphics[width=\linewidth]{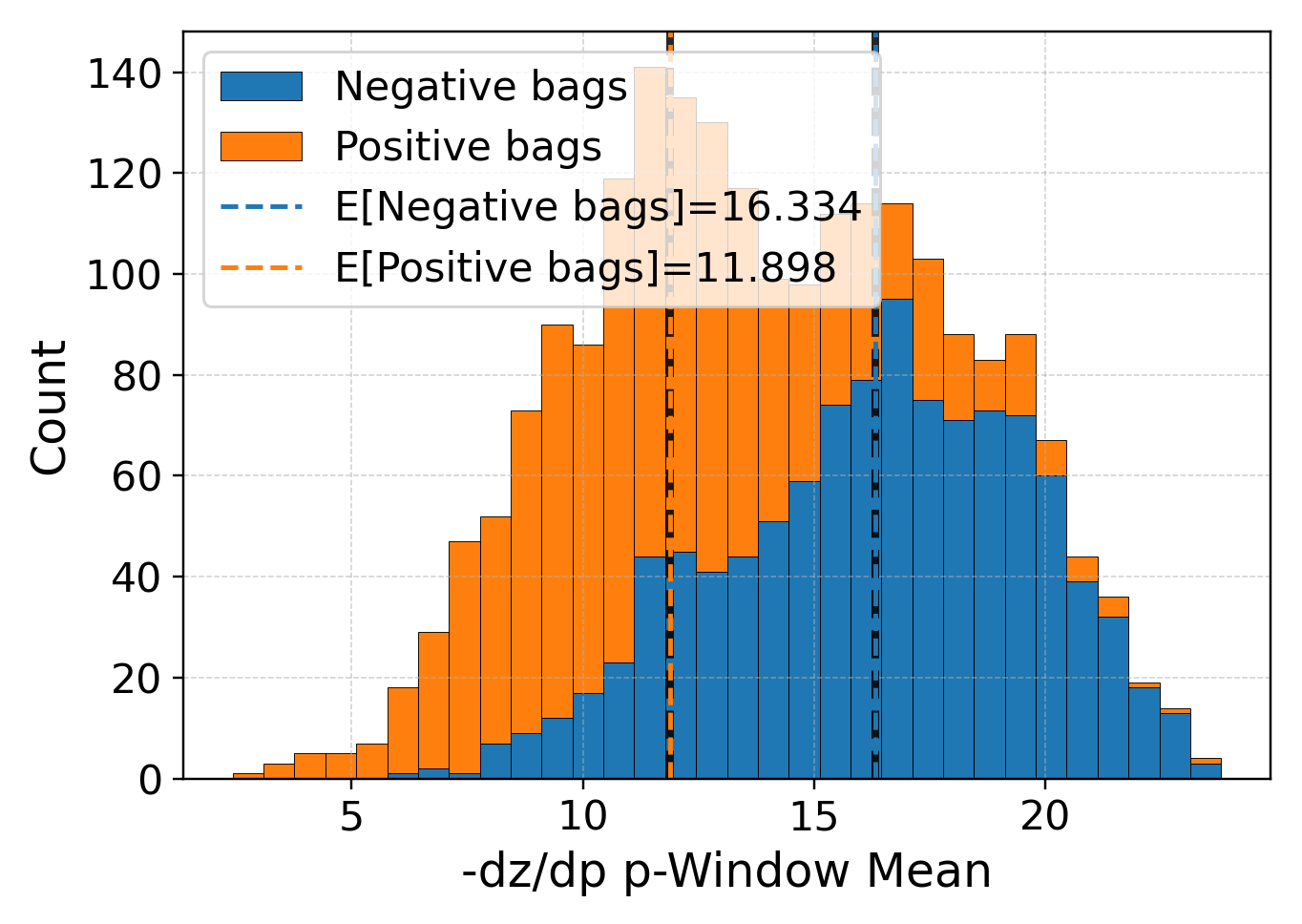}
        \caption{BioASQ\ (Llama-3.1-8B)}
    \end{subfigure}\hfill
    \begin{subfigure}{0.32\textwidth}
        \includegraphics[width=\linewidth]{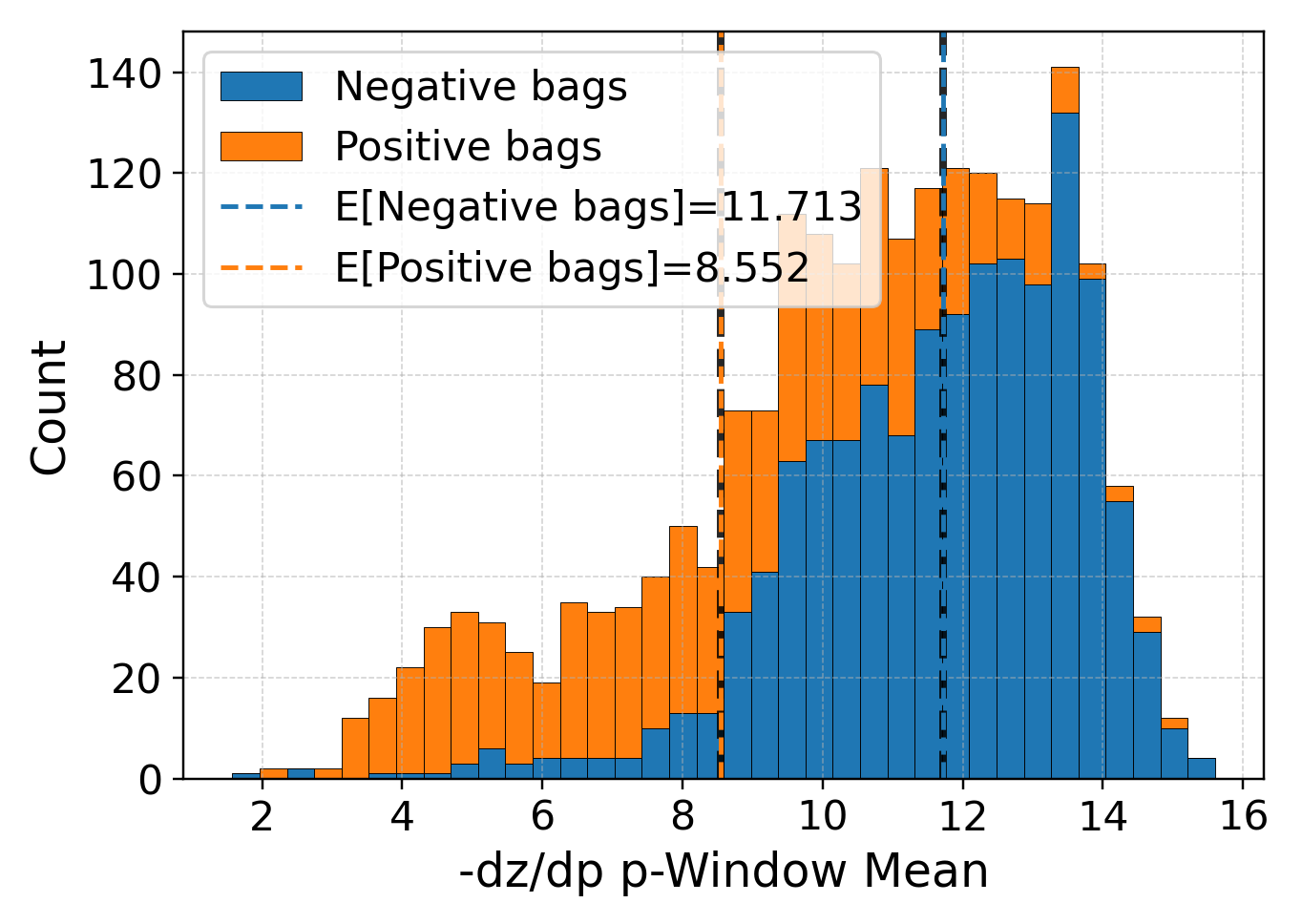}
        \caption{BioASQ\ (Mistral-12B)}
    \end{subfigure}\hfill
    \begin{subfigure}{0.32\textwidth}
        \includegraphics[width=\linewidth]{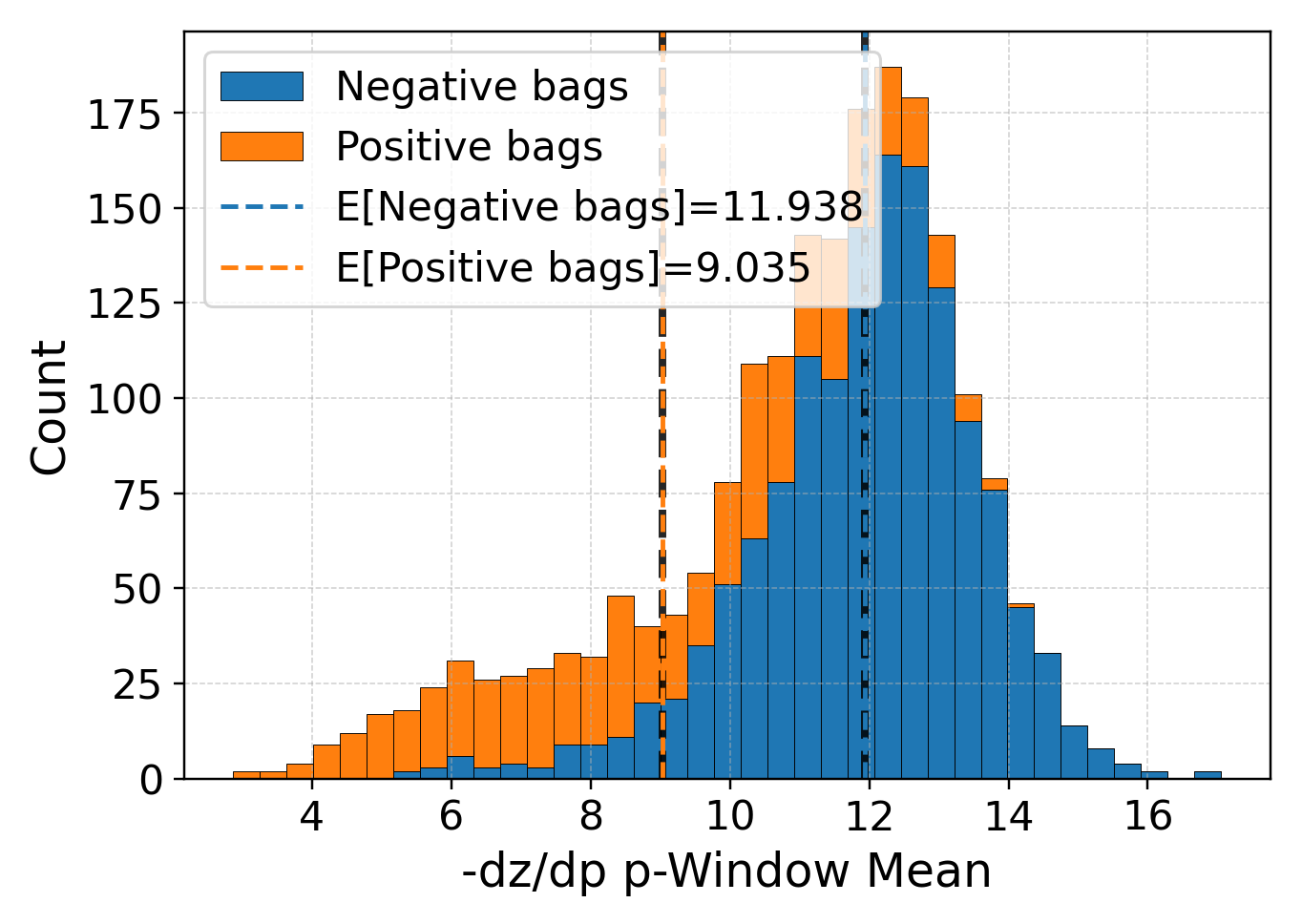}
        \caption{BioASQ\ (Llama-3.3-70B)}
    \end{subfigure}\hfill
    \vspace{0.3cm}
    \caption{Empirical distribution of $\bar{C}_{\mathbf{B}}^{\mathrm{int}}$ in \textbf{train} data (Non-invariant case).}
    \label{fig:dist_sensitivity_train_non_invariant}
\end{figure}

\begin{figure}[H]
    \centering
    \begin{subfigure}{0.32\textwidth}
        \includegraphics[width=\linewidth]{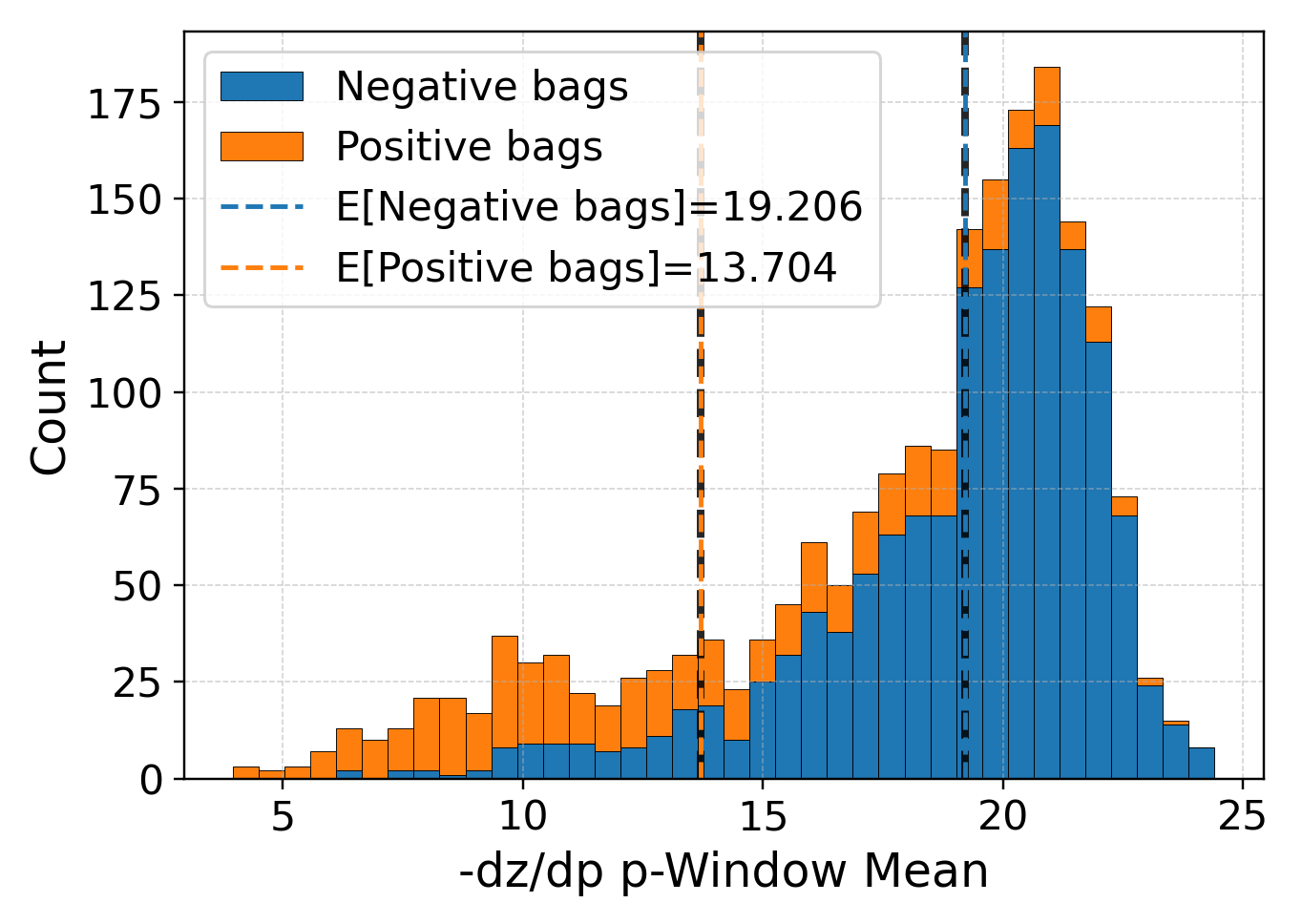}
        \caption{TriviaQA\ (Llama-3.1-8B)}
    \end{subfigure}\hfill
    \begin{subfigure}{0.32\textwidth}
        \includegraphics[width=\linewidth]{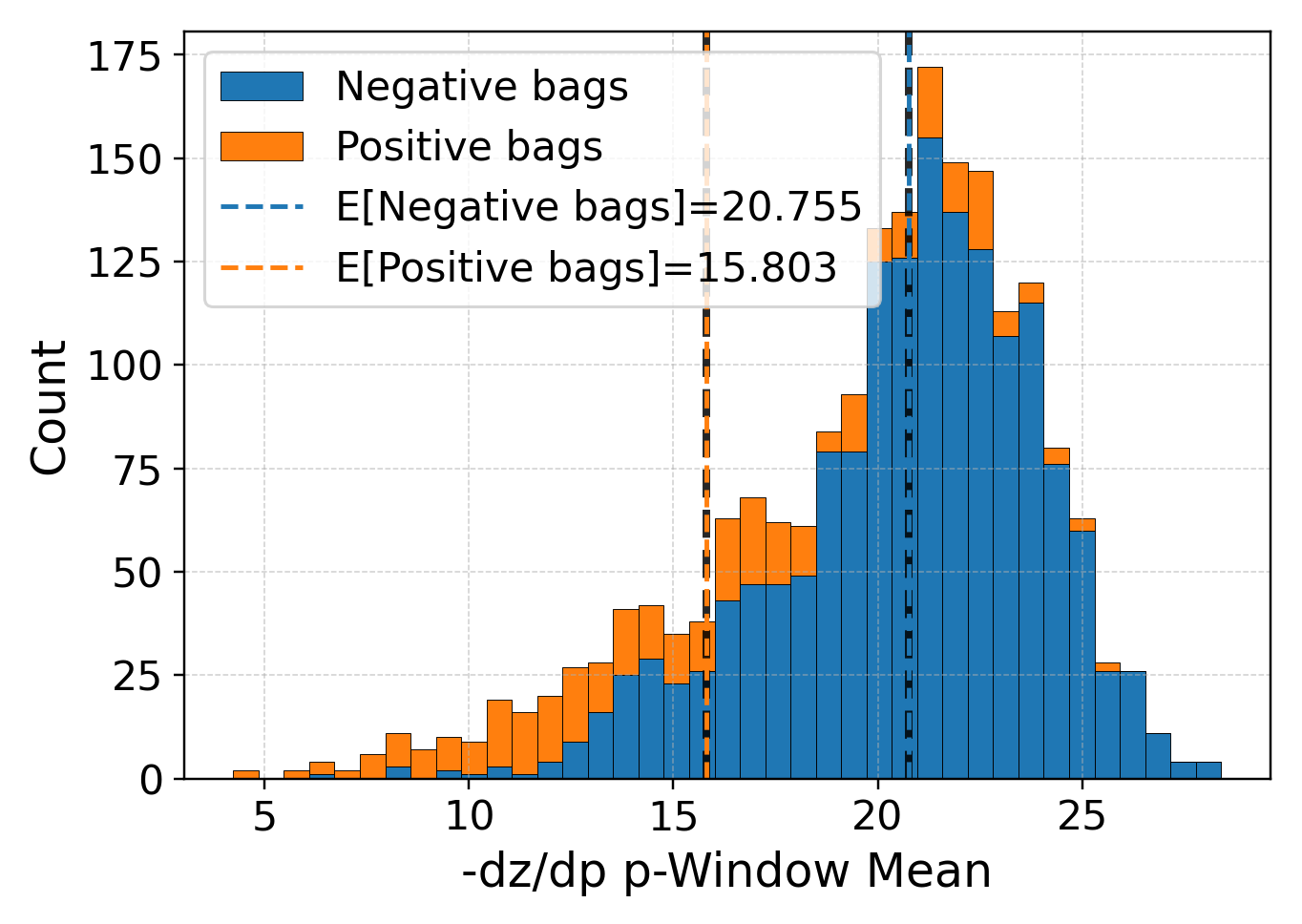}
        \caption{TriviaQA\ (Mistral-12B)}
    \end{subfigure}\hfill
    \begin{subfigure}{0.32\textwidth}
        \includegraphics[width=\linewidth]{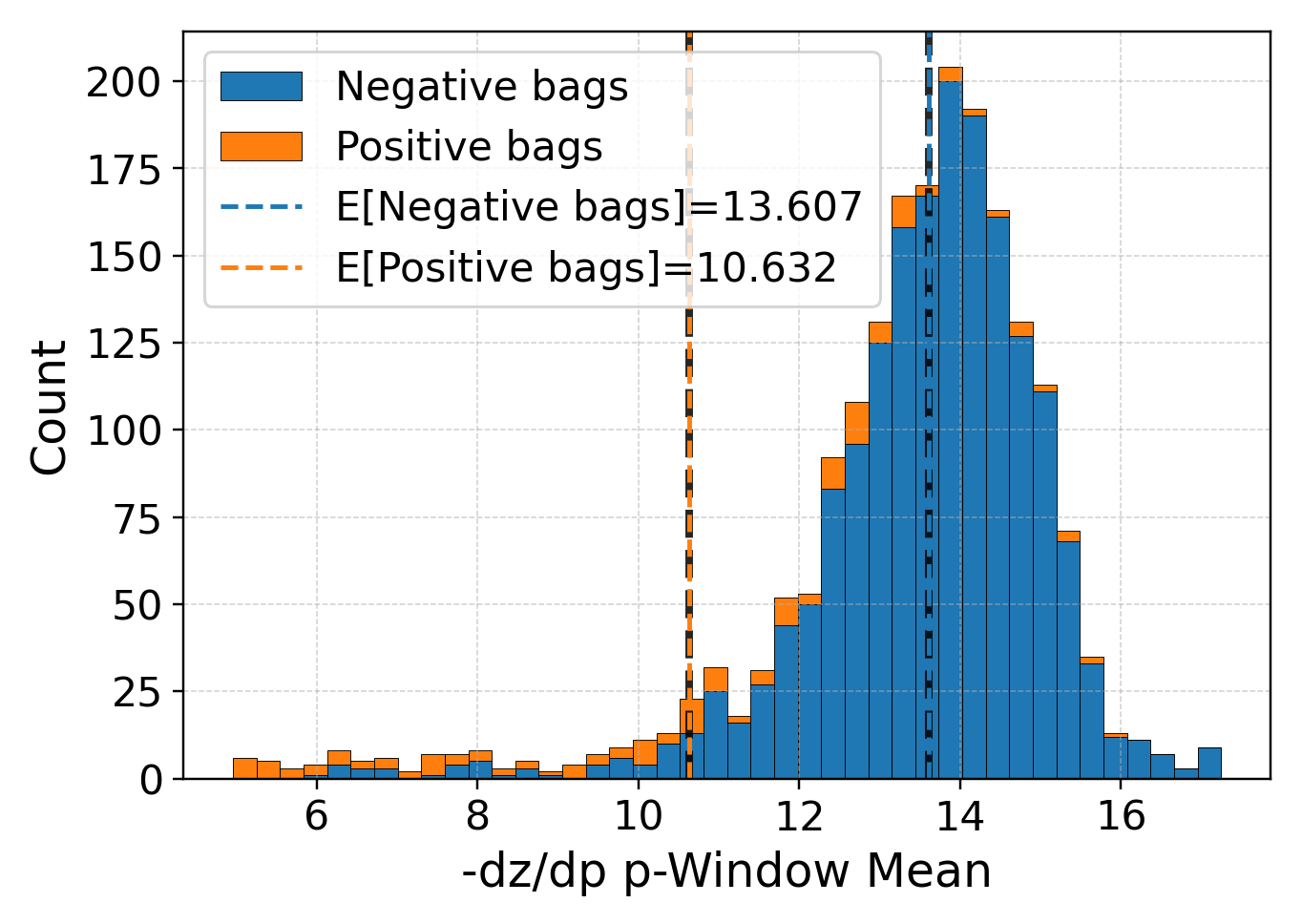}
        \caption{TriviaQA\ (Llama-3.3-70B)}
    \end{subfigure}\hfill
    \vspace{0.3cm}
    \begin{subfigure}{0.32\textwidth}
        \includegraphics[width=\linewidth]{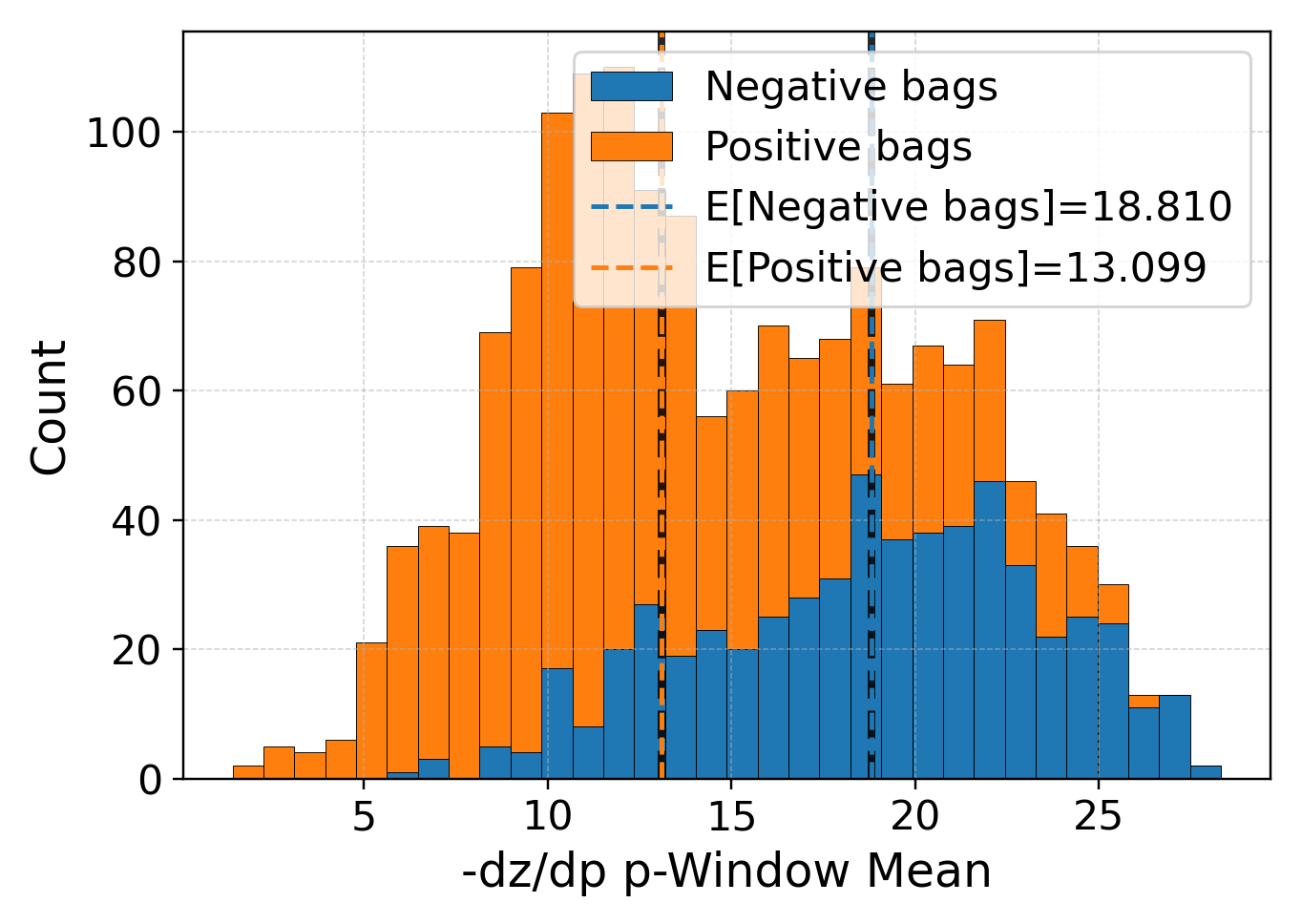}
        \caption{SQuAD\ (Llama-3.1-8B)}
    \end{subfigure}\hfill
    \begin{subfigure}{0.32\textwidth}
        \includegraphics[width=\linewidth]{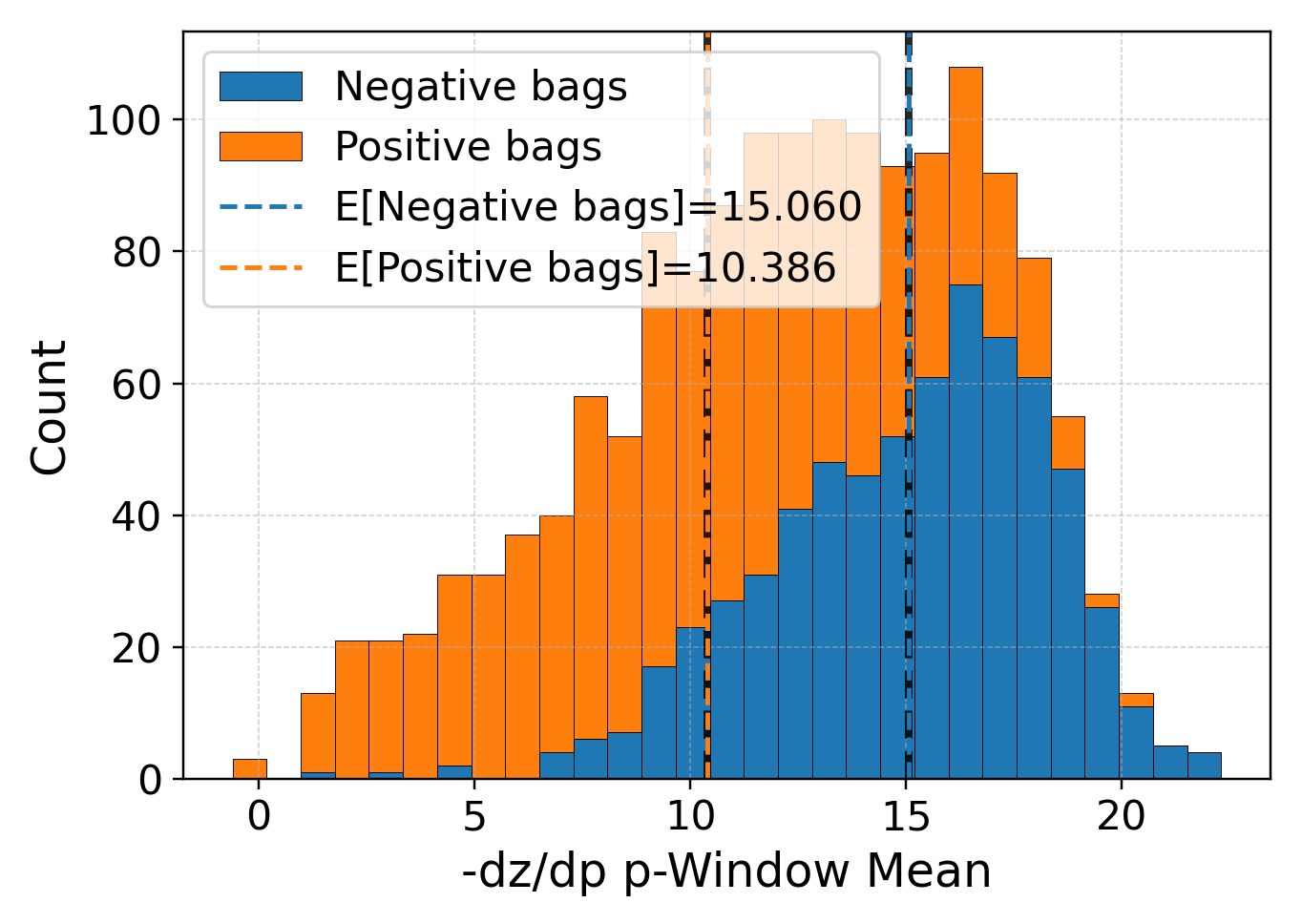}
        \caption{SQuAD\ (Mistral-12B)}
    \end{subfigure}\hfill
    \begin{subfigure}{0.32\textwidth}
        \includegraphics[width=\linewidth]{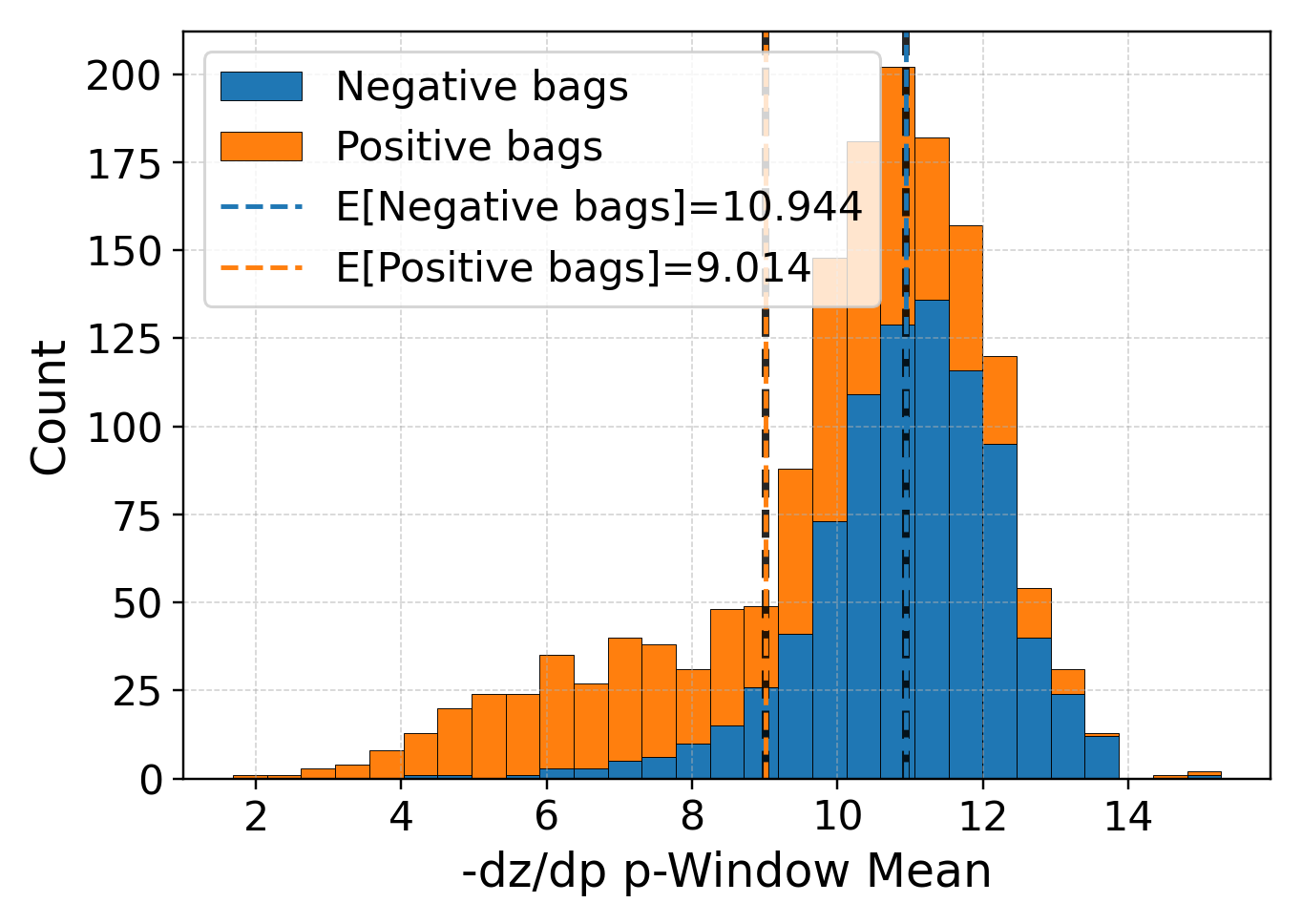}
        \caption{SQuAD\ (Llama-3.3-70B)}
    \end{subfigure}\hfill
    \vspace{0.3cm}
    \begin{subfigure}{0.32\textwidth}
        \includegraphics[width=\linewidth]{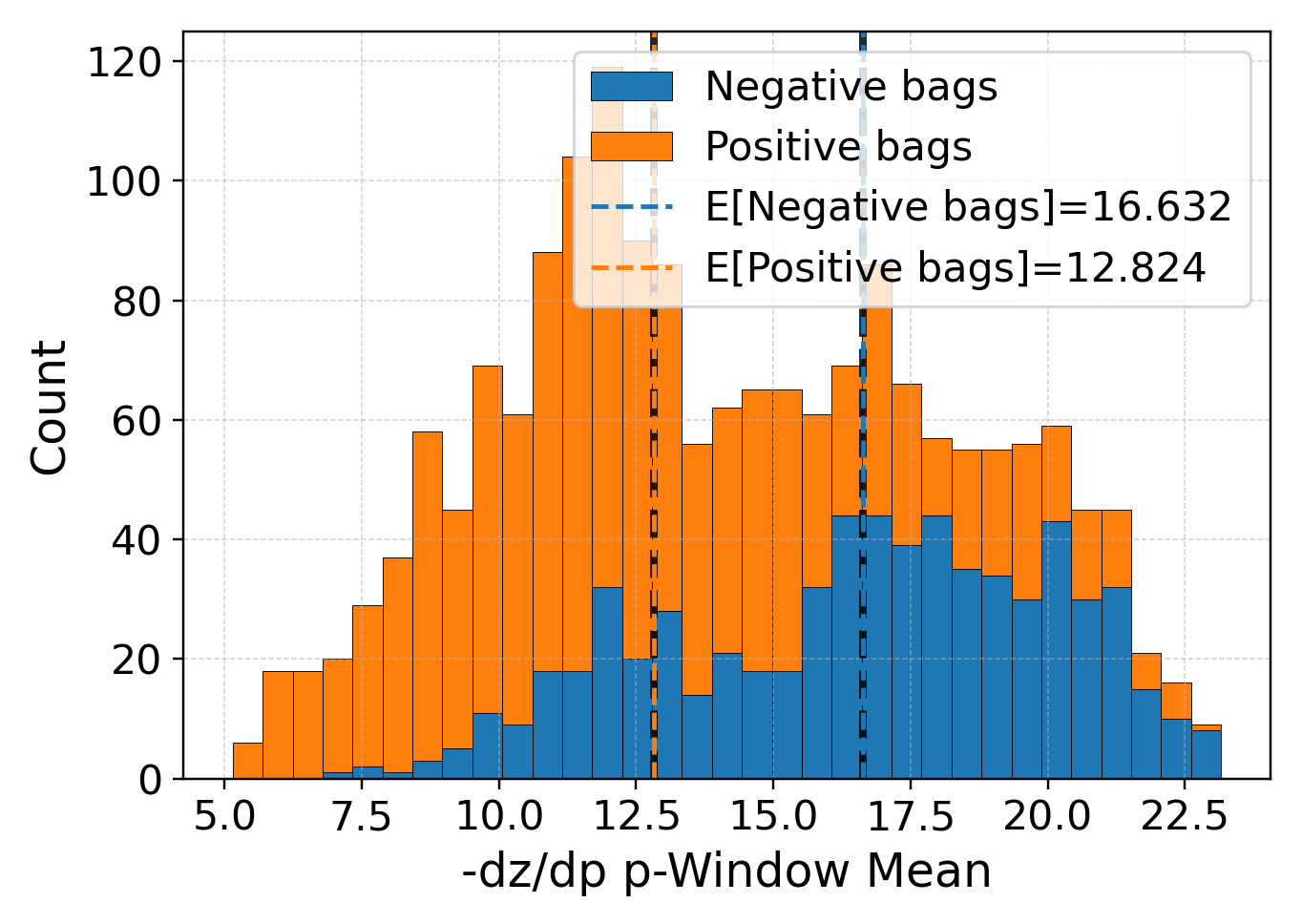}
        \caption{NQ\ (Llama-3.1-8B)}
    \end{subfigure}\hfill
    \begin{subfigure}{0.32\textwidth}
        \includegraphics[width=\linewidth]{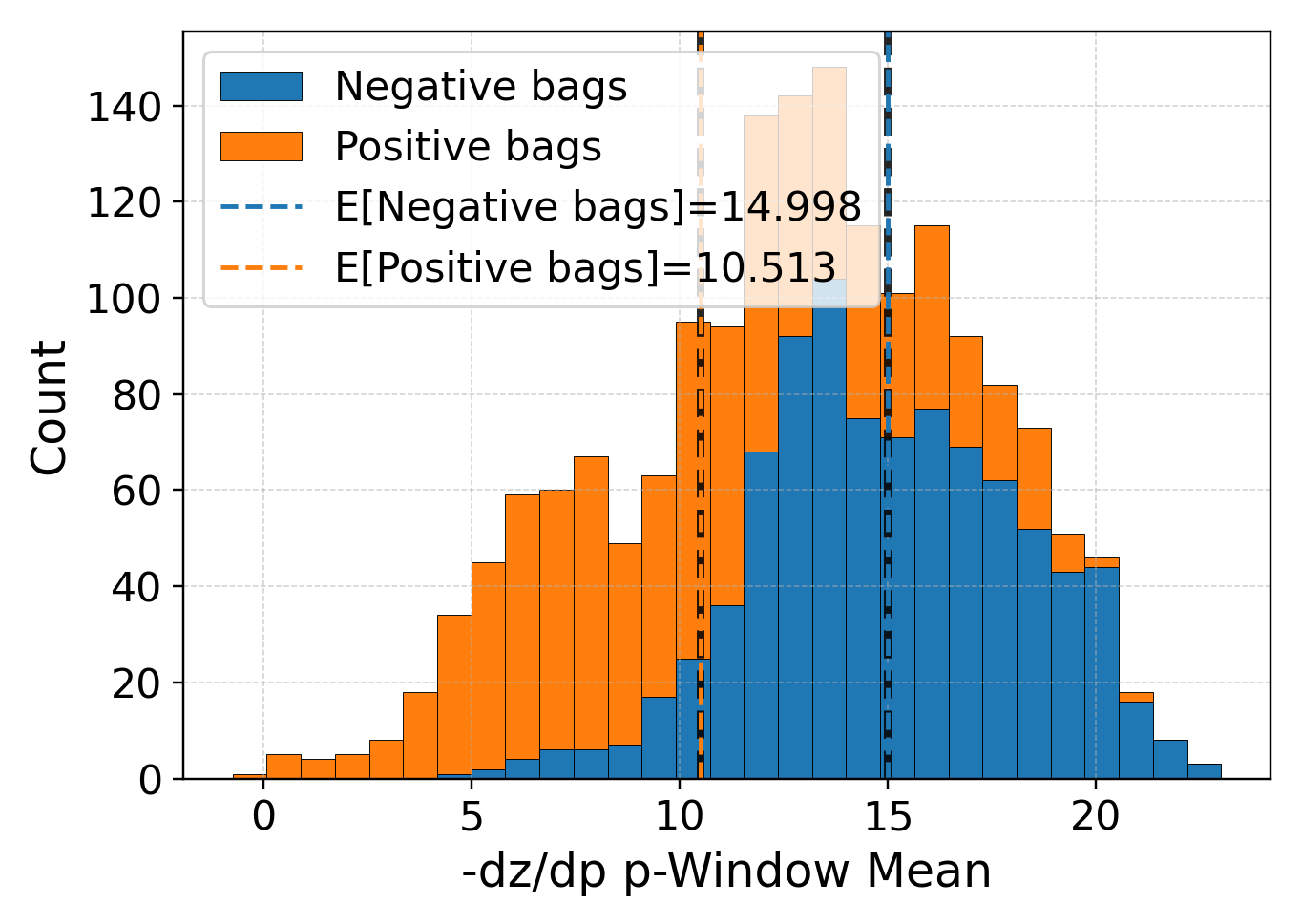}
        \caption{NQ\ (Mistral-12B)}
    \end{subfigure}\hfill
    \begin{subfigure}{0.32\textwidth}
        \includegraphics[width=\linewidth]{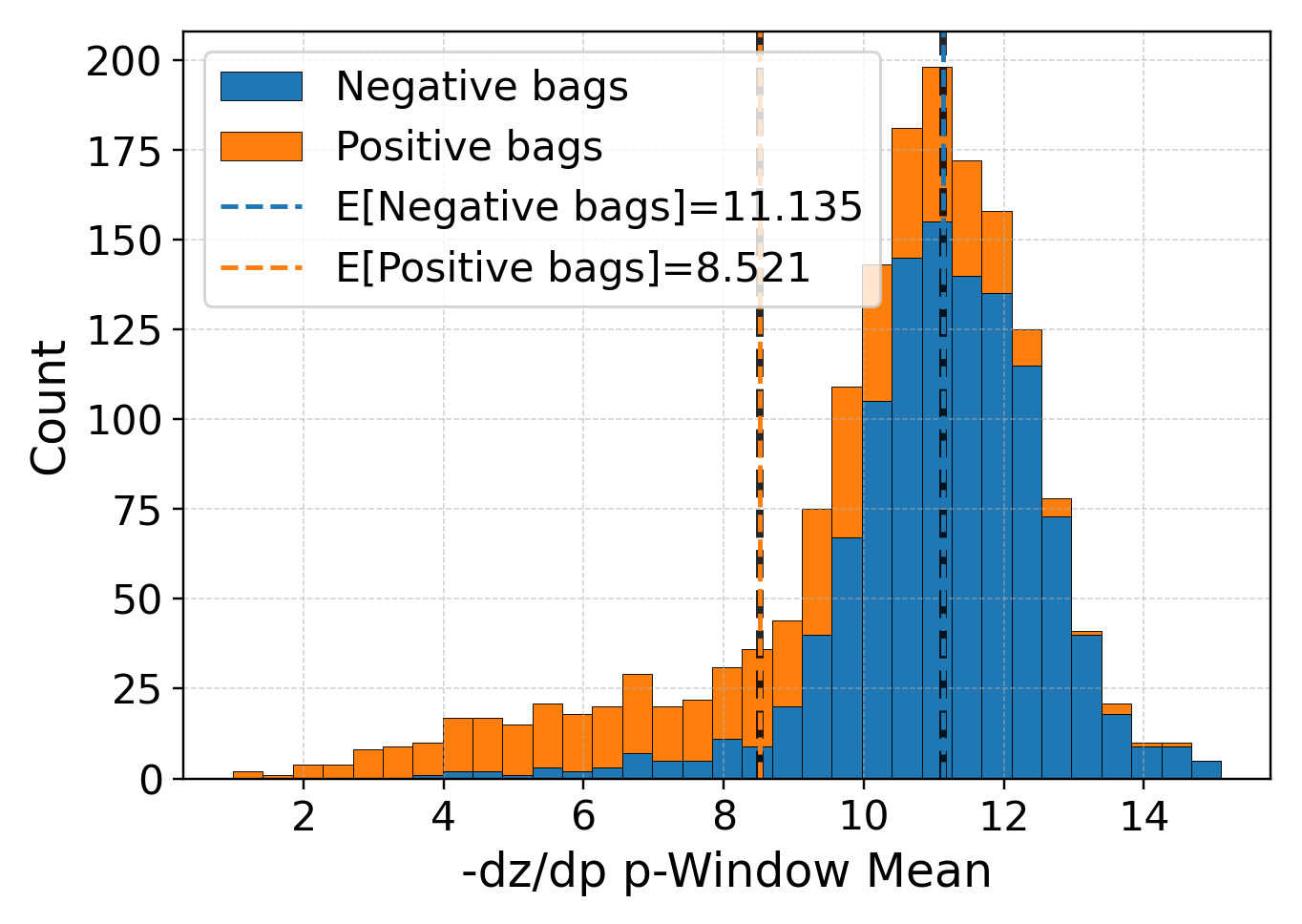}
        \caption{NQ\ (Llama-3.3-70B)}
    \end{subfigure}\hfill
    \vspace{0.3cm}
    \begin{subfigure}{0.32\textwidth}
        \includegraphics[width=\linewidth]{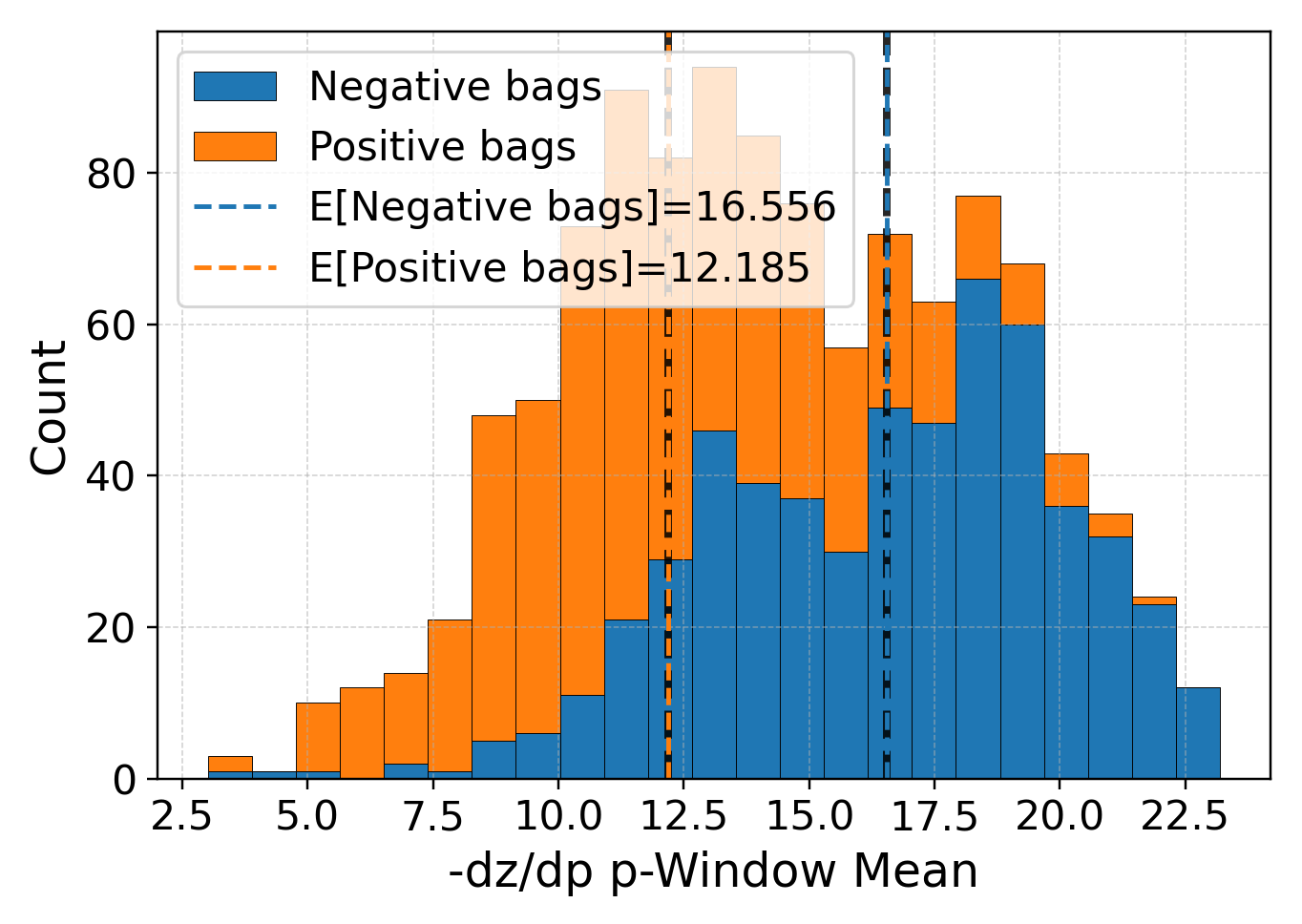}
        \caption{BioASQ\ (Llama-3.1-8B)}
    \end{subfigure}\hfill
    \begin{subfigure}{0.32\textwidth}
        \includegraphics[width=\linewidth]{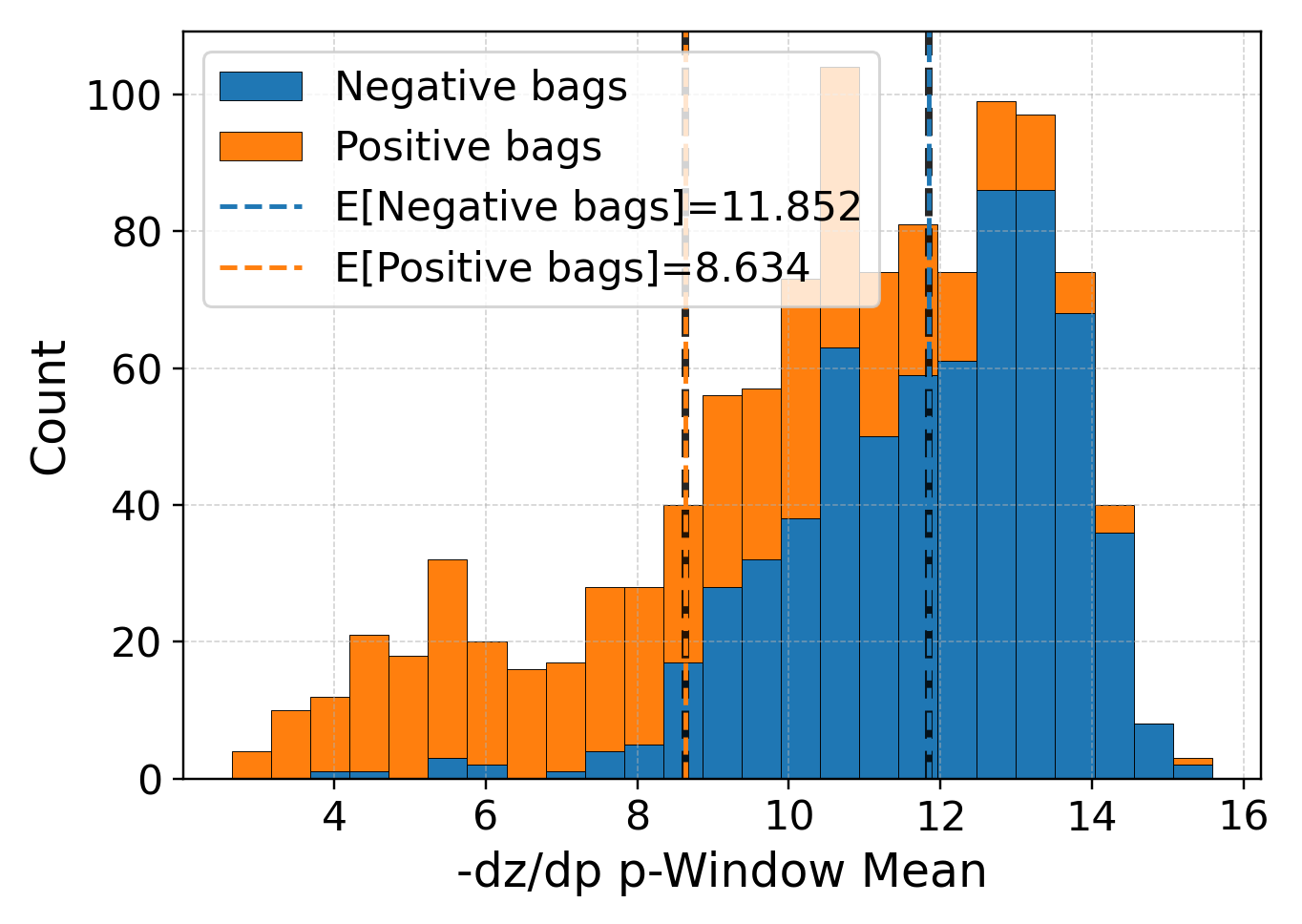}
        \caption{BioASQ\ (Mistral-12B)}
    \end{subfigure}\hfill
    \begin{subfigure}{0.32\textwidth}
        \includegraphics[width=\linewidth]{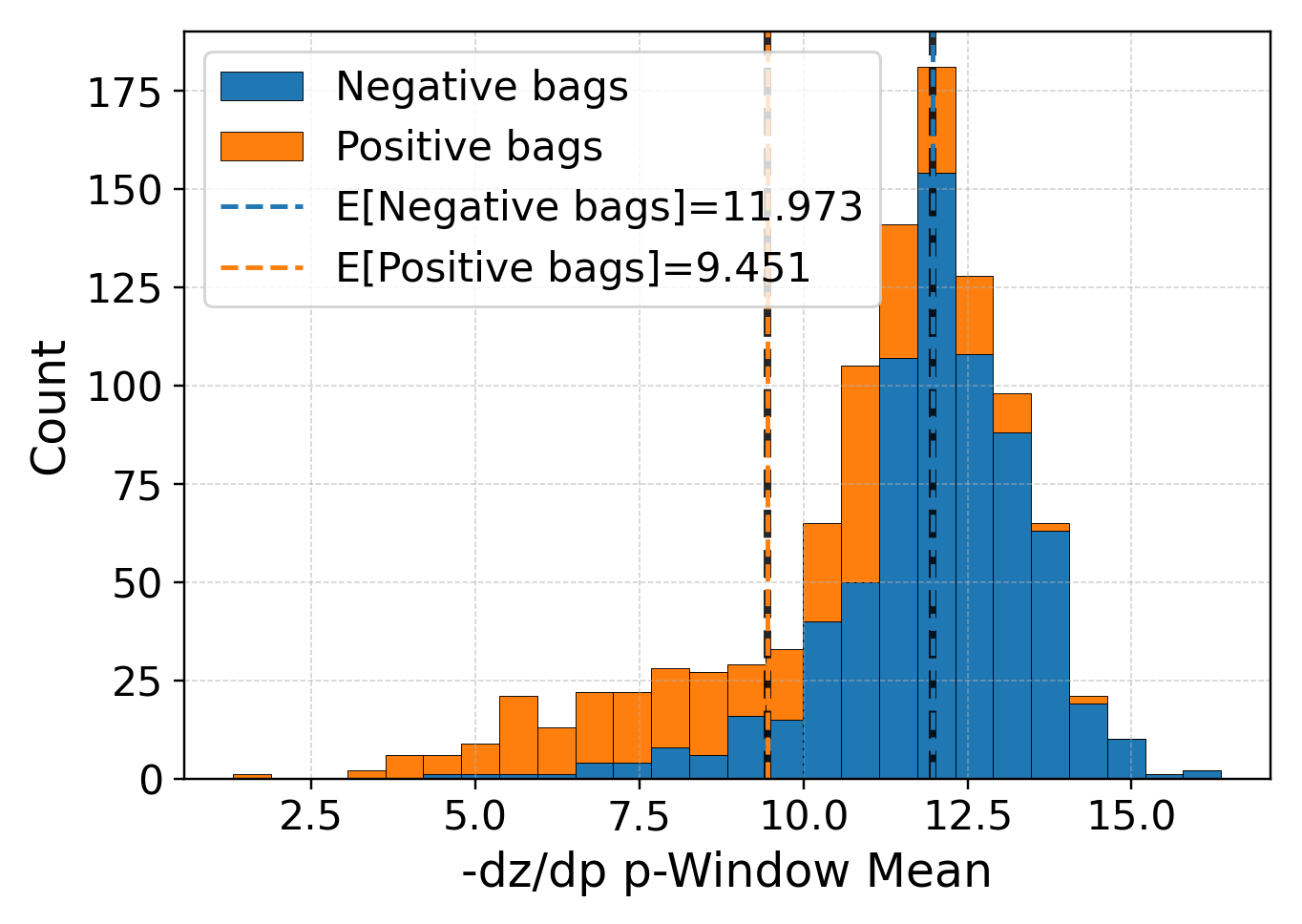}
        \caption{BioASQ\ (Llama-3.3-70B)}
    \end{subfigure}\hfill
    \vspace{0.3cm}
    \caption{Empirical distribution of $\bar{C}_{\mathbf{B}}^{\mathrm{int}}$ in \textbf{test} data (Non-invariant case).}
    \label{fig:dist_sensitivity_val_non_invariant}
\end{figure}
\clearpage

\subsection{Distribution of \texorpdfstring{$P_{\mathrm{sem}}^{\mathbf{B}}\bar{C}_{\mathbf{B}}^{\mathrm{int}}$}{P-sem C-bar int}}

We computed the joint product $P_{\mathrm{sem}}^{\mathbf{B}} \bar{C}_{\mathbf{B}}^{\mathrm{int}}$ for each bag (Figures~\ref{fig:dist_joint_train_non_invariant}--\ref{fig:dist_joint_val_non_invariant}). All values were calculated with $\lambda=1$ and using the integrated cost $\bar{C}_{\mathbf{B}}^{\mathrm{int}}$ via numerical approximation with $N=1000$ steps.

Although individual bag distributions exhibit some overlap, the negative class demonstrates a higher density at larger values. The empirical class-wise expectations satisfy $\mathbb{E}_{\mathrm{neg}}[P_{\mathrm{sem}}^{\mathbf{B}} \bar{C}_{\mathbf{B}}^{\mathrm{int}}] > \mathbb{E}_{\mathrm{pos}}[P_{\mathrm{sem}}^{\mathbf{B}} \bar{C}_{\mathbf{B}}^{\mathrm{int}}]$ across all evaluated datasets and splits. The expected values for each condition are provided in Appendix~\ref{sum}.

\begin{figure}[H]
    \centering
    \begin{subfigure}{0.32\textwidth}
        \includegraphics[width=\linewidth]{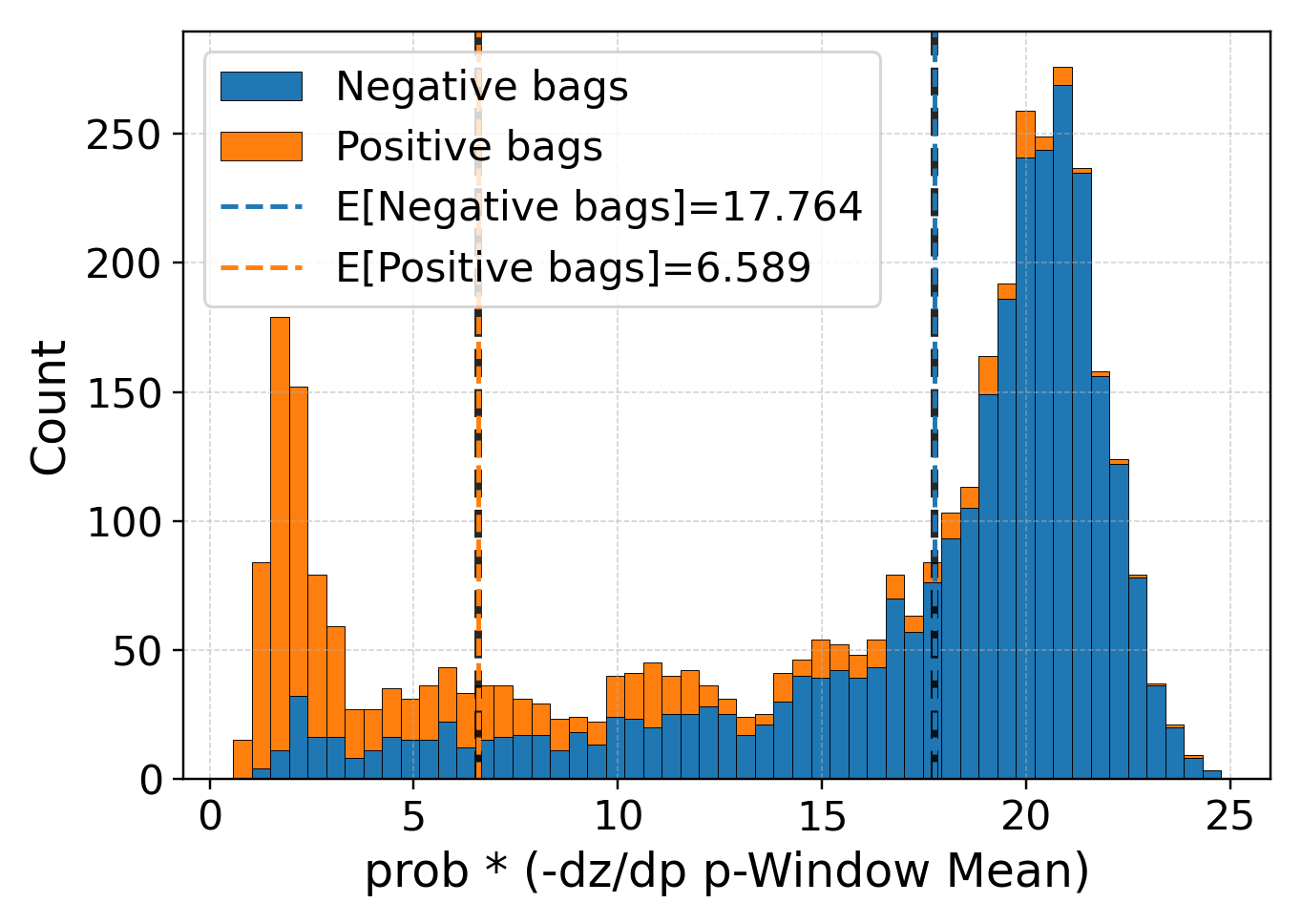}
        \caption{TriviaQA\ (Llama-3.1-8B)}
    \end{subfigure}\hfill
    \begin{subfigure}{0.32\textwidth}
        \includegraphics[width=\linewidth]{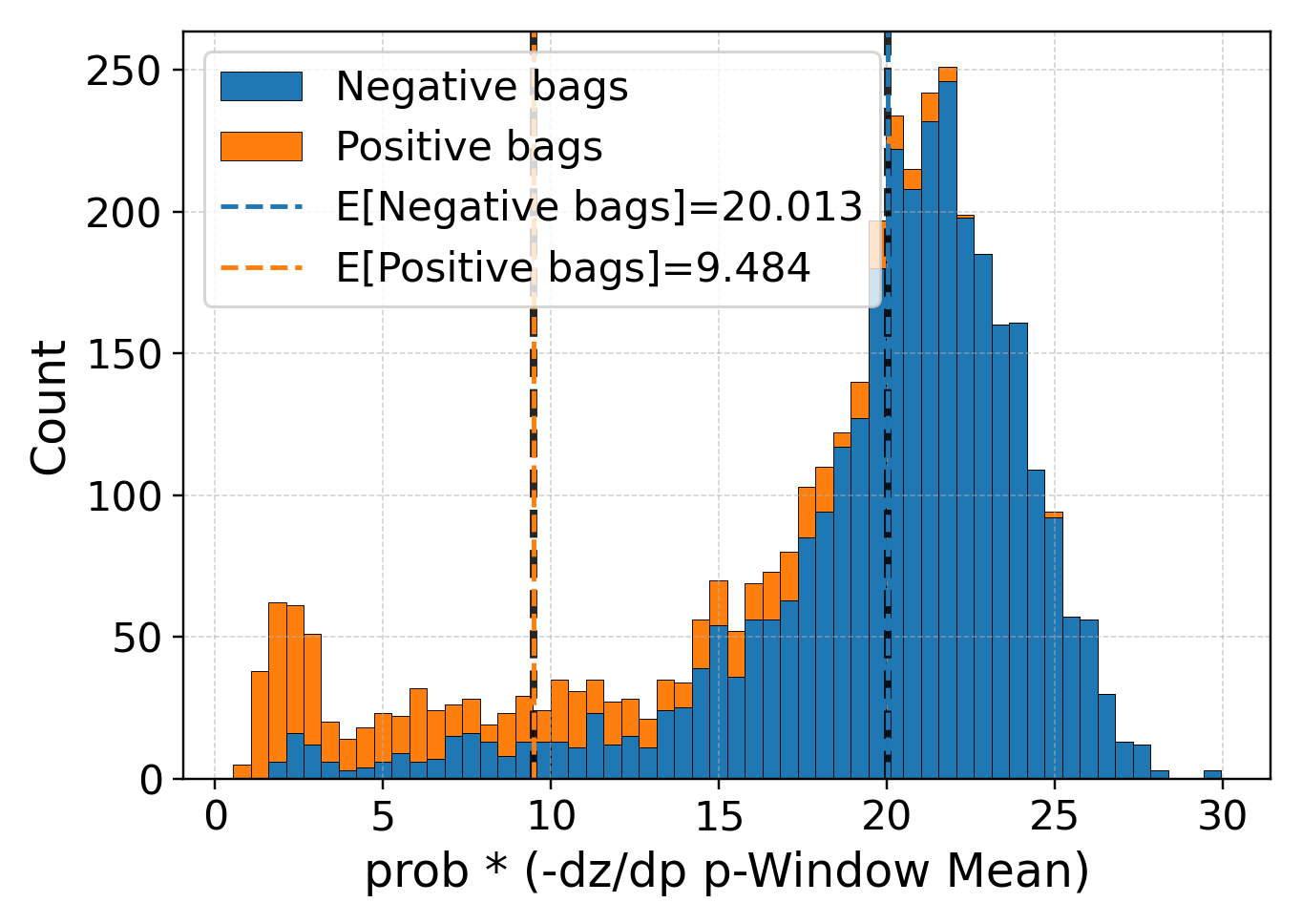}
        \caption{TriviaQA\ (Mistral-12B)}
    \end{subfigure}\hfill
    \begin{subfigure}{0.32\textwidth}
        \includegraphics[width=\linewidth]{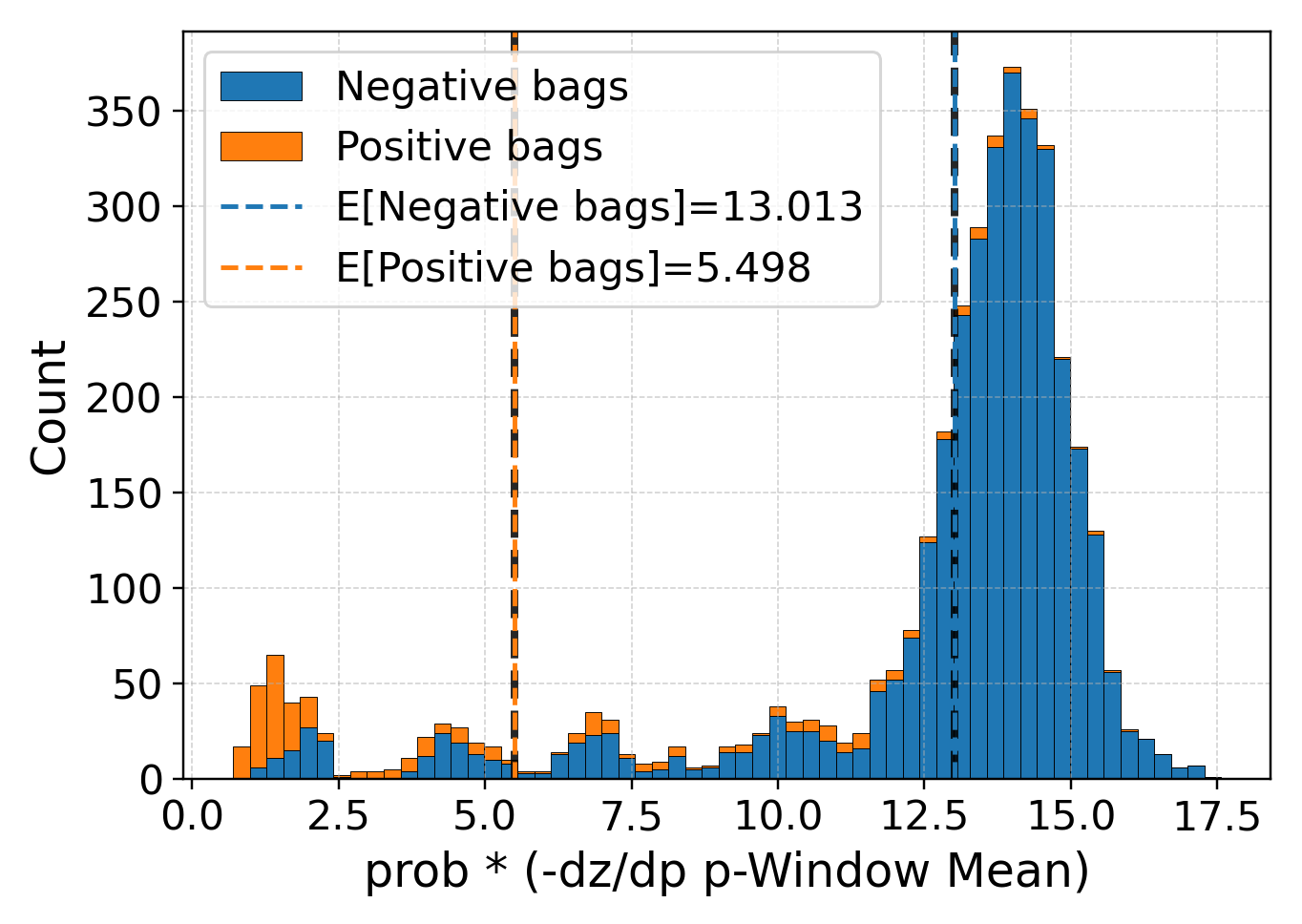}
        \caption{TriviaQA\ (Llama-3.3-70B)}
    \end{subfigure}\hfill
    \vspace{0.3cm}
    \begin{subfigure}{0.32\textwidth}
        \includegraphics[width=\linewidth]{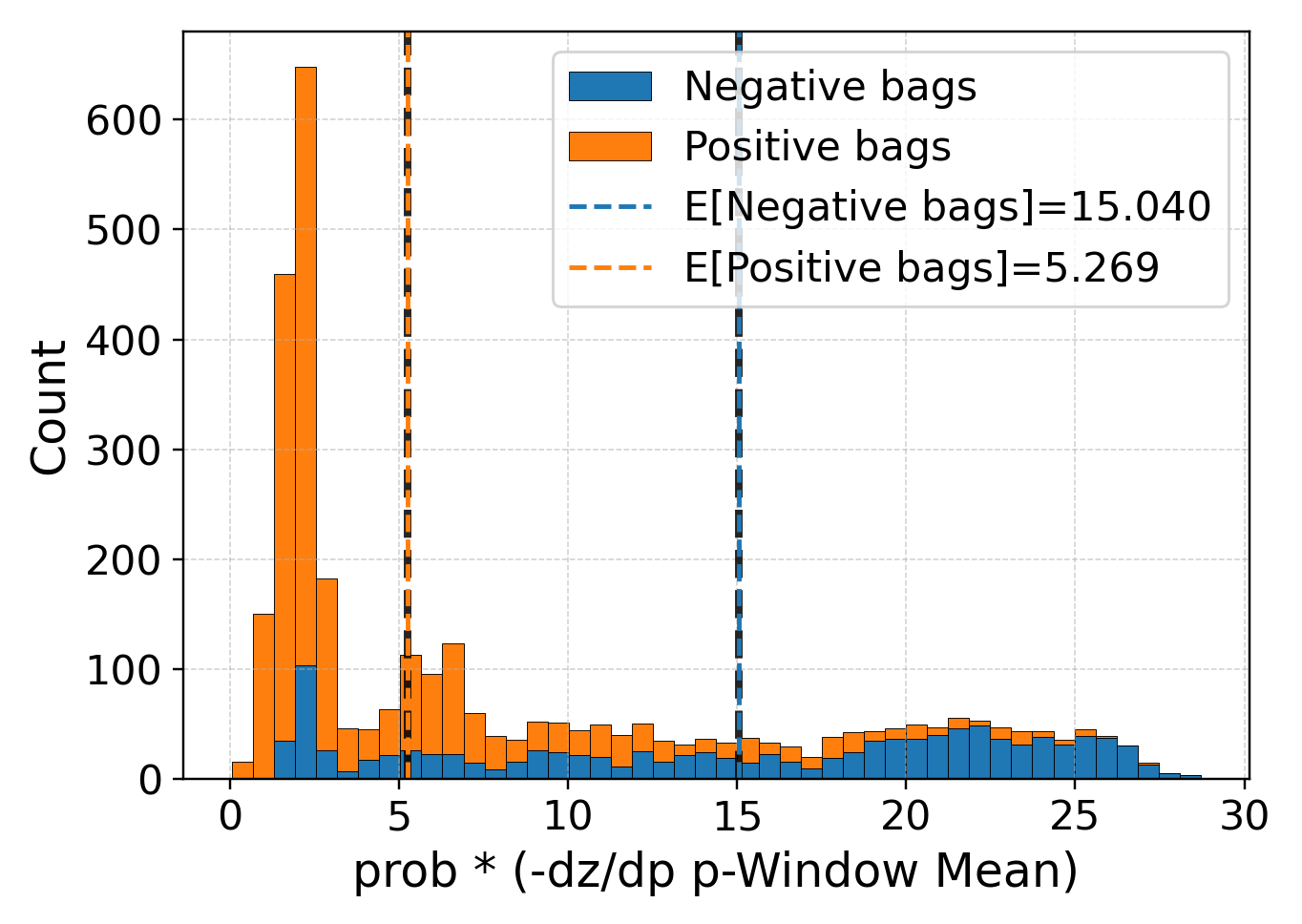}
        \caption{SQuAD\ (Llama-3.1-8B)}
    \end{subfigure}\hfill
    \begin{subfigure}{0.32\textwidth}
        \includegraphics[width=\linewidth]{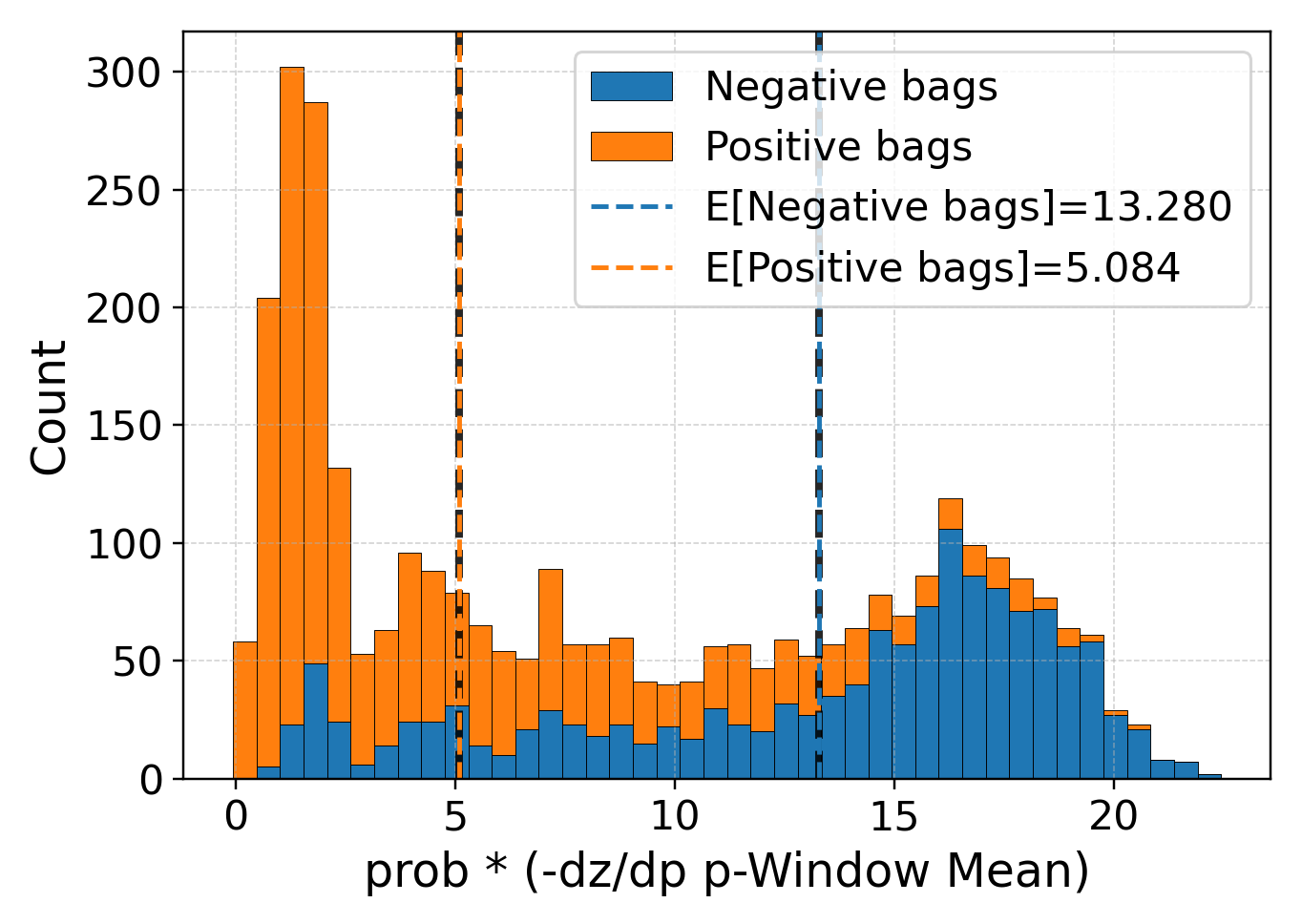}
        \caption{SQuAD\ (Mistral-12B)}
    \end{subfigure}\hfill
    \begin{subfigure}{0.32\textwidth}
        \includegraphics[width=\linewidth]{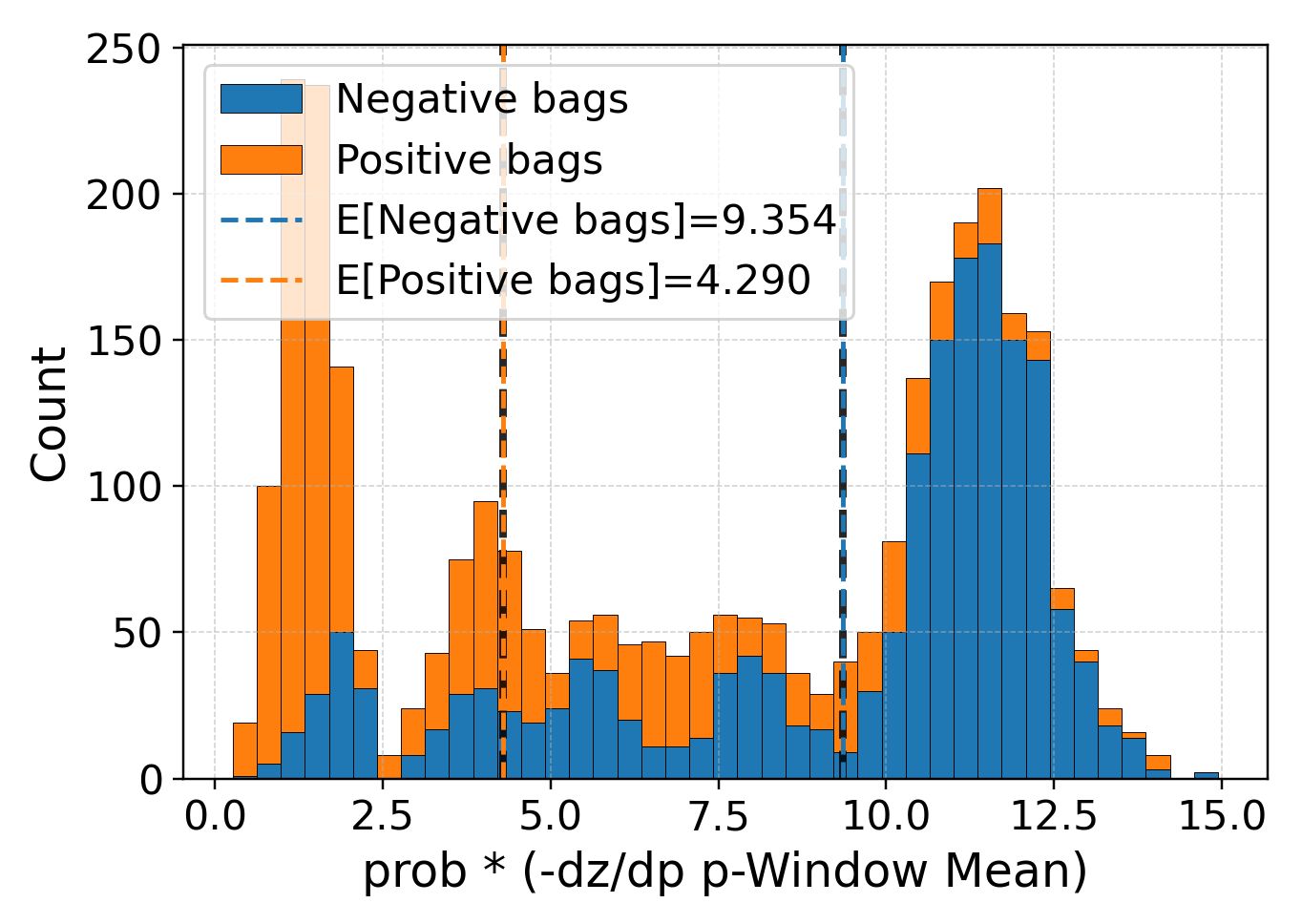}
        \caption{SQuAD\ (Llama-3.3-70B)}
    \end{subfigure}\hfill
    \vspace{0.3cm}
    \begin{subfigure}{0.32\textwidth}
        \includegraphics[width=\linewidth]{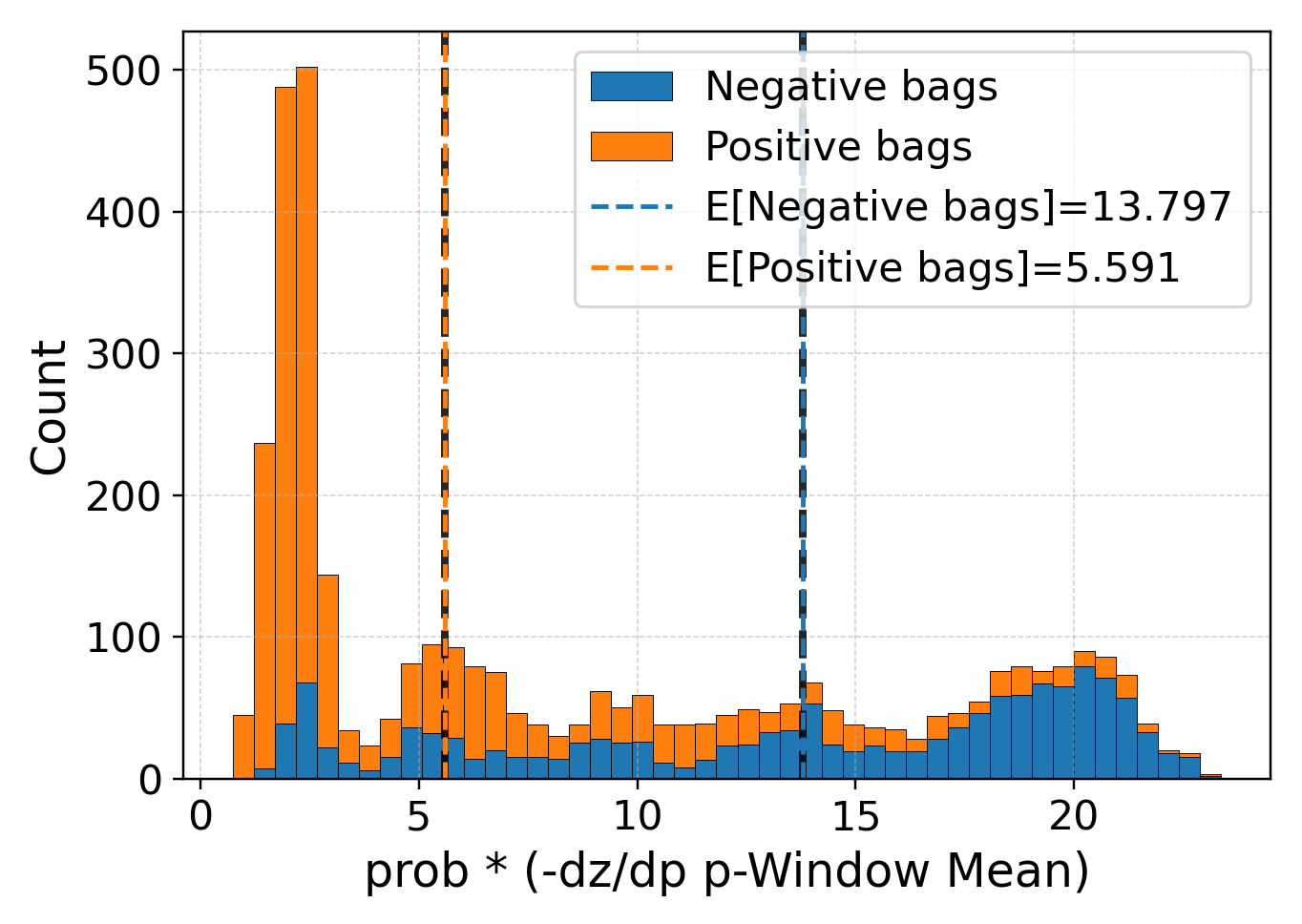}
        \caption{NQ\ (Llama-3.1-8B)}
    \end{subfigure}\hfill
    \begin{subfigure}{0.32\textwidth}
        \includegraphics[width=\linewidth]{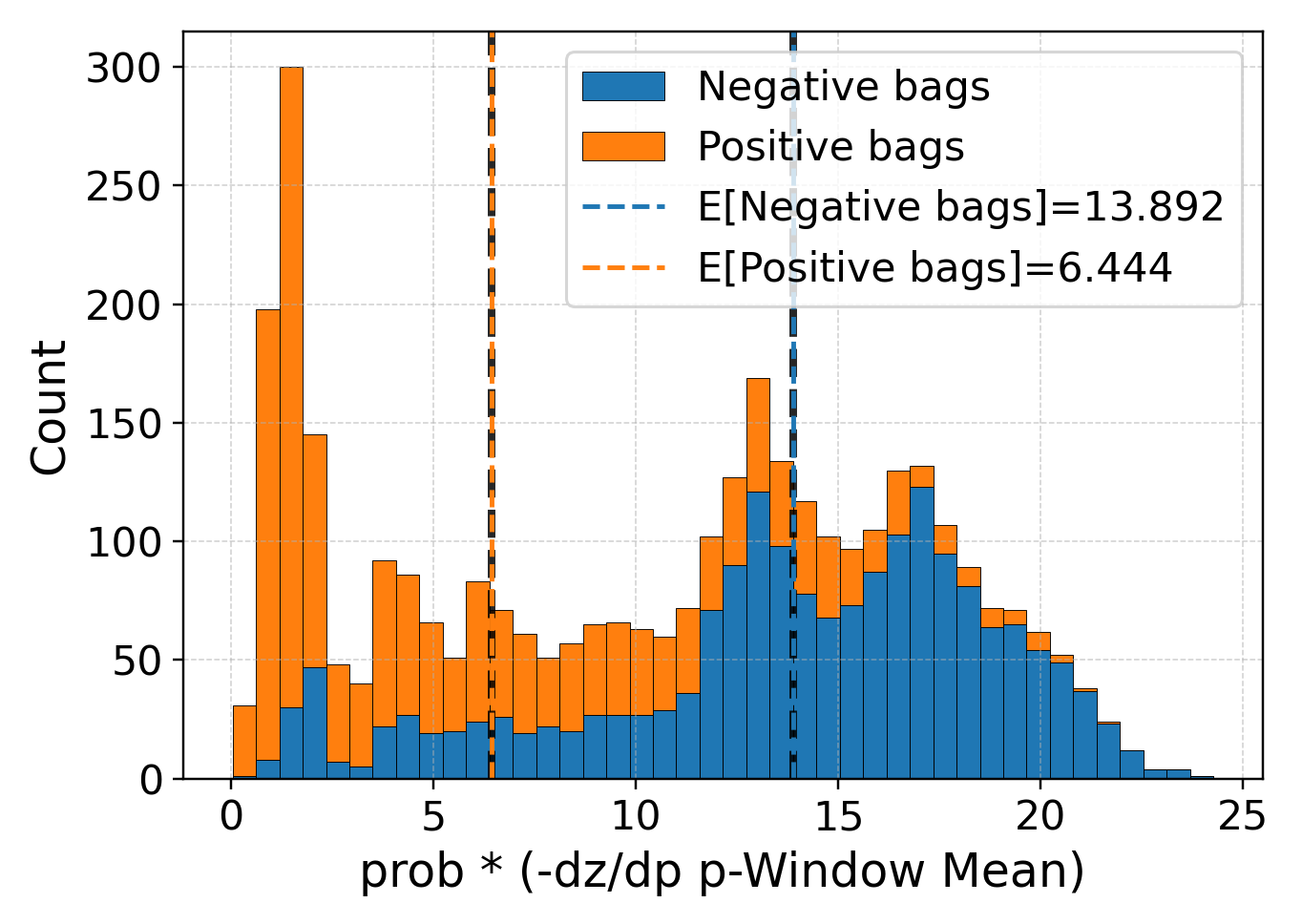}
        \caption{NQ\ (Mistral-12B)}
    \end{subfigure}\hfill
    \begin{subfigure}{0.32\textwidth}
        \includegraphics[width=\linewidth]{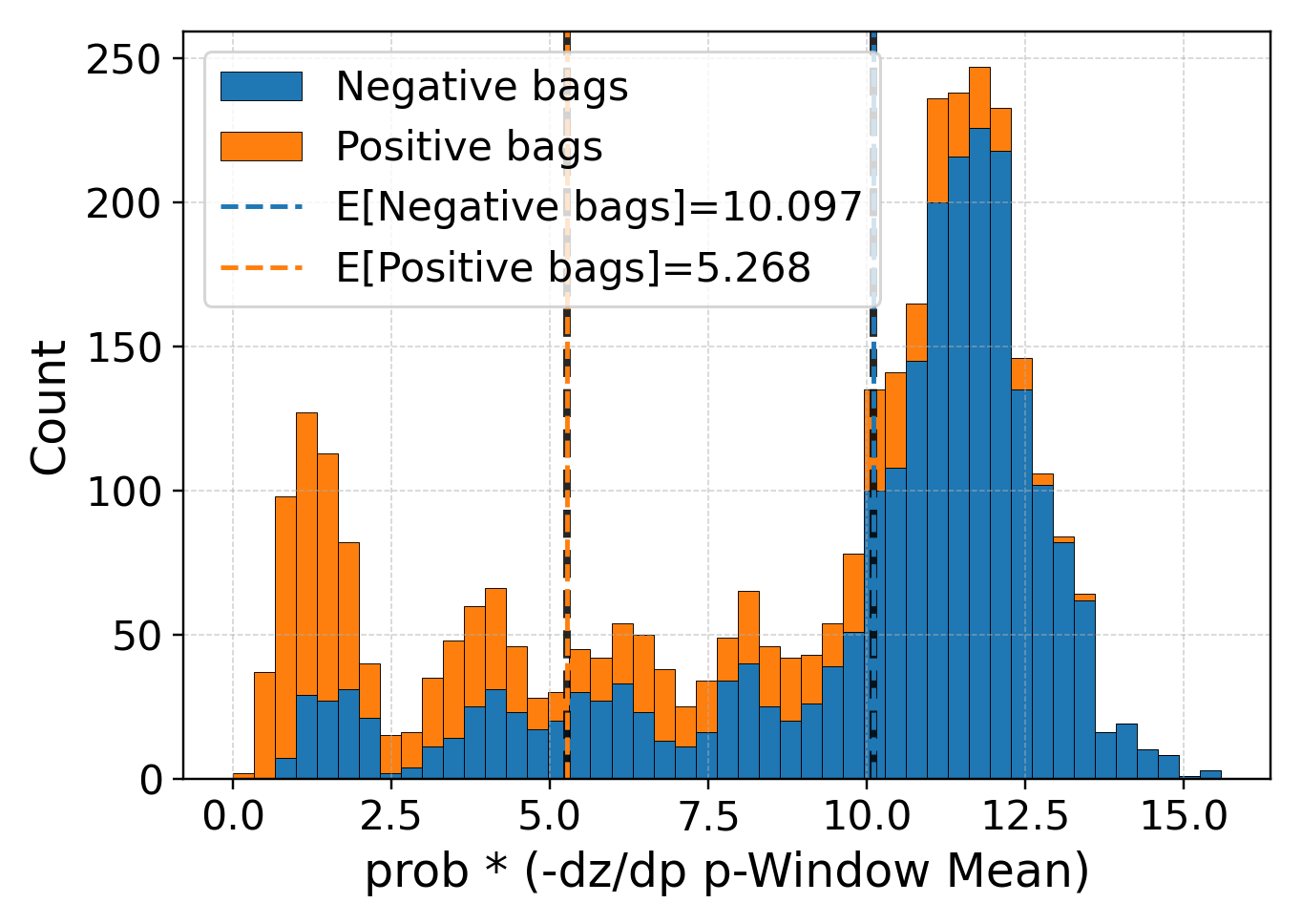}
        \caption{NQ\ (Llama-3.3-70B)}
    \end{subfigure}\hfill
    \vspace{0.3cm}
    \begin{subfigure}{0.32\textwidth}
        \includegraphics[width=\linewidth]{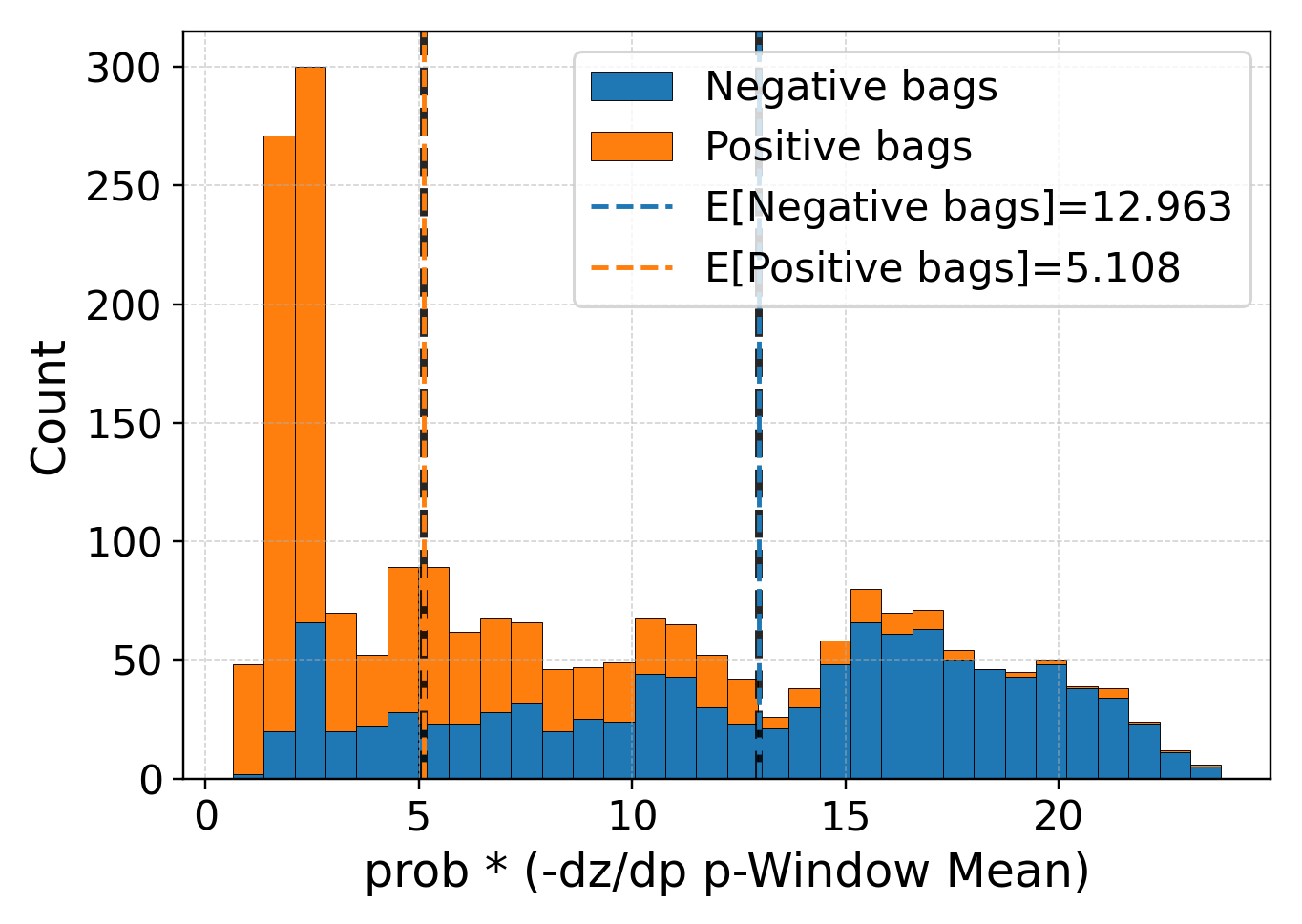}
        \caption{BioASQ\ (Llama-3.1-8B)}
    \end{subfigure}\hfill
    \begin{subfigure}{0.32\textwidth}
        \includegraphics[width=\linewidth]{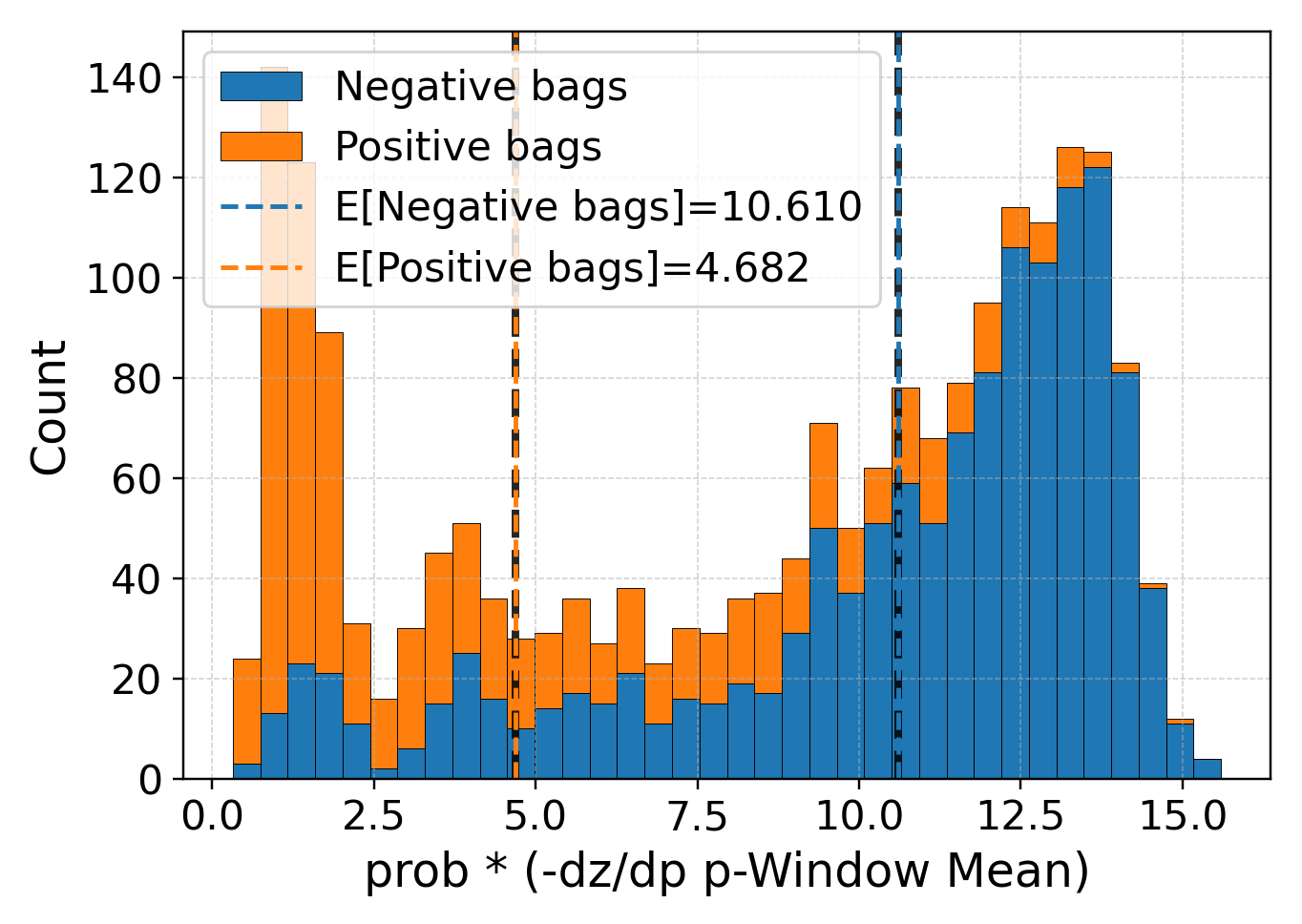}
        \caption{BioASQ\ (Mistral-12B)}
    \end{subfigure}\hfill
    \begin{subfigure}{0.32\textwidth}
        \includegraphics[width=\linewidth]{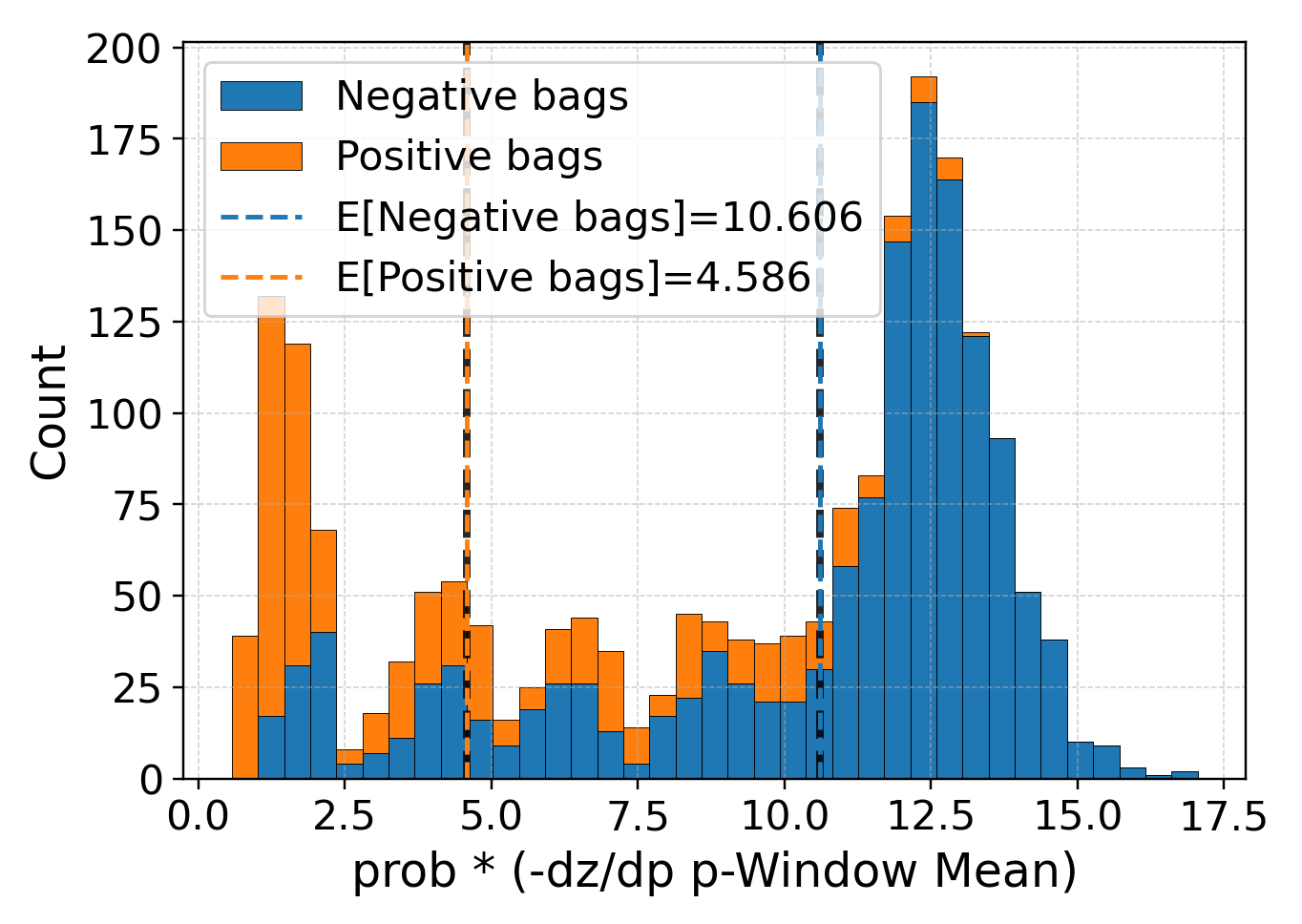}
        \caption{BioASQ\ (Llama-3.3-70B)}
    \end{subfigure}\hfill
    \vspace{0.3cm}
    \caption{Distribution of the joint product $P_{\mathrm{sem}}^{\mathbf{B}} \bar{C}_{\mathbf{B}}^{\mathrm{int}}$ in \textbf{train} data.}
    \label{fig:dist_joint_train_non_invariant}
\end{figure}

\begin{figure}[H]
    \centering
    \begin{subfigure}{0.32\textwidth}
        \includegraphics[width=\linewidth]{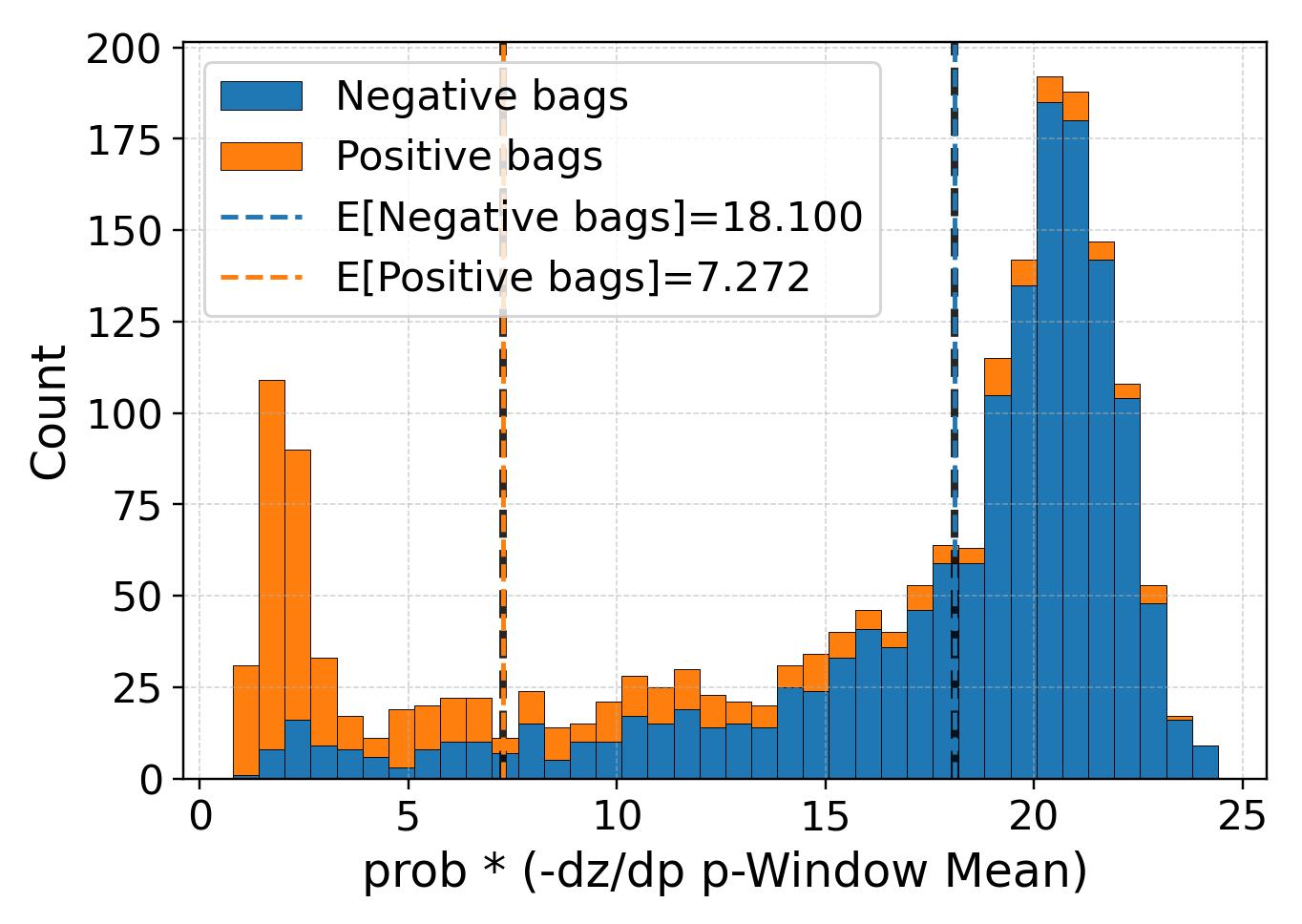}
        \caption{TriviaQA\ (Llama-3.1-8B)}
    \end{subfigure}\hfill
    \begin{subfigure}{0.32\textwidth}
        \includegraphics[width=\linewidth]{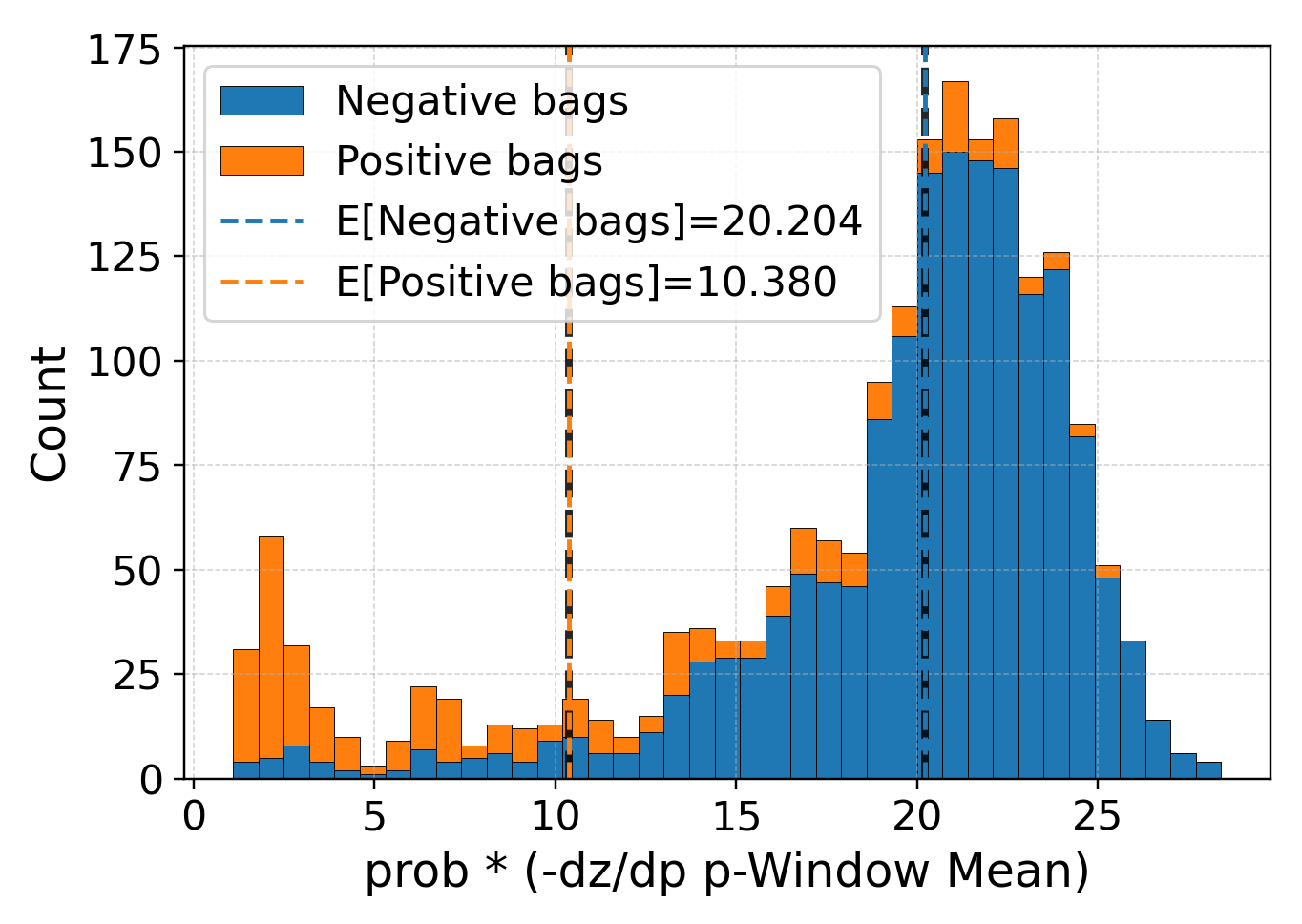}
        \caption{TriviaQA\ (Mistral-12B)}
    \end{subfigure}\hfill
    \begin{subfigure}{0.32\textwidth}
        \includegraphics[width=\linewidth]{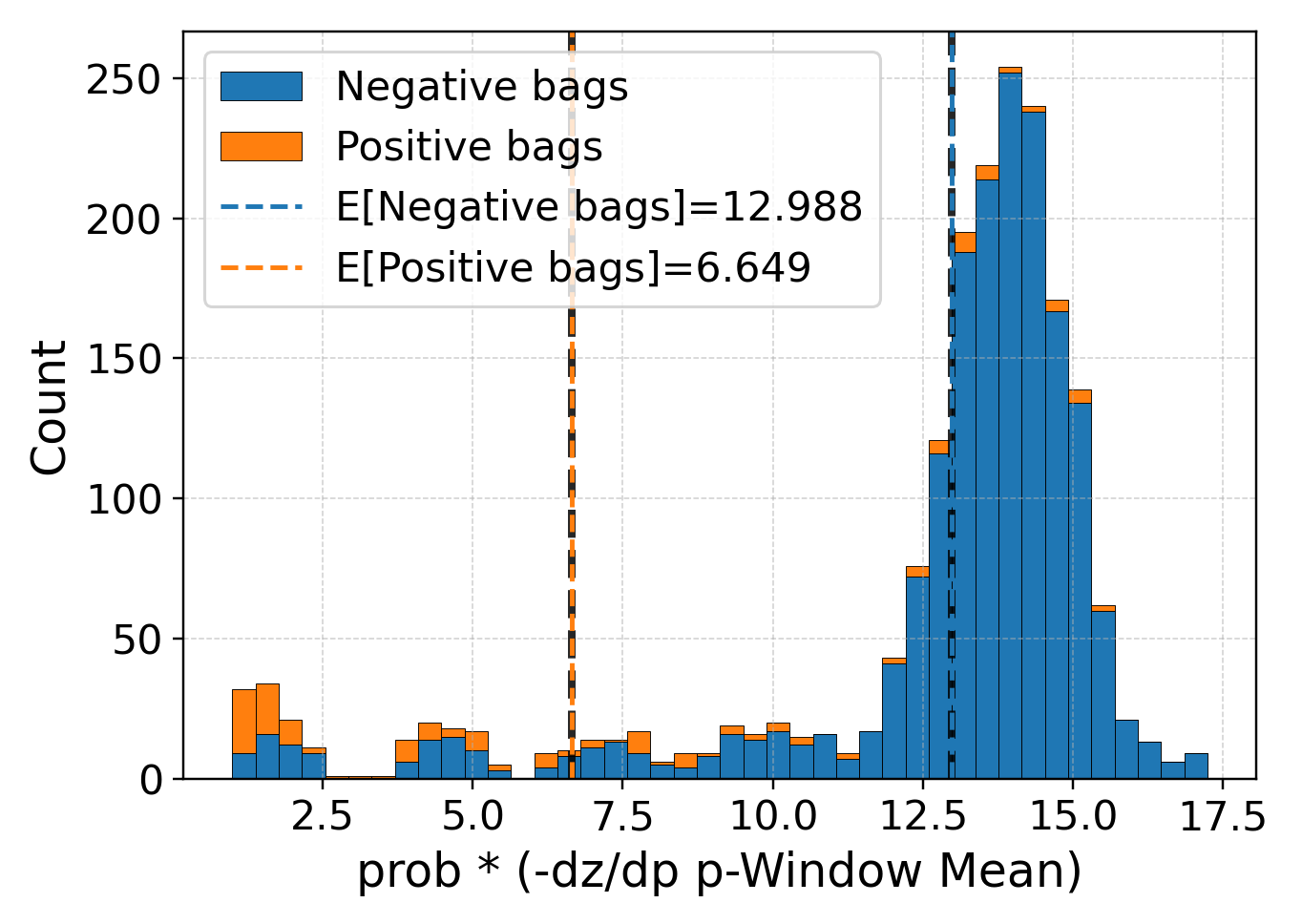}
        \caption{TriviaQA\ (Llama-3.3-70B)}
    \end{subfigure}\hfill
    \vspace{0.3cm}
    \begin{subfigure}{0.32\textwidth}
        \includegraphics[width=\linewidth]{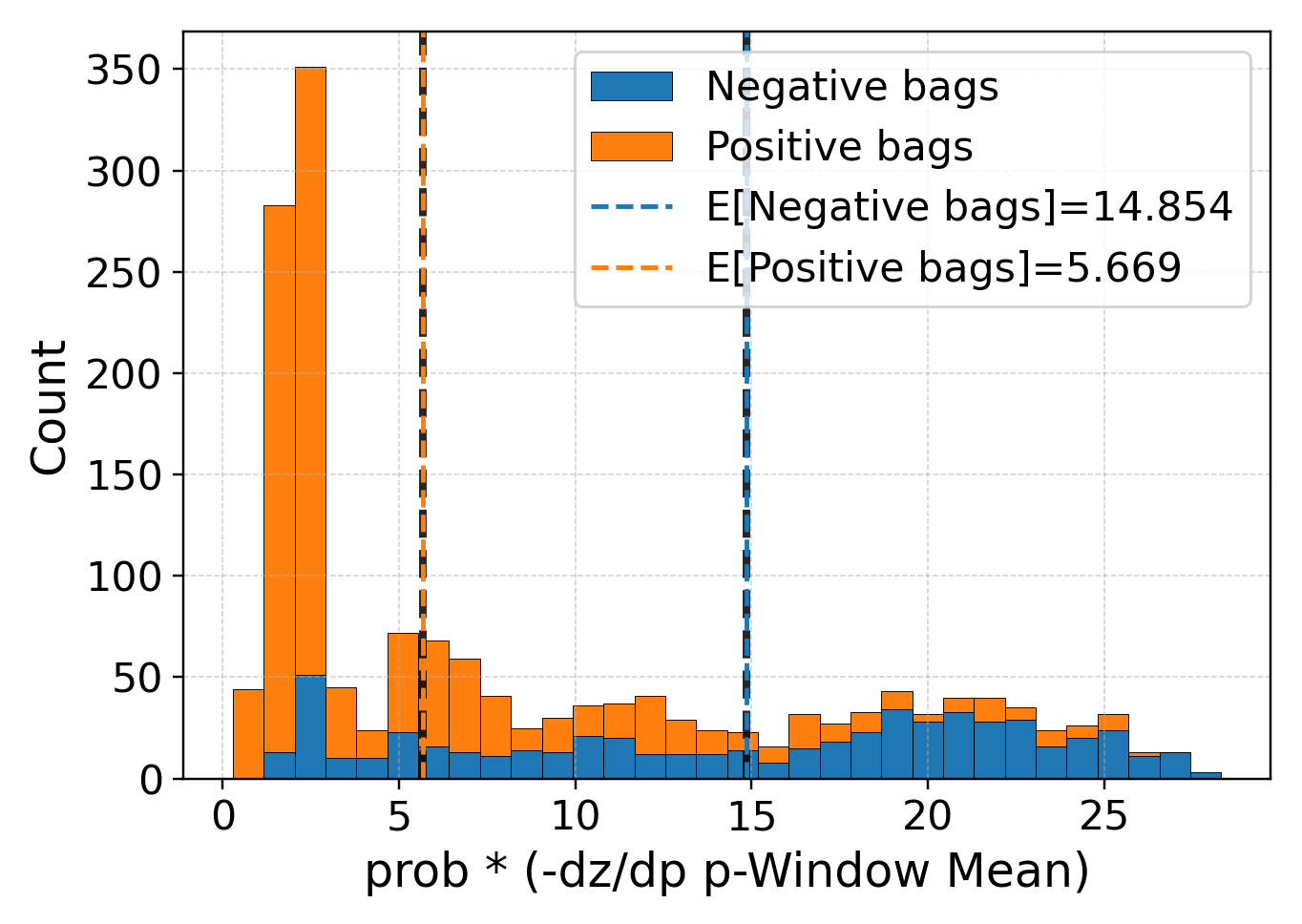}
        \caption{SQuAD\ (Llama-3.1-8B)}
    \end{subfigure}\hfill
    \begin{subfigure}{0.32\textwidth}
        \includegraphics[width=\linewidth]{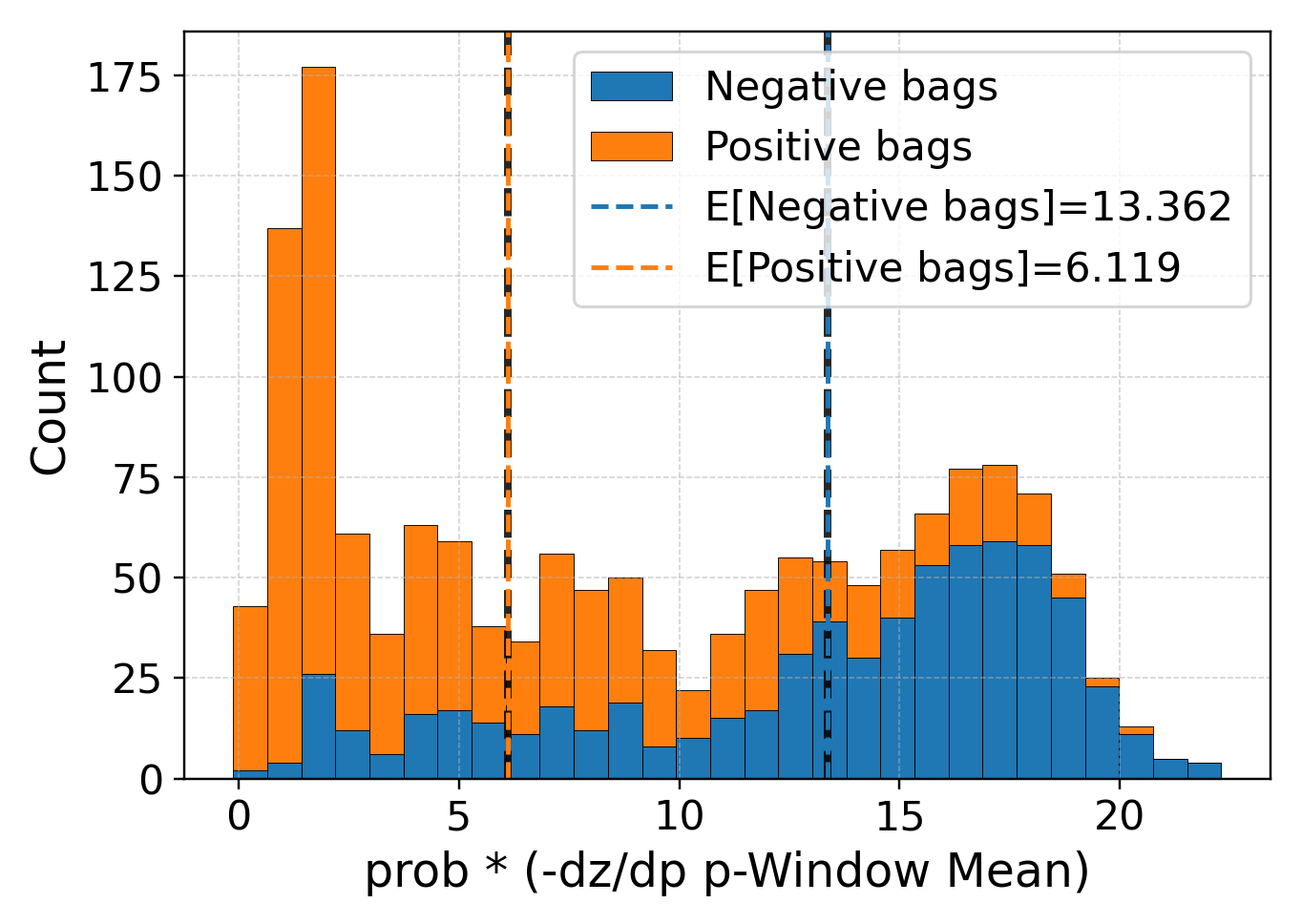}
        \caption{SQuAD\ (Mistral-12B)}
    \end{subfigure}\hfill
    \begin{subfigure}{0.32\textwidth}
        \includegraphics[width=\linewidth]{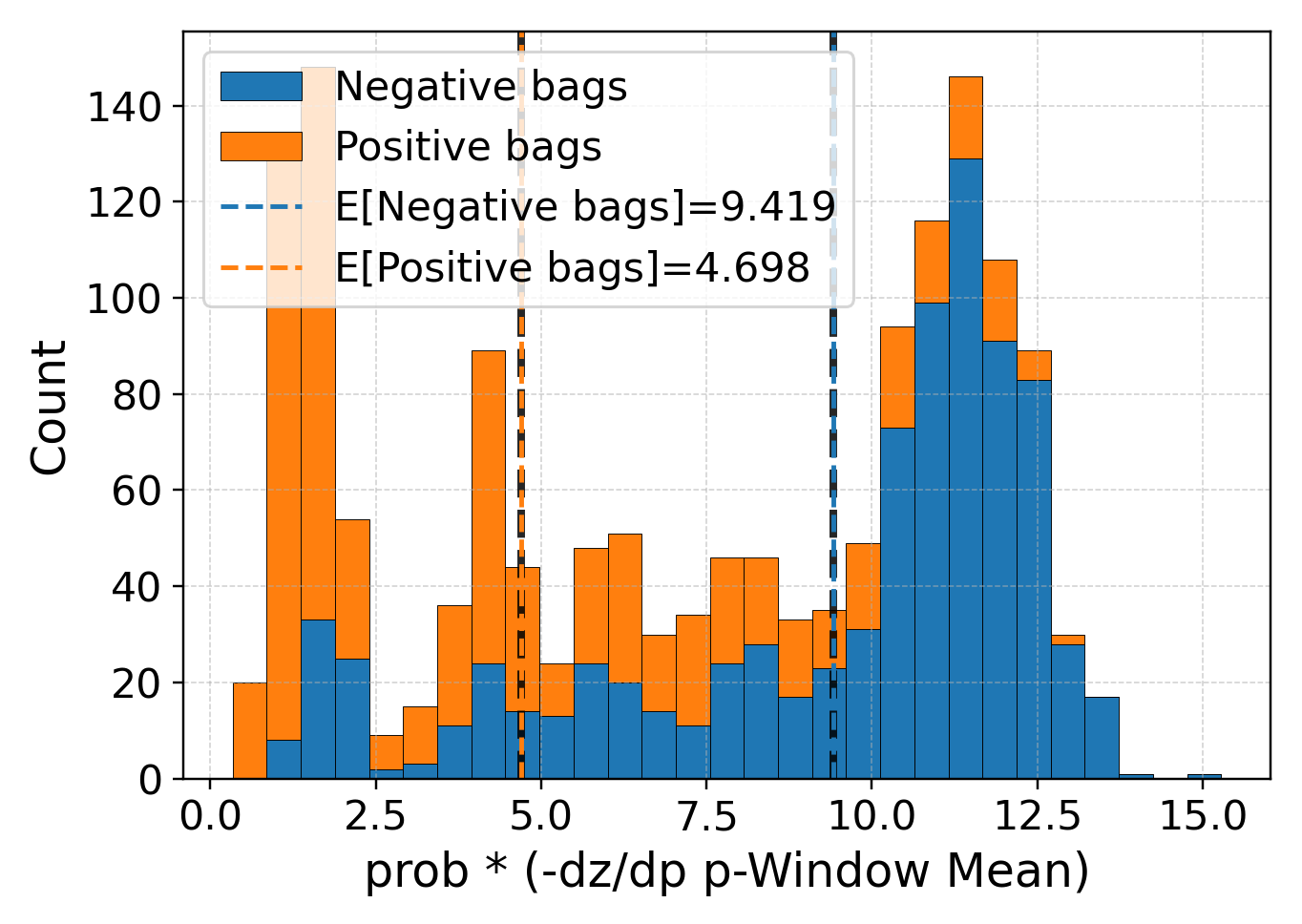}
        \caption{SQuAD\ (Llama-3.3-70B)}
    \end{subfigure}\hfill
    \vspace{0.3cm}
    \begin{subfigure}{0.32\textwidth}
        \includegraphics[width=\linewidth]{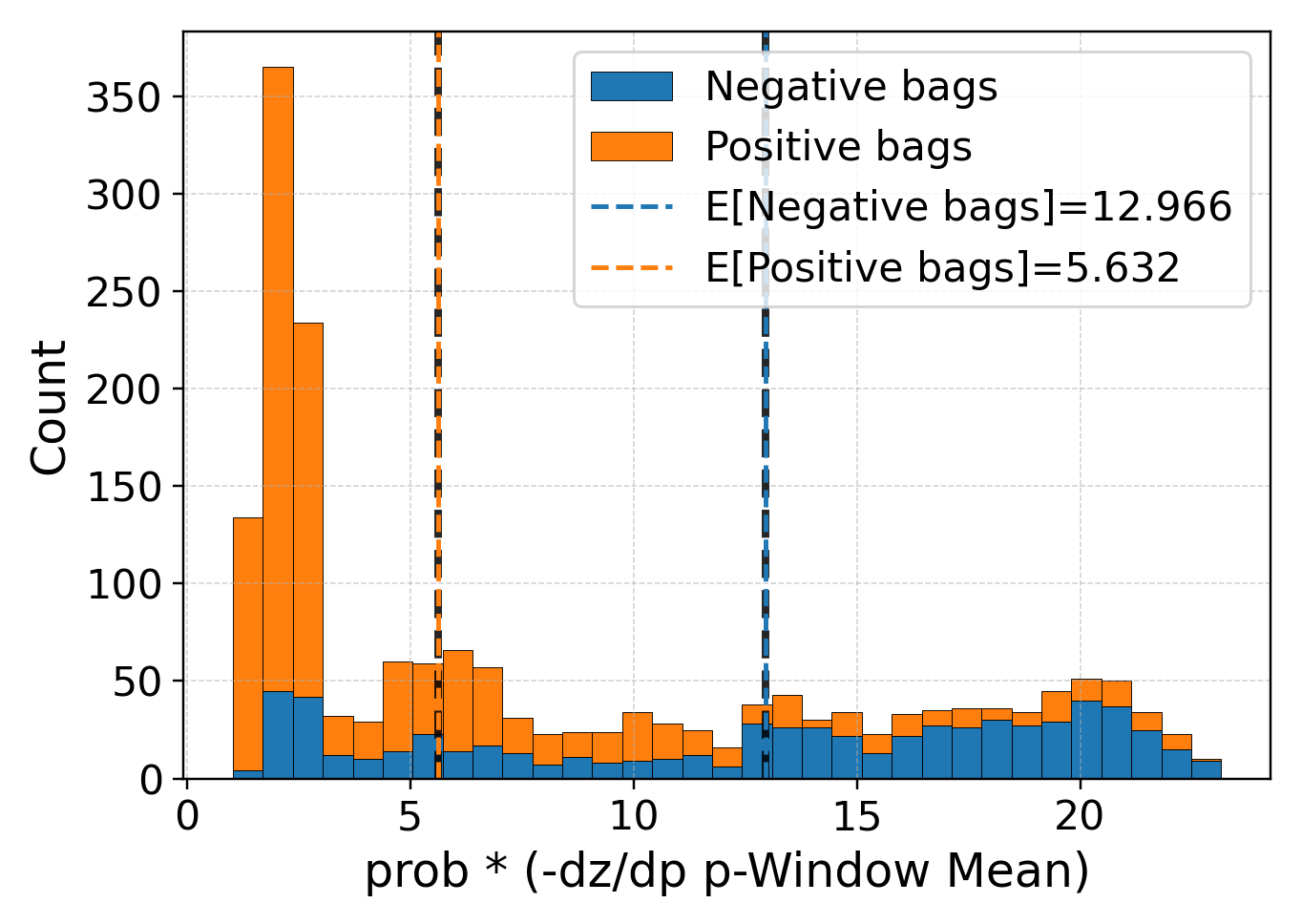}
        \caption{NQ\ (Llama-3.1-8B)}
    \end{subfigure}\hfill
    \begin{subfigure}{0.32\textwidth}
        \includegraphics[width=\linewidth]{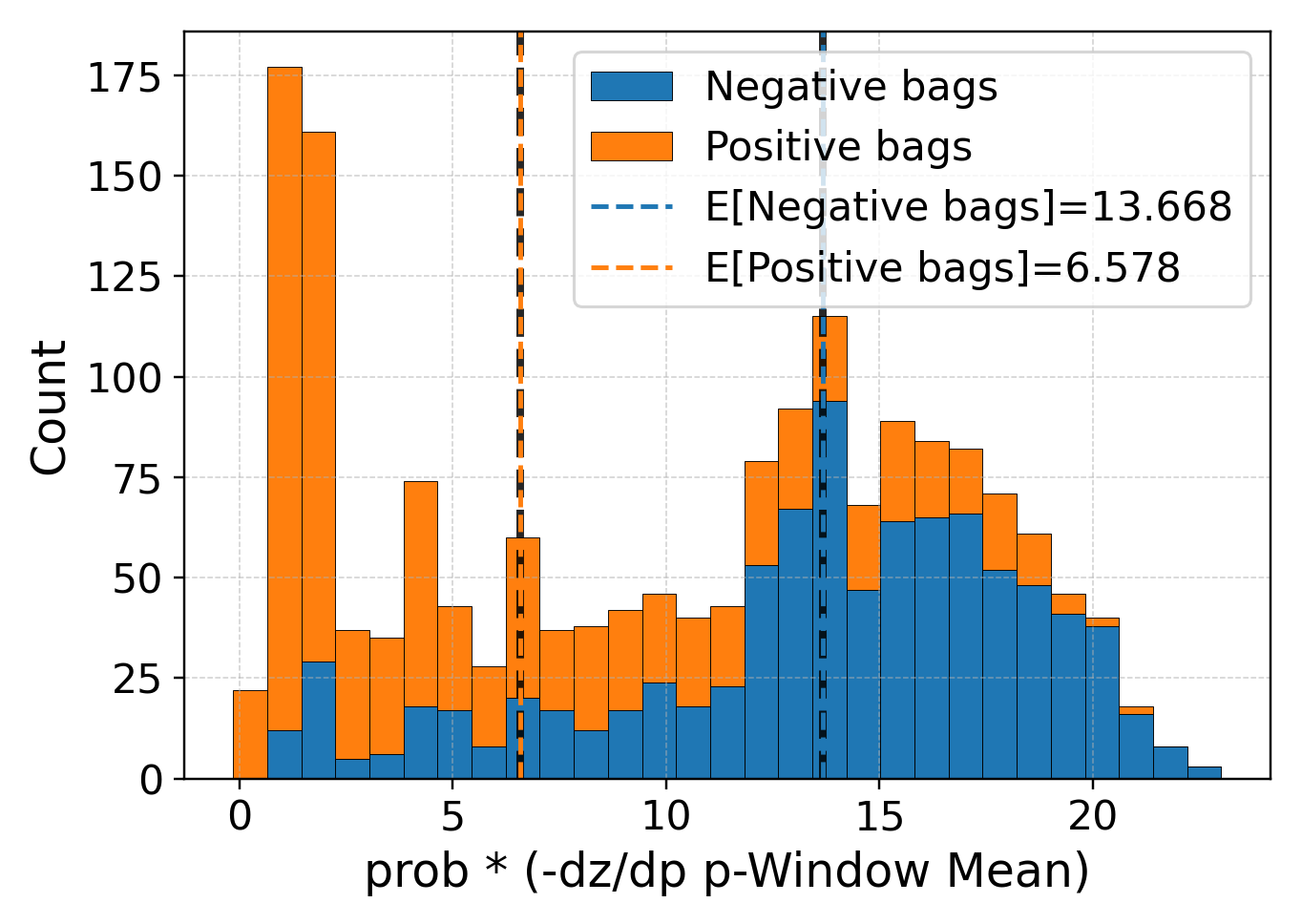}
        \caption{NQ\ (Mistral-12B)}
    \end{subfigure}\hfill
    \begin{subfigure}{0.32\textwidth}
        \includegraphics[width=\linewidth]{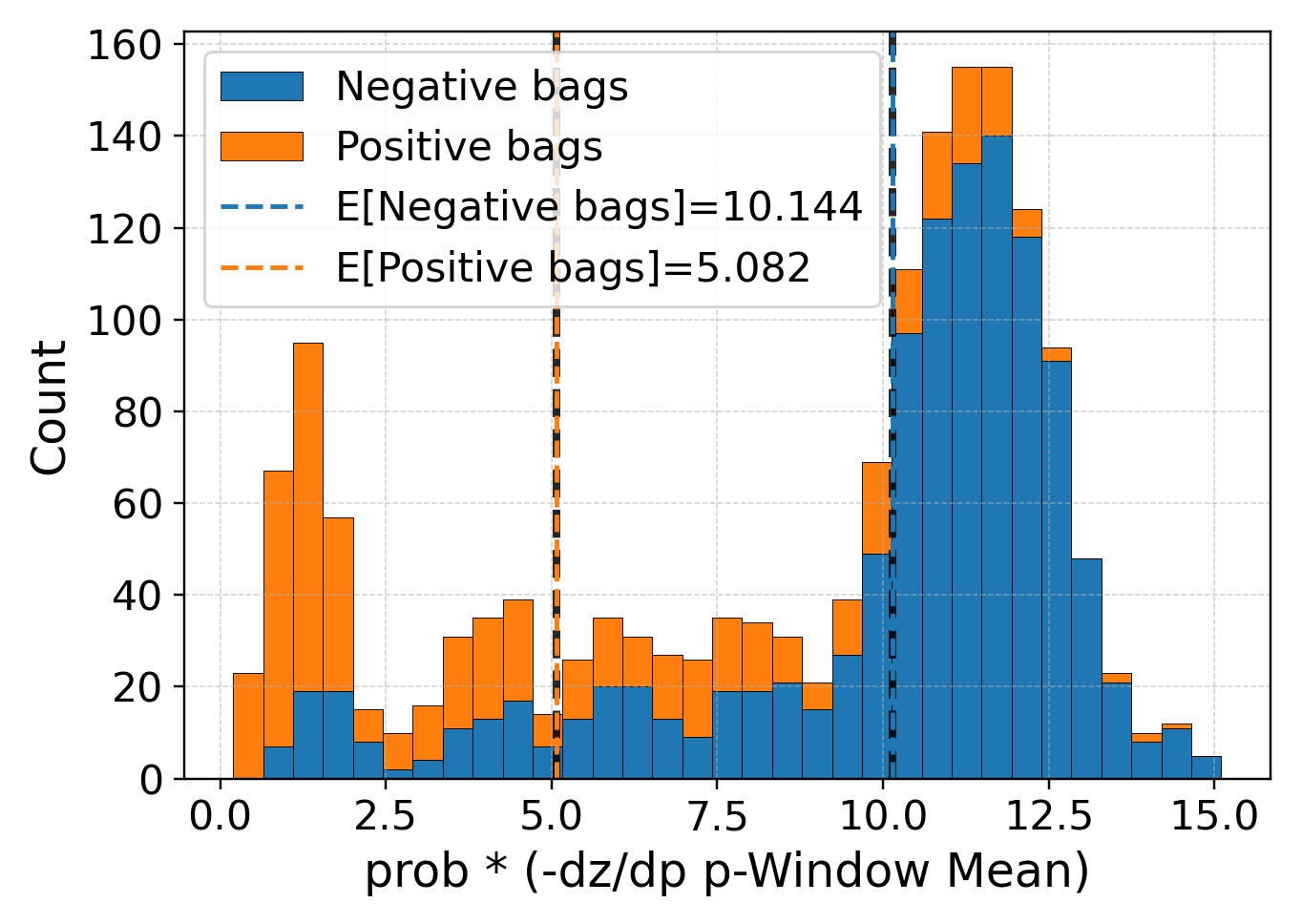}
        \caption{NQ\ (Llama-3.3-70B)}
    \end{subfigure}\hfill
    \vspace{0.3cm}
    \begin{subfigure}{0.32\textwidth}
        \includegraphics[width=\linewidth]{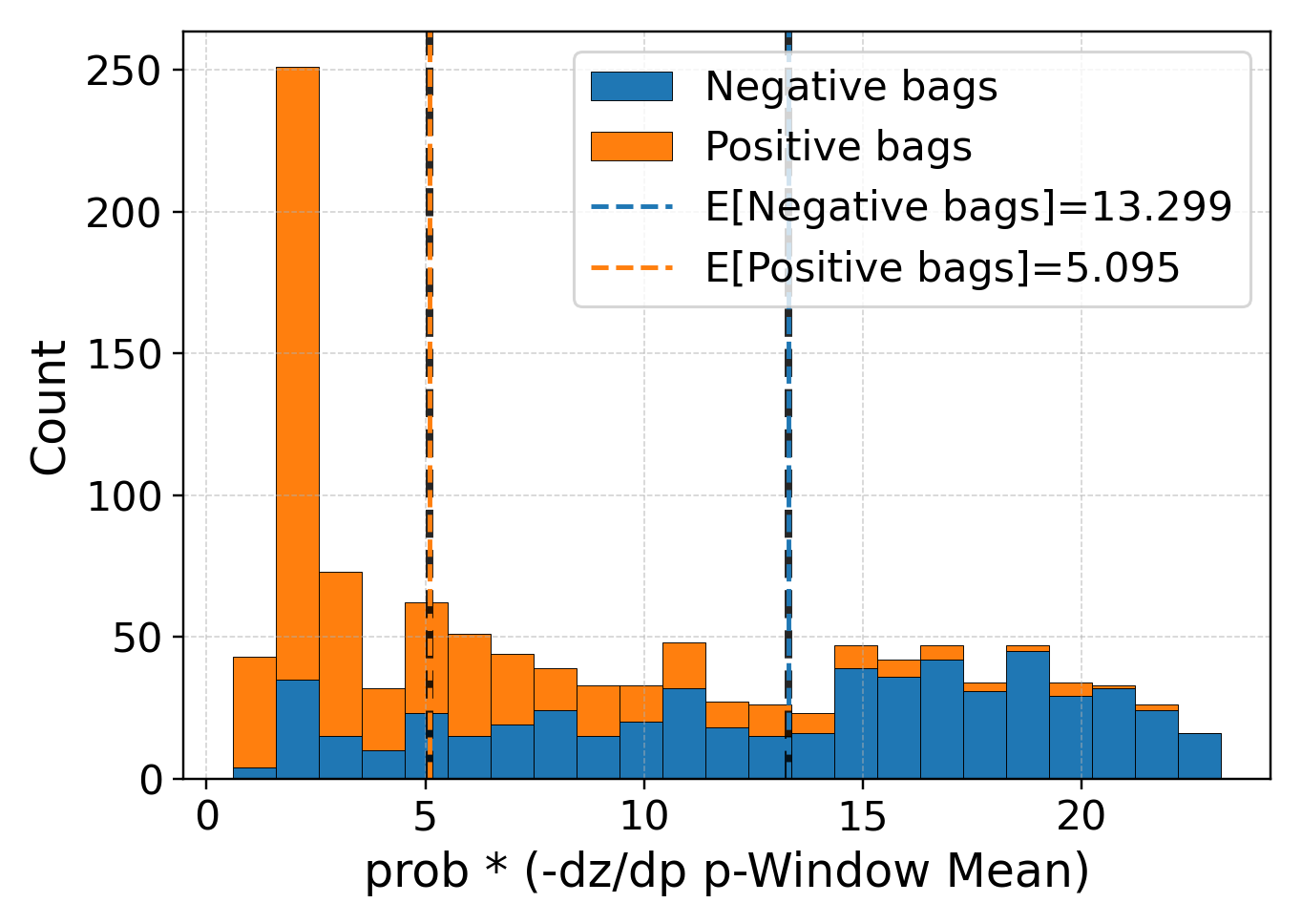}
        \caption{BioASQ\ (Llama-3.1-8B)}
    \end{subfigure}\hfill
    \begin{subfigure}{0.32\textwidth}
        \includegraphics[width=\linewidth]{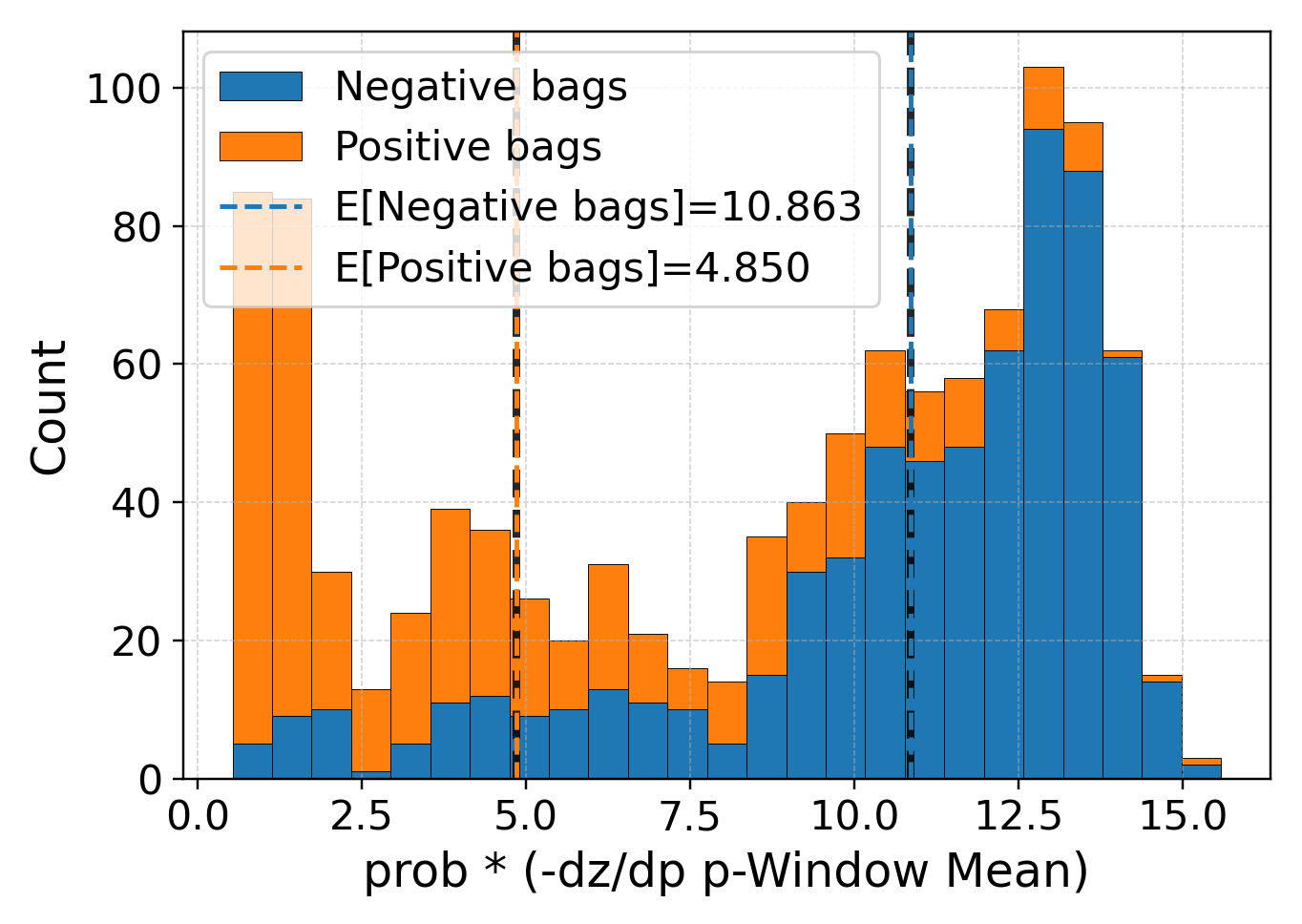}
        \caption{BioASQ\ (Mistral-12B)}
    \end{subfigure}\hfill
    \begin{subfigure}{0.32\textwidth}
        \includegraphics[width=\linewidth]{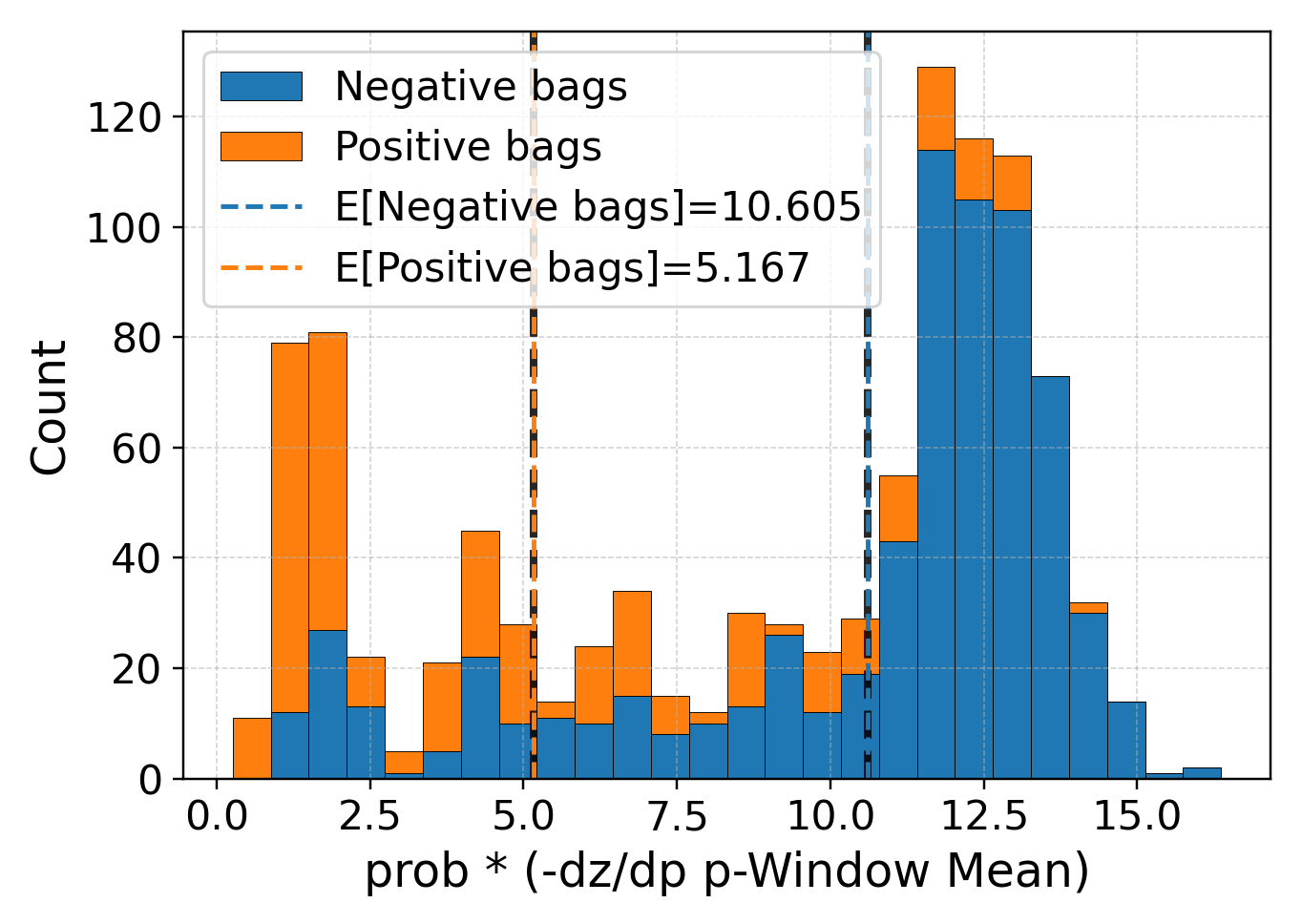}
        \caption{BioASQ\ (Llama-3.3-70B)}
    \end{subfigure}\hfill
    \vspace{0.3cm}
    \caption{Distribution of the joint product $P_{\mathrm{sem}}^{\mathbf{B}} \bar{C}_{\mathbf{B}}^{\mathrm{int}}$ in \textbf{test} data.}
    \label{fig:dist_joint_val_non_invariant}
\end{figure}